\documentclass[a4paper,11pt]{article}

\usepackage[english]{babel}
\usepackage[utf8]{inputenc}
\usepackage{lmodern}

\usepackage{geometry,color}
\usepackage{amsmath,amsthm,amssymb,mathtools,nicefrac,bbm}
\usepackage{titling}

\usepackage[us]{datetime}

\DeclareFontEncoding{LS1}{}{}
\DeclareFontSubstitution{LS1}{stix}{m}{n}
\DeclareMathAlphabet{\mathscr}{LS1}{stixscr}{m}{n}

\usepackage{hyperref}
\hypersetup{colorlinks=true,linkcolor=blue}

\usepackage[shortlabels,inline]{enumitem}
\usepackage[capitalise]{cleveref} 

\crefname{enumi}{item}{items}
\crefname{equation}{}{}

\newcommand{\N}{\mathbb{N}}

\newcommand{\R}{\mathbb{R}}

\renewcommand{\P}{\mathbb{P}}
\newcommand{\E}{\mathbb{E}}
\newcommand{\Var}{\mathrm{Var}}

\newcommand{\ind}{\mathbbm{1}}

\newcommand*\diff{\mathop{}\!\mathrm{d}}

\newcommand{\is}{\curvearrowleft}
\newcommand{\smallsum}{\textstyle \sum}
\newcommand{\smallprod}{\textstyle \prod}
\newcommand{\smallbigcap}{\textstyle \bigcap}
\newcommand{\smallbigcup}{\textstyle \bigcup}

\newcommand{\Exists}{\exists\,}
\newcommand{\Forall}{\forall\,}

\newcommand{\mc}[1]{\mathcal{#1}}
\newcommand{\mf}[1]{\mathfrak{#1}}
\newcommand{\mb}[1]{\mathbf{#1}}
\newcommand{\ms}[1]{\mathscr{#1}}

\newcommand{\maxone}[1]{\cfadd{def:clipping_functions}\lceil #1 \rceil}
\newcommand{\minone}[1]{\cfadd{def:clipping_functions}\lfloor #1 \rfloor}

\newcommand{\const}{\mf C}

\newcommand{\nablam}{\nabla_{\!-}}
\newcommand{\lrmat}{M}
\newcommand{\grad}{g}

\newcommand{\network}{\cfadd{def:ANN}\mathbf{N}}
\newcommand{\multdim}[2]{\cfadd{def:multidimensional_version}\mathfrak{M}_{#1,#2}}
\newcommand{\act}{\sigma}
\newcommand{\deriv}{\mathscr{D}}
\newcommand{\bp}{\kappa}
\newcommand{\nbp}{L}
\newcommand{\slope}{\mathfrak{s}}
\newcommand{\yinter}{\mathfrak{t}}
\newcommand{\lr}{\gamma}

\newcommand{\cmin}{\mathfrak{c}}
\newcommand{\cmax}{\mathfrak{C}}
\newcommand{\cout}{\mathcal{C}}
\newcommand{\cvar}{\mathscr{v}}
\newcommand{\Cvar}{\mathfrak{v}}
\newcommand{\constvarwidth}{\ms c}
\newcommand{\cderiv}{\mathfrak{l}}

\newcommand{\functionANN}[1]{\cfadd{def:realization_ANN}\mathcal{R}_{#1}}

\newcommand{\scalprod}[2]{\cfadd{def:scalar_product_norm} \langle #1, #2 \rangle}
\newcommand{\bscalprod}[2]{\cfadd{def:scalar_product_norm} \big \langle #1, #2 \big \rangle}
\newcommand{\bbscalprod}[2]{\cfadd{def:scalar_product_norm} \Big \langle #1, #2 \Big \rangle}

\newcommand{\eucnorm}[1]{\cfadd{def:scalar_product_norm} \lVert #1 \rVert}
\newcommand{\beucnorm}[1]{\cfadd{def:scalar_product_norm} \big \lVert #1 \big \rVert}

\newcommand{\bbbeucnorm}[1]{\cfadd{def:scalar_product_norm} \bigg \lVert #1 \bigg \rVert}

\newcommand{\lambdamin}{\cfadd{def:lambdamin}\lambda_{\min}}
\newcommand{\subexp}[2]{\cfadd{def:subexponential}$ ( #1, #2 ) $-subexponential}
\newcommand{\gram}{\cfadd{def:gram_matrix}Gram}

\makeatletter
\newcommand{\opnorm}{\@ifstar\@opnorms\@opnorm}
\newcommand{\@opnorms}[1]{%
  \left|\mkern-1.5mu\left|\mkern-1.5mu\left|
   #1
  \right|\mkern-1.5mu\right|\mkern-1.5mu\right|
}
\newcommand{\@opnorm}[2][]{%
  \mathopen{#1|\mkern-1.5mu#1|\mkern-1.5mu#1|}
  #2
  \mathclose{#1|\mkern-1.5mu#1|\mkern-1.5mu#1|}
}
\makeatother

\newcommand{\specnorm}[1]{\cfadd{def:spectral_norm} \opnorm{#1}}

\newcommand{\gramG}{\mc G}
\newcommand{\gramGinf}{\mb G}

\newcommand{\lambdazero}{\boldsymbol{\lambda}}

\newcommand{\width}{\mf d}
\newcommand{\indim}{d}
\newcommand{\outdim}{D}

\newcommand{\ndata}{m}
\newcommand{\indata}{\ms x}
\newcommand{\outdata}{\ms y}

\newcommand{\findata}{\ms X}
\newcommand{\foutdata}{\ms Y}	

\newcommand{\netvalue}{f}
\newcommand{\fnetvalue}{F}

\newcommand{\iteration}{n}
\newcommand{\Iteration}{N}
\newcommand{\poe}{\varepsilon}

\newcommand{\loss}{\mc E}

\newcommand{\abs}[1]{\lvert #1 \rvert}
\newcommand{\babs}[1]{\bigl\lvert #1 \bigr\rvert}
\newcommand{\bbabs}[1]{\Bigl\lvert #1 \Bigr\rvert}
\newcommand{\bbbabs}[1]{\biggl\lvert #1 \biggr\rvert}

\newcommand{\qandq}{\quad\text{and}\quad}
\newcommand{\qqandqq}{\qquad\text{and}\qquad}

\makeatletter
\ExplSyntaxOn
\seq_new:N \g_abbrs
\prop_new:N \g_abbr_counts
\tl_new:N \l_abbr_count_tl

\NewDocumentCommand{\abbr}{m m O{#1} m m O{#4} m}{
	\expandafter\newcommand\csname#3\endcsname[1][]{
		\seq_if_in:NnTF \g_abbrs {#1} {
			\prop_get:NnN \g_abbr_counts {#1} \l_abbr_count_tl
			\prop_gput:Nnx \g_abbr_counts {#1} {\int_eval:n {\l_abbr_count_tl + 1}}
			\hyperref[#1]{#7}
		} {
			\seq_gput_left:Nn \g_abbrs {#1}
			\prop_gput:Nnn \g_abbr_counts {#1} {1}
			\expandafter\gdef\csname#1@def\endcsname{#2}
			\phantomsection\label{#1}
			\str_if_eq:nnTF{##1}{}{\emph{#2}}{##1}~(\hyperref[#1]{#7})
		}
	}
	\expandafter\newcommand\csname#6\endcsname[1][]{
		\seq_if_in:NnTF \g_abbrs {#1} {
			\prop_get:NnN \g_abbr_counts {#1} \l_abbr_count_tl
			\prop_gput:Nnx \g_abbr_counts {#1} {\int_eval:n {\l_abbr_count_tl + 1}}
			\hyperref[#1]{#4}
		} {
			\expandafter\gdef\csname#1@def\endcsname{#5}
			\seq_gput_left:Nn \g_abbrs {#1}
			\prop_gput:Nnn \g_abbr_counts {#1} {1}
			\phantomsection\label{#1}
			\str_if_eq:nnTF{##1}{}{\emph{#5}}{##1}~(\hyperref[#1]{#4})
		}
	}
}

\ExplSyntaxOff
\makeatother

\abbr{ANN}{artificial neural network}{ANNs}{artificial neural networks}{ANN}
\abbr{ReLU}{rectified linear unit}{ReLUs}{rectified linear units}{ReLU}
\abbr{GD}{gradient descent}{GDs}{gradient descents}{GD}
\abbr{SGD}{stochastic gradient descent}{SGDs}{stochastic gradient descents}{SGD}
\abbr{MSE}{mean squared error}{MSEs}{mean squared errors}{MSE}
\abbr{AD}{automatic differentiation}{ADs}{automatic differentiations}{AD}

\theoremstyle{plain}
\newtheorem{theorem}{Theorem}[section]
\newtheorem{lemma}[theorem]{Lemma}
\newtheorem{prop}[theorem]{Proposition}
\newtheorem{cor}[theorem]{Corollary}
\newtheorem{proposition}[theorem]{Proposition}
\newtheorem{setting}[theorem]{Setting}
\theoremstyle{definition}
\newtheorem{definition} [theorem]{Definition}

\usepackage{xparse}
\usepackage{etoolbox}

\NewDocumentEnvironment{cproof}{m}
{\begin{proof}[Proof of \cref{#1}]}%
{\noindent The proof of \cref{#1} is thus complete.
\end{proof}}

\NewDocumentEnvironment{cproof2}{m}
{\begin{proof}[Proof of \cref{#1}]}%
{\noindent This completes the proof of \cref{#1}.
\end{proof}}

\ExplSyntaxOn

\seq_new:N \g_cflist_loaded
\seq_new:N \g_cflist_pending

\NewDocumentCommand{\cfadd}{ m }
{
  \seq_if_in:NnF \g_cflist_loaded { #1 } {
    \seq_if_in:NnF \g_cflist_pending { #1 } {
      \seq_gput_right:Nn \g_cflist_pending { #1 }
    }
  }
}

\NewDocumentCommand{\cfload}{ o }
{
  \seq_if_empty:NTF \g_cflist_pending {\unskip} {
    (cf.\ \cref{\seq_use:Nn \g_cflist_pending {,}})\IfValueTF{#1}{#1~}{\unskip}
    \seq_gconcat:NNN \g_cflist_loaded \g_cflist_loaded \g_cflist_pending
    \seq_gclear:N \g_cflist_pending
  }
}

\NewDocumentCommand{\cfclear} {} {
  \seq_gclear:N \g_cflist_loaded
  \seq_gclear:N \g_cflist_pending
}

\NewDocumentCommand{\cfout}{ o }
{
  \seq_if_empty:NTF \g_cflist_pending {\unskip} {
    (cf.\ \cref{\seq_use:Nn \g_cflist_pending {,}})\IfValueTF{#1}{#1~}{\unskip}
    \seq_gclear:N \g_cflist_pending
  }
}

\NewDocumentCommand{\ifnocf} { m } {
  \seq_if_empty:NT \g_cflist_pending { #1 }
}

\NewDocumentCommand{\cfconsiderloaded}{ m }{

  \seq_gput_right:Nn \g_cflist_loaded {#1}

}

\ExplSyntaxOff

\ExplSyntaxOn

\bool_new:N \g_noteobserve

\NewDocumentCommand{\nobs}{}{
  \bool_if:nTF { \g_noteobserve } {
    \bool_gset_false:N \g_noteobserve
    note~
  } {
    \bool_gset_true:N \g_noteobserve
    observe~
  }
}

\NewDocumentCommand{\Nobs}{}{
  \bool_if:nTF { \g_noteobserve } {
    \bool_gset_false:N \g_noteobserve
    Note~
  } {
    \bool_gset_true:N \g_noteobserve
    Observe~
  }
}

\ExplSyntaxOff

\ExplSyntaxOn

\bool_new:N \g_hencetherefore

\NewDocumentCommand{\hence}{}{
  \bool_if:nTF { \g_hencetherefore } {
    \bool_gset_false:N \g_hencetherefore
    hence~
  } {
    \bool_gset_true:N \g_hencetherefore
    therefore~
  }
}

\NewDocumentCommand{\Hence}{}{
  \bool_if:nTF { \g_hencetherefore } {
    \bool_gset_false:N \g_hencetherefore
    Hence,~we~obtain~
  } {
    \bool_gset_true:N \g_hencetherefore
    Therefore,~we~obtain~
  }
}

\ExplSyntaxOff

\ExplSyntaxOn

\int_new:N \g_furthermore

\NewDocumentCommand{\Moreover}{ o o }{
  \IfValueT{#1}{
    \str_case:nn {#1} {
      {Furthermore} {\int_set:Nn {\g_furthermore} {0}}
      {Moreover} {\int_set:Nn {\g_furthermore} {1}}
      {In~addition} {\int_set:Nn {\g_furthermore} {2}}
      {note} {\bool_gset_true:N \g_noteobserve}
      {observe} {\bool_gset_false:N \g_noteobserve}
    }
    \IfValueT{#2}{
      \str_case:nn {#2} {
        {Furthermore} {\int_set:Nn {\g_furthermore} {0}}
        {Moreover} {\int_set:Nn {\g_furthermore} {1}}
        {In~addition} {\int_set:Nn {\g_furthermore} {2}}
        {note} {\bool_gset_true:N \g_noteobserve}
        {observe} {\bool_gset_false:N \g_noteobserve}
      }
    }
  }
  \int_case:nn { \int_mod:nn {\g_furthermore} {3} } {
    { 0 } { Furthermore,~\nobs that}
    { 1 } { Moreover,~\nobs that}
    { 2 } { In~addition,~\nobs that}
  }
  \int_incr:N \g_furthermore
  \IfValueF{#1}{~}
}

\ExplSyntaxOff

\ExplSyntaxOn

\prop_const_from_keyval:Nn \l__verbs
{
show = shows ,
imply = implies ,
demonstrate = demonstrates ,
prove = proves ,
establish = establishes ,
ensure = ensures ,
assure = assures
}

\seq_const_from_clist:Nn \g_prove_mru {
   establish,
   demonstrate,
   prove,
   show,
   imply,
   ensure
}

\tl_new:N \g_wordtmp
\seq_new:N \l_mytmps

\cs_generate_variant:Nn \str_if_in:nnTF { nVTF }

\NewDocumentCommand{\prove}{ o }{
   \IfValueTF{#1}{
     \seq_clear:N \l_mytmps
     \seq_map_inline:Nn \g_prove_mru {
       \str_if_eq:nnTF {##1} {ensure} {
         \str_set:Nn \l_temps {n}
       } {
         \str_set:Nx \l_temps {\str_head_ignore_spaces:n {##1}}
       }
       \str_if_in:nVTF {#1} \l_temps {
         \seq_put_right:Nn \l_mytmps {##1}
       } { }
     }
     \seq_get_right:NN \l_mytmps \g_wordtmp
   } {
     \seq_get_right:NN \g_prove_mru \g_wordtmp
   }
   \tl_use:N \g_wordtmp
   \seq_gput_left:NV \g_prove_mru \g_wordtmp
   \seq_gremove_duplicates:N \g_prove_mru
   \IfValueF{#1}{~}
}

\NewDocumentCommand{\proves}{ o }{
   \IfValueTF{#1}{
     \seq_clear:N \l_mytmps
     \seq_map_inline:Nn \g_prove_mru {
       \str_if_eq:nnTF {##1} {ensure} {
         \str_set:Nn \l_temps {n}
       } {
         \str_set:Nx \l_temps {\str_head_ignore_spaces:n {##1}}
       }
       \str_if_in:nVTF {#1} \l_temps {
         \seq_put_right:Nn \l_mytmps {##1}
       } { }
     }
     \seq_get_right:NN \l_mytmps \g_wordtmp
   } {
     \seq_get_right:NN \g_prove_mru \g_wordtmp
   }
   \str_set:NV \l_tmpa_str \g_wordtmp
   \prop_get:NVN \l__verbs \l_tmpa_str \l_tmpa_tl
   \tl_use:N \l_tmpa_tl
   \seq_gput_left:NV \g_prove_mru \g_wordtmp
   \seq_gremove_duplicates:N \g_prove_mru
   \IfValueF{#1}{~}
}

\ExplSyntaxOff

\ExplSyntaxOn

\NewDocumentEnvironment{flexmath}{ m o }{
  \str_if_eq:noTF {a} {#1} {
    \begin{equation}
    \IfValueT{#2}{\label{eq:\loc.#2}}
    \begin{aligned}
  } {
    \catcode`&=9
    \renewcommand{\\}{}
    \str_if_eq:noTF {d} {#1} {
      \begin{equation}
      \IfValueT{#2}{\label{eq:\loc.#2}}
    } {
      \begin{math}
    }
  }
}{
  \str_if_eq:noTF {i} {#1} {
    \end{math}
    \catcode`&=4
  } {
    \str_if_eq:noTF {d} {#1} {
      \end{equation}
    } {
      \end{aligned}
      \end{equation}
    }
  }
}

\ExplSyntaxOff

\begin{document}

\title{Convergence rates for gradient descent in the \\ training of overparameterized artificial neural \\ networks with piecewise affine activation}

\author{Arnulf Jentzen$^{1,2}$ and Timo Kr\"oger$^{3}$\bigskip\\
\small{$^1$ School of Data Science and School of Artificial Intelligence,} \vspace{-0.1cm}\\
\small{The Chinese University of Hong Kong, Shenzhen (CUHK-Shenzhen),}\vspace{-0.1cm}\\
\small{Shenzhen, China; e-mail: \texttt{ajentzen}\textcircled{\texttt{a}}\texttt{cuhk.edu.cn}}\smallskip\\
\small{$^2$ Applied Mathematics: Institute for Analysis and Numerics,}\vspace{-0.1cm}\\
\small{Faculty of Mathematics and Computer Science, University of M\"unster,}\vspace{-0.1cm}\\
\small{M\"unster, Germany; e-mail: \texttt{ajentzen}\textcircled{\texttt{a}}\texttt{uni-muenster.de}}\smallskip\\
\small{$^3$ Applied Mathematics: Institute for Analysis and Numerics,}\vspace{-0.1cm}\\
\small{Faculty of Mathematics and Computer Science, University of M\"unster,}\vspace{-0.1cm}\\
\small{M\"unster, Germany; e-mail: \texttt{timo.kroeger}\textcircled{\texttt{a}}\texttt{uni-muenster.de}}}

\date{May 18, 2026}

\maketitle

\begin{abstract}
In recent years, artificial neural networks have developed into a powerful tool for addressing a multitude of problems for which classical solution approaches reach their limits. However, it is still unclear why gradient descent optimization algorithms with random initialization, such as the well-known batch gradient descent, are able to achieve zero training loss in many situations, even though the objective function is non-convex and non-smooth. One of the most promising approaches to solving this issue in the field of supervised learning is the analysis of gradient descent optimization in the so-called overparameterized regime. In this article, we provide a further contribution to this area of research by considering overparameterized fully connected shallow artificial neural networks with piecewise affine activation, such as the rectified linear unit activation. Specifically, given that the activation function is not affine and the training input data are pairwise distinct, we show that, with high probability, the mean squared error of such a randomly initialized artificial neural network optimized via batch gradient descent converges to zero at a linear convergence rate as long as the width of the artificial neural network is sufficiently large and the learning rate is sufficiently small.
\end{abstract}

\begin{center}
empirical risk minimization, optimization, random initialization, Gram matrix
\end{center}

\tableofcontents

\section{Introduction}
\label{sec:introduction}

Many problems, such as recognizing faces, handwritten text, or natural language, as well as performing various activities, such as driving a car, seem very simple to humans but pose a major challenge to computers. A fundamental reason for this lies in the immense variability of real-world data, combined with the fact that computers, unlike humans, are not able to intuitively recognize complex patterns in it due to their structurally precise data processing. As a result, it is very difficult to manually specify a program or provide an explicit function which attaches to the input -- be it an image, an audio file, a game situation, or a traffic situation of an autonomous driving car -- a desired output or guidance. Nevertheless, in recent years, \ANNs\ have become a powerful tool in dealing with such problems. Although many \ANN\ training algorithms have demonstrated great success in practice, the reasons for this are generally not known, as no mathematically rigorous analysis exists for most algorithms. Since these algorithms primarily rely on \GD, we refer to Ruder \cite{ruder2017overview} for an overview of commonly used \GD\ variations for \ANN\ optimization.

There are several promising attempts in the scientific literature that intend to mathematically analyze \GD\ optimization algorithms in the training of \ANNs. In particular, there are various convergence results for \GD\ optimization algorithms in the training of \ANNs\ that assume convexity of the objective functions under consideration (cf., e.g., \cite{JMLR:v18:14-546, NIPS2013_7fe1f8ab, NIPS2011_40008b9a} and the references mentioned therein), there are general abstract convergence results for \GD\ optimization algorithms that do not assume convexity of the objective functions under consideration (cf., e.g., \cite{JMLR:v25:21-1423, doi:10.1142/S021902572150020X, MR4810685, JMLR:v21:19-636, JENTZEN2023127907, 8930994, doi:10.1137/22M1514283} and the references mentioned therein), there are divergence results and lower bounds for \GD\ optimization algorithms in the training of \ANNs\ (cf., e.g., \cite{CHERIDITO2021101540, JENTZEN2020101438, MR4188518} and the references mentioned therein), there are mathematical analyses regarding the parameter initialization in the training of \ANNs\ with \GD\ optimization algorithms (cf., e.g., \cite{NEURIPS2018_13f9896d, NEURIPS2018_d81f9c1b, MR4188518, Shin_2020} and the references mentioned therein), and there are convergence results for \GD\ optimization algorithms in the training of \ANNs\ for constant target functions (cf.\ \cite{CHERIDITO2022101646}).

One of the most promising approaches in the field of supervised learning is the analysis of \GD\ optimization in the so-called overparameterized regime. In this regime, the number of parameters defining the \ANN\ far exceeds the number of available training samples. Under this condition, the convergence of the \GD\ optimization method and some of its variants has been proved in several settings, even though the objective function is non-convex and non-smooth (cf., e.g., \cite{NEURIPS2019_62dad6e2, pmlr-v97-allen-zhu19a, NEURIPS2018_a1afc58c, pmlr-v97-du19c, du2019gradienta, EMaWu2020, zou2018stochastic} and the references mentioned therein). In this overparameterized regime, further progress was made in the analysis of the popular \SGD\ optimization algorithm regarding the influence of the architecture of an \ANN\ on the so-called gradient confusion, which directly affects convergence speed (cf.\ Sankararaman et al.~\cite{pmlr-v119-sankararaman20a}), regarding the characterization of stability properties of global minima (cf.\ Wu et al.~\cite{NEURIPS2018_6651526b}), and regarding multi-class classification with additional considerations for the generalization error (cf.\ Li \& Liang \cite{NEURIPS2018_54fe976b}). Another promising approach in the context of overparameterization is the use of kernel methods. In particular, it was proved that the realization of an \ANN\ during \GD\ optimization follows the kernel gradient of the functional cost with respect to the neural tangent kernel (cf., e.g., Jacot et al.~\cite{NEURIPS2018_5a4be1fa}).

In this work, we provide a further contribution by extending Du et al.~\cite{du2019gradienta} in the following way. We consider \ANNs\ with piecewise affine activation functions, with multidimensional output, and with biases on the hidden layer and output layer and initialize all weights with a normal distribution -- as usual in practice -- with mean $ 0 $ and equal variance within each layer. In addition, we relax the requirements on the training data and consider slightly more general optimization algorithms. This approach contrasts with other articles in this area of research that extend Du et al.~\cite{du2019gradienta}, for example, by considering \ANNs\ with multiple layers without biases and smooth activation functions instead of piecewise affine functions (cf.\ Du et al.~\cite{pmlr-v97-du19c}) or by considering \ANNs\ with multiple layers in the context of \SGD\ (cf.\ Zou et al.~\cite{zou2018stochastic}). To illustrate the findings of this article in a special case, we now present \cref{thm:introduction}, and we refer to \cref{subsec:quantitative_probabilistic_error_analysis} below for the more general convergence results which we develop in this article. Below \cref{thm:introduction} we also provide some explanations regarding the mathematical objects that are introduced within \cref{thm:introduction}.

\begin{samepage}
\cfclear
\begin{theorem}
  \label{thm:introduction}
  Let $ \indim, \outdim \in \N $,
  let $ \mc P \colon \N \to \N $ satisfy for all $ \width \in \N $ that
  $ \mc P ( \width ) = \width \indim + \width + \outdim \width + \outdim $,
  let $ \act \in C( \R, \R ) $ be piecewise affine$\mspace{1.5mu}$\footnote{
    \Nobs that for every $ f \in C( \R, \R ) $ it holds that $ f $ is piecewise affine (piecewise linear) if and only if
    there exists a finite $ S \subseteq \R $ such that $ f|_{\R \backslash S} \in C^{\infty}(\R,\R) $ and
    $ \sup_{ v \in \R \backslash S } \abs{ f''(v) } = 0 $.}$\mspace{-1.5mu}$,
  assume that $ \act $ is not affine,
  for every $ \width \in \N $, $ \theta = ( \theta_1, \ldots, \theta_{\mc P(\width)} ) \in \R^{\mc P(\width)} $
  let $ \mc N_{\theta}^{\width} = ( \mc N_{\theta}^{\width,1}, \ldots, \mc N_{\theta}^{\width,\outdim} ) \colon \R^{\indim} \to \R^{\outdim} $
  satisfy for all $ i \in \{ 1, 2, \ldots, \outdim \} $, $ v = (v_1, \ldots, v_{\indim}) \in \R^{\indim} $ that
  \begin{equation}\label{thm:introduction:realization}
    \mc N_{\theta}^{\width,i} (v) = \theta_{ \width \indim + \width + \outdim \width + i }
      + \smallsum\limits_{j=1}^{\width} \theta_{ \width \indim + \width + (i-1) \width + j } \,
      \act \Bigl( \theta_{ \width \indim + j } + \smallsum\limits_{k=1}^{\indim} \theta_{ (j-1) \indim + k } v_{k} \Bigr),
  \end{equation}
  let $ \const \in (0,\infty) $, for every $ \ndata \in \N $ let
  $ \indata_1^{\ndata}, \indata_2^{\ndata}, \ldots, \indata_{\ndata}^{\ndata} \in [{\const}^{-1}, \const]^{\indim} $,
  $ \outdata_1^{\ndata}, \outdata_2^{\ndata}, \ldots, \outdata_{\ndata}^{\ndata} \in [-\const,\const]^{\outdim} $ satisfy
  $ \# \{ \indata_1^{\ndata}, \indata_2^{\ndata}, \ldots, \indata_{\ndata}^{\ndata} \} = \ndata $,
  for every $ \width, \ndata \in \N $ let $ \loss^{\width, \ndata} \colon \R^{ \mc P ( \width ) } \to \R $
  satisfy$\mspace{1.5mu}$\footnote{
    \Nobs that for every $ p \in \N $, $ u = ( u_1, \ldots, u_p ) $, $ v = ( v_1, \ldots, v_p ) \in \R^p $
    it holds that $ \scalprod{u}{v} = \sum_{i=1}^p u_i v_i $ and
    $ \eucnorm{u} = \bigl( \sum_{i=1}^p \abs{ u_i }^2 \bigr)^{\nicefrac{1}{2}} $.}$\mspace{1.5mu}$
  for all $ \theta \in \R^{ \mc P ( \width ) } $ that
  \begin{equation}\label{thm:introduction:loss}
    \loss^{\width, \ndata} (\theta)
    = \frac{1}{\ndata} \biggl[ \smallsum\limits_{i=1}^{\ndata} \eucnorm{ \mc N_{\theta}^{\width} (\indata_i^{\ndata}) - \outdata_i^{\ndata} }^2 \biggr],
  \end{equation}
  for every $ \width \in \N $ let $ \lrmat^{\width} \in \R^{ \mc P(\width) \times \mc P(\width) } $ satisfy
  $ \lrmat^{\width} = ( \ind_{ \{ j \} \backslash ( \width \indim + \width, \mc P( \width ) - \outdim ] } (i) )_{
    (i,j) \in \{ 1, 2, \ldots, \mc P( \width ) \}^2 } $,
  let $ (\Omega, \mc F, \P) $ be a probability space,
  for every $ \width, \ndata \in \N $, $ \lr \in \R $ let
  $ \Theta^{\width, \ndata, \lr} = ( \Theta_1^{\width, \ndata, \lr}, \ldots, \Theta_{ \mc P( \width ) }^{\width, \ndata, \lr} )
    \colon \N_0 \times \Omega \to \R^{ \mc P( \width ) } $
  be a stochastic process which satisfies$\mspace{1.5mu}$\footnote{
    \Nobs that for every $ p \in \N $, $ f \in C( \R^p, \R ) $, $ u, v, e_1, e_2, \ldots, e_p \in \R^p $ with
    $ e_1 = ( 1, 0, \ldots, 0 ) $, $ e_2 = ( 0, 1, 0, \ldots, 0 )$, $ \dots $, $ e_p = ( 0, \ldots, 0, 1 ) $, and
    $ \limsup_{\varepsilon \nearrow 0} \eucnorm{ v - \sum_{k=1}^p ( f( u + \varepsilon e_k ) - f(u) ) \varepsilon^{-1} e_k } = 0 $
    it holds that $ ( \nablam f )( u ) = v $ (left gradient).}$\mspace{-1.5mu}$
  for all $ \iteration \in \N $ that
  $ \smallsum_{k=1}^{\width} \abs{ \Theta_{ \width \indim + k }^{\width, \ndata, \lr} (0) }
    + \smallsum_{i=1}^{\outdim} \abs{ \Theta_{ \width \indim + \width + \outdim \width + i }^{\width, \ndata, \lr} (0) } = 0 $ and
  \begin{equation}\label{thm:introduction:theta}
    \Theta^{\width, \ndata, \lr} ( \iteration ) = \Theta^{\width, \ndata, \lr} ( \iteration - 1 )
    - \lr \lrmat^{\width} \bigl( \nablam \loss^{\width, \ndata} \bigr)
    \bigl( \Theta^{\width, \ndata, \lr} ( \iteration - 1 ) \bigr),
  \end{equation}
  assume for all $ \width, \ndata \in \N $, $ \lr \in \R $ that
  $ \sqrt{\nicefrac{\indim}{2}} \mspace{2mu} \Theta_1^{\width, \ndata, \lr} (0), \allowbreak
    \sqrt{\nicefrac{\indim}{2}} \mspace{2mu} \Theta_2^{\width, \ndata, \lr} (0), \allowbreak \ldots, \allowbreak
    \sqrt{\nicefrac{\indim}{2}} \mspace{2mu} \Theta_{\width \indim}^{\width, \ndata, \lr} (0), \allowbreak
    \sqrt{\nicefrac{\width}{2}} \mspace{2mu} \Theta_{\width \indim + \width + 1}^{\width, \ndata, \lr} (0), \allowbreak
    \sqrt{\nicefrac{\width}{2}} \mspace{2mu} \Theta_{\width \indim + \width + 2 }^{\width, \ndata, \lr} (0), \allowbreak \ldots, \allowbreak
    \sqrt{\nicefrac{\width}{2}} \mspace{2mu} \Theta_{\width \indim + \width + \outdim \width}^{\width, \ndata, \lr} (0) $
  are independent and standard normal,
  let $ Z \colon \Omega \to \R^{\indim} $ be standard normal, and
  for every $ \ndata \in \N $ let $ \lambda_{\ndata} \in \R $ satisfy
  \begin{equation}\label{thm:introduction:lambda}
    \lambda_{\ndata} = \inf\limits_{ z = ( z_1, \ldots, z_{\ndata} ) \in \R^{\ndata}, \, \eucnorm{ z } = 1 }
      \biggl( \smallsum\limits_{i,j=1}^{\ndata}
      \bigl( 2 + 2 \scalprod{ \indata_i^{\ndata} }{ \indata_j^{\ndata} } \bigr) \,
      \E \biggl[ \smallprod\limits_{ k \in \{ i, j \} } z_k
        \bigl( \nablam \act \bigr) \bigl( \sqrt{\nicefrac{2}{\indim}} \mspace{2mu} \scalprod{ Z }{ \indata_k^{\ndata} } \bigr) \biggr] \biggr).
  \end{equation}
  Then
  \begin{enumerate}[(i)]
    \item \label{thm_introduction:item1} it holds for all $ \ndata \in \N $ that $ \lambda_{\ndata} > 0 $ and
    \item \label{thm_introduction:item2} there exists $ \Lambda \in (0, 1) $ such that for all $ \ndata \in \N $, $ \poe \in (0, 1) $,
      $ \lr \in (0, \Lambda \poe^2 \ndata^{-1} \min \{ \lambda_{\ndata}, ( \lambda_{\ndata} )^{-1} \} ) $,
      $ \width \in \N \cap [ \Lambda^{-1} \poe^{-4} \ndata^7 \max\{ ( \lambda_{\ndata} )^{-1}, ( \lambda_{\ndata} )^{-5} \}, \infty) $
      it holds that
  \begin{equation}
  \label{eq:thm_introduction}
    \P \Bigl( \Forall \iteration \in \N_0 \colon \loss^{\width, \ndata} ( \Theta^{\width,\ndata,\lr} (\iteration) )
      \le \bigl( 1 - \tfrac{ \lr \lambda_{\ndata} }{ \ndata } \bigr)^{\iteration} \loss^{\width, \ndata}
        ( \Theta^{\width,\ndata,\lr} (0) ) \Bigr)
    \ge 1 - \varepsilon.
  \end{equation}
  \end{enumerate}
\end{theorem}
\end{samepage}

\cref{thm:introduction} is an immediate consequence of \cref{cor:vectorized} combined with \cref{lem:convolution} and \cref{lem:left_derivative} in \cref{sec:error_analysis} below. \cref{cor:vectorized} follows from \cref{cor:assumption_pairwise_distinct_const} which, in turn, builds on a series of intermediate results which are all based on \cref{thm:main_theorem}. In the following, we add some comments and explanations regarding the mathematical objects which appear in \cref{thm:introduction}.

The natural numbers $ \indim, \outdim \in \N = \{ 1, 2, 3, \ldots \} $ in \cref{thm:introduction} above specify the dimensions of the input data and output data of the considered training set. The function $ \mc P \colon \N \to \N $ specifies the number of parameters defining a shallow \ANN\ with $ \width \in \N $ neurons on the hidden layer and must therefore satisfy for all $ \width \in \N $ that
\begin{equation}
  \mc P( \width ) = \width \indim + \width + \outdim \width + \outdim.
\end{equation}
Specifically, in \cref{thm:introduction} above we employ fully connected \ANNs\ with one hidden layer and a piecewise affine function $ \act \in C( \R, \R ) $ that is not affine as the activation function in front of the hidden layer, such as the popular \ReLU\ activation function
\begin{equation}
  \R \ni v \mapsto \max\{ v, 0 \} \in \R.
\end{equation}
\Moreover for every $ \width \in \N $, $ \theta \in \R^{\mc P(\width)} $ we have that the function $ \mc N_{\theta}^{\width} \colon \R^{\indim} \to \R^{\outdim} $ in \cref{thm:introduction:realization} describes the realization of a fully connected \ANN\ with $ \indim $ neurons on the input layer, $ \width $ neurons on the hidden layer, $ \outdim $ neurons on the output layer, and $ \act $ as the activation function in front of the hidden layer. The vector $ \theta \in \R^{\mc P(\width)} $ stores the real parameters (weights and biases) for the specific \ANN\ being considered.

\Moreover the natural number $ \ndata \in \N $ specifies the number of input-output data pairs used in the training of the considered \ANN. The corresponding vectors $ \indata_1^{\ndata}, \indata_2^{\ndata}, \ldots, \indata_{\ndata}^{\ndata} \in [\const^{-1},\const]^{\indim} $ constitute the input data used for the training, while the vectors $ \outdata_1^{\ndata}, \outdata_2^{\ndata}, \ldots, \outdata_{\ndata}^{\ndata} \in [-\const,\const]^{\outdim} $ constitute the output data used for the training, where the real number $ \const \in (0,\infty) $ specifies a uniform bound on the training set.
Utilizing these training data, for every $ \width, \ndata \in \N $ we have that the function $ \loss^{\width,\ndata} \colon \R^{\mc P(\width)} \to \R $ in \cref{thm:introduction:loss} measures the \MSE\ between the estimated values $ \mc N_{\theta}^{\width} (\indata_1^{\ndata}), \mc N_{\theta}^{\width} (\indata_2^{\ndata}), \ldots, \mc N_{\theta}^{\width} (\indata_{\ndata}^{\ndata}) $ and the target values $ \outdata_1^{\ndata}, \outdata_2^{\ndata}, \ldots, \outdata_{\ndata}^{\ndata} $.

Next \nobs that the piecewise affine activation function $ \act \in C( \R, \R ) $ is not differentiable at the kinks, and as a consequence, the associated risk function is not everywhere differentiable. To overcome this obstacle, we consider the \GD\ optimization process using the left gradient of the risk function $ \nablam \loss^{\width,\ndata} $ instead of the gradient of the risk function, which does not exist everywhere. This process coincides with the \AD\ algorithm (cf.~Baydin et al.~\cite{JMLR:v18:17-468}) used in practice, which treats $ \act $ as differentiable with derivative given by the left derivative of $ \act $.

\Nobs that for every $ \width, \ndata \in \N $, $ \lr \in \R $ we have that the stochastic process $ \Theta^{\width, \ndata, \lr} \colon \N_0 \times \Omega \to \R^{\mc P( \width )} $ is associated with the \GD\ process of an \ANN\ with $ \width $ neurons on the hidden layer, which is randomly initialized and optimized with \GD\ in \cref{thm:introduction:theta} using $ \ndata $ training data pairs and the constant learning rate $ \lr $. To be more precise, the parameters of the considered \ANN\ in \cref{thm:introduction} are initialized using the Kaiming initialization method (cf.~He et al.~\cite{7410480}), which is a standard initialization method for fully connected \ANNs\ with \ReLU\ activation. All parameters are initialized independently, the weights from the input layer to the hidden layer are normally distributed with mean $ 0 $ and variance $ \nicefrac{2}{\indim} $, the weights from the hidden layer to the output layer are normally distributed with mean $ 0 $ and variance $ \nicefrac{2}{\width} $, and all biases are set to $ 0 $. After that, all parameters except for the weights from the hidden layer to the output layer are optimized with \GD\ using the learning rate $ \lr $. Finally, \nobs that for every $ \ndata \in \N $ the real number $ \lambda_{\ndata} \in \R $ in \cref{thm:introduction:lambda} corresponds to the smallest eigenvalue of an associated Gram matrix (cf.~\cref{lem:lambdazero_positive_onedim} below).

Under these conditions, \cref{thm:introduction} above states that \ref{thm_introduction:item1} for every number of training data pairs $ \ndata \in \N $ it holds that $ \lambda_{\ndata} > 0 $ and that \ref{thm_introduction:item2} there exists a real number $ \Lambda \in (0,1) $ such that for every number of training data pairs $ \ndata \in \N $, every arbitrarily small error probability $ \varepsilon \in (0,1) $, and every \ANN\ with $ \width \in \N $ neurons on the hidden layer, initialized and trained as above with learning rate $ \lr \in (0,\infty) $, satisfying
\begin{equation}\label{introduction:width:lr}
  \lr \le \frac{ \Lambda \poe^2 \min \{ \lambda_{\ndata}, ( \lambda_{\ndata} )^{-1} \} }{ \ndata }
  \qqandqq
  \width \ge \frac{ \ndata^7 }{ \Lambda \poe^{4} \min\{ \lambda_{\ndata} , ( \lambda_{\ndata} )^{5} \} }
\end{equation}
it holds with probability at least $ 1 - \varepsilon $ that in every iteration step $ \iteration \in \N_0 $ the associated empirical risk is bounded by the product of the initial empirical risk and the factor $ ( 1 - \tfrac{\lr \lambda_{\ndata}}{\ndata} )^{\iteration} $. Assuming \cref{introduction:width:lr}, this implies in particular that with probability at least $ 1 - \varepsilon $ the associated empirical risk converges to zero at a linear convergence rate of $ 1 - \tfrac{\lr \lambda_{\ndata}}{\ndata} \in (0,1) $, that is,
\begin{equation}\label{introduction:zero}
  \lim\nolimits_{ \iteration \to \infty } \loss^{\width, \ndata} ( \Theta^{\width,\ndata,\lr} (\iteration) ) = 0.
\end{equation}
For simplicity of presentation, in \cref{thm:introduction} we only assert the existence of $ \Lambda $ and give very rough bounds on the learning rate $ \lr $ and network width $ \width $. However, in the more detailed results in \cref{sec:error_analysis} below, we show the dependency of $ \Lambda $ on the specific activation function, the output dimension, the bounds on the training data, and the initialization method (cf.~\cref{cor:assumption_pairwise_distinct_const} below) and we also give much tighter bounds on $ \lr $ and $ \width $ (cf.~\cref{cor:assumption_pairwise_distinct} below). 

\Nobs that in the more general result in \cref{cor:vectorized} below, in contrast to \cref{thm:introduction}, we do not require positivity of the components of the training input data but only the weaker condition that the norms of the training input data are bounded away from zero. Additionally, in \cref{cor:vectorized} we cover slightly more general initialization methods and, instead of considering the left gradient of the associated risk function, we approximate the piecewise affine activation function $ \act $ with smooth functions, apply the \GD\ method to the resulting smooth risk functions, and then take the limit, which corresponds to more general \AD\ algorithms.

Lastly, we briefly discuss some technical details regarding \cref{thm:introduction}. As we demonstrate in \cref{cor:assumption_pairwise_distinct}, the lower bound on the number of neurons on the hidden layer $ \width $ in \cref{introduction:width:lr} can be relaxed in the sense that it is sufficient for \cref{thm_introduction:item2} in \cref{thm:introduction} to hold if for an arbitrarily small $ \delta \in (0,1) $ we have that
\begin{equation}
  \width \ge \biggl[ \frac{ \ndata^6 }{ \delta \Lambda \poe^{3} \min\{ \lambda_{\ndata} , ( \lambda_{\ndata} )^{4} \} } \biggr]^{1+\delta}.
\end{equation}
This improves the bound in \cref{introduction:width:lr} with regard to the exponents of $ \ndata $, $ \varepsilon $, and $ \lambda_{\ndata} $. \Moreover the assumption in \cref{thm:introduction} that $ \act $ is not affine is necessary since in the case that $ \act $ is affine for every $ \width \in \N $, $ \theta \in \R^{\mc P(\width)} $ the associated realization function $ \mc N_{\theta}^{\width} \colon \R^{\indim} \to \R^{\outdim} $ is also affine.
Combining this with the fact that the subspace
$ \{ f \colon \R^{\indim} \to \R^{\outdim} \colon \text{$f$ is affine} \} \subseteq C( \R^{\indim}, \R^{\outdim} ) $ is closed under the maximum norm on compact sets shows for all $ \width, \ndata \in \N $ that in the case that $ \act $ is affine and there is no affine relationship between the training input data $ \indata_1^{\ndata}, \indata_2^{\ndata}, \ldots, \indata_{\ndata}^{\ndata} $ and the training output data $ \outdata_1^{\ndata}, \outdata_2^{\ndata}, \ldots, \outdata_{\ndata}^{\ndata} $ there exists an $ \varepsilon \in (0,\infty) $ such that
\begin{equation}
  \inf\nolimits_{\theta \in \R^{\mc P(\width)}} \loss^{\width,\ndata} (\theta) > \varepsilon,
\end{equation}
which precludes the convergence of the associated empirical risk to zero established in \cref{introduction:zero}
(cf.~also the universal approximation properties of \ANNs, e.g., in~\cite{HORNIK1991251,HORNIK1989359,LESHNO1993861}).
\Moreover the assumption that the training input data are pairwise distinct is also necessary, since if there exist $ \ndata \in \N $, $ i, j \in \{ 1, 2, \ldots, \ndata \} $ with $ i \neq j $, $ \indata_i^{\ndata} = \indata_j^{\ndata} $, and $ \outdata_i^{\ndata} \neq \outdata_j^{\ndata} $ we have for all $ \width \in \N $ that
\begin{equation}
  \inf\nolimits_{\theta \in \R^{\mc P(\width)}} \loss^{\width,\ndata} (\theta)
  \ge \tfrac{1}{2 \ndata} \eucnorm{ \outdata_i^{\ndata} - \outdata_j^{\ndata} }^2 > 0,
\end{equation}
which again precludes the convergence established in \cref{introduction:zero}.

The remainder of this article is organized as follows. In \cref{sec:mathematical_framework}, we recall a common approach to mathematically describe \ANNs\ and we introduce the setting of \GD\ processes that we often use in our article. In \cref{sec:analysis_of_eigenvalues_of_stochastic_gram_matrices}, we analyze some of the mathematical objects that appear in this setting. In particular, in \cref{sec:analysis_of_eigenvalues_of_stochastic_gram_matrices} we deal with Gram matrices and their eigenvalues, which are essential for the proof of the main result. In \cref{sec:error_analysis}, we combine the results from \cref{sec:analysis_of_eigenvalues_of_stochastic_gram_matrices} to establish error analyses for the considered \GD\ processes.
The arguments in this work are partially inspired by existing techniques in the overparameterized regime, in particular by the arguments in Du et al.~\cite{du2019gradienta}.

\section{Mathematical framework for gradient descent (GD) optimization algorithms}
\label{sec:mathematical_framework}

This section is devoted to the mathematical description of \ANNs\ (cf.~\cref{subsec:mathematical_description_of_ANNs} below), the introduction of the mathematical framework for \GD\ optimization algorithms (cf.~\cref{subsec:mathematical_description_of_gd_processes} below), and the analysis of the smallest eigenvalues of specific deterministic Gram matrices (cf.~\cref{subsec:positive_definiteness_of_deterministic_gram_matrices} below). For the understanding of our main setting, which we often impose in this article, \cref{setting:gradient_descent} below, we show in \cref{lem:gradient_descent_motivation} that the specific algorithm used in \cref{setting:gradient_descent} arises directly from the gradients of the considered risk functions.

In \cref{lem:graminf} and \cref{lem:gram_matrix} we show that the considered matrices are indeed Gram matrices and explain the connection between the deterministic Gram matrices $ \gramGinf $ and the stochastic Gram matrices $ \gramG( \iteration ) $, $ \iteration \in \N_0 $.
In \cref{lem:lambdazero_positive_onedim} we prove in the case of one-dimensional training output data that the deterministic Gram matrices have only positive eigenvalues if the considered training input data are pairwise distinct and the piecewise affine activation function $ \act $ is not affine. In our proof of \cref{lem:lambdazero_positive_onedim} we use a well-known property of Gram matrices, which we establish in \cref{lem:properties_gram_matrix}, and an isolation result on hyperplanes, which we establish in \cref{lem:hyperplane_isolation}.
Lastly, in the elementary result in \cref{lem:eigenvalues_block_matrix} we show that the general case of multidimensional training output data can be reduced to the one-dimensional case to establish in \cref{lem:lambdazero_positive} that, under the same conditions as above, all eigenvalues of the deterministic Gram matrices are positive. Only for completeness we include the proofs of \cref{lem:properties_gram_matrix} and \cref{lem:eigenvalues_block_matrix} in this section.

\subsection{Mathematical description of artificial neural networks (ANNs)}
\label{subsec:mathematical_description_of_ANNs}

\cfclear
\begin{definition}[Standard scalar product and norm]
  \label{def:scalar_product_norm}
  For every $ \indim \in \N $, $ u = ( u_1, \ldots, u_{\indim} ) $, $ v = ( v_1, \ldots, v_{\indim} ) \in \R^{\indim} $ we denote by
  $ \scalprod{ u }{ v } \in \R $ and $ \eucnorm{ v } \in \R $ the real numbers which satisfy that
  $ \scalprod{ u }{ v } = \sum_{i=1}^{\indim} u_i v_i $ and
  $ \eucnorm{ v } = ( \sum_{i=1}^{\indim} \abs{ v_i }^2 )^{\nicefrac{1}{2}} $.
\end{definition}

\cfclear
\begin{definition}[\ANNs]
  \label{def:ANN}
  For every $ \indim, \width, \outdim \in \N $ we denote by $ \network_{ \indim, \width, \outdim } $ the set given by
  $ \network_{ \indim, \width, \outdim } = ( ( \R^{ \width \times \indim } \times \R^{\width} ) \times (\R^{\outdim \times \width} \times \R^{\outdim} ) ) $.
\end{definition}

\cfclear
\begin{definition}[Multidimensional version]
  \label{def:multidimensional_version}
  For every $ \indim \in \N $ and every function $ \act \colon \R \to \R $ we denote by
  $ \multdim{\act}{\indim} \colon \R^{\indim} \to \R^{\indim} $ the function which satisfies for all
  $ v = ( v_1, \ldots, v_{\indim} ) \in \R^{\indim} $ that
  \begin{equation}
    \multdim{\act}{\indim}(v) = ( \act(v_1), \ldots, \act(v_{\indim}) ).
  \end{equation}
\end{definition}

\cfclear
\cfconsiderloaded{def:realization_ANN}
\begin{definition}[Realization associated with an \ANN]
  \label{def:realization_ANN}
  For every function $ \act \colon \R \to \R $ and
  every $ \indim, \width, \outdim \in \N $, $ \Phi = ((W, B), (\mf W, \mf B)) \in \network_{ \indim, \width, \outdim } $ we denote by
  $ \functionANN{\act}(\Phi) = ( \functionANN{\act}^1 (\Phi), \ldots, \functionANN{\act}^{\outdim} (\Phi) ) \colon \R^{\indim} \to \R^{\outdim} $
  the function which satisfies for all $ v \in \R^{\indim} $ that
  \begin{equation}
    \bigl( \functionANN{\act}(\Phi) \bigr) (v) = \mf W \bigl[ \multdim{\act}{\width} ( W v + B) \bigr] + \mf B
  \end{equation}
  \cfload.
\end{definition}

\subsection{Gradients of the considered risk functions}
\label{subsec:gradients_of_risk_functions}

\cfclear
\begin{lemma}
  \label{lem:gradient_descent_motivation}
  Let $ \nbp \in \N $, $ \bp_0, \bp_1, \dots, \bp_{\nbp+1} \in [-\infty,\infty] $,
  $ \slope_1, \slope_2, \dots, \slope_{\nbp+1}, \yinter_1, \yinter_2, \dots, \yinter_{\nbp+1} \in \R $ satisfy
  $ -\infty = \bp_0 < \bp_1 < \dots < \bp_{\nbp+1} = \infty $,
  let $ \deriv \colon \R \to \R $ and $ \act \in C( \R, \R ) $ satisfy for all
  $ i \in \{ 1, 2, \dots, \nbp+1 \} $, $ v \in ( \bp_{i-1}, \bp_i ) $
  that
  \begin{equation}
    \act( v ) = \slope_i v + \yinter_i
    \qqandqq
    \deriv( v ) = \slope_i,
  \end{equation}
  let $ \indim, \width, \outdim, \ndata \in \N $,
  $ \indata = ( ( \indata_i^p )_{ p \in \{ 1, 2, \ldots, \indim \} } )_{ i \in \{ 1, 2, \ldots, \ndata \} } \in ( \R^{\indim} )^{\ndata} $,
  $ \outdata = ( ( \outdata_i^p )_{ p \in \{ 1, 2, \ldots, \outdim \} } )_{ i \in \{ 1, 2, \ldots, \ndata \} } \in ( \R^{\outdim} )^{\ndata} $,
  $ W = ( W_1, \ldots, W_{\width} ) = (W_{i,j})_{(i,j) \in \{ 1, 2, \ldots, \width \} \times
    \{ 1, 2, \ldots, \indim \}} \in \R^{\width \times \indim} $,
  $ B = ( B_1, \ldots, B_{\width} ) \in \R^{\width} $,
  $ \mf W = ( \mf W_{i,j} )_{ (i,j) \in \{ 1, 2, \ldots, \outdim \} \times \{ 1, 2, \ldots, \width \} } \in \R^{\outdim \times \width} $,
  $ \mf B = ( \mf B_1, \ldots, \mf B_{\outdim} ) \in \R^{\outdim} $
  satisfy for all $ k \in \{ 1, 2, \dots, \width \} $, $ i \in \{ 1, 2, \dots, \ndata \} $ that
  $ \scalprod{ W_k }{ \indata_i } + B_k \not\in \{ \bp_1, \bp_2, \ldots, \bp_{\nbp} \} $,
  let $ \Phi \in \network_{ \indim, \width, \outdim } $ satisfy $ \Phi = ((W,B),(\mf W, \mf B)) $,
  let $ \loss \colon \network_{ \indim, \width, \outdim } \to \R $ satisfy for all
  $ \Psi \in \network_{ \indim, \width, \outdim } $ that
  \begin{equation}\label{eq:gradient_descent_motivation_E}
    \loss (\Psi) = \tfrac{1}{\ndata} \smallsum\limits_{i=1}^{\ndata} \eucnorm{ ( \functionANN{\act} (\Psi)) (\indata_i) - \outdata_i }^2,
  \end{equation}
  and let
  $ j \in \{ 1, 2, \ldots, \indim \} $, $ k \in \{ 1, 2, \ldots, \width \} $, $ \ell \in \{ 1, 2, \ldots, \outdim \} $
  \cfload.
  Then
  \begin{enumerate}[label=(\roman{*})]
    \item \label{item:gradient_descent_motivation_1} it holds that
      $ \loss $ is differentiable at $ \Phi $,
    \item \label{item:gradient_descent_motivation_2} it holds that
      $ \frac{\partial \loss (\Phi)}{\partial W_{k,j}}
        = \tfrac{2}{\ndata} \smallsum\limits_{i=1}^{\ndata} \smallsum\limits_{p=1}^{\outdim}
        \bigl( ( \functionANN{\act}^p (\Phi) )(\indata_i) - \outdata_i^p \bigr)
        \mf W_{p,k} \deriv \bigl( \scalprod{W_k}{\indata_i} + B_k \bigr) \indata_i^j, $
    \item \label{item:gradient_descent_motivation_3} it holds that
      $ \frac{\partial \loss (\Phi)}{\partial B_k}
        = \tfrac{2}{\ndata} \smallsum\limits_{i=1}^{\ndata} \smallsum\limits_{p=1}^{\outdim}
        \bigl( ( \functionANN{\act}^p (\Phi) )(\indata_i) - \outdata_i^p \bigr)
        \mf W_{p,k} \deriv \bigl( \scalprod{W_k}{\indata_i} + B_k \bigr), $
      and
    \item \label{item:gradient_descent_motivation_4} it holds that
      $ \frac{\partial \loss (\Phi)}{\partial \mf B_{\ell}}
      = \frac{2}{\ndata} \smallsum\limits_{i=1}^{\ndata}
      \bigl( ( \functionANN{\act}^{\ell} (\Phi) )(\indata_i) - \outdata_i^{\ell} \bigr) $.
  \end{enumerate}
\end{lemma}

\begin{cproof}{lem:gradient_descent_motivation}
  First, \nobs that \cref{eq:gradient_descent_motivation_E} \proves that
  \begin{equation}
  \label{eq:error_phi}
  \begin{split}
    \loss (\Phi)
    &= \tfrac{1}{\ndata} \smallsum\limits_{i=1}^{\ndata} \beucnorm{ (\functionANN{\act} (\Phi)) (\indata_i) - \outdata_i }^2
    = \tfrac{1}{\ndata} \smallsum\limits_{i=1}^{\ndata} \smallsum\limits_{p=1}^{\outdim}
      \babs{ (\functionANN{\act}^p (\Phi)) (\indata_i) - \outdata_i^p }^2 \\
    &= \tfrac{1}{\ndata} \smallsum\limits_{i=1}^{\ndata} \smallsum\limits_{p=1}^{\outdim}
      \biggl( \smallsum\limits_{q=1}^{\width}
      \mf W_{p,q} \, \act( \scalprod{ W_q }{ \indata_i } + B_q ) + \mf B_p - \outdata_i^p \biggr)^2 \\
    &= \tfrac{1}{\ndata} \smallsum\limits_{i=1}^{\ndata} \smallsum\limits_{p=1}^{\outdim}
      \biggl( \smallsum\limits_{q=1}^{\width} \mf W_{p,q} \,
      \act \biggl( \smallsum\limits_{r=1}^{\indim} W_{q,r} \indata_i^r + B_q \biggr) + \mf B_p - \outdata_i^p \biggr)^2.
  \end{split}
  \end{equation}
  This, the fact that for all $ v \in \R \backslash \{ \bp_1, \bp_2, \dots, \bp_{\nbp} \} $ it holds that
  $ \R \ni u \mapsto \act ( u ) \in \R $ is differentiable at $ v $,
  the assumption that for all $ q \in \{ 1, 2, \ldots, \width \} $, $ i \in \{ 1, 2, \ldots, \ndata \} $
  it holds that $ \scalprod{ W_q }{ \indata_i} + B_q \in \R \backslash \{ \bp_1, \bp_2, \ldots, \bp_{\nbp} \} $,
  and the fact that sums and compositions of differentiable functions are differentiable
  \prove \cref{item:gradient_descent_motivation_1}.
  \Moreover \cref{eq:error_phi} and the fact that for all $ v \in \R \backslash \{ \bp_1, \bp_2, \ldots, \bp_{\nbp} \} $
  it holds that $ \frac{\partial \act ( v )}{\partial v} = \deriv( v ) $ \prove that it holds that
  \begin{equation}
  \begin{split}
    \frac{\partial \loss (\Phi)}{\partial W_{k,j}}
    &= \tfrac{2}{\ndata} \smallsum\limits_{i=1}^{\ndata} \smallsum\limits_{p=1}^{\outdim}
      \bigl( ( \functionANN{\act}^p (\Phi)) (\indata_i) - \outdata_i^p \bigr)
      \displaystyle\frac{\partial}{\partial W_{k,j}} \biggl( \smallsum\limits_{q=1}^{\width} \mf W_{p,q} \,
      \act \biggl( \smallsum\limits_{r=1}^{\indim} W_{q,r} \indata_i^r + B_q \biggr) + \mf B_p - \outdata_i^p \biggr) \\
    &= \tfrac{2}{\ndata} \smallsum\limits_{i=1}^{\ndata} \smallsum\limits_{p=1}^{\outdim}
      \bigl( ( \functionANN{\act}^p (\Phi)) (\indata_i) - \outdata_i^p \bigr)
      \mf W_{p,k} \deriv \bigl( \scalprod{W_k}{\indata_i} + B_k \bigr) \indata_i^j,
  \end{split}
  \end{equation}
  \begin{equation}
  \begin{split}
    \frac{\partial \loss (\Phi)}{\partial B_k}
    &= \tfrac{2}{\ndata} \smallsum\limits_{i=1}^{\ndata} \smallsum\limits_{p=1}^{\outdim}
      \bigl( ( \functionANN{\act}^p (\Phi)) (\indata_i) - \outdata_i^p \bigr)
      \displaystyle\frac{\partial}{\partial B_k} \biggl( \smallsum\limits_{q=1}^{\width} \mf W_{p,q} \,
      \act \biggl( \smallsum\limits_{r=1}^{\indim} W_{q,r} \indata_i^r + B_q \biggr) + \mf B_p - \outdata_i^p \biggr) \\
    &= \tfrac{2}{\ndata} \smallsum\limits_{i=1}^{\ndata} \smallsum\limits_{p=1}^{\outdim}
      \bigl( ( \functionANN{\act}^p (\Phi)) (\indata_i) - \outdata_i^p \bigr)
      \mf W_{p,k} \deriv \bigl( \scalprod{W_k}{\indata_i} + B_k \bigr),
  \end{split}
  \end{equation}
  and
  \begin{equation}
  \begin{split}
    \frac{\partial \loss (\Phi)}{\partial \mf B_{\ell}}
    &= \tfrac{2}{\ndata} \smallsum\limits_{i=1}^{\ndata} \smallsum\limits_{p=1}^{\outdim}
      \bigl( ( \functionANN{\act}^p (\Phi)) (\indata_i) - \outdata_i^p \bigr)
      \displaystyle\frac{\partial}{\partial \mf B_{\ell}} \biggl( \smallsum\limits_{q=1}^{\width} \mf W_{p,q} \,
      \act \biggl( \smallsum\limits_{r=1}^{\indim} W_{q,r} \indata_i^r + B_q \biggr) + \mf B_p - \outdata_i^p \biggr) \\
    &= \tfrac{2}{\ndata} \smallsum\limits_{i=1}^{\ndata}
      \bigl( ( \functionANN{\act}^{\ell} (\Phi)) (\indata_i) - \outdata_i^{\ell} \bigr).
  \end{split}
  \end{equation}
  This \proves items \ref{item:gradient_descent_motivation_2}, \ref{item:gradient_descent_motivation_3}, and
  \ref{item:gradient_descent_motivation_4}.
\end{cproof}

\subsection{Mathematical description of GD processes}
\label{subsec:mathematical_description_of_gd_processes}

\cfclear
\begin{definition}[Smallest eigenvalue]
  \label{def:lambdamin}
  For every $ n \in \N $, $ A \in \R^{ n \times n } $ we denote by $ \lambdamin( A ) \in [-\infty,\infty] $
  the extended real number which satisfies that
  \begin{equation}
    \lambdamin ( A )
    = \min \bigl( \bigl\{ \lambda \in \R \colon
      [ \Exists v \in \R^{n} \backslash \{ 0 \} \colon A v = \lambda v ] \bigr\} \cup \{ \infty \} \bigr).
  \end{equation}
\end{definition}

\cfclear
\begin{setting}
  \label{setting:gradient_descent}
  Let $ \nbp \in \N $, $ \bp_0, \bp_1, \dots, \bp_{\nbp+1} \in [-\infty,\infty] $,
  $ \slope = ( \slope_1, \dots, \slope_{\nbp+1} ) \in \R^{\nbp+1} \backslash \{ 0 \} $,
  $ \yinter = ( \yinter_1, \dots, \yinter_{\nbp+1} ) \in \R^{\nbp+1} $,
  $ \cderiv \in (0,\infty) $
  satisfy $ -\infty = \bp_0 < \bp_1 < \dots < \bp_{\nbp+1} = \infty $,
  let $ \deriv \colon \R \to \R $ and $ \act \in C( \R, \R ) $ satisfy for all
  $ i \in \{ 1, 2, \dots, \nbp+1 \} $, $ v \in ( \bp_{i-1}, \bp_i ) $
  that
  \begin{equation}
    \act( v ) = \slope_i v + \yinter_i, \qquad
    \deriv( v ) = \slope_i, \qqandqq
    \sup\nolimits_{ u \in \R } \abs{ \deriv( u ) } \le \cderiv,
  \end{equation}
  for every $ R \in (0,\infty) $ let $ \mc U_R \subseteq \R $ satisfy $ \mc U_R = \smallbigcup_{i=1}^{\nbp+1} [ \bp_i - R, \bp_i + R ] $,
  let $ \indim, \width, \outdim, \ndata \in \N $, $ \lr, \cmin, \cmax, \cout, \cvar, \Cvar \in (0,\infty) $,
  $ \indata = ( ( \indata_i^p )_{ p \in \{ 1, 2, \ldots, \indim \} } )_{ i \in \{ 1, 2, \ldots, \ndata \} } \in ( \R^{\indim} )^{\ndata} $,
  $ \outdata = ( ( \outdata_i^p )_{ p \in \{ 1, 2, \ldots, \outdim \} } )_{ i \in \{ 1, 2, \ldots, \ndata \} } \in ( \R^{\outdim} )^{\ndata} $
  satisfy for all $ i \in \{ 1, 2, \ldots, \ndata \} $ that
  $ \cmin \le \eucnorm{ \indata_i } \le \cmax $ and $ \eucnorm{ \outdata_i } \le \cout $,
  let $ (\Omega, \mc F, \P) $ be a probability space, let
  $ W = ( W_1, \ldots, W_{\width} ) = ( W_{i,j} )_{ (i,j) \in \{ 1, 2, \ldots, \width \} \times \{ 1, 2, \ldots, \indim \} }
    \colon \N_0 \times \Omega \to \R^{\width \times \indim} $,
  $ B = ( B_1, \ldots, B_{\width} ) \colon \N_0 \times \Omega \to \R^{\width} $,
  $ \mf W = ( \mf W_1, \ldots, \mf W_{\outdim} ) = ( \mf W_{i,j} )_{ (i,j) \in \{ 1, 2, \ldots, \outdim \} \times \{ 1, 2, \ldots, \width \} }
    \colon \Omega \to \R^{\outdim \times \width} $, and
  $ \mf B = ( \mf B_1, \ldots, \mf B_{\outdim} ) \colon \N_0 \times \Omega \to \R^{\outdim} $
  be stochastic processes,
  let $ \Phi \colon \N_0 \times \Omega \to \network_{ \indim, \width, \outdim } $ and
  $ \netvalue = ( ( \netvalue_i^p )_{ p \in \{ 1, 2, \ldots, \outdim \} } )_{ i \in \{ 1, 2, \ldots, \ndata \} }
    \colon \N_0 \times \Omega \to ( \R^{\outdim} )^{\ndata} $ 
  satisfy for all $ i \in \{ 1, 2, \ldots, \ndata \} $, $ \iteration \in \N_0 $, $ \omega \in \Omega $ that
  \begin{equation}
    \Phi (\iteration,\omega) = \bigl( (W(\iteration,\omega), B(\iteration,\omega)), (\mf W (\omega), \mf B(\iteration,\omega)) \bigr)
    \qqandqq
    \netvalue_i (\iteration,\omega) = \bigl( \functionANN{\act} ( \Phi(\iteration,\omega) ) \bigr) (\indata_i),
  \end{equation}
  assume that
  $ \sqrt{\nicefrac{1}{\cvar}} W_1(0), \allowbreak
    \sqrt{\nicefrac{1}{\cvar}} W_2(0), \allowbreak \ldots, \allowbreak
    \sqrt{\nicefrac{1}{\cvar}} W_{\width}(0), \allowbreak
    \sqrt{\nicefrac{1}{\Cvar}} \mf W_1, \allowbreak
    \sqrt{\nicefrac{1}{\Cvar}} \mf W_2, \allowbreak \ldots, \allowbreak
    \sqrt{\nicefrac{1}{\Cvar}} \mf W_{\outdim} $
  are independent and standard normal,
  assume for all $ j \in \{ 1, 2, \ldots, \indim \} $, $ k \in \{ 1, 2, \ldots, \width \} $, $ \ell \in \{ 1, 2, \ldots, \outdim \} $,
  $ \iteration \in \N_0 $, $ \omega \in \Omega $ that
  \begin{equation}
  \label{eq:setting_W}
    W_{k,j}(\iteration+1,\omega) = W_{k,j}(\iteration,\omega) - \frac{2 \lr}{\ndata} \biggl(
      \smallsum\limits_{i=1}^{\ndata} \smallsum\limits_{p=1}^{\outdim} ( \netvalue_i^p (\iteration,\omega) - \outdata_i^p ) \mf W_{p,k}(\omega)
      \deriv \bigl( \scalprod{W_k(\iteration,\omega)}{\indata_i} + B_k(\iteration,\omega) \bigr) \, \indata_i^j \biggr),
  \end{equation}
  \begin{equation}
  \label{eq:setting_B}
    B_k(\iteration+1,\omega) = B_k(\iteration,\omega) - \frac{2 \lr}{\ndata} \biggl(
      \smallsum\limits_{i=1}^{\ndata} \smallsum\limits_{p=1}^{\outdim} ( \netvalue_i^p (\iteration,\omega) - \outdata_i^p ) \mf W_{p,k}(\omega)
      \deriv \bigl( \scalprod{W_k(\iteration,\omega)}{\indata_i} + B_k(\iteration,\omega) \bigr) \biggr),
  \end{equation}
  \begin{equation}
  \label{eq:setting_mf_B}
    \mf B_{\ell} (\iteration+1,\omega) = \mf B_{\ell} (\iteration,\omega) - \frac{2 \lr}{\ndata} \biggl(
      \smallsum\limits_{i=1}^{\ndata} ( \netvalue_i^{\ell} (\iteration,\omega) - \outdata_i^{\ell} ) \biggr),
  \end{equation}
  and $ \eucnorm{ B(0,\omega) } = \eucnorm{ \mf B(0,\omega) } = 0 $,
  let $ \fnetvalue = ( \fnetvalue_i )_{ i \in \{ 1, 2, \ldots, \outdim \ndata \} } \colon \N_0 \times \Omega \to \R^{\outdim \ndata} $
  and $ \gramG = ( \gramG_{i,j} )_{ (i,j) \in \{ 1, 2, \ldots, \outdim \ndata \}^2 }
    \colon \N_0 \times \Omega \to \R^{(\outdim \ndata) \times (\outdim \ndata)} $
  satisfy for all $ i, j \in \{ 1, 2, \ldots, \ndata \} $, $ p,q \in \{ 1, 2, \ldots, \outdim \} $,
  $ \iteration \in \N_0 $, $ \omega \in \Omega $ that
  $ \fnetvalue_{ (i-1) \outdim + p } ( \iteration, \omega ) = \netvalue_i^p ( \iteration, \omega ) $
  and
  \begin{equation}\label{eq:setting_gramG}
  \begin{split}
    \gramG_{ (i-1) \outdim + p, (j-1) \outdim + q } (\iteration,\omega)
    &= ( 1 + \scalprod{\indata_i}{\indata_j} ) \smallsum\limits_{k=1}^{\width}
      \mf W_{p,k} (\omega) \mf W_{q,k} (\omega)
      \deriv \bigl( \scalprod{W_k(\iteration,\omega)}{\indata_i} + B_k(\iteration,\omega) \bigr) \\
    &\quad \times \deriv \bigl( \scalprod{W_k(\iteration,\omega)}{\indata_j} + B_k(\iteration,\omega) \bigr),
  \end{split}
  \end{equation}
  and let
  $ \gramGinf = ( \gramGinf_{i,j} )_{ (i,j) \in \{ 1, 2, \ldots, \outdim \ndata \}^2 } \in \R^{(\outdim \ndata) \times (\outdim \ndata)} $,
  $ \findata = ( \findata_i )_{ i \in \{ 1, 2, \ldots, \indim \ndata \} } \in \R^{\indim \ndata} $,
  $ \foutdata = ( \foutdata_i )_{ i \in \{ 1, 2, \ldots, \outdim \ndata \} } \in \R^{\outdim \ndata} $,
  $ \lambdazero \in \R $
  satisfy for all $ i, j \in \{ 1, 2, \ldots, \ndata \} $, $ p, q \in \{ 1, 2, \ldots, \outdim \}$, $ \ell \in \{ 1, 2, \ldots, \indim \}$
  that
  \begin{equation}
    \gramGinf_{ (i-1) \outdim + p, (j-1) \outdim + q }
    = \Cvar \width \bigl( 1 + \scalprod{\indata_i}{\indata_j} \bigr)
      \E \bigl[ \deriv \bigl( \scalprod{W_1(0)}{\indata_i} \bigr) \deriv \bigl( \scalprod{W_1(0)}{\indata_j} \bigr) \bigr] \ind_{ \{ p \} } (q),
  \end{equation}
  $ \findata_{ (i-1) \indim + \ell } = \indata_i^{\ell} $,
  $ \foutdata_{ (i-1) \outdim + p } = \outdata_i^p $, and
  $ \lambdazero = \lambdamin( \tfrac{1}{\Cvar \width} \gramGinf ) $
  \cfload.
\end{setting}

\subsection{Connection between deterministic and stochastic Gram matrices}
\label{subsec:connection_between_gram_matrices}

\begin{definition}[Gram matrix]
  \label{def:gram_matrix}
  Let $ n \in \N $, $ G = ( G_{i,j} )_{ (i,j) \in \{ 1, 2, \ldots, n \}^2 } \in \R^{n \times n} $.
  Then we say that $ G $ is a \gram\ matrix if and only if there exist a vector space $ V $ over $ \R $,
  an inner product space $ \varphi \colon V \times V \to \R $ on $ V $, and $ v_1, v_2, \ldots, v_n \in V $
  such that for all $ i, j \in \{ 1, 2, \ldots, n \}$ it holds that
  \begin{equation}
    G_{i,j} = \varphi( v_i, v_j ).
  \end{equation}
\end{definition}

\cfclear
\begin{lemma}
  \label{lem:graminf}
  Assume \cref{setting:gradient_descent}.
  Then it holds for all $ i, j \in \{ 1, 2, \ldots, \outdim \ndata \} $ that
  $ \E \big[ \gramG_{i,j} (0) \big] = \gramGinf_{i,j} $.
\end{lemma}

\begin{cproof}{lem:graminf}
  \Nobs that the assumption that
  $ \sqrt{\nicefrac{1}{\Cvar}} \mf W_1, \allowbreak
    \sqrt{\nicefrac{1}{\Cvar}} \mf W_2, \allowbreak \ldots, \allowbreak
    \sqrt{\nicefrac{1}{\Cvar}} \mf W_{\outdim} $
  are independent and standard normal
  \proves that for all $ k \in \{ 1, 2, \ldots, \width \} $, $ p, q \in \{ 1, 2, \ldots, \outdim \}$ with $ p \neq q $ it holds that
  \begin{equation}\label{lem:graminf:zero}
    \E[ \mf W_{p,k} \mf W_{q,k} ]
    = \E[ \mf W_{p,k} ] \E [ \mf W_{q,k} ]
    = 0.
  \end{equation}
  \Moreover the fact that for all $ k \in \{ 1, 2, \ldots, \width \}$, $ p \in \{ 1, 2, \ldots, \outdim \} $ it holds that
  $ \mf W_{p,k} $ is a centered normal random variable with $ \Var[ \mf W_{p,k} ] = \Cvar $
  \proves that for all $ k \in \{ 1, 2, \ldots, \width \} $, $ p \in \{ 1, 2, \ldots, \outdim \} $ it holds that
  \begin{equation}\label{lem:graminf:var}
    \E \bigl[ \abs{ \mf W_{p,k} }^2 \bigr] = \Cvar.
  \end{equation}
  \Moreover that the assumption that
  $ \sqrt{\nicefrac{1}{\cvar}} W_1(0), \allowbreak
    \sqrt{\nicefrac{1}{\cvar}} W_2(0), \allowbreak \ldots, \allowbreak
    \sqrt{\nicefrac{1}{\cvar}} W_{\width}(0), \allowbreak
    \sqrt{\nicefrac{1}{\Cvar}} \mf W_1, \allowbreak
    \sqrt{\nicefrac{1}{\Cvar}} \mf W_2, \allowbreak \ldots, \allowbreak
    \sqrt{\nicefrac{1}{\Cvar}} \mf W_{\outdim} $
  are independent \proves that for all
  $ i, j \in \{ 1, 2, \ldots, \ndata \} $, $ k \in \{ 1, 2, \ldots, \width \} $, $ p, q \in \{ 1, 2, \ldots, \outdim \}$ it holds that
  \begin{equation}
    \mf W_{p,k} \mf W_{q,k} 
    \qqandqq
    \deriv( \scalprod{W_k(0)}{\indata_i} ) \deriv( \scalprod{W_k(0)}{\indata_j} )
  \end{equation}
  are independent.
  Combining this with \cref{eq:setting_gramG}, \cref{lem:graminf:zero}, \cref{lem:graminf:var},
  the assumption that $ \eucnorm{ B(0) } = 0 $, and
  the assumption that
  $ \sqrt{\nicefrac{1}{\cvar}} W_1(0), \allowbreak
    \sqrt{\nicefrac{1}{\cvar}} W_2(0), \allowbreak \ldots, \allowbreak
    \sqrt{\nicefrac{1}{\cvar}} W_{\width}(0) $
  are identically distributed
  \proves that for all $ i, j \in \{ 1, 2, \ldots, \ndata \} $, $ p, q \in \{ 1, 2, \ldots, \outdim \} $ it holds that
  \begin{equation}
  \begin{split}
    &\E \bigl[ \gramG_{ (i-1) \outdim + p, (j-1) \outdim + q } (0) \bigr] \\
    &= \E \! \biggl[ \bigl( 1 + \scalprod{\indata_i}{\indata_j} \bigr) \smallsum\limits_{k=1}^{\width}
      \mf W_{p,k} \mf W_{q,k}
       \deriv \bigl( \scalprod{W_k(0)}{\indata_i} + B_k(0) \bigr)
       \deriv \bigl( \scalprod{W_k(0)}{\indata_j} + B_k(0) \bigr) \biggr] \\
    &= \bigl( 1 + \scalprod{\indata_i}{\indata_j} \bigr) \smallsum\limits_{k=1}^{\width} \E \bigl[
       \mf W_{p,k} \mf W_{q,k}
       \deriv \bigl( \scalprod{W_k(0)}{\indata_i} \bigr)
       \deriv \bigl( \scalprod{W_k(0)}{\indata_j} \bigr) \bigr] \\
    &= \bigl( 1 + \scalprod{\indata_i}{\indata_j} \bigr) \smallsum\limits_{k=1}^{\width} \E \bigl[ \mf W_{p,k} \mf W_{q,k} \bigr]
      \E \bigl[ \deriv \bigl( \scalprod{W_k(0)}{\indata_i} \bigr) \deriv \bigl( \scalprod{W_k(0)}{\indata_j} \bigr) \bigr] \\
    &= \bigl( 1 + \scalprod{\indata_i}{\indata_j} \bigr) \smallsum\limits_{k=1}^{\width} \Cvar \ind_{ \{ p \} } (q)
      \E \bigl[ \deriv \bigl( \scalprod{W_1(0)}{\indata_i} \bigr) \deriv \bigl( \scalprod{W_1(0)}{\indata_j} \bigr) \bigr] \\
    &= \Cvar \width \bigl( 1 + \scalprod{\indata_i}{\indata_j} \bigr) \E \bigl[
      \deriv \bigl( \scalprod{W_1(0)}{\indata_i} \bigr) \deriv \bigl( \scalprod{W_1(0)}{\indata_j} \bigr) \bigr] \ind_{ \{ p \} } (q)
    = \gramGinf_{ (i-1) \outdim + p, (j-1) \outdim + q }
  \end{split}
  \end{equation}
  \cfload.
  \Hence that for all $ i, j \in \{ 1, 2, \ldots, \outdim \ndata \} $ it holds that
  $ \E[ \gramG_{i,j} (0) ] = \gramGinf_{i,j} $.
\end{cproof}

\cfclear
\begin{lemma}
  \label{lem:gram_matrix}
  Assume \cref{setting:gradient_descent}.
  Then it holds for all $ \iteration \in \N_0 $, $ \omega \in \Omega $ that
  $ \gramG( \iteration, \omega ) $ and $ \gramGinf $ are \gram\ matrices
  \cfout.
\end{lemma}

\begin{cproof}{lem:gram_matrix}
  Throughout this proof
  let $ V $ satisfy $ V = L^2( \Omega, \mc F, \P; \R^{(\indim+1) \width} ) $,
  let $ \varphi \colon V \times V \to \R $ satisfy for all $ X, Y \in V $ that
  \begin{equation}
    \varphi( X, Y ) = \int_{\Omega} \bscalprod{ X(\omega) }{ Y(\omega) } \, \P( \diff \omega ),
  \end{equation}
  let $ z = ( ( z_i^{\ell} )_{ \ell \in \{ 1, 2, \ldots \indim+1 \} } )_{ i \in \{ 1, 2, \ldots \ndata \} } \in ( \R^{\indim+1} )^{\ndata} $
  satisfy for all $ i \in \{ 1, 2, \ldots, \ndata \} $ that $ z_i = ( \indata_i, 1 ) $ and
  let $ v_i = ( v_i^k )_{ k \in \{ 1, 2, \ldots, (\indim+1) \width \} } \colon \N_0 \times \Omega \to \R^{(\indim+1) \width} $,
  $ i \in \{ 1, 2, \ldots, \outdim \ndata \} $,
  satisfy for all $ i \in \{ 1, 2, \ldots, \ndata \} $, $ p \in \{ 1, 2, \ldots, \outdim \} $,
  $ k \in \{ 1, 2, \ldots, \width \} $, $ \ell \in \{ 1, 2, \ldots, \indim+1 \} $, $ \iteration \in \N_0 $, $ \omega \in \Omega $ that
  \begin{equation}\label{lem:gram_matrix:v}
    v_{ (i-1) \outdim + p }^{ (k-1) (\indim+1) + \ell } ( \iteration, \omega )
    = \mf W_{p,k} (\omega) \deriv \bigl( \scalprod{W_k(\iteration,\omega)}{\indata_i} + B_k(\iteration,\omega) \bigr) z_i^{\ell}
  \end{equation}
  \cfload.
  \Nobs that \cref{lem:gram_matrix:v} \proves that
  for all $ i, j \in \{ 1, 2, \ldots, \ndata \} $, $ p, q \in \{ 1, 2, \ldots, \outdim \} $,
  $ \iteration \in \N_0 $, $ \omega \in \Omega $ it holds that
  \begin{equation}\label{lem:gram_matrix:gramG}
  \begin{split}
    &\gramG_{ (i-1) \outdim + p, (j-1) \outdim + q } (\iteration,\omega) \\
    &= ( 1 + \scalprod{\indata_i}{\indata_j} ) \smallsum\limits_{k=1}^{\width}
      \mf W_{p,k} (\omega) \mf W_{q,k} (\omega)
      \deriv \bigl( \scalprod{W_k(\iteration,\omega)}{\indata_i} + B_k(\iteration,\omega) \bigr)
      \deriv \bigl( \scalprod{W_k(\iteration,\omega)}{\indata_j} + B_k(\iteration,\omega) \bigr) \\
    &= \smallsum\limits_{\ell=1}^{\indim+1} z_i^{\ell} z_j^{\ell} \smallsum\limits_{k=1}^{\width}
      \mf W_{p,k} (\omega) \mf W_{q,k} (\omega)
      \deriv \bigl( \scalprod{W_k(\iteration,\omega)}{\indata_i} + B_k(\iteration,\omega) \bigr)
      \deriv \bigl( \scalprod{W_k(\iteration,\omega)}{\indata_j} + B_k(\iteration,\omega) \bigr) \\
    &= \smallsum\limits_{k=1}^{\width} \smallsum\limits_{\ell=1}^{\indim+1}
      \Bigl( \mf W_{p,k} (\omega) \deriv \bigl( \scalprod{W_k(\iteration,\omega)}{\indata_i} + B_k(\iteration,\omega) \bigr) z_i^{\ell} \Bigr)
      \Bigl( \mf W_{q,k} (\omega) \deriv \bigl( \scalprod{W_k(\iteration,\omega)}{\indata_j} + B_k(\iteration,\omega) \bigr) z_j^{\ell} \Bigr)\\
    &= \smallsum\limits_{k=1}^{\width} \smallsum\limits_{\ell=1}^{\indim+1}
      v_{ (i-1) \outdim + p }^{ (k-1) (\indim+1) + \ell } ( \iteration, \omega )
      v_{ (j-1) \outdim + q }^{ (k-1) (\indim+1) + \ell } ( \iteration, \omega )
    = \bscalprod{ v_{ (i-1) \outdim + p } ( \iteration, \omega ) }{ v_{ (j-1) \outdim + q } ( \iteration, \omega ) }.
  \end{split}
  \end{equation}
  \Hence that for all $ i, j \in \{ 1, 2, \ldots, \outdim \ndata \} $, $ \iteration \in \N_0 $, $ \omega \in \Omega $ it holds that
  \begin{equation}\label{lem:gram_matrix:gramG_scalprod}
    \gramG_{i,j} ( \iteration, \omega )
    = \bscalprod{ v_i ( \iteration, \omega ) }{ v_j ( \iteration, \omega ) }.
  \end{equation}
  This \proves that for all $ \iteration \in \N_0 $, $ \omega \in \Omega $ it holds that
  $ \gramG( \iteration, \omega ) $ is a \gram\ matrix
  \cfload.
  \Moreover \cref{lem:gram_matrix:v},
  the fact that for all $ p \in \{ 1, 2, \ldots, \outdim \} $, $ k \in \{ 1, 2, \ldots, \width \} $ it holds that
  $ \mf W_{p,k} $ is a normal random variable, and
  the fact that $ \deriv $ is bounded
  \prove that for all $ i \in \{ 1, 2, \ldots, \outdim \ndata \} $ it holds that
  \begin{equation}\label{lem:gram_matrix:v_L2}
    v_i (0) \in V.
  \end{equation}
  \Moreover \cref{lem:graminf} and \cref{lem:gram_matrix:gramG_scalprod} \prove that
  for all $ i, j \in \{ 1, 2, \ldots, \outdim \ndata \} $ it holds that
  \begin{equation}
    \gramGinf_{i,j}
    = \E \bigl[ \gramG_{i,j} (0) \bigr]
    = \E \bigl[ \bscalprod{ v_i (0) }{ v_j (0) } \bigr]
    = \int_{\Omega} \bscalprod{ v_i (0,\omega) }{ v_j (0,\omega) } \, \P( \diff \omega )
    = \varphi \bigl( v_i (0), v_j (0) \bigr).
  \end{equation}
  Combining this, \cref{lem:gram_matrix:v_L2}, and
  the fact that $ \varphi $ is an inner product on $ V $
  \proves that $ \gramGinf $ is a \gram\ matrix.
\end{cproof}

\subsection{Positive definiteness of deterministic Gram matrices for pairwise distinct data}
\label{subsec:positive_definiteness_of_deterministic_gram_matrices}

\cfclear
\begin{lemma}
  \label{lem:properties_gram_matrix}
  Let $ V $ be a vector space over $ \R $, let $ \varphi \colon V \times V \to \R $ be an inner product on $ V $,
  let $ n \in \N $, $ v_1, v_2, \ldots, v_n \in V $, and let
  $ G = ( G_{i,j} )_{ (i,j) \in \{ 1, 2, \ldots, n \}^2 }\in \R^{n \times n} $ satisfy for all
  $ i, j \in \{1, 2, \ldots, n \} $ that $ G_{i,j} = \varphi( v_i, v_j ) $. Then
  \begin{enumerate}[label=(\roman{*})]
    \item \label{item:prop1} it holds that $ G $ is symmetric and positive semidefinite and
    \item \label{item:prop2} it holds that $ G $ is positive definite if and only if $ v_1, v_2, \ldots, v_n $ are linearly independent.
  \end{enumerate}
\end{lemma}

\begin{cproof}{lem:properties_gram_matrix}
  First, \nobs that the assumption that $ \varphi $ is symmetric \proves that $ G $ is symmetric.
  \Moreover the assumption that for all $ i, j \in \{ 1, 2, \ldots, n \} $ it holds that $ G_{i,j} = \varphi( v_i, v_j ) $ and
  the assumption that $ \varphi $ is an inner product on $ V $
  \prove that for all $ z = ( z_1, \ldots, z_n ) \in \R^n $ it holds that
  \begin{equation}
    \label{eq:properties_gram_matrix}
    \scalprod{ z }{ Gz }
    = \smallsum\limits_{i=1}^n z_i \smallsum\limits_{j=1}^n G_{i,j} z_j
    = \smallsum\limits_{i=1}^n \smallsum\limits_{j=1}^n z_i z_j \varphi( v_i, v_j )
    = \smallsum\limits_{i=1}^n \smallsum\limits_{j=1}^n \varphi( z_i v_i, z_j v_j )
    = \varphi \biggl( \smallsum\limits_{i=1}^n z_i v_i, \smallsum\limits_{j=1}^n z_j v_j \biggr)
    \ge 0
  \end{equation}
  \cfload.
  This \proves \cref{item:prop1}.
  \Moreover \cref{eq:properties_gram_matrix} and the assumption that $ \varphi $
  is positive definite \prove that it holds that
  \begin{equation}
  \begin{split}
    \Bigl[ G \text{ is positive definite } \Bigr]
    &\Leftrightarrow \Bigl[ \Forall z = ( z_1, \ldots, z_n ) \in \R^n \backslash \{ 0 \} \colon
      \scalprod{z}{Gz} > 0 \Bigr] \\
    &\Leftrightarrow \biggl[ \Forall z = ( z_1, \ldots, z_n ) \in \R^n \backslash \{ 0 \} \colon
      \varphi \biggl( \smallsum\limits_{i=1}^n z_i v_i, \smallsum\limits_{i=1}^n z_i v_i \biggr) > 0 \biggr] \\
    &\Leftrightarrow \Bigl[ \Forall z = ( z_1, \ldots, z_n ) \in \R^n \backslash \{ 0 \} \colon
      \smallsum\limits_{i=1}^n z_i v_i \neq 0 \Bigr] \\
    &\Leftrightarrow \Bigl[ v_1, v_2, \ldots, v_n \text{ are linearly independent} \Bigr].
  \end{split}
  \end{equation}
  This \proves \cref{item:prop2}.
\end{cproof}

\cfclear
\begin{lemma}
  \label{lem:hyperplane_isolation}
  Let $ \indim, \ndata \in \N $, $ z_1, z_2, \ldots, z_{\ndata} \in \R^{\indim} \backslash \{ 0 \} $
  satisfy for all $ \lambda \in \R $, $ i, j \in \{ 1, 2, \ldots, \ndata \}$ with $ i \neq j $ that $ z_i \neq \lambda z_j $,
  let $ \nbp \in \N $, $ \bp_1, \bp_2, \ldots, \bp_{\nbp} \in \R $ satisfy $ \# \{ \bp_1, \bp_2, \ldots, \bp_{\nbp} \} = \nbp $, and
  let $ H_{i,j} \subseteq \R^{\indim} $, $ i \in \{1, 2, \ldots, \ndata \} $, $ j \in \{ 1, 2, \ldots, \nbp \} $,
  satisfy for all $ i \in \{1, 2, \ldots, \ndata \} $, $ j \in \{ 1, 2, \ldots, \nbp \} $
  that $ H_{i,j} = \{ w \in \R^{\indim} \colon \scalprod{w}{z_i} = \bp_j \} $.
  \cfload.
  Then there exist
  $ ( p_{i,j} )_{ (i,j) \in \{ 1, 2, \ldots, \ndata \} \times \{ 1, 2, \ldots, \nbp \} } \in ( \R^{\indim} )^{\ndata \times \nbp} $,
  $ \varepsilon \in (0,\infty) $
  such that for all $ i \in \{ 1, 2, \ldots, \ndata \} $, $ j \in \{ 1, 2, \ldots, \nbp \} $ it holds that
  \begin{equation}\label{eq:lem:hyperplane_isolation}
    p_{i,j} \in H_{i,j}
    \qqandqq
    \bigl\{ w \in \R^{\indim} \colon \eucnorm{ w - p_{i,j} } \le \varepsilon \bigr\} \cap
      \bigl( \smallbigcup_{ (k,\ell) \in ( \{ 1, 2, \ldots, \ndata \} \times \{ 1, 2, \ldots, \nbp \} ) \backslash \{ (i,j) \} }
      H_{k,\ell} \bigr) = \emptyset.
  \end{equation}
\end{lemma}

\begin{cproof}{lem:hyperplane_isolation}
  Throughout this proof
  let $ G_i \subseteq \R^{\indim} $, $ i \in \{ 1, 2, \ldots, \ndata \} $, satisfy for all $ i \in \{ 1, 2, \ldots, \ndata \} $ that
  $ G_i = \{ w \in \R^{\indim} \colon \scalprod{w}{z_i} = 0 \} $,
  let $ A^{i,j} = ( A^{i,j}_1, A^{i,j}_2 ) \in \R^{ 2 \times \indim } $, $ (i,j) \in \{ 1, 2, \ldots, \ndata \}^2 $, satisfy for all
  $ i, j \in \{ 1, 2, \ldots, \ndata \} $ that
  \begin{equation}\label{eq:lem:hyperplane_isolation_A}
    A^{i,j}_1 = z_i
    \qqandqq
    A^{i,j}_2 = z_j
  \end{equation}
  and let $ \mc S = \{ 1, 2, \ldots, \ndata \} \times \{ 1, 2, \ldots, \nbp \} $.
  \Nobs that the fact that for all $ i \in \{ 1, 2, \ldots, \ndata \}$ it holds that $ z_i \neq 0 $
  \proves that for every $ (i,j) \in \mc S $ there exists $ q_{i,j} \in \R^{\indim} $ which satisfies that
  \begin{equation}\label{eq:lem:hyperplane_isolation_q}
    \scalprod{ q_{i,j} }{ z_i } = \bp_j.
  \end{equation}
  Let $ K_{i,j} \subseteq \mc S $, $ (i,j) \in \mc S $, satisfy for all $ (i,j) \in \mc S $ that
  \begin{equation}
    K_{i,j} = \bigl\{ (k,\ell) \in \mc S \colon \scalprod{ q_{i,j} }{ z_k } = \bp_{\ell} \bigr\}.
  \end{equation}
  \Nobs that the fact that for all $ i \in \{ 1, 2, \ldots, \ndata \} $ it holds that
  $ \R^{\indim} \ni w \mapsto \scalprod{ w }{ z_i } \in \R $ is continuous
  \proves that there exists $ \delta \in (0,\infty) $ which satisfies
  for all $ (i,j) \in \mc S $, $ (k,\ell) \in \mc S \backslash K_{i,j} $,
  $ w \in \{ v \in \R^{\indim} \colon \eucnorm{ w - q_{i,j} } \le \delta \} $ that
  \begin{equation}\label{eq:lem:hyperplane_isolation_delta}
    \abs{ \scalprod{ w }{ z_k } - \bp_{\ell} } > 0.
  \end{equation}   
  In the following we distinguish between the case $ \indim = 1 $ and the case $ \indim > 1 $.
  We first prove \cref{eq:lem:hyperplane_isolation} in the case
  \begin{equation}\label{eq:lem:hyperplane_isolation_case_1}
    \indim = 1.
  \end{equation}
  \Nobs that \cref{eq:lem:hyperplane_isolation_case_1},
  the assumption that for all $ i \in \{ 1, 2, \ldots, \ndata \} $ it holds that $ z_i \neq 0 $, and
  the assumption that for all $ \lambda \in \R $, $ i, j \in \{ 1, 2, \ldots, \ndata \}$ with $ i \neq j $
  it holds that $ z_i \neq \lambda z_j $
  \prove that $ \ndata = 1 $.
  The assumption that $ \# \{ \bp_1, \bp_2, \ldots, \bp_{\nbp} \} = \nbp $
  \hence \proves that for all $ (i,j) \in \mc S $ it holds that $ K_{i,j} = \{ (i,j) \} $.
  Combining this with \cref{eq:lem:hyperplane_isolation_q} and \cref{eq:lem:hyperplane_isolation_delta} \proves that
  for all $ (i,j) \in \mc S $ it holds that
  \begin{equation}
    q_{i,j} \in H_{i,j}
    \qqandqq
    \bigl\{ w \in \R^{\indim} \colon \eucnorm{ w - q_{i,j} } \le \delta \bigr\} \cap
      \bigl( \smallbigcup_{ (k,\ell) \in \mc S \backslash \{ (i,j) \} }
      H_{k,\ell} \bigr) = \emptyset.
  \end{equation}
  This \proves \cref{eq:lem:hyperplane_isolation} in the case $ \indim = 1 $.
  In the next step we prove \cref{eq:lem:hyperplane_isolation} in the case
  \begin{equation}\label{eq:lem:hyperplane_isolation_case_2}
    \indim > 1.
  \end{equation}
  Let $ \mu \colon \mc B( \R^{\indim-1} ) \to [0,\infty] $ be the Lebesgue measure.
  \Nobs that the rank-nullity theorem and the fact that for all $ i \in \{ 1, 2, \ldots, \ndata \} $ it holds that $ z_i \neq 0 $
  \prove that for all $ i \in \{ 1, 2, \ldots, \ndata \} $ it holds that
  \begin{equation}
    \dim_{\R}( G_i )
    = \dim_{\R}( \{ w \in \R^{\indim} \colon \scalprod{ w }{ z_i } = 0 \} )
    = \indim - \dim_{\R}( \{ u \in \R \colon [ \Exists w \in \R^{\indim} \colon \scalprod{ w }{ z_i } = u ] \} )
    = \indim - 1.
  \end{equation}
  \Hence that there exist $ \varphi_i \colon G_i \to \R^{\indim-1} $, $ i \in \{ 1, 2, \ldots, \ndata \} $, which satisfy
  for all $ i \in \{ 1, 2, \ldots, \ndata \} $, $ v, w  \in G_i $, $ \lambda \in \R $ that
  \begin{equation}\label{eq:lem:hyperplane_isolation_phi}
    \varphi_i ( \lambda w ) = \lambda \varphi_i( w ), \qquad
    \varphi_i ( v + w ) = \varphi_i ( v ) + \varphi_i ( w ), \qquad
    \ker( \varphi_i ) = \{ 0 \}, \qqandqq
    \varphi_i( G_i ) = \R^{\indim-1}.
  \end{equation}
  Combining this, the fact that for all $ i, j \in \{ 1, 2, \ldots, \ndata \} $ it holds that
  $ G_i \cap G_j = \{ w \in \R^{\indim} \colon \scalprod{w}{z_i} = 0 \} \cap \{ w \in \R^{\indim} \colon \scalprod{w}{z_j} = 0 \}
    = \ker( A^{i,j} ) $,
  the assumption that for all $ \lambda \in \R $, $ i,j \in \{ 1, 2, \ldots, \ndata \} $ with $ i \neq j $
  it holds that $ z_i \neq \lambda z_j $, and the rank-nullity theorem \proves that
  for all $ i, j \in \{ 1, 2, \ldots, \ndata \} $ with $ i \neq j $ it holds that
  \begin{equation}
    \dim_{\R} \bigl( \varphi_i( G_i \cap G_j ) \bigr) = \dim_{\R} ( G_i \cap G_j ) = \dim_{\R} ( \ker( A^{i,j} ) ) = \indim - 2.
  \end{equation}
  \Hence that for all $ i, j \in \{ 1, 2, \ldots, \ndata \} $ with $ i \neq j $ it holds that
  $ \mu \bigl( \varphi_i(G_i \cap G_j)\bigr) = 0 $.
  This \proves that for all $ i \in \{ 1, 2, \ldots, \ndata \} $ it holds that
  \begin{equation}\label{eq:lem:hyperplane_isolation_lebesque_zero}
  \begin{split}
    0 &\le \mu \bigl( \varphi_i \bigl( G_i \cap \smallbigcup_{ j \in \{ 1, 2, \ldots, \ndata \}, j \neq i } G_j \bigr) \bigr)
    = \mu \bigl( \smallbigcup_{ j \in \{ 1, 2, \ldots, \ndata \}, j \neq i } \varphi_i ( G_i \cap G_j ) \bigr) \\
    &\le \smallsum_{j \in \{ 1, 2, \ldots, \ndata \}, j \neq i } \mu ( \varphi_i ( G_i \cap G_j ) ) = 0.
  \end{split}
  \end{equation}
  \Moreover \cref{eq:lem:hyperplane_isolation_phi} \proves that for all $ i \in \{ 1, 2, \ldots, \ndata \} $
  it holds that $ \mu ( \varphi_i ( G_i ) ) = \mu( \R^{\indim-1} ) = \infty $.
  Combining this with \cref{eq:lem:hyperplane_isolation_lebesque_zero} \proves that for all $ i \in \{ 1, 2, \ldots, \ndata \} $
  it holds that $ G_i \not\subseteq \smallbigcup_{ j \in \{ 1, 2, \ldots, \ndata \} \backslash \{ i \} } G_j $.
  \Hence that there exist $ u_1, u_2, \ldots, u_{\ndata} \in \R^{\indim} \backslash \{ 0 \} $ which satisfy
  for all $ i, j \in \{ 1, 2, \ldots, \ndata \} $ with $ i \neq j $ that
  \begin{equation}\label{eq:lem:hyperplane_isolation_u}
    \scalprod{u_i}{z_i} = 0, \qquad
    \abs{ \scalprod{u_i}{z_j} } > 0, \qqandqq
    \eucnorm{u_i} \le \nicefrac{\delta}{2}.
  \end{equation}
  \Nobs that \cref{eq:lem:hyperplane_isolation_u} and
  the fact that for all $ i \in \{ 1, 2, \ldots, \ndata \} $ it holds that
  $ \R^{\indim} \ni w \mapsto \scalprod{ w }{ z_i } \in \R $ is continuous
  \prove that there exists $ \varepsilon \in (0,\nicefrac{\delta}{2}] $ which satisfies
  for all $ i, j \in \{ 1, 2, \ldots, \ndata \} $, $ w \in \{ v \in \R^{\indim} \colon \eucnorm{ v - u_i } \le \varepsilon \} $
  with $ i \neq j $ that
  \begin{equation}\label{eq:lem:hyperplane_isolation_eps}
    \abs{ \scalprod{ w }{ z_j } } > 0.
  \end{equation}
  \Nobs that
  the assumption that $ \# \{ \bp_1, \bp_2, \ldots, \bp_{\nbp} \} = \nbp $
  \proves that for all $ (i,j) \in \mc S $, $ (k,\ell) \in K_{i,j} \backslash \{ (i,j) \} $ it holds that $ i \neq k $.
  Combining this, \cref{eq:lem:hyperplane_isolation_q}, and \cref{eq:lem:hyperplane_isolation_eps} \proves that for all
  $ (i,j) \in \mc S $, $ (k,\ell) \in K_{i,j} \backslash \{ (i,j) \} $,
  $ w \in \{ v \in \R^{\indim} \colon \eucnorm{ v - ( q_{i,j} + u_i ) } \le \varepsilon \} $ it holds that
  \begin{equation}\label{eq:lem:hyperplane_isolation_local}
    \abs{ \scalprod{ w }{ z_k } - \bp_{\ell} }
    = \abs{ \scalprod{ w }{ z_k } - \scalprod{ q_{i,j} }{ z_k } }
    = \abs{ \scalprod{ w - q_{i,j} }{ z_k } }
    > 0.
  \end{equation}
  \Moreover \cref{eq:lem:hyperplane_isolation_eps} \proves that for all $ (i,j) \in \mc S $,
  $ w \in \{ v \in \R^{\indim} \colon \eucnorm{ v - ( q_{i,j} + u_i ) } \le \varepsilon \} $ it holds that
  \begin{equation}
    \eucnorm{ w - q_{i,j} }
    = \eucnorm{ w - ( q_{i,j} + u_i ) - u_i }
    \le \eucnorm{ w - ( q_{i,j} + u_i ) } + \eucnorm{ u_i }
    \le \varepsilon + \nicefrac{\delta}{2}
    \le \delta.
  \end{equation}
  Combining this, \cref{eq:lem:hyperplane_isolation_delta}, and \cref{eq:lem:hyperplane_isolation_local}
  \proves that for all $ (i,j) \in \mc S $ it holds that
  \begin{equation}\label{eq:lem:hyperplane_isolation_ball}
    \bigl\{ w \in \R^{\indim} \colon \eucnorm{ w - ( q_{i,j} + u_i ) } \le \varepsilon \bigr\} \cap
      \bigl( \smallbigcup_{ (k,\ell) \in \mc S \backslash \{ (i,j) \} } H_{k,\ell} \bigr) = \emptyset.
  \end{equation}
  \Moreover \cref{eq:lem:hyperplane_isolation_q} and \cref{eq:lem:hyperplane_isolation_u}
  \prove that for all $ (i,j) \in \mc S $ it holds that
  \begin{equation}
    \scalprod{ q_{i,j} + u_i }{ z_i }
    = \scalprod{ q_{i,j} }{ z_i } + \scalprod{ u_i }{ z_i }
    = \bp_j.
  \end{equation}
  This and \cref{eq:lem:hyperplane_isolation_ball} \prove \cref{eq:lem:hyperplane_isolation} in the case $ \indim > 1 $.
\end{cproof}

\cfclear
\begin{lemma}
  \label{lem:lambdazero_positive_onedim}
  Let $ \indim, \ndata, \nbp \in \N $,
  $ \indata_1, \indata_2, \ldots, \indata_{\ndata} \in \R^{\indim} \backslash \{ 0 \} $,
  $ \bp_0, \bp_1, \ldots, \bp_{\nbp+1} \in [-\infty,\infty] $, $ \slope_1, \slope_2, \ldots, \allowbreak \slope_{\nbp+1} \in \R $ satisfy
  \begin{equation}
    \# \{ \indata_1, \indata_2, \ldots, \indata_{\ndata} \} = \ndata, \qquad
    - \infty = \bp_0 < \bp_1 < \dots < \bp_{\nbp+1} = \infty, \qqandqq
    \# \{ \slope_1, \slope_2, \ldots, \slope_{\nbp+1} \} \ge 2,
  \end{equation}
  let $ \deriv \colon \R \to \R $ satisfy for all $ i \in \{ 1, 2, \ldots, \nbp+1 \} $, $ v \in ( \bp_{i-1}, \bp_i ) $ that
  $ \deriv( v ) = \slope_i $,
  let $ ( \Omega, \mc F, \P ) $ be a probability space,
  let $ Z \colon \Omega \to \R^{\indim} $ be a random variable, and
  let $ G = ( G_{i,j} )_{ (i,j) \in \{ 1, 2, \ldots, \ndata \}^2 } \in \R^{\ndata \times \ndata} $
  satisfy for all $ i, j \in \{ 1, 2, \ldots, \ndata \} $ that
  $ G_{i,j} = ( 1 + \scalprod{\indata_i}{\indata_j} ) \E[ \deriv( \scalprod{ Z }{ \indata_i } ) \deriv( \scalprod{ Z }{ \indata_j } ) ] $
  \cfload.
  Then $ \lambdamin( G ) > 0 $
  \cfout.
\end{lemma}

\begin{cproof}{lem:lambdazero_positive_onedim}
  Throughout this proof
  let $ z_1, z_2, \ldots, z_{\ndata} \in \R^{\indim+1} $ satisfy for all $ i \in \{ 1, 2, \ldots, \ndata \} $ that
  $ z_i = ( \indata_i, 1 ) $,
  let $ P \colon \R^{\indim+1} \to \R^{\indim} $ satisfy for all $ w = ( w_1, \ldots, w_{\indim+1} ) \in \R^{\indim+1} $ that
  $ P(w) = ( w_1, \ldots, w_{\indim} ) $,
  let $ \mu \colon \mc B( \R^{\indim+1} ) \to [0,\infty] $ satisfy for all $ A \in \mc B( \R^{\indim+1} ) $ that
  $ \mu( A ) = \P( Z \in P(A) ) $,
  let $ H_{i,j} \subseteq \R^{\indim} $, $ i \in \{ 1, 2, \ldots, \ndata \} $, $ j \in \{ 1, 2, \ldots, \nbp \}$, satisfy for all
  $ i \in \{ 1, 2, \ldots, \ndata \} $, $ j \in \{ 1, 2, \ldots, \nbp \} $ that
  \begin{equation}
    H_{i,j} = \{ w \in \R^{\indim+1} \colon \scalprod{ w }{ z_i } = \bp_j \},
  \end{equation}
  let $ V $ satisfy
  $ V = L^2( \R^{\indim+1}, \mc B (\R^{\indim+1}), \mu ; \R^{\indim+1} ) $,
  let $ \varphi \colon V \times V \to \R $
  satisfy for all $ p, q \in V $ that
  \begin{equation}
  \label{eq:positive_definite_phi}
    \varphi( p, q ) = \int_{\R^{\indim+1}} \scalprod{p(w)}{q(w)} \, \mu( \diff  w )
  \end{equation}
  let $ g_v \colon \R^{\indim+1} \to \R^{\indim+1} $, $ v \in \R^{\indim+1} $, satisfy for all $ v, w \in \R^{\indim+1} $ that
  \begin{equation}
  \label{eq:positive_definite_g}
    g_v (w) = \deriv ( \scalprod{ w }{ v } ) \, v,
  \end{equation}
  and let $ \alpha_1, \alpha_2, \ldots, \alpha_{\ndata} \in \R $ satisfy
  $ \sum_{j=1}^{\ndata} \alpha_j g_{z_j} = 0 $.
  First, \nobs that $ \varphi $ is an inner product on $ V $.  
  \Moreover for all $ v \in \R^{\indim} $ it holds that $ g_v \in V $.
  Next \nobs that \cref{eq:positive_definite_phi} and \cref{eq:positive_definite_g}
  \prove that for all $ i, j \in \{ 1, 2, \ldots, \ndata \} $ it holds that
  \begin{equation}
  \begin{split}
    G_{i,j}
    &= ( 1 + \scalprod{\indata_i}{\indata_j} ) \E \bigl[ \deriv( \scalprod{Z}{\indata_i} ) \deriv( \scalprod{Z}{\indata_j} ) \bigr]
    = \scalprod{z_i}{z_j} \int_{\R^{\indim+1}} \deriv( \scalprod{w}{z_i} ) \deriv( \scalprod{w}{z_j} ) \, \mu ( \diff  w ) \\
    &= \int_{\R^{\indim+1}} \scalprod{z_i}{z_j} \deriv( \scalprod{w}{z_i} ) \deriv( \scalprod{w}{z_j} ) \, \mu ( \diff w )
    = \int_{\R^{\indim+1}} \scalprod{ \deriv( \scalprod{w}{z_i} ) z_i  }{ \deriv( \scalprod{w}{z_j} ) z_j } \, \mu ( \diff w ) \\
    &= \int_{\R^{\indim+1}} \scalprod{ g_{z_i} (w) }{ g_{z_j} (w) } \, \mu ( \diff w )
    = \varphi( g_{z_i}, g_{z_j} ).
  \end{split}
  \end{equation}
  In the next step, we show that $ g_{z_1}, g_{z_2}, \ldots, g_{z_{\ndata}} $ are linearly independent.
  \Nobs that the assumption that $ \# \{ \indata_1, \indata_2, \ldots, \indata_{\ndata} \} = \ndata $ \proves that for all
  $ \lambda \in \R $, $ i, j \in \{ 1, 2, \ldots, \ndata \} $ with $ i \neq j $ it holds that
  \begin{equation}\label{eq:positive_definite_linearly_independent}
    z_i \neq \lambda z_j
  \end{equation}
  \Moreover the assumption that $ \# \{ \slope_1, \slope_2, \ldots, \slope_{\nbp+1} \} \ge 2 $
  \proves that there exists $ \ell \in \{ 1, 2, \ldots, \nbp \} $ which satisfies
  \begin{equation}\label{eq:positive_definite_ell}
    \abs{ \slope_{\ell} - \slope_{\ell+1} } > 0.
  \end{equation}
  Combining \cref{eq:positive_definite_linearly_independent}, \cref{eq:positive_definite_ell}, and \cref{lem:hyperplane_isolation}
  \proves that there exist $ p_1, p_2, \ldots, p_{\ndata} \in \R^{\indim} $, $ \varepsilon \in (0,\infty) $
  which satisfy for all $ i \in \{ 1, 2, \ldots, \ndata \}$ that
  \begin{equation}\label{eq:positive_definite_p}
    p_i \in H_{i,\ell}
    \qqandqq
    \bigl\{ w \in \R^{\indim} \colon \eucnorm{ w - p_i } \le \varepsilon \bigr\} \cap
      \bigl( \smallbigcup_{ (k,r) \in ( \{ 1, 2, \ldots, \ndata \} \times \{ 1, 2, \ldots, \nbp \} ) \backslash \{ (i,\ell) \} }
      H_{k,r} \bigr) = \emptyset.
  \end{equation}
  \Nobs that \cref{eq:positive_definite_p} \proves that
  for all $ i, j \in \{ 1, 2, \ldots, \ndata \} $, $ w \in \R^{\indim+1} $ with $ i \neq j $ and $ \eucnorm{ w - p_i } \le \varepsilon $
  it holds that
  \begin{equation}
  \label{eq:positive_definite_g_const}
    g_{z_j}(w)
    = \deriv ( \scalprod{w}{z_j} ) z_j
    = \deriv ( \scalprod{p_i}{z_j} ) z_j
    = g_{z_j}(p_i).
  \end{equation}
  Now for every $ i \in \{ 1, 2, \ldots, \ndata \}$ let $ A_i \subseteq \R^{\indim} $ and $ B_i \subseteq \R^{\indim} $ satisfy
  \begin{equation}
  \label{eq:positive_definite_A_B}
    A_i = \{ w \in \R^{\indim+1} \colon \eucnorm{ w - p_i } \le \varepsilon, \scalprod{w}{z_i} < \bp_{\ell} \}
    \qqandqq
    B_i = \{ w \in \R^{\indim} \colon \eucnorm{ w - p_i } \le \varepsilon, \scalprod{w}{z_i} > \bp_{\ell} \},
  \end{equation}
  \Nobs that \cref{eq:positive_definite_g_const} and \cref{eq:positive_definite_A_B} \prove that for all
  $ i, j \in \{ 1, 2, \ldots, \ndata \} $ with $ i \neq j $ it holds that
  \begin{equation}
  \label{eq:positive_definite_1}
  \begin{split}
    &\frac{1}{\mu ( A_i )} \int_{ A_i } g_{z_j} (w) \, \mu ( \diff w )
      - \frac{1}{\mu ( B_i )} \int_{ B_i } g_{z_j} (w) \, \mu ( \diff w ) \\
    &= \frac{1}{\mu ( A_i )} \int_{ A_i } g_{z_j} (p_i) \, \mu ( \diff w )
      - \frac{1}{\mu ( B_i \bigr)} \int_{ B_i } g_{z_j} (p_i) \, \mu ( \diff w ) \\
    &= g_{z_j} (p_i) - g_{z_j} (p_i)
    = 0.
  \end{split}
  \end{equation}
  \Moreover \cref{eq:positive_definite_p} and \cref{eq:positive_definite_A_B} \prove that
  for all $ i \in \{ 1, 2, \ldots, \ndata \} $ it holds that
  \begin{equation}
  \label{eq:positive_definite_2}
  \begin{split}
    &\frac{1}{\mu ( A_i )} \int_{ A_i } g_{z_i} (w) \, \mu ( \diff w )
      - \frac{1}{\mu ( B_i ) } \int_{ B_i } g_{z_i} (w) \, \mu ( \diff w ) \\
    &= \frac{1}{\mu ( A_i )} \int_{ A_i } \slope_{\ell} z_i \, \mu ( \diff w )
      - \frac{1}{\mu ( B_i )} \int_{ B_i } \slope_{\ell+1} z_i \, \mu ( \diff w )
    = ( \slope_{\ell} - \slope_{\ell+1} ) z_i.
  \end{split}
  \end{equation}
  Combining this, \cref{eq:positive_definite_1}, and
  the assumption that $ \sum_{j=1}^{\ndata} \alpha_j g_{z_j} = 0 $ 
  \proves that for all $ i \in \{ 1, 2, \ldots, \ndata \} $ it holds that
  \begin{equation}
  \begin{split}
    0
    &= \frac{1}{\mu ( A_i ) } \int_{ A_i } \smallsum\limits_{j=1}^{\ndata}
      \alpha_j g_{z_j} (w) \, \mu ( \diff w ) \displaystyle - \frac{1}{\mu ( B_i ) } \int_{ B_i }
      \smallsum\limits_{j=1}^{\ndata} \alpha_j g_{z_j} (w) \, \mu ( \diff w ) \\
    &= \smallsum\limits_{j=1}^{\ndata} \displaystyle \alpha_j \biggl( \frac{1}{\mu ( A_i ) } \int_{ A_i }
      g_{z_j} (w) \, \mu ( \diff w ) - \frac{1}{\mu ( B_i ) } \int_{ B_i } g_{z_j} (w) \, \mu ( \diff w ) \biggr)
    = \alpha_i ( \slope_{\ell} - \slope_{\ell+1} ) z_i.
  \end{split}
  \end{equation}
  This, \cref{eq:positive_definite_ell}, and
  the fact that for all $ i \in \{ 1, 2, \ldots, \ndata \} $ it holds that $ z_i \neq 0 $ \prove that
  for all $ i \in \{ 1, 2, \ldots, \ndata \} $ it holds that $ \alpha_i = 0 $.
  \Hence that $ g_{z_1}, g_{z_2}, \ldots, g_{z_{\ndata}} $ are linearly independent.
  Item \ref{item:prop2} in \cref{lem:properties_gram_matrix} (applied with
  $ V \is V $, $ \varphi \is \varphi $, $ n \is \ndata $,
  $ (v_i)_{ i \in \{ 1, 2, \ldots, n \} } \is (g_{z_i})_{ i \in \{ 1, 2, \ldots, \ndata \} } $, $ G \is G $
  in the notation of \cref{lem:properties_gram_matrix})
  \hence \proves that $ G $ is positive definite. 
  This \proves that $ \lambdamin( G ) > 0 $ \cfload.
\end{cproof}

\cfclear
\begin{lemma}
  \label{lem:eigenvalues_block_matrix}
  Let $ \outdim, \ndata \in \N $,
  $ A = ( A_{i,j} )_{ (i,j) \in \{ 1, 2, \ldots, \outdim \ndata \}^2 } \in \R^{(\outdim \ndata) \times (\outdim \ndata)} $,
  $ B = ( B_{i,j} )_{ (i,j) \in \{ 1, 2, \ldots, \ndata \}^2 } \in \R^{\ndata \times \ndata} $
  satisfy for all $ i, j \in \{ 1, 2, \ldots, \ndata \} $, $ p, q \in \{ 1, 2, \ldots, \outdim \} $ that
  \begin{equation}
    A_{ (i-1) \outdim + p, (j-1) \outdim + q } = B_{i,j} \ind_{ \{ p \} } (q).
  \end{equation}
  Then
  $ \{ \lambda \in \R \colon [ \Exists v \in \R^{(\outdim \ndata) \times (\outdim \ndata)} \backslash \{ 0 \}
    \colon A v = \lambda v ] \}
    = \{ \lambda \in \R \colon [ \Exists v \in \R^{\ndata \times \ndata} \backslash \{ 0 \} \colon B v = \lambda v ] \} $.
\end{lemma}

\begin{cproof}{lem:eigenvalues_block_matrix}
  Throughout this proof for every $ n \in \N $ let $ I_n \in \R^{n \times n} $ be the identity matrix in $ \R^{n \times n} $ and
  let $ K = ( K_{i,j} )_{ (i,j) \in \{ 1, 2, \ldots, \outdim \ndata \}^2 } \in \R^{(\outdim \ndata) \times (\outdim \ndata)} $
  satisfy for all $ p, q \in \{ 1, 2, \ldots, \outdim \} $, $ i, j \in \{ 1, 2, \ldots, \ndata \} $ that
  \begin{equation}\label{lem:eigenvalues_block_matrix:K}
    K_{ (p-1) \ndata + i, (q-1) \ndata + j }
    = B_{i,j} \ind_{ \{ p \} } (q).
  \end{equation}
  \Nobs that \cref{lem:eigenvalues_block_matrix:K} \proves that $ K $ is a diagonal block matrix
  consisting of $ \outdim $ blocks on the diagonal, each given by the matrix $ B $.
  Combining this and
  the fact that for all $ \lambda \in \R $ it holds that
  $ A - \lambda I_{\outdim \ndata} $ is derived from $ K - \lambda I_{\outdim \ndata} $
  by an even number of row and column interchanges
  \proves that for all $ \lambda \in \R $ it holds that
  \begin{equation}
    \det \bigl( A - \lambda I_{\outdim \ndata} \bigr)
    = \det \bigl( K - \lambda I_{\outdim \ndata} \bigr)
    = \prod_{p=1}^{\outdim} \det \bigl( B - \lambda I_{\ndata} \bigr)
    = \Bigl( \det \bigl( B - \lambda I_{\ndata} \bigr) \Bigr)^{\outdim}.
  \end{equation}
  This \proves that for all $ \lambda \in \R $ it holds that
  \begin{equation}
    \Bigl[ \det \bigl( A - \lambda I_{\outdim \ndata} \bigr) = 0 \Bigr]
    \Leftrightarrow
    \Bigl[ \det \bigl( B - \lambda I_{\ndata} \bigr) = 0 \Bigr].
  \end{equation}
  \Hence that
  $ \{ \lambda \in \R \colon [ \Exists v \in \R^{(\outdim \ndata) \times (\outdim \ndata)} \backslash \{ 0 \}
    \colon A v = \lambda v ] \}
    = \{ \lambda \in \R \colon [ \Exists v \in \R^{\ndata \times \ndata} \backslash \{ 0 \} \colon B v = \lambda v ] \} $.
\end{cproof}

\begin{lemma}
  \label{lem:lambdazero_positive}
  Assume \cref{setting:gradient_descent},
  assume $ \# \{ \indata_1, \indata_2, \ldots, \indata_{\ndata} \} = \ndata $, and
  assume $ \# \{ \slope_1, \slope_2, \ldots, \slope_{\nbp+1} \} \ge 2 $.
  Then $ \lambdazero > 0 $.
\end{lemma}

\begin{cproof}{lem:lambdazero_positive}
  Throughout this proof let
  $ G = ( G_{i,j} )_{ (i,j) \in \{ 1, 2, \ldots, \ndata \} } \in \R^{ \ndata \times \ndata } $
  satisfy for all $ i, j \in \{ 1, 2, \ldots, \ndata \} $ that
  \begin{equation}
    G_{i,j} = \bigl( 1 + \scalprod{ \indata_i }{ \indata_j } \bigr)
      \E \bigl[ \deriv \bigl( \scalprod{ W_1(0) }{ \indata_i } \bigr) \deriv \bigl( \scalprod{ W_1(0) }{ \indata_j } \bigr) \bigr].
  \end{equation}
  \Nobs that \cref{lem:lambdazero_positive_onedim} (applied with
  $ \indim \is \indim $, $ \ndata \is \ndata $, $ \nbp \is \nbp $,
  $ ( \indata_i )_{ i \in \{ 1, 2, \ldots, \ndata \} } \is ( \indata_i )_{ i \in \{ 1, 2, \ldots, \ndata \} } $,
  $ ( \bp_i )_{ i \in \{ 0, 1, \ldots, \nbp+1 \} } \is ( \bp_i )_{ i \in \{ 0, 1, \ldots, \nbp+1 \} } $,
  $ ( \slope_i )_{ i \in \{ 1, 2, \ldots, \nbp+1 \} } \is ( \slope_i )_{ i \in \{ 1, 2, \ldots, \nbp+1 \} } $,
  $ \deriv \is \deriv $, $ ( \Omega, \mc F, \P ) \is ( \Omega, \mc F, \P ) $, $ Z \is W_1(0) $, $ G \is G $
  in the notation of \cref{lem:lambdazero_positive_onedim}),
  the assumption that $ \# \{ \indata_1, \indata_2, \ldots, \indata_{\ndata} \} = \ndata $, and
  the assumption that $ \# \{ \slope_1, \slope_2, \ldots, \slope_{\nbp+1} \} \ge 2 $
  \prove that
  \begin{equation}
    \lambdamin( G ) > 0.
  \end{equation}
  \cref{lem:eigenvalues_block_matrix} (applied with
  $ \outdim \is \outdim $, $ \ndata \is \ndata $,
  $ A \is \gramGinf $, $ B \is G $
  in the notation of \cref{lem:eigenvalues_block_matrix}) and
  the assumption that $ \lambdazero = \lambdamin( \tfrac{1}{\Cvar \width} \gramGinf ) $
  \hence \prove that
  \begin{equation}
    \lambdazero
    = \lambdamin \bigl( \tfrac{1}{\Cvar \width} \gramGinf \bigr)
    = \tfrac{1}{\Cvar \width} \lambdamin( \gramGinf )
    = \tfrac{1}{\Cvar \width} \lambdamin( G )
    > 0.
  \end{equation}
\end{cproof}

\section{Analysis of eigenvalues of stochastic Gram matrices}
\label{sec:analysis_of_eigenvalues_of_stochastic_gram_matrices}

In this section, we study the error at initialization (cf.~\cref{subsec:probabilistic_error_analysis_at_initialization} below), the evolution of the weights and biases of the considered \ANNs\ during training (cf.~\cref{subsec:analysis_of_weights_and_biases_of_ANNs_during_training} below), and the eigenvalues of the considered stochastic Gram matrices (cf.~\cref{subsec:analysis_of_eigenvalues_of_stochastic_gram_matrices_during_training} below). In particular, the main result of \cref{subsec:analysis_of_eigenvalues_of_stochastic_gram_matrices_during_training}, \cref{lem:lambdamin_gramG} below, establishes a lower bound for the eigenvalues of the Gram matrices $ \gramG(n) $, $ n \in \N_0 $, from \cref{setting:gradient_descent} in \cref{subsec:mathematical_description_of_gd_processes} above.

In order to obtain these estimates, we first analyze the deviation of $ \gramG $ at initialization from its deterministic counterpart with respect to the spectral norm (cf.~\cref{subsec:concentration_type_inequalities_for_stochastic_gram_matrices_at_initialization} below) using various techniques and properties of subexponential random variables (cf.~\cref{subsec:properties_of_subexponential_random_variables} below). Then we analyze the deviation of $ \gramG $ during training from its initialization with respect to the spectral norm (cf.~\cref{subsec:analysis_of_stochastic_gram_matrices_during_training} below) using the previously studied evolution of the weights and biases of the considered \ANNs\ during training. These estimates will be used in our error analysis for \GD\ optimization algorithms in \cref{sec:error_analysis} below.

We also \nobs that the results in \cref{lem:eucnorm_inequality}, \cref{lem:gaussian_tail_bound}, \cref{lem:subexponential_tail_bound}, \cref{lem:spectral_norm_inequality}, \cref{lem:gaussian_anti_concentration_inequality}, \cref{lem:expectation_absolute_normal}, \cref{lem:lambdamin_eucnorm}, and \cref{lem:lambdamin_specnorm} are well-known, while the results in \cref{lem:subexponential_scalar} and \cref{lem:subexponential_weighted_sum} are elementary. Only for completeness we include the detailed proofs for these lemmas in this section.

\subsection{Probabilistic error analysis at initialization}
\label{subsec:probabilistic_error_analysis_at_initialization}

\cfclear
\begin{lemma}
  \label{lem:expectation_act_normal_squared}
  Let $ ( \Omega, \mc F, \P ) $ be a probability space,
  let $ s \in \R $,
  let $ X \colon \Omega \to \R $ be a normal random variable with $ \Var[ X ] = s^2 $, and
  let $ \nbp \in \N $, $ \bp_0, \bp_1, \dots, \bp_{\nbp+1} \in [-\infty,\infty] $,
  $ \slope = ( \slope_1, \dots, \slope_{\nbp+1} ) $, $ \yinter = ( \yinter_1, \dots, \yinter_{\nbp+1} ) \in \R^{\nbp+1} $,
  $ \act \in C( \R, \R ) $ satisfy for all $ i \in \{ 1, 2, \dots, \nbp+1 \} $, $ v \in ( \bp_{i-1}, \bp_i ) $
  that
  \begin{equation}\label{eq:lem:expectation_act_normal_squared}
    -\infty = \bp_0 < \bp_1 < \dots < \bp_{\nbp+1} = \infty
    \qqandqq
    \act( v ) = \slope_i v + \yinter_i.
  \end{equation}
  Then $ \E[ \abs{ \act( X - \E[X] ) }^2 ] \le s^2 \eucnorm{ \slope }^2 + \eucnorm{ \yinter }^2 $.
\end{lemma}

\begin{cproof}{lem:expectation_act_normal_squared}
  \Nobs that \cref{eq:lem:expectation_act_normal_squared} and
  the assumption that $ X $ is a normal random variable with $ \Var[ X ] = s^2 $
  \prove that
  \begin{equation}
  \begin{split}
    \E \bigl[ \abs{ \act( X - \E[X] ) }^2 \bigr]
    &= \smallsum\limits_{i=1}^{\nbp+1} \E \bigl[ \abs{ \act( X - \E[X] ) }^2
      \ind_{ \{ X - \E[X] \in ( \bp_{i-1}, \bp_i ) \} } \bigr] \\
    &= \smallsum\limits_{i=1}^{\nbp+1} \E \bigl[ \abs{ \slope_i( X - \E[X] ) + \yinter_i }^2
      \ind_{ \{ X - \E[X] \in ( \bp_{i-1}, \bp_i ) \} } \bigr] \\
    &\le \smallsum\limits_{i=1}^{\nbp+1} \E \bigl[ \abs{ \slope_i( X - \E[X] ) + \yinter_i }^2 \bigr] \\
    &= \smallsum\limits_{i=1}^{\nbp+1} \bigl( \slope_i^2 \E \bigl[ \abs{ X - \E[X] }^2 \bigr]
      + 2 \slope_i \yinter_i \E \bigl[ X - \E[X] \bigr] + \yinter_i^2 \bigr) \\
    &= \smallsum\limits_{i=1}^{\nbp+1} ( \slope_i^2 s^2 + \yinter_i^2 )
    = s^2 \eucnorm{ \slope }^2 + \eucnorm{ \yinter }.
  \end{split}
  \end{equation}
\end{cproof}

\cfclear
\begin{lemma}
  \label{lem:expectation_rect_normal_squared}
  Let $ ( \Omega, \mc F, \P ) $ be a probability space,
  let $ s \in \R $,
  let $ X \colon \Omega \to \R $ be a normal random variable with $ \Var[ X ] = s^2 $, and
  let $ \act \colon \R \to \R $ satisfy for all $ v \in \R $ that
  \begin{equation}
    \act ( v ) = \max \{ v, 0 \}.
  \end{equation}
  Then $ \E[ \abs{ \act( X - \E[X] ) }^2 ] = \nicefrac{s^2}{2} $.
\end{lemma}

\begin{cproof}{lem:expectation_rect_normal_squared}
  \Nobs that the assumption that $ X $ is a normal random variable with $ \Var[ X ] = s^2 $
  and the fact that $ X - \E[X] $ and $ \E[X] - X $ are identically distributed
 \prove that
  \begin{equation}
  \begin{split}
    s^2
    &= \E \bigl[ \abs{ X - \E[X] }^2 \bigr]
    = \E \bigl[ \abs{ X - \E[X] }^2 \bigl( \ind_{ \{ X - \E[X] \ge 0 \} } + \ind_{ \{ X - \E[X] \le 0 \} } \bigr) \bigr] \\
    &= \E \bigl[ \abs{ X - \E[X] }^2 \ind_{ \{ X - \E[X] \ge 0 \} } \bigr]
      + \E \bigl[ \abs{ X - \E[X] }^2 \ind_{ \{ X - \E[X] \le 0 \} } \bigr] \\
    &= \E \bigl[ \abs{ X - \E[X] }^2 \ind_{ \{ X - \E[X] \ge 0 \} } \bigr]
      + \E \bigl[ \abs{ \E[X] - X }^2 \ind_{ \{ \E[X] - X \ge 0 \} } \bigr]
    = 2 \, \E \bigl[ \abs{ X - \E[X] }^2 \ind_{ \{ X - \E[X] \ge 0 \} } \bigr].
  \end{split}
  \end{equation}
  \Hence that
  \begin{equation}
    \E \bigl[ \abs{ \act( X - \E[X] ) }^2 \bigr]
    = \E \Bigl[ \babs{ ( X - \E[X] ) \ind_{ \{ X - \E[X] \ge 0 \} } }^2 \Bigr]
    = \E \bigl[ \abs{ X - \E[X] }^2 \ind_{ \{ X - \E[X] \ge 0 \} } \bigr]
    = \tfrac{s^2}{2}.
  \end{equation}
\end{cproof}

\cfclear
\begin{lemma}
  \label{lem:error_at_initialization}
  Assume \cref{setting:gradient_descent} and let $ \poe \in (0,1) $.
  Then
  \begin{equation}
    \P \Bigl( \eucnorm{ \fnetvalue(0) - \foutdata }^2 \le \bigl(
      \cvar \Cvar \width \outdim \eucnorm{\slope}^2 \eucnorm{ \findata }^2
      + \Cvar \width \outdim \eucnorm{\yinter}^2 \ndata + \eucnorm{ \foutdata }^2 \bigr) \poe^{-1} \Bigr)
    \ge 1 - \poe
  \end{equation}
  \cfout.
\end{lemma}

\begin{cproof}{lem:error_at_initialization}
  \Nobs that the assumption that $ \eucnorm{ B(0) } = \eucnorm{ \mf B(0) } = 0 $ \proves that
  for all $ i \in \{ 1, 2, \ldots, \allowbreak \ndata \} $, $ p \in \{ 1, 2, \ldots, \outdim \} $ it holds that
  \begin{equation}
  \label{eq:error_at_initialization_f_i}
    \netvalue_i^p (0)
    = \bigl( \functionANN{\act}^p ( \Phi(0) ) \bigr) ( \indata_i )
    = \smallsum\limits_{k=1}^{\width} \mf W_{p,k} \, \act\bigl( \scalprod{ W_k(0) }{ \indata_i } + B_k(0) \bigr) + \mf B_p(0)
    = \smallsum\limits_{k=1}^{\width} \mf W_{p,k} \, \act\bigl( \scalprod{ W_k(0) }{ \indata_i } \bigr)
  \end{equation}
  \cfload.
  Combining this with the assumption that
  $ \sqrt{\nicefrac{1}{\cvar}} W_1(0), \allowbreak
    \sqrt{\nicefrac{1}{\cvar}} W_2(0), \allowbreak \ldots, \allowbreak
    \sqrt{\nicefrac{1}{\cvar}} W_{\width}(0), \allowbreak
    \sqrt{\nicefrac{1}{\Cvar}} \mf W_1, \allowbreak
    \sqrt{\nicefrac{1}{\Cvar}} \mf W_2, \allowbreak \ldots, \allowbreak
    \sqrt{\nicefrac{1}{\Cvar}} \mf W_{\outdim} $
  are independent and standard normal \proves that
  for all $ i \in \{ 1, 2, \ldots, \ndata \} $, $ p \in \{ 1, 2, \ldots, \outdim \} $ it holds that
  \begin{equation}
  \label{eq:error_at_initialization_f_i_expectation}
    \E[ \netvalue_i^p (0) ]
    = \smallsum\limits_{k=1}^{\width}  \E \bigl[ \mf W_{p,k} \act\bigl( \scalprod{ W_k(0) }{ \indata_i } \bigr) \bigr]
    = \smallsum\limits_{k=1}^{\width}  \E[ \mf W_{p,k} ] \E \bigl[ \act\bigl( \scalprod{ W_k(0) }{ \indata_i } \bigr) \bigr]
    = 0.
  \end{equation}
  \Moreover \cref{eq:error_at_initialization_f_i} \proves that
  for all $ i \in \{ 1, 2, \ldots, \ndata \} $, $ p \in \{ 1, 2, \ldots, \outdim \} $ it holds that
  \begin{equation}
  \label{eq:error_at_initialization_f_i_squared}
  \begin{split}
    \abs{ \netvalue_i^p (0) }^2
    &= \bbbabs{ \smallsum\limits_{k=1}^{\width} \mf W_{p,k} \act\bigl( \scalprod{ W_k(0) }{ \indata_i } \bigr) }^2
    = \smallsum\limits_{k=1}^{\width} \smallsum\limits_{\ell=1}^{\width}
      \mf W_{p,k} \act\bigl( \scalprod{ W_k(0) }{ \indata_i } \bigr) \mf W_{j,\ell} \act\bigl( \scalprod{ W_{\ell}(0) }{ \indata_i } \bigr) \\
    &= \smallsum\limits_{k=1}^{\width}  \abs{ \mf W_{p,k} }^2 \babs{ \act\bigl( \scalprod{ W_k(0) }{ \indata_i } \bigr) }^2
      + \smallsum\limits_{ \substack{ k, \ell \in \{1, 2, \ldots, \width \}, \\ k \neq \ell } }
        \mf W_{p,k} \mf W_{p,\ell} \act\bigl( \scalprod{ W_k(0) }{ \indata_i } \bigr) \act\bigl( \scalprod{ W_{\ell}(0) }{ \indata_i } \bigr).
  \end{split}
  \end{equation}
  \Moreover the assumption that
  $ \sqrt{\nicefrac{1}{\cvar}} W_1(0), \sqrt{\nicefrac{1}{\cvar}} W_2(0), \ldots, \sqrt{\nicefrac{1}{\cvar}} W_{\width}(0) $
  are standard normal \proves that for all
  $ i \in \{ 1, 2, \ldots, \ndata \} $, $ k \in \{ 1, 2, \ldots, \width \} $ it holds that $ \scalprod{ W_k(0) }{ \indata_i } $ is a
  centered normal random variable with
  \begin{equation}
    \Var[ \scalprod{ W_k(0) }{ \indata_i } ]
    = \cvar \Var[ \scalprod{ \textstyle \sqrt{\nicefrac{1}{\cvar}} W_k(0) }{ \indata_i } ]
    = \cvar \eucnorm{ \indata_i }^2.
  \end{equation}
  Combining this with \cref{lem:expectation_act_normal_squared} (applied for every $ i \in \{ 1, 2, \ldots, \ndata \} $,
  $ k \in \{ 1, 2, \ldots, \width \} $ with
  $ ( \Omega, \mc F, \P ) \is ( \Omega, \mc F, \P ) $,
  $ X \is \scalprod{ W_k(0) }{ \indata_i } $,
  $ s \is \sqrt{\cvar} \eucnorm{ \indata_i } $
  in the notation of \cref{lem:expectation_act_normal_squared}),
  the assumption that
  $ \sqrt{\nicefrac{1}{\cvar}} W_1(0), \allowbreak
    \sqrt{\nicefrac{1}{\cvar}} W_2(0), \allowbreak \ldots, \allowbreak
    \sqrt{\nicefrac{1}{\cvar}} W_{\width}(0), \allowbreak
    \sqrt{\nicefrac{1}{\Cvar}} \mf W_1, \allowbreak
    \sqrt{\nicefrac{1}{\Cvar}} \mf W_2, \allowbreak \ldots, \allowbreak
    \sqrt{\nicefrac{1}{\Cvar}} \mf W_{\outdim} $
  are independent and standard normal, and \cref{eq:error_at_initialization_f_i_squared}
  \proves that for all $ i \in \{ 1, 2, \ldots, \ndata \} $, $ p \in \{ 1, 2, \ldots, \outdim \} $ it holds that
  \begin{equation}
  \begin{split}
    \E \big[ \abs{ \netvalue_i^p (0) }^2 \bigr]
    &= \smallsum\limits_{k=1}^{\width} \E \Bigl[ \abs{ \mf W_{p,k} }^2 \babs{ \act \bigl( \scalprod{ W_k(0) }{ \indata_i } \bigr) }^2 \Bigr]
      + \smallsum\limits_{ \substack{ k, \ell \in  \{1, 2, \ldots, \width \}, \\ k \neq \ell } } \E \bigl[ \mf W_{p,k} \mf W_{p,\ell}
      \act \bigl( \scalprod{ W_k(0) }{ \indata_i } \bigr) \act \bigl( \scalprod{ W_{\ell}(0) }{ \indata_i } \bigr) \bigr] \\
    &= \smallsum\limits_{k=1}^{\width} \E \bigl[ \abs{ \mf W_{p,k} }^2 \bigr]
      \E \bigl[ \babs{ \act \bigl( \scalprod{ W_k(0) }{ \indata_i } \bigr) }^2 \Bigr]
      + \smallsum\limits_{ \substack{ k, \ell \in  \{1, 2, \ldots, \width \}, \\ k \neq \ell } } \E[ \mf W_{p,k} ] \E \bigl[ \mf W_{p,\ell}
      \act \bigl( \scalprod{ W_k(0) }{ \indata_i } \bigr) \act \bigl( \scalprod{ W_{\ell}(0) }{ \indata_i } \bigr) \bigr] \\
    &\le \smallsum\limits_{k=1}^{\width} \Cvar \bigl(
      \cvar \eucnorm{ \indata_i }^2 \eucnorm{\slope}^2 + \eucnorm{\yinter}^2 \bigr)
    = \Cvar \width \bigl( \cvar \eucnorm{ \indata_i }^2 \eucnorm{\slope}^2 + \eucnorm{\yinter}^2 \bigr).
  \end{split}
  \end{equation}
  This and \cref{eq:error_at_initialization_f_i_expectation} \prove that
  \begin{equation}
  \begin{split}
    \E \bigl[ \eucnorm{ \fnetvalue(0) - \foutdata }^2 \bigr]
    &= \E \! \biggl[ \smallsum\limits_{i=1}^{\ndata} \smallsum\limits_{p=1}^{\outdim} \abs{ \netvalue_i^p (0) - \outdata_i^p }^2 \biggr]
    = \smallsum\limits_{i=1}^{\ndata} \smallsum\limits_{p=1}^{\outdim} \E \bigl[ \abs{ \netvalue_i^p (0) - \outdata_i^p }^2 \bigr] \\
    &= \smallsum\limits_{i=1}^{\ndata} \smallsum\limits_{p=1}^{\outdim} \E \bigl[ \abs{ \netvalue_i^p (0) }^2 - 2 \netvalue_i^p (0) \outdata_i^p
      + \abs{ \outdata_i^p }^2 \bigr] \\
    &= \smallsum\limits_{i=1}^{\ndata} \smallsum\limits_{p=1}^{\outdim}
      \bigl( \E \bigl[ \abs{ \netvalue_i^p (0) }^2 \bigr] - 2 \outdata_i^p \E[ \netvalue_i^p (0) ] + \abs{ \outdata_i^p }^2 \bigr) \\
    &\le \smallsum\limits_{i=1}^{\ndata} \smallsum\limits_{p=1}^{\outdim} \bigl(
      \Cvar \width \bigl( \cvar \eucnorm{ \indata_i }^2 \eucnorm{\slope}^2 + \eucnorm{\yinter}^2 \bigr) + \abs{ \outdata_i^p }^2 \bigr) \\
    &= \cvar \Cvar \width \outdim \eucnorm{\slope}^2 \eucnorm{ \findata }^2
      + \Cvar \width \outdim \eucnorm{\yinter}^2 \ndata + \eucnorm{ \foutdata }^2.
  \end{split}
  \end{equation}
  The Markov inequality thus \proves that
  \begin{equation}
  \begin{split}
    &\P \Bigl( \eucnorm{ \fnetvalue(0) - \foutdata }^2 \le \bigl( 
      \cvar \Cvar \width \outdim \eucnorm{\slope}^2 \eucnorm{ \findata }^2
      + \Cvar \width \outdim \eucnorm{\yinter}^2 \ndata + \eucnorm{ \foutdata }^2 \bigr) \poe^{-1} \Bigr) \\
    &\ge 1 - \P \Bigl( \eucnorm{ \fnetvalue(0) - \foutdata }^2 \ge \bigl(
      \cvar \Cvar \width \outdim \eucnorm{\slope}^2 \eucnorm{ \findata }^2
      + \Cvar \width \outdim \eucnorm{\yinter}^2 \ndata + \eucnorm{ \foutdata }^2 \bigr) \poe^{-1} \Bigr)  \\
    &\ge 1 - \frac{\E \bigl[ \eucnorm{ \fnetvalue(0) - \foutdata }^2 \bigr]}
      { \bigl( \cvar \Cvar \width \outdim \eucnorm{\slope}^2 \eucnorm{ \findata }^2
      + \Cvar \width \outdim \eucnorm{\yinter}^2 \ndata + \eucnorm{ \foutdata }^2 \bigr) \poe^{-1} }
    \ge 1 - \poe.
  \end{split}
  \end{equation}
\end{cproof}

\subsection{Analysis of weights and biases of ANNs during training}
\label{subsec:analysis_of_weights_and_biases_of_ANNs_during_training}

\cfclear
\begin{lemma}
  \label{lem:eucnorm_inequality}
  Let $ n \in \N $, $ v = ( v_1, \ldots, v_n ) \in \R^n $.
  Then $ \eucnorm{ v } \le \sum_{i=1}^n \abs{ v_i } \le \sqrt{n} \eucnorm v $
  \cfout.
\end{lemma}

\begin{cproof}{lem:eucnorm_inequality}
  Throughout this proof let $ p = ( p_1, \ldots, p_n ) $, $ q = ( q_1, \ldots, q_n ) \in \R^n $
  satisfy for all $ i \in \{ 1, 2, \ldots, n \} $ that $ p_i = \abs{ v_i } $ and $ q_i = 1 $.
  \Nobs that the fact that for all $ a, b \in [0,\infty) $ it holds that
  $ ( a + b )^{\nicefrac{1}{2}} \le a^{\nicefrac{1}{}2} + b^{\nicefrac{1}{2}} $ inductively \proves that
  \begin{equation}
    \eucnorm{ v }
    = \biggl( \smallsum\limits_{i=1}^n \abs{ v_i }^2 \biggr)^{\! \nicefrac{1}{2}}
    \le \smallsum\limits_{i=1}^n \bigl( \abs{ v_i }^2 \bigr)^{\nicefrac{1}{2}}
    = \smallsum\limits_{i=1}^n \abs{ v_i }
  \end{equation}
  \cfload.
  \Moreover the Cauchy Schwarz inequality \proves that
  \begin{equation}
    \smallsum\limits_{i=1}^n \abs{ v_i }
    = \scalprod{ p } { q }
    \le \eucnorm{ p } \eucnorm{ q }
    = \biggl( \smallsum\limits_{i=1}^n \abs{ v_i }^2 \biggr)^{\! \nicefrac{1}{2}}
      \biggl( \smallsum\limits_{i=1}^n 1^2 \biggr)^{\! \nicefrac{1}{2}}
    = \sqrt{n} \eucnorm{ v }.
  \end{equation}
\end{cproof}

\cfclear
\begin{lemma}
  \label{lem:gaussian_tail_bound}
  Let $ ( \Omega, \mc F, \P ) $ be a probability space,
  let $ X \colon \Omega \to \R $ be a standard normal random variable, and
  let $ \varepsilon \in (0,\infty) $.
  Then 
  $ \P ( \abs X \ge \varepsilon ) \le 2 \exp( - \nicefrac{\varepsilon^2}{2} ) $.
\end{lemma}

\begin{cproof}{lem:gaussian_tail_bound}
  Throughout this proof let $ f \colon (0,\infty) \to \R $ satisfy for all $ \lambda \in (0,\infty) $ that
  $ f(\lambda) = \tfrac{1}{2} \lambda^2 - \lambda \varepsilon $.
  \Nobs that the assumption that $ X $ is a standard normal random variable \proves that for all $ \lambda \in \R $
  it holds that
  \begin{equation}
  \begin{split}
    \E \bigl[ \exp( \lambda X ) \bigr]
    &= \int_{-\infty}^\infty \frac{1}{\sqrt{2 \pi}} \exp \Bigl( - \frac{x^2}{2} \Bigr) \exp( \lambda x ) \diff x
    = \int_{-\infty}^\infty \frac{1}{\sqrt{2 \pi}} \exp \biggl( - \frac{x^2 - 2 \lambda x}{2} \biggr) \diff x \\
    &= \exp \Bigl( \frac{\lambda^2}{2} \Bigr) \int_{-\infty}^\infty \frac{1}{\sqrt{2 \pi}}
      \exp \biggl( - \frac{(x - \lambda)^2}{2} \biggr) \diff x
    = \exp \Bigl( \frac{\lambda^2}{2} \Bigr) < \infty.
  \end{split}
  \end{equation}
  This and the Markov inequality \prove that for all $ \lambda \in (0,\infty) $ it holds that
  \begin{equation}
  \begin{split}
    \P ( \abs X \ge \varepsilon )
    &= \P ( X \ge \varepsilon ) + \P ( - X \ge \varepsilon )
    = \P \bigl( \exp( \lambda X ) \ge \exp( \lambda \varepsilon ) \bigr) + \P \bigl(\exp( -\lambda X ) \ge \exp( \lambda \varepsilon ) \bigr) \\
    &\le \frac{\E \bigl[ \exp( \lambda X ) \bigr]}{\exp( \lambda \varepsilon)}
      + \frac{\E \bigl[ \exp( -\lambda X ) \bigr]}{\exp( \lambda \varepsilon)}
    = \exp \! \biggl( \frac{\lambda^2}{2} - \lambda \varepsilon \biggr) + \exp \! \biggl(\frac{(-\lambda)^2}{2} - \lambda \varepsilon \biggr) \\
    &= 2 \exp \! \biggl( \frac{\lambda^2}{2} - \lambda \varepsilon \biggr)
    = 2 \exp \bigl( f(\lambda) \bigr).
  \end{split}
  \end{equation}
  Combining this with the fact that $ \R \ni x \mapsto \exp (x) \in \R $ is strictly increasing \proves that
  \begin{equation}
  \label{eq:gaussian_tail_bound}
    \P ( \abs X \ge \varepsilon )
    \le \inf\nolimits_{ \lambda \in (0,\infty) } 2 \exp \bigl( f(\lambda) \bigr)
    = 2 \exp \bigl( \inf\nolimits_{ \lambda \in (0,\infty) } f(\lambda) \bigr).
  \end{equation}
  \Moreover for all $ \lambda \in (0,\infty) $ it holds that
  \begin{equation}
    f(\lambda)
    = \tfrac{\lambda^2}{2} - \lambda \varepsilon
    = \tfrac{1}{2} \bigl[ \lambda^2 - 2 \lambda \varepsilon \bigr]
    = \tfrac{1}{2} \bigl[ (\lambda - \varepsilon)^2 - \varepsilon^2 \bigr]
    = \tfrac{1}{2} (\lambda - \varepsilon)^2 - \tfrac{1}{2} \varepsilon^2
    \ge - \tfrac{1}{2} \varepsilon^2.
  \end{equation}
  This and the fact that $ f(\varepsilon) = - \nicefrac{\varepsilon^2}{2} $ \prove that it holds that
  $ \inf_{ \lambda \in (0,\infty) } f(\lambda) = - \nicefrac{\varepsilon^2}{2} $.
  Combining this with \cref{eq:gaussian_tail_bound} \proves that
  $ \P ( \abs X \ge \varepsilon )
    \le 2 \exp \bigl( \inf\nolimits_{ \lambda \in (0,\infty) } f(\lambda) \bigr)
    = 2 \exp ( - \nicefrac{\varepsilon^2}{2} ) $.
\end{cproof}

\cfclear
\begin{lemma}
  \label{lem:probability_A_W}
  Assume \cref{setting:gradient_descent} and let $ \poe \in (0, 1) $.
  Then
  $ \P \bigl( \bigcap_{p=1}^{\outdim} \bigcap_{k=1}^{\width} \bigl\{ \abs{ \mf W_{p,k} }^2
    \le 2 \Cvar \ln( \frac{2 \outdim \width}{\poe} ) \bigr\} \bigr) \ge 1 - \poe $.
\end{lemma}

\begin{cproof}{lem:probability_A_W}
  \Nobs that the assumption that
  $ \sqrt{\nicefrac{1}{\Cvar}} \mf W_1, \sqrt{\nicefrac{1}{\Cvar}} \mf W_2, \ldots, \sqrt{\nicefrac{1}{\Cvar}} \mf W_{\outdim} $
  are standard normal and
  \cref{lem:gaussian_tail_bound} (applied for every $ p \in \{ 1, 2, \dots, \outdim \} $, $ k \in \{ 1, 2, \ldots, \width \} $ with
  $ ( \Omega, \mc F, \P ) \is ( \Omega, \mc F, \P ) $,
  $ X \is \sqrt{\nicefrac{1}{\Cvar}} \mf W_{p,k} $,
  $ \varepsilon \is ( 2 \ln( \frac{2 \outdim \width}{\poe} ) )^{\nicefrac{1}{2}} $
  in the notation of \cref{lem:gaussian_tail_bound})
  \prove that
  \begin{equation}
  \begin{split}
    &\P \biggl( \smallbigcap\limits_{p=1}^{\outdim} \smallbigcap\limits_{k=1}^{\width} \Bigl\{ \abs{ \mf W_{p,k} }^2 \le
      2 \Cvar \ln( \tfrac{2 \outdim \width}{\poe} ) \Bigr\} \biggr)
    \ge 1 - \P \biggl( \smallbigcap\limits_{p=1}^{\outdim} \smallbigcup\limits_{k=1}^{\width} \Bigl\{ \abs{ \mf W_{p,k} }^2 \ge
      2 \Cvar \ln( \tfrac{2 \outdim \width}{\poe} ) \Bigr\} \biggr) \\
    &\ge 1 - \smallsum\limits_{p=1}^{\outdim} \smallsum\limits_{k=1}^{\width} \P \Bigl( \abs{ \mf W_{p,k} }^2 \ge
      2 \Cvar \ln( \tfrac{2 \outdim \width}{\poe} ) \Bigr)
    = 1 - \smallsum\limits_{p=1}^{\outdim} \smallsum\limits_{k=1}^{\width} \P \Bigl( \abs{ \sqrt{\nicefrac{1}{\Cvar}} \mf W_{p,k} } \ge
      \bigl( 2 \ln( \tfrac{2 \outdim \width}{\poe} ) \bigr)^{\nicefrac{1}{2}} \Bigr) \\
    &\ge 1 - \smallsum\limits_{p=1}^{\outdim} \smallsum\limits_{k=1}^{\width}
      2 \exp \Bigl( - \tfrac{1}{2} \bigl( 2 \ln( \tfrac{2 \outdim \width}{\poe} ) \bigr) \Bigr)
    = 1 - \smallsum\limits_{p=1}^{\outdim} \smallsum\limits_{k=1}^{\width} 2 \exp \bigl( \ln( \tfrac{\poe}{2 \outdim \width} ) \bigr)
    = 1 - \smallsum\limits_{p=1}^{\outdim} \smallsum\limits_{k=1}^{\width} \tfrac{\poe}{\outdim \width}
    = 1 - \poe.
  \end{split}
  \end{equation}
\end{cproof}

\cfclear
\begin{lemma}
  \label{lem:first_layer}
  Assume \cref{setting:gradient_descent},
  assume $ \lambdazero \in (0, \infty) $,
  assume $ \lr < \frac{\ndata}{\Cvar \width \lambdazero} $,
  let $ \poe \in (0, 1) $, $ A \in \mc F $ satisfy
  $ A = \bigcap_{p=1}^{\outdim} \bigcap_{k=1}^{\width} \bigl\{ \abs{ \mf W_{p,k} }^2
    \le 2 \Cvar \ln( \frac{2 \outdim \width}{\poe} ) \bigr\} $,
  let $ \omega \in A $, $ \Iteration \in \N_0 $, and
  assume for all $ \iteration \in \{ 0, 1, 2, \ldots, \Iteration \} $ that
  $ \eucnorm{ \fnetvalue(\iteration,\omega) - \foutdata }^2 \le \bigl( 1 - \frac{\lr \Cvar \width \lambdazero}{\ndata} \bigr)^{\iteration}
    \eucnorm{ \fnetvalue(0,\omega) - \foutdata }^2 $
  \cfload.
  Then
  \begin{enumerate}[label=(\roman{*})]
    \item \label{item:first_layer_1} it holds for all $ k \in \{ 1, 2, \ldots, \width \} $ that
      $ \eucnorm{ W_k(\Iteration+1,\omega) - W_k(0,\omega) }
        \le \frac{4 \cmax \cderiv \eucnorm{ \fnetvalue(0,\omega) - \foutdata }}{\width \lambdazero}
          \bigl( \frac{2 \outdim \ndata}{\Cvar} \ln( \frac{2 \outdim \width}{\poe} ) \bigr)^{\nicefrac{1}{2}} $ and
    \item \label{item:first_layer_2} it holds for all $ k \in \{ 1, 2, \ldots, \width \} $ that
      $ \abs{ B_k(\Iteration+1,\omega) - B_k(0,\omega) } 
        \le \frac{4 \cderiv \eucnorm{ \fnetvalue(0,\omega) - \foutdata }}{\width \lambdazero}
          \bigl( \frac{2 \outdim \ndata}{\Cvar} \ln( \frac{2 \outdim \width}{\poe} ) \bigr)^{\nicefrac{1}{2}} $.
  \end{enumerate}
\end{lemma}

\begin{cproof}{lem:first_layer}
  \Nobs that for all $ q \in (0,1) $ it holds that $ (1-q)^{\nicefrac{1}{2}} \le 1 - \tfrac{q}{2} $ and
  \begin{equation}
    \smallsum\limits_{\iteration = 0}^{\Iteration} (1-q)^{\nicefrac{\iteration}{2}}
    \le \smallsum\limits_{\iteration = 0}^\infty (1-q)^{\nicefrac{\iteration}{2}}
    = \displaystyle \frac{1}{1 - ( 1 - q )^{\nicefrac{1}{2}}}
    \le \frac{1}{1 - ( 1 - \frac{q}{2} )}
    = \frac{2}{q}.
  \end{equation}
  This, the fact that it holds that $ \frac{\lr \Cvar \width \lambdazero}{\ndata} \in (0,1) $,
  \cref{lem:eucnorm_inequality} (applied for every $ \iteration \in \{ 0, 1, 2, \ldots, \Iteration \} $ with
  $ m \is \outdim \ndata $, $ v \is \fnetvalue(\iteration,\omega) - \foutdata $
  in the notation of \cref{lem:eucnorm_inequality}),
  the assumption that for all $ \iteration \in \{ 0, 1, 2, \ldots, \Iteration \} $ it holds that
  $ \eucnorm{ \fnetvalue(\iteration,\omega) - \foutdata }^2 \le \bigl( 1 - \frac{\lr \Cvar \width \lambdazero}{\ndata} \bigr)^{\iteration}
    \eucnorm{ \fnetvalue(0,\omega) - \foutdata }^2 $, and
  the assumption that $ \omega \in A $
  \prove that it holds for all $ k \in \{ 1, 2, \ldots, \width \} $ that
  \begin{equation}
  \begin{split}
    &\smallsum\limits_{\iteration = 0}^{\Iteration} \smallsum\limits_{i=1}^{\ndata} \smallsum\limits_{p=1}^{\outdim}
      \abs{ \netvalue_i^p (\iteration,\omega) - \outdata_i^p }
    \le \smallsum\limits_{\iteration = 0}^{\Iteration} \sqrt{\outdim \ndata} \eucnorm{ \fnetvalue(\iteration,\omega) - \foutdata }
    = \sqrt{\outdim \ndata} \smallsum\limits_{\iteration = 0}^{\Iteration} \bigl(
      \eucnorm{ \fnetvalue(\iteration,\omega) - \foutdata }^2 \bigr)^{\nicefrac{1}{2}} \\
    &\le \sqrt{\outdim \ndata} \smallsum\limits_{\iteration = 0}^{\Iteration}
      \Bigl(1 - \tfrac{\lr \Cvar \width \lambdazero}{\ndata} \Bigr)^{\nicefrac{\iteration}{2}} \eucnorm{ \fnetvalue(0,\omega) - \foutdata }
    = \sqrt{\outdim \ndata} \eucnorm{ \fnetvalue(0,\omega) - \foutdata } \smallsum\limits_{\iteration = 0}^{\Iteration}
      \Bigl(1 - \tfrac{\lr \Cvar \width \lambdazero}{\ndata} \Bigr)^{\nicefrac{\iteration}{2}} \\
    &\le \sqrt{\outdim \ndata} \eucnorm{ \fnetvalue(0,\omega) - \foutdata } \frac{ 2 \ndata }{\lr \Cvar \width \lambdazero}
    = \frac{ 2 \ndata \sqrt{\outdim \ndata} \eucnorm{ \fnetvalue(0,\omega) - \foutdata } }{\lr \Cvar \width \lambdazero}.
  \end{split}
  \end{equation}
  The assumption that $ \omega \in A $ \hence \proves that for all $ k \in \{ 1, 2, \ldots, \width \} $ it holds that
  \begin{equation}\label{eq:first_layer_term}
  \begin{split}
    \smallsum\limits_{\iteration = 0}^{\Iteration} \smallsum\limits_{i=1}^{\ndata} \smallsum\limits_{p=1}^{\outdim}
      \abs{ \netvalue_i^p (\iteration,\omega) - \outdata_i^p } \abs{ \mf W_{p,k}(\omega) }
    &\le \bigl( 2 \Cvar \ln( \tfrac{2 \outdim \width}{\poe} ) \bigr)^{\nicefrac{1}{2}}
      \smallsum\limits_{\iteration = 0}^{\Iteration} \smallsum\limits_{i=1}^{\ndata} \smallsum\limits_{p=1}^{\outdim}
      \abs{ \netvalue_i^p (\iteration,\omega) - \outdata_i^p } \\
    &\le \displaystyle \frac{2 \ndata \eucnorm{ \fnetvalue(0,\omega) - \foutdata }}
      {\lr \width \lambdazero} \bigl( \tfrac{2 \outdim \ndata}{\Cvar} \ln( \tfrac{2 \outdim \width}{\poe} ) \bigr)^{\nicefrac{1}{2}}.
  \end{split}
  \end{equation}
  Combining this with \cref{eq:setting_W} and
  the fact that for all $ v \in \R $ it holds that $ \abs{ \deriv(v) } \le \cderiv $
  \proves for all $ k \in \{ 1, 2, \ldots, \width \} $ that 
  \begin{equation}
  \begin{split}
    &\eucnorm{ W_k(\Iteration+1,\omega) - W_k(0,\omega) }
    = \bbbeucnorm{ \smallsum\limits_{\iteration = 0}^{\Iteration}
      \Bigl[ W_k(\iteration+1,\omega) - W_k(\iteration,\omega) \Bigr] } \\
    &= \bbbeucnorm{ \smallsum\limits_{\iteration = 0}^{\Iteration}
      - \displaystyle \frac{2 \lr}{\ndata} \smallsum\limits_{i=1}^{\ndata} \smallsum\limits_{p=1}^{\outdim}
      ( \netvalue_i^p (\iteration,\omega) - \outdata_i^p ) \mf W_{p,k} (\omega)
      \deriv\bigl( \scalprod{ W_k(\iteration,\omega) }{ \indata_i } + B_k(\iteration,\omega) \bigr) \, \indata_i } \\
    &\le \frac{2 \lr}{\ndata} \smallsum\limits_{\iteration = 0}^{\Iteration} \smallsum\limits_{i=1}^{\ndata} \smallsum\limits_{p=1}^{\outdim}
      \babs{ ( \netvalue_i^p (\iteration,\omega) - \outdata_i^p ) \mf W_{p,k} (\omega)
      \deriv\bigl( \scalprod{ W_k(\iteration,\omega) }{ \indata_i } + B_k(\iteration,\omega) \bigr) } \eucnorm{ \indata_i } \\
    &\le \frac{2 \lr \cmax \cderiv}{\ndata} \smallsum\limits_{\iteration = 0}^{\Iteration} \smallsum\limits_{i=1}^{\ndata}
      \smallsum\limits_{p=1}^{\outdim} \abs{ \netvalue_i^p (\iteration,\omega) - \outdata_i^p } \abs{ \mf W_{p,k} (\omega) } 
    \le \displaystyle \frac{4 \cmax \cderiv \eucnorm{ \fnetvalue(0,\omega) - \foutdata }}
      {\width \lambdazero} \bigl( \tfrac{2 \outdim \ndata}{\Cvar} \ln( \tfrac{2 \outdim \width}{\poe} ) \bigr)^{\nicefrac{1}{2}}.
  \end{split}
  \end{equation}
  This \proves \cref{item:first_layer_1}.
  Next \nobs that \cref{eq:setting_B}, \cref{eq:first_layer_term}, and
  the fact that for all $ v \in \R $ it holds that $ \abs{ \deriv(v) } \le \cderiv $
  \prove that
  \begin{equation}
  \begin{split}
    &\abs{ B_k(\Iteration+1,\omega) - B_k(0,\omega) }
    = \bbbabs{ \smallsum\limits_{\iteration = 0}^{\Iteration}
      \Bigl[ B_k(\iteration+1,\omega) - B_k(\iteration,\omega) \Bigr] } \\
    &= \bbbabs{ \smallsum\limits_{\iteration = 0}^{\Iteration}
      - \displaystyle \frac{2 \lr}{\ndata} \smallsum\limits_{i=1}^{\ndata} \smallsum\limits_{p=1}^{\outdim}
      ( \netvalue_{i,p} (\iteration,\omega) - \outdata_{i,p} ) \mf W_{p,k} (\omega)
      \deriv\bigl( \scalprod{ W_k(\iteration,\omega) }{ \indata_i } + B_k(\iteration,\omega) \bigr) } \\
    &\le \frac{2 \lr}{\ndata} \smallsum\limits_{\iteration = 0}^{\Iteration} \smallsum\limits_{i=1}^{\ndata} \smallsum\limits_{p=1}^{\outdim}
      \babs{ ( \netvalue_{i,p} (\iteration,\omega) - \outdata_{i,p} ) \mf W_{p,k} (\omega)
      \deriv\bigl( \scalprod{ W_k(\iteration,\omega) }{ \indata_i } + B_k(\iteration,\omega) \bigr) } \\
    &\le \frac{2 \lr \cderiv}{\ndata} \smallsum\limits_{\iteration = 0}^{\Iteration} \smallsum\limits_{i=1}^{\ndata}
      \smallsum\limits_{p=1}^{\outdim} \abs{ \netvalue_{i,p} (\iteration,\omega) - \outdata_{i,p} } \abs{ \mf W_{p,k} (\omega) } 
    \le \displaystyle \frac{4 \cderiv \eucnorm{ \fnetvalue(0,\omega) - \foutdata }}
      {\width \lambdazero} \bigl( \tfrac{2 \outdim \ndata}{\Cvar} \ln( \tfrac{2 \outdim \width}{\poe} ) \bigr)^{\nicefrac{1}{2}}.
  \end{split}
  \end{equation}
  This \proves \cref{item:first_layer_2}.
\end{cproof}

\subsection{Properties of subexponential random variables}
\label{subsec:properties_of_subexponential_random_variables}

\cfclear
\cfconsiderloaded{def:subexponential}
\begin{definition}[Subexponential random variable]
  \label{def:subexponential}
  Let $ ( \Omega, \mc F, \P ) $ be a probability space and let $ \nu, b \in (0,\infty) $.
  Then we say that $ X $ is \subexp{\nu}{b} with respect to $ \P $
  (we say that $ X $ is \subexp{\nu}{b}) if and only if we have that
  \begin{enumerate}[label=(\roman{*})]
    \item it holds that $ X \colon \Omega \to \R $ is a function from $ \Omega $ to $ \R $,
    \item it holds that $ X $ is measurable,
    \item it holds that $ \E [ \abs{X} ] < \infty $, and
    \item it holds for all $ \lambda \in ( - \nicefrac{1}{b}, \nicefrac{1}{b} ) $ that
      $ \E[ \exp( \lambda( X - \E[X] ) ) ] \le \exp( \frac{1}{2} \lambda^2 \nu^2 ) $.
  \end{enumerate}
\end{definition}

\cfclear
\begin{lemma}
  \label{lem:subexponential_scalar}
  Let $ ( \Omega, \mc F, \P ) $ be a probability space,
  let $ \nu, b \in (0,\infty) $, $ a \in \R \backslash \{ 0 \}$,
  and let $ X $ be \subexp{\nu}{b} \cfload.
  Then $ a X $ is \subexp{ \abs{a} \nu}{ \abs{a} b}.
\end{lemma}

\begin{cproof}{lem:subexponential_scalar}
  \Nobs that for all $ \lambda \in ( - \frac{1}{\abs{a} b}, \frac{1}{\abs{a} b} ) $
  it holds that $ a \lambda \in ( - \nicefrac{1}{b}, \nicefrac{1}{b} ) $.
  \Hence that for all $ \lambda \in ( - \frac{1}{\abs{a} b}, \frac{1}{\abs{a} b} ) $ it holds that
  \begin{equation}
    \E \bigl[ \exp \bigl( \lambda ( aX - \E[aX] ) \bigr) \bigr]
    = \E \bigl[ \exp \bigl( a \lambda ( X - \E[X] ) \bigr) \bigr]
    \le \exp \bigl( \tfrac{1}{2} (a \lambda)^2 \nu^2 \bigr)
    = \exp \bigl( \tfrac{1}{2} \lambda^2 ( \abs{a} \nu)^2 \bigr).
  \end{equation}
  This \proves that $ a X $ is \subexp{ \abs{a} \nu }{ \abs{a} b }.
\end{cproof}

\cfclear
\begin{lemma}
  \label{lem:subexponential_weighted_sum}
  Let $ ( \Omega, \mc F, \P ) $ be a probability space,
  let $ n \in \N $,
  $ \nu = (\nu_1, \nu_2, \ldots, \nu_n)$, $ b = (b_1, b_2, \ldots, b_n) \in (0,\infty)^n $,
  $ a = ( a_1, a_2, \ldots, a_n ) \in \R^n \backslash \{ 0 \} $,
  let $ X_i \colon \Omega \to \R $, $ i \in \{ 1, 2, \ldots, n \} $, be independent random variables,
  and assume for all $ i \in \{ 1, 2, \ldots, n \} $ that $ X_i $ is \subexp{\nu_i}{b_i} \cfload.
  Then $ \sum_{i=1}^n a_i X_i $ is
  \subexp{( \sum_{i=1}^n \abs{ a_i \nu_i }^2 )^{1/2}}{\max \{ \abs{a_1} b_1, \abs{a_2} b_2, \ldots, \abs{a_n} b_n \} }.
\end{lemma}

\begin{cproof}{lem:subexponential_weighted_sum}
  Throughout this proof let $ B \in (0,\infty) $ satisfy
  $ B = \max \{ \abs{a_1} b_1, \abs{a_2} b_2, \ldots, \abs{a_n} b_n \} $ and
  let $ I \subseteq \{ 1, 2, \ldots, n \} $ satisfy $ I = \{ i \in \{ 1, 2, \ldots, n \} \colon a_i \neq 0 \} $.
  \Nobs that \cref{lem:subexponential_scalar} (applied for every $ i \in I $ with
  $ ( \Omega, \mc F, \P ) \is ( \Omega, \mc F, \P ) $,
  $ \nu \is \nu_i $, $ b \is b_i $, $ a \is a_i $, $ X \is X_i $
  in the notation of \cref{lem:subexponential_scalar})
  \proves that for all $ i \in I $ it holds that
  \begin{equation}
    a_i X_i
  \end{equation}
  is \subexp{ \abs{a_i} \nu_i }{ \abs{a_i} b_i }.
  Combining this, the fact that $ X_1, X_2, \ldots, X_n $ are independent, and
  the fact that for all $ i \in I $ it holds that
  $ (-\nicefrac{1}{B}, \nicefrac{1}{B}) \subseteq (-\nicefrac{1}{\abs{a_i} b_i}, \nicefrac{1}{\abs{a_i} b_i}) $
  \prove that for all $ \lambda \in (-\nicefrac{1}{B}, \nicefrac{1}{B} ) $ it holds that
  \begin{equation}
  \begin{split}
    \E \biggl[ \exp\biggl( \lambda \biggl(\smallsum\limits_{i=1}^n a_i X_i
      - \E \biggl[ \smallsum\limits_{i=1}^n a_i X_i \biggr] \biggr) \biggr) \biggr]
    &= \E \biggl[ \exp \Bigl( \lambda \smallsum\limits_{i \in I} \bigl( a_i X_i - \E[a_i X_i] \bigr) \Bigr) \biggr] \\
    &= \E \biggl[ \smallprod\limits_{i \in I} \exp \bigl( \lambda ( a_i X_i - \E[a_i X_i] ) \bigr) \biggr] \\
    &= \prod_{i \in I} \E \Bigl[ \exp \bigl( \lambda ( a_i X_i - \E[a_i X_i] ) \bigr) \Bigr]
    \le \prod_{i \in I} \exp \biggl( \frac{\lambda^2 \abs{a_i \nu_i}^2}{2} \biggr) \\
    &= \exp \biggl( \frac{\lambda^2 \sum_{i \in I} \abs{a_i \nu_i}^2}{2} \biggr)
    = \exp \biggl( \frac{\lambda^2 \sum_{i=1}^n \abs{a_i \nu_i}^2}{2} \biggr).
  \end{split}
  \end{equation}
  This \proves that $ \sum_{i=1}^n a_i X_i $ is \subexp{( \sum_{i=1}^n \abs{ a_i \nu_i }^2 )^{1/2}}{B}.
\end{cproof}

\cfclear
\begin{lemma}
  \label{lem:subexponential_chi_squared}
  Let $ ( \Omega, \mc F, \P ) $ be a probability space,
  let $ X \colon \Omega \to \R $ and $ Y \colon \Omega \to \R $ be standard normal random variables,
  let $ A \in \mc F $, and assume that $ X $, $ Y $, and $ \ind_A $ are independent.
  Then
  \begin{enumerate}[(i)]
    \item \label{lem:subexponential_chi_squared:item1} it holds that $ X^2 \ind_A $ is \subexp{2}{4} and 
    \item \label{lem:subexponential_chi_squared:item2} it holds that $ X Y \ind_A $ is \subexp{\sqrt{2}}{\sqrt{2}}
  \end{enumerate}
  \cfout. 
\end{lemma}

\begin{cproof}{lem:subexponential_chi_squared}
  Throughout this proof
  let $ p, q \in [0,1] $ satisfy $ p = \P(A) $ and $ q = 1 - \P(A) $,
  and let $ f \colon ( - \infty, \nicefrac{1}{2} ) \to \R $ and $ g \colon ( -\infty, \nicefrac{1}{2} ) \to \R $
  satisfy for all $ \lambda \in ( -\infty, \nicefrac{1}{2} ) $ that
  \begin{equation}
  \label{eq:subexponential_f_g}
    f(\lambda) = 2 \lambda^2 + \lambda + \tfrac{1}{2} \ln( 1 - 2 \lambda)
    \qqandqq
    g(\lambda) = q \lambda + \ln\bigl( p + q\sqrt{1 - 2 \lambda} \bigr).
  \end{equation}
  \Nobs that for all $ \lambda \in ( -\nicefrac{1}{4}, \nicefrac{1}{4} ) $ it holds that
  $ 1 - 2 \lambda > 0 $, $ \nicefrac{1}{4} - \lambda > 0$, and
  \begin{equation}
  \label{eq:subexponential_f_prime}
  \begin{split}
    f'(\lambda)
    &= 4 \lambda + 1 - \frac{1}{(1 - 2 \lambda)}
    = \frac{(4 \lambda + 1)(1 - 2 \lambda) - 1}{(1 - 2 \lambda)} \\
    &= \frac{4 \lambda + 1 - 8 \lambda^2 - 2 \lambda - 1}{(1 - 2 \lambda)}
    = \frac{2 \lambda - 8 \lambda^2}{(1 - 2 \lambda)}
    = \frac{8 \lambda ( \frac{1}{4} - \lambda )}{(1 - 2 \lambda)}.
  \end{split}
  \end{equation}
  This \proves for all $ \lambda \in (-\nicefrac{1}{4},0) $ that $ f'(\lambda) < 0 $.
  \Hence that $ f|_{(-\nicefrac{1}{4},0)} $ is strictly decreasing.
  \Moreover \cref{eq:subexponential_f_prime} \proves that for all $ \lambda \in (0,\nicefrac{1}{4}) $
  it holds that $ f'(\lambda) > 0 $.
  \Hence that $ f|_{(0,\nicefrac{1}{4})} $ is strictly increasing.
  Combining this, the fact that $ f|_{(-\nicefrac{1}{4},0)} $ is strictly decreasing, and the fact that $ f(0) = 0 $
  \proves that for all $ \lambda \in ( -\nicefrac{1}{4}, \nicefrac{1}{4} ) $ it holds that
  \begin{equation}
  \label{eq:subexponential_f_nonnegative}
    f(\lambda) \ge f(0) = 0.
  \end{equation}    
  Next \nobs that \cref{eq:subexponential_f_g} \proves  that for all $ \lambda \in ( -\infty, \nicefrac{1}{2} ) $ it holds that
  \begin{equation}
    g'(\lambda)
    = q + \frac{ \frac{q}{2} \bigl[ (1 - 2 \lambda)^{-\nicefrac{1}{2}} (-2) \bigr]}{p + q \sqrt{1 - 2 \lambda}}
    = q - q \Bigl[ p (1 - 2 \lambda)^{\nicefrac{1}{2}} + q (1 - 2 \lambda) \Bigr]^{-1}
  \end{equation}
  and
  \begin{equation}
  \label{eq:subexponential_g_prime_prime}
    g''(\lambda)
    = \frac{q \bigl[ \frac{p}{2} (1 - 2 \lambda)^{-\nicefrac{1}{2}} (-2) - 2q \bigr]}
      {\bigl ( p (1 - 2 \lambda)^{\nicefrac{1}{2}} + q (1 - 2 \lambda) \bigr)^2}
    = - \frac{q \bigl[ p ( 1 - 2 \lambda)^{-\nicefrac{1}{2}} + 2q \bigr]}
      {\bigl ( p (1 - 2 \lambda)^{\nicefrac{1}{2}} + q (1 - 2 \lambda) \bigr)^2}
    \le 0.
  \end{equation}
  \Moreover the fundamental theorem of calculus and the fact that $ g(0) = g'(0) = 0 $ \prove that
  for all $ \lambda \in (-\infty,\nicefrac{1}{2}) $ it holds that
  \begin{equation}
  \label{eq:subexponential_g_ftc}
  \begin{split}
    g(\lambda)
    = g(0) + \int_0^\lambda g'(s) \diff s
    = g(0) + \int_0^\lambda \biggl( g'(0) + \int_0^s g''(r) \diff r \biggr) \diff s
    = \int_0^\lambda \int_0^s g''(r) \diff r \diff s.
  \end{split}
  \end{equation}
  Combining this with \cref{eq:subexponential_g_prime_prime} \proves that for all $ \lambda \in [0, \nicefrac{1}{2}) $ it holds that
  \begin{equation}
  \label{eq:subexponential_g_nonpositive_1}
    g(\lambda)
    = \int_0^\lambda \int_0^s g''(r) \diff r \diff s
    \le 0.
  \end{equation}
  \Moreover \cref{eq:subexponential_g_prime_prime} and \cref{eq:subexponential_g_ftc}
 \prove for all $ \lambda \in (-\infty,0) $ that
  \begin{equation}
  \label{eq:subexponential_g_nonpositive_2}
    g(\lambda)
    = \int_0^\lambda \int_0^s g''(r) \diff r \diff s
    = - \int_\lambda^0 \int_0^s g''(r) \diff r \diff s
    = \int_\lambda^0 \int_s^0 g''(r) \diff r \diff s
    \le 0.
  \end{equation}
  Next \nobs that the fact that $ X $ is a standard normal random variable \proves that
  for all $ \lambda \in (-\infty, \nicefrac{1}{2}) $ it holds that
  \begin{equation}
  \label{eq:chi_squared_subexponential}
  \begin{split}
    \E \bigl[ \exp( \lambda X^2 ) \bigr]
    &= \int_{-\infty}^{\infty} \frac{1}{\sqrt{2\pi}} \exp \Bigl( -\frac{x^2}{2} \Bigr) \exp(\lambda x^2) \diff x
    = \int_{-\infty}^{\infty} \frac{1}{\sqrt{2\pi}} \exp \biggl( -\frac{x^2 ( 1 - 2 \lambda ) }{2} \biggr) \diff x \\
    &= \frac{1}{\sqrt{1 - 2\lambda}} \int_{-\infty}^{\infty} \frac{1}{\sqrt{2\pi}} \exp \Bigl( -\frac{y^2}{2} \Bigr) \diff y
    = \frac{1}{\sqrt{1 - 2\lambda}} < \infty.
  \end{split}
  \end{equation}
  \Moreover the assumption that $ X $ and $ \ind_A $ are independent
  and the assumption that $ X $ is a standard normal random variable
  \prove that it holds that $ \E[ X^2 \ind_A ] = \E [ X^2 ] \E [ \ind_A ] = \E [ \ind_A ] = \P(A) = p $.
  This, the assumption that $ X $ and $ \ind_A $ are independent,
  \cref{eq:subexponential_f_nonnegative},
  \cref{eq:subexponential_g_nonpositive_1},
  \cref{eq:subexponential_g_nonpositive_2}, and
  \cref{eq:chi_squared_subexponential}
  \prove that for all $ \lambda \in (-\nicefrac{1}{4}, \nicefrac{1}{4}) $ it holds that
  \begin{equation}
  \begin{split}
    \ln \bigl( \E \bigl[ \exp\bigl( \lambda ( X^2 \ind_A - \E[X^2 \ind_A] ) \bigr) \bigr] \bigr)
    &= \ln \bigl( \E \bigl[ \exp \bigl( \lambda X^2 \ind_A - p \lambda \bigr) \bigr] \bigl)
    = \ln \bigl( \E \bigl[ \exp( \lambda X^2 \ind_A ) \bigr] \bigr) - p \lambda \\
    &= \ln \bigl( \E \bigl[ \exp( \lambda X^2 \ind_A ) \bigl( \ind_A + \ind_{ \Omega \backslash A } \bigr) \bigr] \bigr) - p \lambda \\
    &= \ln \bigl( \E \bigl[ \exp( \lambda X^2 ) \ind_A + \ind_{ \Omega \backslash A } \bigr] \bigr) - p \lambda \\
    &= \ln \bigl( p \E \bigl[ \exp( \lambda X^2 ) \bigr] + q \bigr) - p \lambda
    = \ln \biggl( \frac{p}{\sqrt{1 - 2 \lambda}} + q \biggr) - p \lambda \\
    &= \ln \biggl( \frac{p + q \sqrt{1 - 2 \lambda}}{\sqrt{1 - 2 \lambda}} \biggr) - p \lambda
    = \ln \bigl( p + q \sqrt{1-2\lambda} \bigr) - \tfrac{1}{2} \ln( 1 - 2 \lambda ) - p \lambda \\
    &= \Bigl( q \lambda + \ln \bigl( p + q \sqrt{1 - 2 \lambda} \bigr) \Bigl)
      - \Bigl( \tfrac{1}{2} \ln( 1 - 2 \lambda) + \lambda + 2 \lambda^2 \Bigr) + 2 \lambda^2 \\
    &= g(\lambda) - f(\lambda) + 2 \lambda^2 \le g(\lambda) + 2 \lambda^2 \le 2 \lambda^2.
  \end{split}
  \end{equation}
  This \proves that for all $ \lambda \in ( -\nicefrac{1}{4}, \nicefrac{1}{4} ) $ it holds that
  \begin{equation}
    \E \bigl[ \exp\bigl( \lambda ( X^2 \ind_A - \E[X^2 \ind_A] ) \bigr) \bigr] \le \exp( 2 \lambda^2 ) = \exp( \tfrac{1}{2} \lambda^2 2^2 ).
  \end{equation}
  \Hence that $ X^2 \ind_A $ is \subexp{2}{4} \cfload.
  This \proves \cref{lem:subexponential_chi_squared:item1}.
  Next let $ h \colon (-1,1) \to \R $ satisfy for all $ \lambda \in (-1,1) $ that
  \begin{equation}
  \label{eq:subexponential_h}
    h( \lambda ) = \lambda^2 + \tfrac{1}{2} \ln( 1 - \lambda^2 ).
  \end{equation}
  \Nobs that \cref{eq:subexponential_h} \proves that for all $ \lambda \in (-\nicefrac{1}{\sqrt{2}}, \nicefrac{1}{\sqrt{2}}) $
  it holds that $ 2 - \frac{1}{1-\lambda^2} > 0 $ and
  \begin{equation}
  \label{eq:subexponential_h_prime}
    h'( \lambda )
    = 2 \lambda - \frac{2 \lambda}{2 ( 1 - \lambda^2 )}
    = \lambda \biggl( 2 - \frac{1}{1 - \lambda^2} \biggr).
  \end{equation}
  This \proves for all $ \lambda \in (-\nicefrac{1}{\sqrt{2}},0) $ that $ h'(\lambda) < 0 $.
  \Hence that $ h|_{(-\nicefrac{1}{\sqrt{2}},0)} $ is strictly decreasing.
  \Moreover \cref{eq:subexponential_h_prime} \proves for all $ \lambda \in (0,\nicefrac{1}{\sqrt{2}}) $ that $ h'(\lambda) > 0 $.
  \Hence that $ h|_{(0,\nicefrac{1}{\sqrt{2}})} $ is strictly increasing.
  Combining this, the fact that $ h|_{(-\nicefrac{1}{\sqrt{2}},0)} $ is strictly decreasing, and
  the fact that $ h(0) = 0 $
  \proves that for all $ \lambda \in (-\nicefrac{1}{\sqrt{2}},\nicefrac{1}{\sqrt{2}}) $ it holds that
  \begin{equation}
  \label{eq:subexponential_h_nonnegative}
    h( \lambda ) \ge h(0) = 0.
  \end{equation}
  \Moreover the assumption that $ X $ is a standard normal random variable \proves that for all $ \lambda \in \R $ it holds that
  \begin{equation}
  \label{eq:subexponential_mgf_normal}
  \begin{split}
    \E \bigl[ \exp( \lambda X ) \bigr]
    &= \int_{-\infty}^{\infty} \frac{1}{\sqrt{2\pi}} \exp \Bigl( -\frac{x^2}{2} \Bigr) \exp(\lambda x) \diff x
    = \int_{-\infty}^{\infty} \frac{1}{\sqrt{2\pi}} \exp \biggl( -\frac{x^2 - 2 \lambda x }{2} \biggr) \diff x \\
    &= \exp \bigl( \tfrac{1}{2} \lambda^2 \bigr)
      \int_{-\infty}^{\infty} \frac{1}{\sqrt{2\pi}} \exp \Bigl( -\frac{(x-\lambda)^2}{2} \Bigr) \diff x
    = \exp \bigl( \tfrac{1}{2} \lambda^2 \bigr).
  \end{split}
  \end{equation}
  Combining this, \cref{eq:chi_squared_subexponential},
  the assumption that $ X $, $ Y $, and $ \ind_A $ are independent, and
  the law of total expectation
  \proves that for all $ \lambda \in (-1,1) $ it holds that
  \begin{equation}
  \label{eq:subexponential_exp_XYI}
    \E \bigl[ \exp( \lambda X Y \ind_A ) \bigr]
    = \E \bigl[ \E[ \exp( \lambda X Y \ind_A ) | X \ind_A ] \bigr]
    = \E \bigl[ \exp \bigl( \tfrac{1}{2} \lambda^2 X^2 \ind_A \bigr) \bigr]
    \le \E \bigl[ \exp \bigl( \tfrac{1}{2} \lambda^2 X^2 \bigr) \bigr]
    = \frac{1}{\sqrt{1-\lambda^2}}.
  \end{equation}
  \Moreover the assumption that $ X $, $ Y $, and $ \ind_A $ are independent and
  the assumption that $ X $ is a standard normal random variable
  \prove that
  \begin{equation}
  \label{eq:subexponential_expectation_XYI}
    \E[ X Y \ind_A ] = \E[ X ] \E[ Y \ind_A ] = 0.
  \end{equation}
  This, \cref{eq:subexponential_h_nonnegative}, and \cref{eq:subexponential_exp_XYI}
  \prove that for all $ \lambda \in (-\nicefrac{1}{\sqrt{2}},\nicefrac{1}{\sqrt{2}}) $ it holds that
  \begin{equation}
  \begin{split}
    \ln \bigl( \E \bigl[ \exp\bigl( \lambda ( X Y \ind_A - \E[X Y \ind_A] ) \bigr) \bigr] \bigr)
    &= \ln \bigl( \E \bigl[ \exp( \lambda X Y \ind_A ) \bigr] \bigl)
    \le \ln \Bigl( \tfrac{1}{\sqrt{1-\lambda^2}} \Bigr) \\
    &= - \tfrac{1}{2} \ln \bigl( 1 - \lambda^2 \bigr)
    = \lambda^2 - h(\lambda)
    \le \lambda^2.
  \end{split}
  \end{equation}
  This \proves that for all $ \lambda \in ( -\nicefrac{1}{\sqrt{2}}, \nicefrac{1}{\sqrt{2}} ) $ it holds that
  \begin{equation}
    \E \bigl[ \exp\bigl( \lambda ( X Y \ind_A - \E[X Y \ind_A] ) \bigr) \bigr]
    \le \exp \bigl( \lambda^2 \bigr)
    = \exp \bigl( \tfrac{1}{2} \lambda^2 (\sqrt{2})^2 \bigr).
  \end{equation}
  \Hence that $ X Y \ind_A $ is \subexp{\sqrt{2}}{\sqrt{2}} \cfload.
  This \proves \cref{lem:subexponential_chi_squared:item2}.
\end{cproof}

\cfclear
\begin{lemma}
  \label{lem:subexponential_sum_chi_squared}
  Let $ ( \Omega, \mc F, \P ) $ be a probability space,
  let $ n \in \N $, $ s = ( s_1, \ldots, s_n ) \in \R^n $, $ A_1, A_2, \ldots, A_n \in \mc F $,
  for every $ i \in \{ 1, 2, \ldots, n \} $ let $ X_i \colon \Omega \to \R $ and $ Y_i \colon \Omega \to \R $ be normal random variables,
  assume for all $ i \in \{ 1, 2, \ldots, n \} $ that $ \Var[ X_i ] = \Var[ Y_i ] = \abs{ s_i }^2 $,
  and assume that $ X_1, X_2, \ldots, X_n, Y_1, Y_2, \ldots, Y_n, \ind_{A_1}, \ind_{A_2}, \ldots, \ind_{A_n} $ are independent.
  Then
  \begin{enumerate}[(i)]
    \item \label{lem:subexponential_sum_chi_squared:item1} it holds that $ \sum_{i=1}^n \abs{ X_i - \E[ X_i ] }^2 \ind_{A_i} $
      is \subexp{ 2 ( \sum_{i=1}^n \abs{ s_i }^4 )^{\nicefrac{1}{2}} }{ 4 \max_{ i \in \{ 1, 2, \ldots, n \} } \abs{ s_i }^2 } and
    \item \label{lem:subexponential_sum_chi_squared:item2} it holds that $ \sum_{i=1}^n ( X_i - \E[ X_i ] )( Y_i - \E[ Y_i ] ) \ind_{A_i} $
      is \subexp{ \sqrt{2} ( \sum_{i=1}^n \abs{ s_i }^4 )^{\nicefrac{1}{2}} }
        { \sqrt{2} \max_{ i \in \{ 1, 2, \ldots, n \} } \abs{ s_i }^2 }
  \end{enumerate}
  \cfout.
\end{lemma}

\begin{cproof}{lem:subexponential_sum_chi_squared}
  \Nobs that the assumption that for all $ i \in \{ 1, 2, \ldots, n \} $ it holds that $ X_i $ and $ Y_i $ are normal random
  variable with $ \Var[ X_i ] = \Var[ Y_i ] = \abs{ s_i }^2 $ \proves that for all $ i \in \{ 1, 2, \ldots, n \} $ it holds that
  \begin{equation}
    \tfrac{1}{s_i} ( X_i - \E[ X_i ] )
    \qqandqq
    \tfrac{1}{s_i} ( Y_i - \E[ Y_i ] )
  \end{equation}
  are standard normal random variables.
  This, the fact that for all $ i \in \{ 1, 2, \ldots, n \} $ it holds that $ X_i $, $ Y_i $, and $ \ind_{A_i} $ are independent,
  and \cref{lem:subexponential_chi_squared} (applied for every $ i \in \{ 1, 2, \ldots, n \} $ with
  $ ( \Omega, \mc F, \P ) \is ( \Omega, \mc F, \P ) $, $ X \is \tfrac{1}{s_i} ( X_i - \E[ X_i ] ) $,
  $ Y \is \tfrac{1}{s_i} ( Y_i - \E[ Y_i ] ) $, $ A \is A_i $
  in the notation of \cref{lem:subexponential_chi_squared})
  \prove that for all $ i \in \{ 1, 2, \ldots, n \} $ it holds that
  \begin{equation}
    \tfrac{1}{\abs{s_i}^2} \abs{ X_i - \E[ X_i ] }^2 \ind_{A_i}
  \end{equation}
  is \subexp{2}{4} and
  \begin{equation}
    \tfrac{1}{\abs{s_i}^2} ( X_i - \E[ X_i ] ) ( Y_i - \E[ Y_i ] ) \ind_{A_i}
  \end{equation}
  is \subexp{\sqrt{2}}{\sqrt{2}}
  \cfload.
  \cref{lem:subexponential_scalar} (applied for every $ i \in \{ 1, 2, \ldots, n \} $ with
  $ ( \Omega, \mc F, \P ) \is ( \Omega, \mc F, \P) $, $ \nu \is 2 $, $ b \is 4 $, $ a \is \abs{ s_i }^2 $,
  $ X \is \tfrac{1}{\abs{s_i}^2} \abs{ X_i - \E[ X_i ] }^2 \ind_{A_i} $
  in the notation of \cref{lem:subexponential_scalar})
  and
  \cref{lem:subexponential_scalar} (applied for every $ i \in \{ 1, 2, \ldots, n \} $ with
  $ ( \Omega, \mc F, \P ) \is ( \Omega, \mc F, \P) $, $ \nu \is \sqrt{2} $, $ b \is \sqrt{2} $, $ a \is \abs{ s_i }^2 $,
  $ X \is \tfrac{1}{\abs{s_i}^2} ( X_i - \E[ X_i ] ) ( Y_i - \E[ Y_i ] ) \ind_{A_i} $
  in the notation of \cref{lem:subexponential_scalar})
  thus \prove that for all $ i \in \{ 1, 2, \ldots, n \} $ it holds that
  \begin{equation}
    \abs{ X_i - \E[ X_i ] }^2 \ind_{A_i}
  \end{equation}
  is \subexp{ 2 \abs{ s_i }^2 }{ 4 \abs{ s_i }^2 } and
  \begin{equation}
    ( X_i - \E[ X_i ] )( Y_i - \E[ Y_i ] ) \ind_{A_i}
  \end{equation}
  is \subexp{ \sqrt{2} \abs{ s_i }^2 }{ \sqrt{2} \abs{ s_i }^2 }.
  Combining this with
  the assumption that $ X_1, X_2, \ldots, X_n, Y_1, Y_2, \ldots, \allowbreak Y_n, \allowbreak \ind_{A_1}, \ind_{A_2}, \ldots, \ind_{A_n} $ are independent,
  \cref{lem:subexponential_weighted_sum} (applied with
  $ ( \Omega, \mc F, \P ) \is ( \Omega, \mc F, \P ) $, $ n \is n $,
  $ \nu \is ( 2 \abs{s_i}^2 )_{ i \in \{ 1, 2, \ldots, n \} } $,
  $ b \is ( 4 \abs{s_i}^2 )_{ i \in \{ 1, 2, \ldots, n \} } $,
  $ a \is ( 1, \ldots, 1 ) $,
  $ ( X_i )_{ i \in \{ 1, 2, \ldots, n \} } \is ( \abs{ X_i - \E[ X_i ] }^2 \ind_{A_i} )_{ i \in \{ 1, 2, \ldots, n \} } $
  in the notation of \cref{lem:subexponential_weighted_sum}),
  and
  \cref{lem:subexponential_weighted_sum} (applied with
  $ ( \Omega, \mc F, \P ) \is ( \Omega, \mc F, \P ) $, $ n \is n $,
  $ \nu \is ( \sqrt{2} \abs{s_i}^2 )_{ i \in \{ 1, 2, \ldots, n \} } $,
  $ b \is ( \sqrt{2} \abs{s_i}^2 )_{ i \in \{ 1, 2, \ldots, n \} } $,
  $ a \is ( 1, \ldots, 1 ) $,
  $ ( X_i )_{ i \in \{ 1, 2, \ldots, n \} } \is ( ( X_i - \E[ X_i ] ) ( Y_i - \E[ Y_i ] ) \ind_{A_i} )_{ i \in \{ 1, 2, \ldots, n \} } $
  in the notation of \cref{lem:subexponential_weighted_sum})
  \proves that
  \begin{equation}
    \smallsum_{i=1}^n \abs{ X_i - \E[ X_i ] }^2 \ind_{A_i}
  \end{equation}
  is \subexp{2 ( \smallsum_{i=1}^n \abs{s_i}^4 )^{\nicefrac{1}{2}}}{ 4 \max_{i \in \{ 1, 2, \ldots, n \}} \abs{s_i}^2 } and
  \begin{equation}
    \smallsum_{i=1}^n ( X_i - \E[ X_i ] ) ( Y_i - \E[ Y_i ] ) \ind_{A_i}
  \end{equation}
  is \subexp{ \sqrt{2} ( \smallsum_{i=1}^n \abs{s_i}^4 )^{\nicefrac{1}{2}}}{ \sqrt{2} \max_{i \in \{ 1, 2, \ldots, n \}} \abs{s_i}^2 }.
  This \proves \cref{lem:subexponential_sum_chi_squared:item1} and \cref{lem:subexponential_sum_chi_squared:item2}.
\end{cproof}

\cfclear
\begin{lemma}
  \label{lem:subexponential_tail_bound}
  Let $ ( \Omega, \mc F, \P ) $ be a probability space,
  let $ \nu, b \in (0, \infty) $, and
  let $ X $ be \subexp{\nu}{b} \cfload.
  Then it holds for all $ \varepsilon \in (0,\infty) $ that
  $ \P(\abs{ X - \E[X] } \ge \varepsilon) \le 2 \exp( -\frac{1}{2} \min \{ \nicefrac{\varepsilon^2}{\nu^2}, \nicefrac{\varepsilon}{b} \})$.
\end{lemma}

\begin{cproof}{lem:subexponential_tail_bound}
  Throughout this proof let $ f_{\varepsilon} \colon (0,\infty) \to \R $, $ \varepsilon \in (0,\infty) $,
  satisfy for all $ \varepsilon, \lambda \in (0,\infty) $ that
  $ f_{\varepsilon} (\lambda) = \frac{1}{2} \lambda^2 \nu^2 - \lambda \varepsilon $.
  \Nobs that the Markov inequality \proves for all $ \varepsilon \in (0,\infty) $, $ \lambda \in (0, \frac{1}{b}) $ that
  \begin{equation}
  \begin{split}
    \P( \abs{ X - \E[X] } \ge \varepsilon )
    &= \P( X - \E[X] \ge \varepsilon ) + \P ( - ( X - \E[X] ) \ge \varepsilon ) \\
    &= \P \bigl( \exp( \lambda ( X - \E[X] ) ) \ge \exp( \lambda \varepsilon ) \bigr)
      + \P \bigl( \exp( -\lambda ( X - \E[X] ) ) \ge \exp( \lambda \varepsilon ) \bigr) \\
    &\le \frac{ \E \bigl[ \exp( \lambda ( X - \E[X] ) ) \bigr] }{ \exp( \lambda \varepsilon ) }
      + \frac{ \E \bigl[ \exp( -\lambda ( X - \E[X] ) ) \bigr] }{ \exp( \lambda \varepsilon ) } \\
    &\le \exp \biggl( \frac{\lambda^2 \nu^2}{2} - \lambda \varepsilon \biggr)
      + \exp \biggl( \frac{(-\lambda)^2 \nu^2}{2} - \lambda \varepsilon \biggr) \\
    &= 2 \exp \biggl( \frac{\lambda^2 \nu^2}{2} - \lambda \varepsilon \biggr)
    = 2 \exp \bigl( f_{\varepsilon} (\lambda) \bigr).
  \end{split}
  \end{equation}
  This and the fact that $ \R \ni x \mapsto \exp(x) \in \R $ is strictly increasing \prove that
  for all $ \varepsilon \in (0,\infty) $ it holds that
  \begin{equation}
  \label{eq:subexponential_tail_bound_1}
    \P( \abs{ X - \E[X] } \ge \varepsilon )
    \le \inf\nolimits_{\lambda \in (0, \frac{1}{b})} 2 \exp \bigl( f_{\varepsilon} (\lambda) \bigr)
    = 2 \exp \bigl( \inf\nolimits_{ \lambda \in (0,\frac{1}{b}) } f_{\varepsilon} (\lambda) \bigr).
  \end{equation}
  \Moreover the fact that for all $ \varepsilon, \lambda \in (0,\infty) $ it holds that
  $ f_{\varepsilon}'(\lambda) = \lambda \nu^2 - \varepsilon $ \proves that
  for all $ \varepsilon \in (0,\infty) $, $ \lambda \in ( 0, \nicefrac{\varepsilon}{\nu^2} ) $ it holds that
  $ f_{\varepsilon}'( \lambda ) < 0 $.
  \Hence for all $ \varepsilon \in (0,\infty) $ that
  $ f_{\varepsilon}|_{(0,\nicefrac{\varepsilon}{\nu^2})} $ is strictly decreasing.
  This and the fact that for all $ \varepsilon \in (\nicefrac{\nu^2}{b},\infty) $
  it holds that $ \nicefrac{1}{b} < \nicefrac{\varepsilon}{\nu^2} $
  \prove that for all $ \varepsilon \in (\nicefrac{\nu^2}{b},\infty) $ it holds that
  \begin{equation}
  \label{eq:subexponential_tail_bound_2}
    \inf\nolimits_{ \lambda \in (0, \frac{1}{b}) } f_{\varepsilon} (\lambda)
    = f_{\varepsilon} \bigl(\tfrac{1}{b}\bigr)
    = \tfrac{\nu^2}{2b^2} - \tfrac{\varepsilon}{b}
    \le \tfrac{\varepsilon}{2b} - \tfrac{\varepsilon}{b}
    = - \tfrac{\varepsilon}{2b}.
  \end{equation}
  \Moreover the fact that for all $ \varepsilon, \lambda \in (0,\infty) $ it holds that
  $ f_{\varepsilon}' (\lambda) = \lambda \nu^2 - \varepsilon $
  \proves that for all $ \varepsilon \in (0,\infty) $, $ \lambda \in ( \nicefrac{\varepsilon}{\nu^2}, \infty) $ it holds that
  $ f_{\varepsilon}' (\lambda) > 0 $.
  \Hence for all $ \varepsilon \in (0,\infty) $ that
  $ f_{\varepsilon}|_{(\nicefrac{\varepsilon}{\nu^2},\infty)} $ is strictly increasing.
  This, the fact that for all $ \varepsilon \in (0,\infty) $ it holds that
  $ f_{\varepsilon}|_{(0,\nicefrac{\varepsilon}{\nu^2})} $ is strictly decreasing, and
  the fact that for all $ \varepsilon \in (0, \nicefrac{\nu^2}{b} ] $ it holds that
  $ \nicefrac{1}{b} \ge \nicefrac{\varepsilon}{\nu^2} $\prove that
  for all $ \varepsilon \in (0, \nicefrac{\nu^2}{b} ] $ it holds that
  \begin{equation}
  \label{eq:subexponential_tail_bound_3}
  \inf\nolimits_{ \lambda \in (0, \frac{1}{b}) } f_{\varepsilon} (\lambda)
    = f_{\varepsilon} \bigl( \tfrac{\varepsilon}{\nu^2} \bigr)
    = \tfrac{\varepsilon^2}{2 \nu^2} - \tfrac{\varepsilon^2}{\nu^2}
    = - \tfrac{\varepsilon^2}{2 \nu^2}.
  \end{equation}
  Combining this with \cref{eq:subexponential_tail_bound_1} and \cref{eq:subexponential_tail_bound_2}
  \proves that for all $ \varepsilon \in (0, \infty) $ it holds that
    \begin{equation}
    \P ( \abs{ X - \E[X] } \ge \varepsilon )
    \le \max \bigl\{ 2 \exp \bigl( - \tfrac{\varepsilon^2}{2 \nu^2} \bigr), 2 \exp \bigl( - \tfrac{\varepsilon}{2 b} \bigr) \bigr\}
    = 2 \exp \bigl( - \tfrac{1}{2} \min \bigl\{ \tfrac{\varepsilon^2}{\nu^2}, \tfrac{\varepsilon}{b} \bigr\} \bigr).
  \end{equation}
\end{cproof}

\subsection{Concentration type inequalities for stochastic Gram matrices at initialization}
\label{subsec:concentration_type_inequalities_for_stochastic_gram_matrices_at_initialization}

\cfclear
\cfconsiderloaded{def:spectral_norm}
\begin{definition}[Spectral norm]
  \label{def:spectral_norm}
  For every $ m, n \in \N $, $ A \in \R^{ m \times n } $
  we denote by $ \specnorm{ A } \in \R $ the real number which satisfies that
  $ \specnorm{ A } = \sup_{ v \in \R^n \backslash \{ 0 \} } \nicefrac{ \eucnorm{ A v } }{ \eucnorm v } $
  \cfload.
\end{definition}

\cfclear
\begin{lemma}
  \label{lem:spectral_norm_inequality}
  Let $ m, n \in \N $, $ A = ( A_{i,j} )_{ (i,j) \in \{ 1, 2, \ldots, m \} \times \{ 1, 2, \ldots, n \} } \in \R^{ m \times n } $.
  Then
  $ \specnorm{ A } \le ( \sum_{i=1}^m \sum_{j=1}^n \abs{ A_{i,j} }^2 )^{\nicefrac{1}{2}} \allowbreak
    \le \sum_{i=1}^m \sum_{j=1}^n \abs{ A_{i,j} } $
  \cfout.
\end{lemma}

\begin{cproof}{lem:spectral_norm_inequality}
  Throughout this proof let
  $ e_1 = ( 1, 0, \ldots, 0 ) $, $ e_2 = ( 0, 1, 0, \ldots, 0 ) $, \ldots, $ e_n = ( 0, \ldots, 0, \allowbreak 1 ) \in \R^n $.
  \Nobs that the Cauchy-Schwarz inequality \proves that for all $ v = ( v_1, \ldots, v_n ) \in \R^n $ it holds that
  \begin{equation}
  \begin{split}
    \eucnorm{ A v }
    &= \bbbeucnorm{ A \bigg( \smallsum\limits_{j=1}^n v_j e_j \bigg) }
    = \bbbeucnorm{ \smallsum\limits_{j=1}^n v_j A e_j } 
    \le \smallsum\limits_{j=1}^n \abs{ v_j } \eucnorm{ A e_j } \\
    &\le \bigg( \smallsum\limits_{j=1}^n \abs{ v_j }^2 \bigg)^{ \nicefrac{1}{2} }
      \bigg( \smallsum\limits_{j=1}^n \eucnorm{ A e_j }^2 \bigg)^{ \nicefrac{1}{2} }
    = \eucnorm{ v } \bigg( \smallsum\limits_{j=1}^n \smallsum\limits_{i=1}^m \abs{ A_{i,j} }^2 \bigg)^{ \nicefrac{1}{2} }
  \end{split}
  \end{equation}
  \cfload.
  \Hence for all $ v \in \R^n \backslash \{ 0 \} $ that
  $ \nicefrac{ \eucnorm{ A v }}{ \eucnorm v } \le ( \sum_{i=1}^m \sum_{j=1}^n \abs{ A_{i,j} }^2 )^{\nicefrac{1}{2}} $.
  This \proves that
  \begin{equation}
    \specnorm{ A }
    = \sup\nolimits_{ v \in \R^n \backslash \{ 0 \} } \nicefrac{ \eucnorm{ A v } }{ \eucnorm{ v } }
    \le \bigl( \smallsum\nolimits_{i=1}^m \smallsum\nolimits_{j=1}^n \abs{ A_{i,j} }^2 \bigr)^{\nicefrac{1}{2}}
  \end{equation}
  \cfload.
  \Moreover the fact that for all $ a, b \in [0, \infty) $ it holds that 
  $ (a + b)^{\nicefrac{1}{2}} \le a^{\nicefrac{1}{2}} + b^{\nicefrac{1}{2}} $
  inductively \proves that
  $ ( \sum_{i=1}^m \sum_{j=1}^n \abs{ A_{i,j} }^2 )^{\nicefrac{1}{2}}
    \le \sum_{i=1}^m \sum_{j=1}^n ( \abs{ A_{i,j} }^2 )^{\nicefrac{1}{2}}
    = \sum_{i=1}^m \sum_{j=1}^n \abs{ A_{i,j} } $.
\end{cproof}

\cfclear
\begin{lemma}
  \label{lem:probability_A_gramG}
  Assume \cref{setting:gradient_descent},
  assume $ \lambdazero \in (0,\infty) $,
  let $ \poe \in (0,1) $, and
  assume
  \begin{equation}
    \width \ge \tfrac{32 (1+\cmax^2) \outdim \ndata}{\lambdazero}
    \ln \bigl( \tfrac{2 \outdim^2 \ndata^2}{\poe} \bigr)
    \max \bigl\{ \tfrac{4 (1+\cmax^2) \eucnorm{\slope}^4 \ndata}{\lambdazero}, \cderiv \bigr\}.
  \end{equation}
  Then $ \P \bigl( \bigcap_{i=1}^{\outdim \ndata} \bigcap_{j=1}^{\outdim \ndata} \{ \abs{ \gramG_{i,j} (0) - \gramGinf_{i,j} }
    \le \tfrac{\Cvar \width \lambdazero}{4 \outdim \ndata} \} \bigr) \ge 1 - \poe $.
\end{lemma}

\begin{cproof}{lem:probability_A_gramG}
  Throughout this proof for every $ i, j \in \{ 1, 2, \ldots, \ndata \} $, $ p, q \in \{ 1, 2, \ldots, \outdim \} $
  let $ X_{i,j}^{p,q} \colon \Omega \to \R $ satisfy
  \begin{equation}
    X_{i,j}^{p,q} = \smallsum\limits_{k=1}^{\width} \mf W_{p,k} \mf W_{q,k}
      \deriv( \scalprod{W_k(0)}{\indata_i} + B_k(0) ) \deriv( \scalprod{W_k(0)}{\indata_j} + B_k(0) ).
  \end{equation}
  \Nobs that \cref{eq:setting_gramG}, the assumption that $ \eucnorm{ B(0) } = 0 $, and
  the fact that for all $ k \in \{ 1, 2, \ldots, \width \} $, $ i \in \{ 1, 2, \ldots, \ndata \}$,
  $ s \in \{ 1, 2, \ldots, \nbp+1 \} $ it holds that
  $ \P( \scalprod{W_k(0)}{\indata_i} = \bp_s ) = 0 $
  \prove that for all
  $ i, j \in \{ 1, 2, \ldots, \ndata \} $, $ p, q \in \{ 1, 2, \ldots, \outdim \} $ it holds $ \P $-a.s.\ that
  \begin{equation}\label{eq:lem:probability_A_gramG:ij}
  \begin{split}
    X_{i,j}^{p,q}
    &= \smallsum\limits_{k=1}^{\width} \mf W_{p,k} \mf W_{q,k}
      \deriv\bigl( \scalprod{W_k(0)}{\indata_i} + B_k(0) \bigr) \deriv\bigl( \scalprod{W_k(0)}{\indata_j} + B_k(0) \bigr) \\
    &= \smallsum\limits_{k=1}^{\width} \mf W_{p,k} \mf W_{q,k}
      \deriv\bigl( \scalprod{W_k(0)}{\indata_i} \bigr) \deriv\bigl( \scalprod{W_k(0)}{\indata_j} \bigr) \\
    &= \smallsum\limits_{k=1}^{\width} \mf W_{p,k} \mf W_{q,k}
      \Bigl( \smallsum_{s=1}^{\nbp+1} \slope_s \ind_{ (\bp_{s-1}, \bp_s ) } ( \scalprod{W_k(0)}{\indata_i} ) \Bigr)
      \Bigl( \smallsum_{t=1}^{\nbp+1} \slope_t \ind_{ (\bp_{t-1}, \bp_t ) } ( \scalprod{W_k(0)}{\indata_j} ) \Bigr) \\
    &= \smallsum\limits_{k=1}^{\width} \mf W_{p,k} \mf W_{q,k}
      \smallsum\limits_{s=1}^{\nbp+1} \smallsum\limits_{t=1}^{\nbp+1} \slope_s  \slope_t
      \ind_{ (\bp_{s-1}, \bp_s ) } ( \scalprod{W_k(0)}{\indata_i} ) \ind_{ (\bp_{t-1}, \bp_t ) } ( \scalprod{W_k(0)}{\indata_j} ) \\
    &= \smallsum\limits_{s=1}^{\nbp+1} \slope_s \smallsum\limits_{t=1}^{\nbp+1} \slope_t
      \smallsum\limits_{k=1}^{\width} \mf W_{p,k} \mf W_{q,k}
      \ind_{ \{ \scalprod{W_k(0)}{\indata_i} \in (\bp_{s-1}, \bp_s ), \, \scalprod{W_k(0)}{\indata_j} \in (\bp_{t-1}, \bp_t ) \} }
  \end{split}
  \end{equation}
  \cfload.
  \Moreover \cref{lem:subexponential_sum_chi_squared} (applied for every
  $ i, j \in \{ 1, 2, \ldots, \ndata \} $, $ p, q \in \{ 1, 2, \ldots, \outdim \} $,
  $ s, t \in \{ 1, 2, \ldots \nbp+1 \}$ with
  $ ( \Omega, \mc F, \P ) \is ( \Omega, \mc F, \P ) $,
  $ n \is \width $,
  $ s \is ( \sqrt{ \Cvar }, \ldots, \sqrt{ \Cvar } ) $,
  $ ( X_i )_{ i \in \{ 1, 2, \ldots, n \} } \is ( \mf W_{p,k} )_{ k \in \{ 1, 2, \ldots, \width \} } $,
  $ ( Y_i )_{ i \in \{ 1, 2, \ldots, n \} } \is ( \mf W_{q,k} )_{ k \in \{ 1, 2, \ldots, \width \} } $,
  $ ( A_i )_{ i \in \{1, 2, \ldots, n \} } $ $ \is
    ( \{ \scalprod{W_k(0)}{\indata_i} \in (\bp_{s-1}, \bp_s ), \, \scalprod{W_k(0)}{\indata_j} \in (\bp_{t-1}, \bp_t ) \} \allowbreak )_{
      k \in \{ 1, 2, \ldots, \width \} } $
  in the notation of \cref{lem:subexponential_sum_chi_squared}),
  the fact that for all $ k \in \{ 1, 2, \ldots, \width \} $, $ p \in \{ 1, 2, \ldots, \outdim \} $
  it holds that
  \begin{equation}
    \mf W_{p,k} \text{ is a centered normal random variable with } \Var[ \mf W_{p,k} ] = \Cvar,
  \end{equation}
  and the assumption that
  $ \sqrt{\nicefrac{1}{\cvar}} W_1(0), \allowbreak
    \sqrt{\nicefrac{1}{\cvar}} W_2(0), \allowbreak \ldots, \allowbreak
    \sqrt{\nicefrac{1}{\cvar}} W_{\width}(0), \allowbreak
    \sqrt{\nicefrac{1}{\Cvar}} \mf W_1, \allowbreak
    \sqrt{\nicefrac{1}{\Cvar}} \mf W_2, \allowbreak \ldots, \allowbreak
    \sqrt{\nicefrac{1}{\Cvar}} \mf W_{\outdim} $
  are independent
  \prove that for all $ i, j \in \{ 1, 2, \ldots, \ndata \} $, $ p, q \in \{ 1, 2 , \ldots, \outdim \} $,
  $ s, t \in \{ 1, 2, \ldots, \nbp+1 \} $ with $ p \neq q $ it holds that
  \begin{equation}
    \smallsum\limits_{k=1}^{\width} \abs{ \mf W_{p,k} }^2 
      \ind_{ \{ \scalprod{W_k(0)}{\indata_i} \in (\bp_{s-1}, \bp_s ), \, \scalprod{W_k(0)}{\indata_j} \in (\bp_{t-1}, \bp_t ) \} }
  \end{equation}
  is \subexp{ 2 \Cvar \sqrt{\width} }{ 4 \Cvar } and
  \begin{equation}
    \smallsum\limits_{k=1}^{\width} \mf W_{p,k} \mf W_{q,k} 
      \ind_{ \{ \scalprod{W_k(0)}{\indata_i} \in (\bp_{s-1}, \bp_s ), \, \scalprod{W_k(0)}{\indata_j} \in (\bp_{t-1}, \bp_t ) \} }
  \end{equation}
  is \subexp{ \sqrt{2} \Cvar \sqrt{\width} }{ \sqrt{2} \Cvar }
  \cfload.
  The fact that $ \sqrt{2} \le 2 \le 4 $ \hence \proves that
  for all $ i, j \in \{ 1, 2, \ldots, \ndata \} $, $ p, q \in \{ 1, 2 , \ldots, \outdim \} $,
  $ s, t \in \{ 1, 2, \ldots, \nbp+1 \} $ it holds that
  \begin{equation}
    \smallsum\limits_{k=1}^{\width} \mf W_{p,k} \mf W_{q,k} 
      \ind_{ \{ \scalprod{W_k(0)}{\indata_i} \in (\bp_{s-1}, \bp_s ), \, \scalprod{W_k(0)}{\indata_j} \in (\bp_{t-1}, \bp_t ) \} }
  \end{equation}
  is \subexp{ 2 \Cvar \sqrt{\width} }{ 4 \Cvar }.
  \cref{lem:subexponential_weighted_sum} (applied with
  $ ( \Omega, \mc F, \P ) \is ( \Omega, \mc F, \P ) $,
  $ n \is \nbp+1 $,
  $ \nu \is ( 2 \Cvar \sqrt{\width}, \ldots, \allowbreak 2 \Cvar \sqrt{\width} ) $,
  $ b \is ( 4 \Cvar, \ldots, 4 \Cvar ) $,
  $ a \is \slope $
  in the notation of \cref{lem:subexponential_weighted_sum})
  and the fact that $ \max_{ s \in \{ 1, 2, \ldots, \nbp+1 \} } \abs{ \slope_s } \le \cderiv $
  \hence \prove that for all
  $ i, j \in \{ 1, 2, \ldots, \ndata \} $, $ p, q \in \{ 1, 2, \ldots, \outdim \} $, $ s \in \{ 1, 2, \ldots, \nbp+1 \} $ it holds that
  \begin{equation}
    \smallsum\limits_{t=1}^{\nbp+1} \slope_t
      \smallsum\limits_{k=1}^{\width} \mf W_{p,k} \mf W_{q,k}
      \ind_{ \{ \scalprod{W_k(0)}{\indata_i} \in (\bp_{s-1}, \bp_s ), \, \scalprod{W_k(0)}{\indata_j} \in (\bp_{t-1}, \bp_t ) \} }
  \end{equation}
  is \subexp{ 2 \Cvar \sqrt{\width} \eucnorm{\slope} }{ 4 \Cvar \cderiv }.
  Combining this,
   \cref{lem:subexponential_weighted_sum} (applied with
  $ ( \Omega, \mc F, \P ) \is ( \Omega, \mc F, \P ) $,
  $ n \is \nbp+1 $,
  $ \nu \is ( 2 \Cvar \sqrt{\width} \eucnorm{\slope}, \ldots, 2 \Cvar \sqrt{\width} \eucnorm{\slope} ) $,
  $ b \is ( 4 \Cvar \cderiv, \ldots, 4 \Cvar \cderiv ) $,
  $ a \is \slope $
  in the notation of \cref{lem:subexponential_weighted_sum}),
  \cref{eq:lem:probability_A_gramG:ij}, and
  the fact that $ \max_{ s \in \{ 1, 2, \ldots, \nbp+1 \} } \abs{ \slope_s } \le \cderiv $
  \proves that for all
  $ i, j \in \{ 1, 2, \ldots, \ndata \} $, $ p, q \in \{ 1, 2, \ldots, \outdim \} $ it holds that
  \begin{equation}\label{eq:lem:probability_A_gramG_subexp}
    X_{i,j}^{p,q}
  \end{equation}
  is \subexp{ 2 \Cvar \sqrt{\width} \eucnorm{\slope}^2 }{ 4 \Cvar \cderiv^2 }.
  \Moreover \cref{lem:graminf} \proves for all $ i, j \in \{ 1, 2, \ldots, \ndata \} $, $ p, q \in \{ 1, 2, \ldots, \outdim \} $ that
  \begin{equation}
    \gramGinf_{ (i-1) \outdim + p, (j-1) \outdim + q }
    = \E \bigl[ \gramG_{ (i-1) \outdim + p, (j-1) \outdim + q } (0) \bigr]
    = ( 1 + \scalprod{\indata_i}{\indata_j} ) \E \bigl[ X_{i,j}^{p,q} \bigr].
  \end{equation}
  \cref{lem:subexponential_tail_bound} (applied for every $ i, j \in \{ 1, 2, \ldots, \ndata \} $, $ p, q \in \{ 1, 2, \ldots, \outdim \} $ with
  $ ( \Omega, \mc F, \P ) \is ( \Omega, \mc F, \P ) $,
  $ \nu \is 2 \Cvar \sqrt{\width} \eucnorm{\slope}^2 $,
  $ b \is 4 \Cvar \cderiv^2 $,
  $ X \is X_{i,j}^{p,q} $,
  $ \varepsilon \is \tfrac{\width \lambdazero}{ 4 (1+\cmax^2) \ndata} $
  in the notation of \cref{lem:subexponential_tail_bound}),
  \cref{eq:lem:probability_A_gramG_subexp},
  the assumption that
  $ \width \ge \tfrac{32 (1+\cmax^2) \outdim \ndata}{\lambdazero}
    \ln( \frac{2 \outdim^2 \ndata^2}{\poe} ) \max \bigl\{
    \frac{4 (1+\cmax^2) \eucnorm{\slope}^4 \ndata}{\lambdazero}, \cderiv \bigr\} $, and
  the fact that for all $ i, j \in \{ 1, 2, \ldots, \ndata \} $ it holds that
  $ \abs{ 1 + \scalprod{ \indata_i }{ \indata_j } } \le 1 + \eucnorm{ \indata_i } \eucnorm{ \indata_j } \le 1 + \cmax^2 $
  \hence \prove that for all $ i, j \in \{ 1, 2, \ldots, \ndata \} $, $ p, q \in \{ 1, 2, \ldots, \outdim \} $ it holds that
  \begin{equation}
  \begin{split}
    &\P \Bigl( \abs{ \gramG_{ (i-1) \outdim + p, (j-1) \outdim + q } (0) - \gramGinf_{ (i-1) \outdim + p, (j-1) \outdim + q } }
      \ge \tfrac{\Cvar \width \lambdazero}{4 \outdim \ndata} \Bigr) \\
    &= \P \Bigl( \abs{ 1 + \scalprod{\indata_i}{\indata_j} } \babs{ X_{i,j}^{p,q} - \E \bigl[ X_{i,j}^{p,q} \bigr] }
      \ge \tfrac{\Cvar \width \lambdazero}{4 \outdim \ndata} \Bigr) \\
    &\le \P \Bigl( \babs{ X_{i,j}^{p,q} - \E \bigl[ X_{i,j}^{p,q} \bigr] }
      \ge \tfrac{\Cvar \width \lambdazero}{4 (1+\cmax^2) \outdim \ndata} \Bigr) \\
    &\le 2 \exp \Bigl( - \tfrac{1}{2} \min \Bigl\{ \tfrac{1}{4 \Cvar^2 \width \eucnorm{\slope}^4}
      \bigl( \tfrac{\Cvar \width \lambdazero}{4 (1+\cmax^2) \outdim \ndata} \bigr)^2,
      \tfrac{1}{4 \Cvar \cderiv^2} \tfrac{\Cvar \width \lambdazero}{4 (1+\cmax^2) \outdim \ndata} \Bigr\} \Bigr) \\
    &= 2 \exp \Bigl( - \tfrac{1}{2} \min \Bigl\{ \tfrac{\width \lambdazero^2}{64 (1+\cmax^2)^2 \eucnorm{\slope}^4 \outdim^2 \ndata^2},
      \tfrac{\width \lambdazero}{16 (1+\cmax^2) \cderiv^2 \outdim \ndata} \Bigr\} \Bigr) \\
    &= 2 \exp \Bigl( - \tfrac{\width \lambdazero}{32 (1+\cmax^2) \outdim \ndata}
      \min \Bigl\{ \tfrac{\lambdazero}{4 (1+\cmax^2) \eucnorm{\slope}^4 \outdim \ndata}, \cderiv^{-1} \Bigr\} \Bigr) \\
    &\le 2 \exp \Bigl( - \ln \bigl( \tfrac{2 \outdim^2 \ndata^2}{\poe} \bigr)
      \max \Bigl\{ \tfrac{4 (1+\cmax^2) \eucnorm{\slope}^4 \outdim \ndata}{\lambdazero}, \cderiv \Bigr\}
      \min \Bigl\{ \tfrac{\lambdazero}{4 (1+\cmax^2) \eucnorm{\slope}^4 \outdim \ndata}, \cderiv^{-1} \Bigr\} \Bigr) \\
    &= 2 \exp \Bigl( - \ln \bigl( \tfrac{2 \outdim^2 \ndata^2}{\poe} \bigr) \Bigr)
    = 2 \exp \Bigl( \ln \bigl( \tfrac{\poe}{2 \outdim^2 \ndata^2} \bigr) \Bigr)
    = \tfrac{\poe}{\outdim^2 \ndata^2}.
  \end{split}
  \end{equation}
  \Hence that
  \begin{equation}
  \begin{split}
    \P \biggl( \smallbigcap\limits_{i=1}^{\outdim \ndata} \smallbigcap\limits_{j=1}^{\outdim \ndata}
      \Bigl\{ \abs{ \gramG_{i,j} (0) - \gramGinf_{i,j} } \le \tfrac{\Cvar \width \lambdazero}{4 \outdim \ndata} \Bigr\} \biggl)
    &\ge 1 - \P \biggl( \smallbigcup\limits_{i=1}^{\outdim \ndata} \smallbigcup\limits_{j=1}^{\outdim \ndata}
      \Bigl\{ \abs{ \gramG_{i,j} (0) - \gramGinf_{i,j} } \ge \tfrac{\Cvar \width \lambdazero}{4 \outdim \ndata} \Bigr\} \biggr) \\
    &\ge 1 - \smallsum\limits_{i=1}^{\outdim \ndata} \smallsum\limits_{j=1}^{\outdim \ndata}
      \P \Bigl( \abs{ \gramG_{i,j} (0) - \gramGinf_{i,j} } \ge \tfrac{\Cvar \width \lambdazero}{4 \outdim \ndata} \Bigr) \\
    &\ge 1 - \smallsum\limits_{i=1}^{\outdim \ndata} \smallsum\limits_{j=1}^{\outdim \ndata} \tfrac{\poe}{\outdim^2 \ndata^2}
    = 1 - \poe.
  \end{split}
  \end{equation}
\end{cproof}

\cfclear
\begin{lemma}
  \label{lem:distance_gramGinf}
  Assume \cref{setting:gradient_descent},
  let $ \poe \in (0,1) $, $ A \in \mc F $ satisfy
  $ A = \bigcap_{i=1}^{\outdim \ndata} \bigcap_{j=1}^{\outdim \ndata} \bigl\{
    \abs{ \gramG_{i,j} (0) - \gramGinf_{i,j} } \le \frac{\Cvar \width \lambdazero}{4 \outdim \ndata} \bigr\} $,
  and let $ \omega \in A $.
  Then $ \specnorm{ \gramG (0,\omega) ) - \gramGinf } \le \nicefrac{\Cvar \width \lambdazero}{4} $ \cfout.
\end{lemma}

\begin{cproof}{lem:distance_gramGinf}
  \Nobs that \cref{lem:spectral_norm_inequality} (applied with
  $ m \is \outdim \ndata $, $ n \is \outdim \ndata $, $ A \is \gramG (0,\omega) - \gramGinf $
  in the notation of \cref{lem:spectral_norm_inequality})
  and the assumption that $ \omega \in A $ \prove that
  \begin{equation}
  \begin{split}
    \specnorm{ \gramG (0,\omega) - \gramGinf }^2
    &\le \smallsum\limits_{i=1}^{\outdim \ndata} \smallsum\limits_{j=1}^{\outdim \ndata}
      \abs{ \gramG_{i,j} (0,\omega) - \gramGinf_{i,j} }^2
    \le \smallsum\limits_{i=1}^{\outdim \ndata} \smallsum\limits_{j=1}^{\outdim \ndata}
      \Bigl( \tfrac{\Cvar \width \lambdazero}{4 \outdim \ndata} \Bigr)^2 \\
    &= \smallsum\limits_{i=1}^{\outdim \ndata} \smallsum\limits_{j=1}^{\outdim \ndata}
      \tfrac{\Cvar^2 \width^2 \lambdazero^2}{16 \outdim^2 \ndata^2}
    = \tfrac{\Cvar^2 \width^2 \lambdazero^2}{16}.
  \end{split}
  \end{equation}
  \cfload.
  \Hence that $ \specnorm{ \gramG (0,\omega) - \gramGinf } \le \nicefrac{\Cvar \width \lambdazero}{4} $.
\end{cproof}

\subsection{Analysis of stochastic Gram matrices during training}
\label{subsec:analysis_of_stochastic_gram_matrices_during_training}

\cfclear
\begin{lemma}
  \label{lem:gaussian_anti_concentration_inequality}
  Let $ ( \Omega, \mc F, \P ) $ be a probability space,
  let $ X \colon \Omega \to \R $ be a standard normal random variable, and
  let $ \bp \in \R $, $ \varepsilon \in (0,\infty) $.
  Then
  $ \P ( \abs{ X - \bp } \le \varepsilon ) \le \nicefrac{2 \varepsilon}{\sqrt{2\pi}} $.
\end{lemma}

\begin{cproof}{lem:gaussian_anti_concentration_inequality}
  \Nobs that the fact that $ X $ is a standard normal random variable and
  the fact that for all $ x \in \R $ it holds that $ \exp( \nicefrac{-x^2}{2} ) \le 1 $
 \prove that it holds that
  \begin{equation}
    \P ( \abs{ X - \bp } \le \varepsilon )
    = \P( \bp - \varepsilon \le X \le \bp + \varepsilon )
    = \int_{\bp-\varepsilon}^{\bp+\varepsilon} \frac{1}{\sqrt{2\pi}} \exp \Bigl( - \frac{x^2}{2} \Bigr) \diff x
    \le \int_{\bp-\varepsilon}^{\bp+\varepsilon} \frac{1}{\sqrt{2\pi}} \diff x
    = \frac{2 \varepsilon}{\sqrt{2\pi}}.
  \end{equation}
\end{cproof}

\cfclear
\begin{lemma}
  \label{lem:expectation_absolute_normal}
  Let $ ( \Omega, \mc F, \P ) $ be a probability space and
  let $ X \colon \Omega \to \R $ be a standard normal random variable.
  Then
  $ \E [ \abs{ X } ] = \sqrt{\nicefrac{2}{\pi}} $.
\end{lemma}

\begin{cproof}{lem:expectation_absolute_normal}
  \Nobs that the assumption that $ X $ is a standard normal random variable \proves that
  \begin{equation}
  \begin{split}
    \E[ \abs{ X } ]
    &= \int_{-\infty}^{\infty} \abs{x} \frac{1}{\sqrt{2 \pi}} \exp \Bigl( - \frac{x^2}{2} \Bigr) \diff x
    = \frac{1}{\sqrt{2 \pi}} \int_0^{\infty} x \exp \Bigl( - \frac{x^2}{2} \Bigr) \diff x
      + \frac{1}{\sqrt{2 \pi}} \int_{-\infty}^0 -x \exp \Bigl( - \frac{x^2}{2} \Bigr) \diff x \\
    &= \frac{2}{\sqrt{2 \pi}} \int_0^{\infty} x \exp \Bigl( - \frac{x^2}{2} \Bigr) \diff x
    = \sqrt{\frac{2}{\pi}} \biggl[ - \exp \Bigl( - \frac{x^2}{2} \Bigr) \biggr]_{x=0}^{x=\infty}
    = \sqrt{\frac{2}{\pi}}.
  \end{split}
  \end{equation}
\end{cproof}

\cfclear
\begin{lemma}
  \label{lem:probability_A_ind}
  Assume \cref{setting:gradient_descent}, and
  let $ \poe \in (0,1) $, $ R \in (0,\infty) $ satisfy
  $ R \le \tfrac{ \sqrt{\pi \cvar} \cmin \poe \lambdazero }
    { 16 \sqrt{2} (1+\cmax^2) \cderiv^2 \nbp \outdim^2 \ndata^2 } $.
  Then
  \begin{equation}
    \P \biggl( \smallsum\limits_{i,j=1}^{\ndata} \smallsum\limits_{p,q=1}^{\outdim} \smallsum\limits_{k=1}^{\width}
      \abs{ \mf W_{p,k} \mf W_{q,k} } \bigl(
    \ind_{ \{ \abs{ \scalprod{ W_k(0) } { \indata_i } } \in \, \mc U_R \} }
    + \ind_{ \{ \abs{ \scalprod{ W_k(0) } { \indata_j } } \in \, \mc U_R \} }
    \le \tfrac{\Cvar \width \lambdazero}{8 (1+\cmax^2) \cderiv^2} \biggr)
    \ge 1 - \poe.
  \end{equation}
\end{lemma}

\begin{cproof}{lem:probability_A_ind}
  Throughout this proof let $ Z_{k,i} \colon \Omega \to \R $, $ k \in \{ 1, 2, \ldots, \width \} $, $ i \in \{ 1, 2, \ldots, \ndata \} $,
  satisfy for all $ k \in \{ 1, 2, \ldots, \width \} $, $ i \in \{ 1, 2, \ldots, \ndata \} $ that
  \begin{equation}
    Z_{k,i} = \tfrac{ \scalprod{ W_k(0) }{ \indata_i } }{ \sqrt{\cvar} \eucnorm{ \indata_i } }.
  \end{equation}
  \Nobs that the assumption that
  $ \sqrt{ \nicefrac{1}{\cvar} } W_1(0), \sqrt{ \nicefrac{1}{\cvar} } W_2(0), \ldots, \sqrt{ \nicefrac{1}{\cvar} } W_{\width}(0) $
  are standard normal \proves that for all $ k \in \{ 1, 2, \ldots, \width \} $, $ i \in \{ 1, 2, \ldots, \ndata \} $
  it holds that $ \scalprod{ W_k(0) }{ \indata_i } $ is a centered normal random variable with
  \begin{equation}
    \Var[\scalprod{ W_k(0) }{ \indata_i }]
    = \cvar \Var \bigl[ \scalprod{ \textstyle \sqrt{ \nicefrac{1}{\cvar} } W_k(0) }{ \indata_i } \bigr]
    = \cvar \eucnorm{\indata_i}^2.
  \end{equation}
  \Hence for all $ k \in \{ 1, 2, \ldots, \width \} $, $ i \in \{ 1, 2, \ldots, \ndata \} $ that
  $ Z_{k,i} $ is a standard normal random variable.
  This, \cref{lem:gaussian_anti_concentration_inequality} (applied for every $ k \in \{ 1, 2, \ldots, \width \} $,
  $ i \in \{ 1, 2, \ldots, \ndata \} $, $ n \in \{ 1, 2, \ldots, \nbp \} $ with
  $ ( \Omega, \mc F, \P ) \is ( \Omega, \mc F, \P ) $,
  $ X \is Z_{k,i} $,
  $ \kappa \is \tfrac{\bp_n}{\sqrt{\cvar} \eucnorm{\indata_i}} $,
  $ \varepsilon \is \tfrac{R}{\sqrt{\cvar} \cmin} $
  in the notation of \cref{lem:gaussian_anti_concentration_inequality}),
  and the assumption that
  $ R \le \tfrac{ \sqrt{\pi \cvar} \cmin \poe \lambdazero }
    { 16 \sqrt{2} (1+\cmax^2) \cderiv^2 \nbp \outdim^2 \ndata^2 } $
  \prove that for all $ k \in \{ 1, 2, \ldots, \width \} $, $ i \in \{ 1, 2, \ldots, \ndata \} $ it holds that
  \begin{equation}
  \label{eq:distance_gramG_expectation_ind}
  \begin{split}
    \E \bigl[ \ind_{ \{ \abs{ \scalprod{ W_k(0) } { \indata_i } } \in \, \mc U_R \} } \bigr]
    &= \P \bigl( \abs{ \scalprod{ W_k(0) }{ \indata_i } } \in \cup_{s=1}^{\nbp} [ \bp_s - R, \bp_s + R ] \bigr) \\
    &\le \smallsum\limits_{s=1}^{\nbp} \P \bigl( \abs{ \scalprod{ W_k(0) }{ \indata_i } - \bp_s } \le R \bigr) \\
    &= \smallsum\limits_{s=1}^{\nbp} \P \Bigl( \babs{ Z_{k,i} - \tfrac{\bp_s}{\sqrt{\cvar} \eucnorm{\indata_i}} }
      \le \tfrac{R}{\sqrt{\cvar} \eucnorm{ \indata_i }} \Bigr) \\
    &\le \smallsum\limits_{s=1}^{\nbp} \P \Bigl( \babs{ Z_{k,i} - \tfrac{\bp_s}{\sqrt{\cvar} \eucnorm{\indata_i}} }
      \le \tfrac{R}{\sqrt{\cvar} \cmin} \Bigr) \\
    &\le \smallsum\limits_{s=1}^{\nbp} \tfrac{2 R}{\sqrt{2\pi} \sqrt{\cvar} \cmin}
    = \tfrac{ \sqrt{2} \nbp R }{ \sqrt{\pi \cvar} \cmin }
    \le \tfrac{\poe \lambdazero}{16 (1+\cmax^2) \cderiv^2 \outdim^2 \ndata^2}.
  \end{split}
  \end{equation}
  \Moreover \cref{lem:expectation_absolute_normal} (applied with
  $ ( \Omega, \mc F, \P ) \is ( \Omega, \mc F, \P ) $,
  $ X \is \nicefrac{1}{\sqrt{\Cvar}} \mf W_{1,1} $
  in the notation of \cref{lem:expectation_absolute_normal}) and
  the assumption that
  $ \sqrt{\nicefrac{1}{\Cvar}} \mf W_1, \allowbreak
    \sqrt{\nicefrac{1}{\Cvar}} \mf W_2, \allowbreak \ldots, \allowbreak
    \sqrt{\nicefrac{1}{\Cvar}} \mf W_{\outdim} $
  are independent and standard normal
  \prove that for all $ p, q \in \{ 1, 2, \ldots, \outdim \} $, $ k \in \{ 1, 2, \ldots, \width \} $ with $ p \neq q $ it holds that
  \begin{equation}\label{eq:distance_gramG_expectation_abs_normal}
    \E \bigl[ \abs{ \mf W_{p,k} \mf W_{q,k} } \bigr]
    = \E \bigl[ \abs{ \mf W_{p,k} } \bigr] \E \bigl[ \abs{ \mf W_{q,k} } \bigr]
    = \bigl( \E \bigl[ \abs{ \mf W_{1,1} } \bigr] \bigr)^2
    = \Cvar \bigl( \E \bigl[ \abs{ \nicefrac{1}{\sqrt{\Cvar}} \mf W_{1,1} } \bigr] \bigr)^2
    = \tfrac{2 \Cvar}{\pi}
    \le \Cvar.
  \end{equation}
  \Moreover the assumption that
  $ \sqrt{\nicefrac{1}{\cvar}} W_1(0), \allowbreak
    \sqrt{\nicefrac{1}{\cvar}} W_2(0), \allowbreak \ldots, \allowbreak
    \sqrt{\nicefrac{1}{\cvar}} W_{\width}(0), \allowbreak
    \sqrt{\nicefrac{1}{\Cvar}} \mf W_1, \allowbreak
    \sqrt{\nicefrac{1}{\Cvar}} \mf W_2, \allowbreak \ldots, \allowbreak
    \sqrt{\nicefrac{1}{\Cvar}} \mf W_{\outdim} $
  are independent \proves that for all
  $ i \in \{ 1, 2, \ldots, \ndata \} $, $ k \in \{ 1, 2, \ldots, \width \} $, $ p, q \in \{ 1, 2, \ldots, \outdim \} $
  it holds that $ \mf W_{p,k} \mf W_{q,k} $ and $ \ind_{ \{ \abs{ \scalprod{ W_k(0) } { \indata_i } } \in \, \mc U_R \} } $
  are independent.
  Combining this, \cref{eq:distance_gramG_expectation_ind}, \cref{eq:distance_gramG_expectation_abs_normal}, and
  the fact that for all $ k \in \{ 1, 2, \ldots, \width \} $, $ p \in \{ 1, 2, \ldots, \outdim \} $ it holds that
  $ \mf W_{p,k} $ is a centered normal random variable with $ \Var[ \mf W_{p,k} ] = \Cvar $
  \proves that it holds that
  \begin{equation}
  \begin{split}
    &\E \! \biggl[ \smallsum\limits_{i,j=1}^{\ndata} \smallsum\limits_{p,q=1}^{\outdim} \smallsum\limits_{k=1}^{\width}
      \abs{ \mf W_{p,k} \mf W_{q,k} } \bigl(
      \ind_{ \{ \abs{ \scalprod{ W_k(0) } { \indata_i } } \in \, \mc U_R \} }
      + \ind_{ \{ \abs{ \scalprod{ W_k(0) } { \indata_j } } \in \, \mc U_R \} } \bigr) \biggr] \\
    &= \smallsum\limits_{i,j=1}^{\ndata} \smallsum\limits_{p,q=1}^{\outdim} \smallsum\limits_{k=1}^{\width}
      \E \bigl[ \abs{ \mf W_{p,k} \mf W_{q,k} } \bigr] \Bigl(
      \E \bigl[ \ind_{ \{ \abs{ \scalprod{ W_k(0) } { \indata_i } } \in \, \mc U_R \} } \bigr]
      + \E \bigl[ \ind_{ \{ \abs{ \scalprod{ W_k(0) } { \indata_j } } \in \, \mc U_R \} } \bigr]
      \Bigr) \\
    &\le \smallsum\limits_{i,j=1}^{\ndata} \smallsum\limits_{p,q=1}^{\outdim} \smallsum\limits_{k=1}^{\width}
      \Cvar \Bigl( \tfrac{\poe \lambdazero}{16 \Cvar (1+\cmax^2) \cderiv^2 \outdim^2 \ndata^2}
      + \tfrac{\poe \lambdazero}{16 \Cvar (1+\cmax^2) \cderiv^2 \outdim^2 \ndata^2} \Bigr)
    = \smallsum\limits_{i,j=1}^{\ndata} \smallsum\limits_{p,q=1}^{\outdim}
      \tfrac{\poe \Cvar \width \lambdazero}{8 (1+\cmax^2) \cderiv^2 \outdim^2 \ndata^2}
    = \tfrac{\poe \Cvar \width \lambdazero}{8 (1+\cmax^2) \cderiv^2}.
  \end{split}
  \end{equation}
  The Markov inequality thus \proves that
  \begin{equation}
  \begin{split}
    &\P \biggl( \smallsum\limits_{i,j=1}^{\ndata} \smallsum\limits_{p,q=1}^{\outdim} \smallsum\limits_{k=1}^{\width}
      \abs{ \mf W_{p,k} \mf W_{q,k} } \bigl(
      \ind_{ \{ \abs{ \scalprod{ W_k(0) } { \indata_i } } \in \, \mc U_R \} }
      + \ind_{ \{ \abs{ \scalprod{ W_k(0) } { \indata_j } } \in \, \mc U_R \} } \bigr)
      \le \tfrac{\Cvar \width \lambdazero}{8 (1+\cmax^2) \cderiv^2} \biggr) \\
    &\ge 1 - \P \biggl( \smallsum\limits_{i,j=1}^{\ndata} \smallsum\limits_{p,q=1}^{\outdim} \smallsum\limits_{k=1}^{\width}
      \abs{ \mf W_{p,k} \mf W_{q,k} } \bigl(
      \ind_{ \{ \abs{ \scalprod{ W_k(0) } { \indata_i } } \in \, \mc U_R \} }
      + \ind_{ \{ \abs{ \scalprod{ W_k(0) } { \indata_j } } \in \, \mc U_R \} } \bigr)
      \ge \tfrac{\Cvar \width \lambdazero}{8 (1+\cmax^2) \cderiv^2} \biggr) \\
    &\ge 1 - \frac{ \E \bigl[ \smallsum\nolimits_{i,j=1}^{\ndata} \smallsum\nolimits_{p,q=1}^{\outdim} \smallsum\nolimits_{k=1}^{\width}
      \abs{ \mf W_{p,k} \mf W_{q,k} } ( 
      \ind_{ \{ \abs{ \scalprod{ W_k(0) } { \indata_i } } \in \, \mc U_R \} }
      + \ind_{ \{ \abs{ \scalprod{ W_k(0) } { \indata_j } } \in \, \mc U_R \} } ) \bigr]}
      { \Cvar \width \lambdazero 8^{-1} (1+\cmax^2)^{-1} \cderiv^{-2} }
    \ge 1 - \poe.
  \end{split}
  \end{equation}
\end{cproof}

\cfclear
\begin{lemma}
  \label{lem:distance_gramG}
  Assume \cref{setting:gradient_descent},
  let $ \poe \in (0,1) $, $ R \in (0,\infty) $, $ A \in \mc F $ satisfy
  \begin{equation}
    A = \biggl\{ \smallsum\limits_{i,j=1}^{\ndata} \smallsum\limits_{p,q=1}^{\outdim} \smallsum\limits_{k=1}^{\width}
      \abs{ \mf W_{p,k} \mf W_{q,k} } \bigl(
      \ind_{ \{ \abs{ \scalprod{ W_k(0) } { \indata_i } } \in \, \mc U_R \} }
      + \ind_{ \{ \abs{ \scalprod{ W_k(0) } { \indata_j } } \in \, \mc U_R \} } \bigr)
      \le \tfrac{\Cvar \width \lambdazero}{8 (1+\cmax^2) \cderiv^2} \biggr\},
  \end{equation}
  let $ \omega \in A $, $ \iteration \in \N_0 $,
  and assume
  $ \max_{ k \in \{ 1, 2, \ldots, \width \} } [ \cmax \eucnorm{ W_k(\iteration,\omega) - W_k(0,\omega) }
    + \abs{ B_k(\iteration,\omega) - B_k(0,\omega) } ] \le R $
  \cfload.
  Then $ \specnorm{ \gramG (\iteration,\omega) - \gramG (0,\omega) } \le \nicefrac{\Cvar \width \lambdazero}{4} $ \cfout.
\end{lemma}

\begin{cproof}{lem:distance_gramG}
  \Nobs that \cref{eq:setting_gramG},
  the assumption that $ \eucnorm{ B(0) } = 0 $, and
  the fact that for all $ i, j \in \{ 1, 2, \ldots, \ndata \} $ it holds that
  $ \abs{ 1 + \scalprod{\indata_i}{\indata_j}} \le 1 + \eucnorm{\indata_i} \eucnorm{\indata_j} \le 1 + \cmax^2 $
  \prove that for all $ i, j \in \{ 1, 2, \ldots, \ndata \} $, $ p, q \in \{ 1, 2, \ldots, \outdim \} $ it holds that
  \begin{equation}
  \label{eq:distance_gramG_entries}
  \begin{split}
    &\abs{ \gramG_{ (i-1) \outdim + p, (j-1) \outdim + q } (\iteration,\omega)
      - \gramG_{ (i-1) \outdim + p, (j-1) \outdim + q } (0,\omega) } \\
    &= \bbbabs{ \bigl( 1 + \scalprod{\indata_i}{\indata_j} \bigr) \smallsum\limits_{k=1}^{\width}
      \mf W_{p,k} (\omega) \mf W_{q,k} (\omega)
      \deriv \bigl( \scalprod{ W_k(\iteration,\omega) }{ \indata_i } + B_k(\iteration,\omega) \bigr)
      \deriv \bigl( \scalprod{ W_k(\iteration,\omega) }{ \indata_j } + B_k(\iteration,\omega) \bigr) \\
    &\quad - \bigl( 1 + \scalprod{\indata_i}{\indata_j} \bigr) \smallsum\limits_{k=1}^{\width}
      \mf W_{p,k} (\omega) \mf W_{q,k} (\omega)
      \deriv \bigl( \scalprod{ W_k(0,\omega) }{ \indata_i } + B_k(0,\omega) \bigr)
      \deriv \bigl( \scalprod{ W_k(0,\omega) }{ \indata_j } + B_k(0,\omega) \bigr) } \\
    &= \abs{1 + \scalprod{\indata_i}{\indata_j}} \bbbabs{ \smallsum\limits_{k=1}^{\width}
      \mf W_{p,k} (\omega) \mf W_{q,k} (\omega) \bigl(
      \deriv \bigl( \scalprod{ W_k(\iteration,\omega) }{ \indata_i } + B_k(\iteration,\omega) \bigr)
      \deriv \bigl( \scalprod{ W_k(\iteration,\omega) }{ \indata_j } + B_k(\iteration,\omega) \bigr) \\
    &\quad - \deriv \bigl( \scalprod{ W_k(0,\omega) }{ \indata_i } \bigr)
      \deriv \bigl( \scalprod{ W_k(0,\omega) }{ \indata_j } \bigr) } \\
    &\le ( 1 + \cmax^2 ) \smallsum\limits_{k=1}^{\width} \abs{ \mf W_{p,k} (\omega) \mf W_{q,k} (\omega) } \babs{
      \deriv \bigl( \scalprod{ W_k(\iteration,\omega) }{ \indata_i } + B_k(\iteration,\omega) \bigr)
      \deriv \bigl( \scalprod{ W_k(\iteration,\omega) }{ \indata_j } + B_k(\iteration,\omega) \bigr) \\
    &\quad - \deriv \bigl( \scalprod{ W_k(0,\omega) }{ \indata_i } \bigr)
      \deriv \bigl( \scalprod{ W_k(0,\omega) }{ \indata_j } \bigr) }.
  \end{split}
  \end{equation}
  \Moreover the assumption that
  $ \max_{ k \in \{ 1, 2, \ldots, \width \} } [ \cmax \eucnorm{ W_k(\iteration,\omega) - W_k(0,\omega) }
    + \abs{ B_k(\iteration,\omega) - B_k(0,\omega) } ] \le R $
  and the Cauchy Schwarz inequality
  \prove that for all $ i \in \{ 1, 2, \ldots, \ndata \} $, $ k \in \{ 1, 2, \ldots, \width \} $ it holds that
  \begin{equation}\label{eq:distance_gramG_diff_R}
  \begin{split}
    \abs{ \scalprod{ W_k(\iteration,\omega) }{ \indata_i } + B_k(\iteration,\omega) - \scalprod{ W_k(0,\omega) }{ \indata_i } }
    &=\abs{ \scalprod{ W_k(\iteration,\omega) - W_k(0,\omega) }{ \indata_i } + B_k(\iteration,\omega) } \\
    &\le \abs{ \scalprod{ W_k(\iteration,\omega) - W_k(0,\omega) }{ \indata_i } }
      + \abs{ B_k(\iteration,\omega) - B_k(0,\omega) } \\
    &\le \eucnorm{ W_k(\iteration,\omega) - W_k(0,\omega) } \eucnorm{ \indata_i }
      + \abs{ B_k(\iteration,\omega) - B_k(0,\omega) } \\
    &\le \cmax \eucnorm{ W_k(\iteration,\omega) - W_k(0,\omega) } + \abs{ B_k(\iteration,\omega) - B_k(0,\omega) }
    \le R.
  \end{split}
  \end{equation}
  \Moreover
  the fact that for all $ u, v \in \R $ it holds that $ \abs{ \deriv(v) } \le \cderiv $ and
  $ \abs{ \deriv(u) - \deriv(v) } \le \abs{ \deriv(u) } + \abs{ \deriv(v) } \le 2 \cderiv $
  \proves that for all $ a, b, \ms a, \ms b \in \R $
  with $ \abs{ a - \ms a } \le R $ and $ \abs{ b - \ms b } \le R $ it holds that
  \begin{equation}
  \begin{split}
    \babs{ \deriv(a) \deriv(b) - \deriv(\ms a) \deriv(\ms b) }
    &= \babs{ \deriv(a) \deriv(b) - \deriv(\ms a) \deriv(b) + \deriv(\ms a) \deriv(b) - \deriv(\ms a) \deriv(\ms b) } \\
    &= \babs{ \deriv(b) \bigl( \deriv(a) - \deriv(\ms a) \bigr) + \deriv(\ms a) \bigl( \deriv(b) - \deriv(\ms b) \bigr) } \\
    &\le \babs{ \deriv(b) } \babs{ \deriv(a) - \deriv(\ms a) }
      + \babs{ \deriv(\ms a) } \babs{ \deriv(b) - \deriv(\ms b) } \\
    &\le 2 \cderiv^2 \ind_{ \mc U_R }( \ms a )
      + 2 \cderiv^2 \ind_{ \mc U_R }( \ms b )
    = 2 \cderiv^2 \bigl( \ind_{ \mc U_R }( \ms a )
      + \ind_{ \mc U_R }( \ms b ) \bigr).
  \end{split}
  \end{equation}
  Combining this, \cref{lem:spectral_norm_inequality} (applied with
  $ m \is \outdim \ndata $, $ n \is \outdim \ndata $, $ A \is \gramG (\iteration,\omega) - \gramG (0,\omega) ) $
  in the notation of \cref{lem:spectral_norm_inequality}),
  \cref{eq:distance_gramG_entries}, \cref{eq:distance_gramG_diff_R}, and
  the assumption that $ \omega \in A $ \proves that
  \begin{equation}
  \label{eq:distance_gramG_sum}
  \begin{split}
    &\specnorm{ \gramG (\iteration,\omega) - \gramG (0,\omega) } \\
    &\le \smallsum\limits_{i,j=1}^{\ndata} \smallsum\limits_{p,q=1}^{\outdim}
      \abs{ \gramG_{ (i-1) \outdim + p, (j-1) \outdim + q } (\iteration,\omega)
      - \gramG_{ (i-1) \outdim + p, (j-1) \outdim + q } (0,\omega) } \\
    &\le (1+\cmax^2) \smallsum\limits_{i,j=1}^{\ndata} \smallsum\limits_{p,q=1}^{\outdim}
      \smallsum\limits_{k=1}^{\width} \abs{ \mf W_{p,k} (\omega) \mf W_{q,k} (\omega) } \babs{
      \deriv \bigl( \scalprod{ W_k(\iteration,\omega) }{ \indata_i } + B_k(\iteration,\omega) \bigr)
      \deriv \bigl( \scalprod{ W_k(\iteration,\omega) }{ \indata_j } + B_k(\iteration,\omega) \bigr) \\
    &\quad - \deriv \bigl( \scalprod{ W_k(0,\omega) }{ \indata_i } \bigr)
      \deriv \bigl( \scalprod{ W_k(0,\omega) }{ \indata_j } \bigr) } \\
    &\le 2 (1+\cmax^2) \cderiv^2 \smallsum\limits_{i,j=1}^{\ndata} \smallsum\limits_{p,q=1}^{\outdim} \smallsum\limits_{k=1}^{\width}
      \abs{ \mf W_{p,k} (\omega) \mf W_{q,k} (\omega) } \bigl(
      \ind_{ \{ \abs{ \scalprod{ W_k(0) } { \indata_i } } \in \, \mc U_R \} }
      + \ind_{ \{ \abs{ \scalprod{ W_k(0) } { \indata_j } } \in \, \mc U_R \} } \bigr) \\
    &\le 2 (1+\cmax^2) \cderiv^2 \tfrac{\Cvar \width \lambdazero}{8 (1+\cmax^2) \cderiv^2}
    = \tfrac{\Cvar \width \lambdazero}{4}
  \end{split}
  \end{equation}
  \cfload.
\end{cproof}

\subsection{Analysis of eigenvalues of stochastic Gram matrices during training}
\label{subsec:analysis_of_eigenvalues_of_stochastic_gram_matrices_during_training}

\cfclear
\begin{lemma}
  \label{lem:lambdamin_eucnorm}
  Let $ n \in \N $ and let $ A \in \R^{ n \times n } $ be a symmetric matrix.
  Then
  \begin{enumerate}[label=(\roman{*})]
    \item \label{item:lambdamin_1} it holds for all $ v \in \R^n $ that $ \scalprod{ v }{ A v } \ge \lambdamin ( A ) \eucnorm{ v }^2 $ and
    \item \label{item:lambdamin_2} it holds that
      $ \lambdamin ( A ) = \min_{ v \in \R^n, \, \eucnorm{z} = 1 } \scalprod{ v }{ A v } $
  \end{enumerate}
  \cfout.
\end{lemma}

\begin{cproof}{lem:lambdamin_eucnorm}
  \Nobs that the spectral theorem \proves that there exist
  $ \lambda_1, \lambda_2, \ldots, \lambda_{n} \in \R $ and a basis
  $ \{ z_1, z_2, \ldots, z_{n} \} $ of the Euclidean space $ ( \R^{n}, \scalprod \cdot \cdot ) $
  which satisfy for all $ i, j \in \{ 1, 2, \ldots, n \} $ that
  \begin{equation}
  \label{eq:lambdamin_eucnorm}
    \bigl\{ \lambda \in \R \colon [ \Exists v \in \R^{n} \backslash \{ 0 \} \colon A v = \lambda v ] \bigr\}
      = \{ \lambda_1, \lambda_2, \ldots, \lambda_{n} \bigr\}, \quad
    A z_i = \lambda_i z_i, \quad \text{and} \quad
    \scalprod{z_i}{z_j} = \begin{cases} 1 &\colon i = j, \\ 0 &\colon i \neq j \end{cases}
  \end{equation}
  \cfload.
  This \proves that for all $ v \in \R^{n} $ there exist $ \alpha_1, \alpha_2, \ldots, \alpha_{n} \in \R $
  such that $ v = \sum_{i=1}^{n} \alpha_i z_i $.
  Combining this with \cref{eq:lambdamin_eucnorm} and the fact that for all $ i \in \{ 1, 2, \ldots, n \} $ it holds that
  $ \lambda_i \ge \min\{ \lambda_1, \lambda_2, \ldots, \lambda_{n} \}
    = \min\{ \lambda \in \R \colon [ \Exists v \in \R^{n} \colon A v = \lambda v ] \} = \lambdamin (A) $
  \proves that for all $ v \in \R^{n} $ there exist $ \alpha_1, \alpha_2, \ldots, \alpha_{n} \in \R $
  such that $ v = \sum_{i=1}^{n} \alpha_i z_i $ and
  \begin{equation}
  \begin{split}
    \scalprod{ v }{ A v }
    &= \bbscalprod{\smallsum\limits_{i=1}^{n} \alpha_i z_i}{A \Bigl( \smallsum\limits_{j=1}^{n} \alpha_j z_j \Bigr) }
    = \bbscalprod{\smallsum\limits_{i=1}^{n} \alpha_i z_i}{\smallsum\limits_{j=1}^{n} \alpha_j \lambda_j z_j}
    = \smallsum\limits_{i=1}^{n} \smallsum\limits_{j=1}^{n} \alpha_i \alpha_j \lambda_j \scalprod{z_i}{z_j}
    = \smallsum\limits_{i=1}^{n} \lambda_i \abs{\alpha_i}^2 \\
    &\ge \lambdamin ( A ) \smallsum\limits_{i=1}^{n} \abs{\alpha_i}^2
    = \lambdamin ( A ) \smallsum\limits_{i=1}^{n} \smallsum\limits_{j=1}^{n} \alpha_i \alpha_j \scalprod{z_i}{z_j}
    = \lambdamin ( A ) \bbscalprod{ \smallsum\limits_{i=1}^{n} \alpha_i z_i}{\smallsum\limits_{j=1}^{n} \alpha_j z_j} \\
    &= \lambdamin ( A ) \scalprod{v}{v}
    = \lambdamin ( A ) \eucnorm{ v }^2
  \end{split}
  \end{equation}
  \cfload.
  This \proves \cref{item:lambdamin_1}.
  Next \nobs that \cref{item:lambdamin_1} \proves that for all $ v \in \R^{n} $ with $ \eucnorm{ v } = 1 $
  it holds that $ \scalprod{ v }{ A v } \ge \lambdamin (A) $.
  \Hence that
  \begin{equation}
  \label{eq:lambdamin_eucnorm_ge}
    \min\nolimits_{ v \in \R^{n}, \, \eucnorm{v} = 1 } \scalprod{ v }{ A v } \ge \lambdamin ( A ).
  \end{equation}
  \Moreover \cref{eq:lambdamin_eucnorm} \proves that there exists
  $ i \in \{ 1, 2, \ldots, n \} $ which satisfies $ A z_i = \lambdamin ( A ) z_i $.
  Thus, we obtain that there exists $ u \in \R^{n} $ which satisfies $ \eucnorm{ u } = 1 $ and $ A u = \lambdamin ( A ) u $.
  This \proves that
  \begin{equation}
  \begin{split}
    \lambdamin ( A )
    &= \lambdamin ( A ) \eucnorm{ u }^2
    = \lambdamin ( A ) \scalprod{ u }{ u }
    = \scalprod{u}{\lambdamin ( A ) u} \\
    &= \scalprod{u}{Au}
    \ge \min\nolimits_{ v \in \R^{n}, \, \eucnorm{v} = 1 } \scalprod{ v }{ A v }.
  \end{split}
  \end{equation}
  Combining this with \cref{eq:lambdamin_eucnorm_ge} \proves that
  $ \lambdamin ( A ) = \min_{ v \in \R^{n}, \, \eucnorm{v} = 1 } \scalprod{ v }{ A v } $.
  This \proves \cref{item:lambdamin_2}.
\end{cproof}

\cfclear
\begin{lemma}
  \label{lem:lambdamin_specnorm}
  Let $ n \in \N $ and let $ A, B \in \R^{ n \times n } $ be symmetric matrices.
  Then $ \lambdamin ( A ) \ge \lambdamin ( B ) - \specnorm{ A - B } $ \cfout.
\end{lemma}

\begin{cproof}{lem:lambdamin_specnorm}
  \Nobs that the Cauchy-Schwarz inequality \proves for all $ v \in \R^{n} \backslash \{ 0 \} $ that
  \begin{equation}
    - \bscalprod{ v }{ (A - B ) v }
    \le \babs{ \bscalprod{ v }{ ( A - B ) v } }
    \le \eucnorm{ v } \eucnorm{ ( A - B ) v }
    = \eucnorm{ v }^2 \tfrac{ \eucnorm{ ( A - B ) v } }{ \eucnorm{ v } }
    \le \eucnorm{ v }^2 \specnorm{ A - B }
  \end{equation}
  \cfload.
  \Hence for all $ v \in \R^{n} $ with $ \eucnorm{ v } = 1 $ that
  $ \scalprod{ v }{ ( A - B ) v } \ge - \specnorm{ A - B } $.
  This and \cref{item:lambdamin_1} in \cref{lem:lambdamin_eucnorm} (applied with
  $ n \is n $, $ A \is B $
  in the notation of \cref{lem:lambdamin_eucnorm}) \prove for all $ v \in \R^{n} $ with $ \eucnorm v = 1 $ that
  \begin{equation}
    \scalprod{ v }{ A v }
    = \bscalprod{ v }{ ( A - B ) v } + \scalprod{ v }{ B v }
    \ge \lambdamin ( B ) - \specnorm{ A - B }.
  \end{equation}
  Combining this with \cref{item:lambdamin_2} in \cref{lem:lambdamin_eucnorm} (applied with
  $ n \is n $, $ A \is A $
  in the notation of \cref{lem:lambdamin_eucnorm})
  \proves that
  \begin{equation}
    \lambdamin ( A )
    = \min\nolimits_{ v \in \R^{n}, \, \eucnorm{z} = 1 } \scalprod{ v }{ A v }
    \ge \lambdamin ( B ) - \specnorm{ A - B }.
  \end{equation}
\end{cproof}

\cfclear
\begin{lemma}
  \label{lem:lambdamin_gramG}
  Assume \cref{setting:gradient_descent},
  let $ \poe \in (0,1) $, $ R \in (0,\infty) $, $ A_1, A_2 \in \mc F $ satisfy
  $ A_1 = \bigcap_{i=1}^{\outdim \ndata} \bigcap_{j=1}^{\outdim \ndata} \bigl\{ \abs{ \gramG_{i,j} (0) - \gramGinf_{i,j} }
    \le \frac{\Cvar \width \lambdazero}{4 \outdim \ndata} \bigr\} $
  and
  $ A_2 = \bigl\{ \sum_{i,j=1}^{\ndata} \sum_{p,q=1}^{\outdim} \sum_{k=1}^{\width} \abs{ \mf W_{p,k} \mf W_{q,k} } \bigl(
    \ind_{ \{ \abs{ \scalprod{ W_k(0) } { \indata_i } } \in \, \mc U_R \} }
    + \ind_{ \{ \abs{ \scalprod{ W_k(0) } { \indata_j } } \in \, \mc U_R \} } \bigr)
    \le \frac{\Cvar \width \lambdazero}{8 (1+\cmax^2) \cderiv^2} \bigr\} $,
  let $ \omega \in A_1 \cap A_2 $, $ \iteration \in \N_0 $, and
  assume
  $ \max_{ k \in \{ 1, 2, \ldots, \width \} } [ \cmax \eucnorm{ W_k(\iteration,\omega) - W_k(0,\omega) }
    + \abs{ B_k(\iteration,\omega) - B_k(0,\omega) } ] \le R $
  \cfload.
  Then $ \lambdamin( \gramG (\iteration,\omega) ) \ge \nicefrac{\Cvar \width \lambdazero}{2} $ \cfout.
\end{lemma}

\begin{cproof}{lem:lambdamin_gramG}
  \Nobs that \cref{lem:distance_gramGinf} (applied with
  $ \poe \is \poe $, $ A \is A_1 $, $ \omega \is \omega $
  in the notation of \cref{lem:distance_gramGinf}) and
  the fact that $ \omega \in A_1 $ \prove that
  \begin{equation}\label{eq:lem:lambdamin_gramG_inf}
    \specnorm{ \gramG (0,\omega) - \gramGinf } \le \tfrac{\Cvar \width \lambdazero}{4}
  \end{equation}
  \cfload.
  \Moreover \cref{lem:distance_gramG} (applied with
  $ \poe \is \poe $, $ R \is R $, $ A \is A_2 $, $ \omega \is \omega $, $ \iteration \is \iteration $
  in the notation of \cref{lem:distance_gramG}) and
  the fact that $ \omega \in A_2 $ \prove that
  \begin{equation}\label{eq:lem:lambdamin_gramG_stoch}
    \specnorm{ \gramG (\iteration,\omega) - \gramG (0,\omega) } \le \tfrac{\Cvar \width \lambdazero}{4}.
  \end{equation}
  Combining this, \cref{eq:lem:lambdamin_gramG_inf}, \cref{eq:lem:lambdamin_gramG_stoch}, and
  \cref{lem:lambdamin_specnorm} (applied with
  $ n \is \outdim \ndata $,
  $ A \is \gramG (\iteration,\omega) $,
  $ B \is \gramGinf $
  in the notation of \cref{lem:lambdamin_specnorm}), and
  the fact that $ \lambdamin( \gramGinf ) = \Cvar \width \lambdazero $.
  \proves that it holds that
  \begin{equation}
  \begin{split}
    \lambdamin( \gramG (\iteration,\omega) )
    &\ge \lambdamin( \gramGinf ) - \specnorm{ \gramG (\iteration,\omega) - \gramGinf }
    = \Cvar \width \lambdazero - \specnorm{ \gramG (\iteration,\omega) - \gramG (0,\omega)
      + \gramG (0,\omega) - \gramGinf } \\
    &\ge \Cvar \width \lambdazero - \specnorm{ \gramG (\iteration,\omega) - \gramG (0,\omega) }
      - \specnorm{ \gramG (0,\omega) - \gramGinf }
    \ge \Cvar \width \lambdazero - \tfrac{\Cvar \width \lambdazero}{4} - \tfrac{\Cvar \width \lambdazero}{4}
    = \tfrac{\Cvar \width \lambdazero}{2}.
  \end{split}
  \end{equation}
\end{cproof}

\section{Error analysis for GD optimization algorithms in the training of ANNs}
\label{sec:error_analysis}

In this section, we combine the estimates of the errors at initialization, the analysis of the evolution of the weights and biases of the considered \ANNs\ during training, and the lower bounds for the eigenvalues of the considered stochastic Gram matrices to obtain an error analysis for the considered \GD\ optimization algorithms.

Specifically, in \cref{prop:main_proposition} below we prove a pathwise upper bound on the squared error at each iteration step on a specific measurable set. In the main result of this article, \cref{thm:main_theorem} in \cref{subsec:quantitative_probabilistic_error_analysis} below, we combine this result with the analysis of the probability of this measurable set in \cref{lem:probability_A_cap} to obtain a quantitative probabilistic error analysis for the considered \GD\ optimization algorithms. In the remainder of \cref{subsec:quantitative_probabilistic_error_analysis}, we progressively simplify the constraints on the width $ \width \in \N $ and the learning rate $ \lr \in (0,\infty) $ and additionally consider pairwise distinct training input data in order to obtain more easily accessible results. In the proof of these convergence results, we use, among other things, the well-known upper bound on the natural logarithm in \cref{lem:log_inequality} below, whose proof we include only for completeness.

We conclude \cref{sec:error_analysis} with the qualitative error analysis for the considered \GD\ optimization algorithms, which we provide in \cref{cor:vectorized} below. Specifically, we show that the considered \GD\ optimization algorithms can be derived by approximating the piecewise affine activation function $ \act $ with smooth functions, applying the \GD\ method to the resulting smooth risk functions, and then taking the limit. In \cref{lem:convolution}, we give an example of appropriate approximations of the activation function and in \cref{lem:left_derivative} we show some elementary properties of the left gradient, which together allow us to deduce \cref{thm:introduction} above from \cref{cor:vectorized}.

\subsection{Quantitative pathwise error analysis for GD optimization algorithms}
\label{subsec:quantitative_pathwise_error_analysis}

\cfclear
\begin{lemma}
  \label{lem:probability_A_cap}
  Assume \cref{setting:gradient_descent},
  assume $ \lambdazero \in (0,\infty) $,
  let $ \poe \in (0,1) $, $ R \in (0,\infty) $ satisfy
  $ R \le \tfrac{ \sqrt{\pi \cvar} \cmin \poe \lambdazero }
    { 16 \sqrt{2} (1+\cmax^2) \cderiv^2 \nbp \outdim^2 \ndata^2 } $,
  assume
  $ \width \ge \tfrac{32 (1+\cmax^2) \outdim \ndata}{\lambdazero}
    \ln( \frac{2 \outdim^2 \ndata^2}{\poe} ) \max \bigl\{
    \frac{4 (1+\cmax^2) \eucnorm{\slope}^4 \outdim \ndata}{\lambdazero}, \cderiv \bigr\} $,
  and let $ A_1, A_2, A_3, \allowbreak A_4, A_5, A_6 \in \mc F $ satisfy
  \begin{equation}
  \begin{split}
    A_1 &= \Bigl \{ \eucnorm{ \fnetvalue(0) - \foutdata }^2 \le
      \bigl( \cvar \Cvar \width \outdim \eucnorm{\slope}^2 \eucnorm{ \findata }^2
      + \Cvar \width \outdim \eucnorm{\yinter}^2 \ndata + \eucnorm{ \foutdata }^2 \bigr) \poe^{-1} \Bigr \}, \\
    A_2 &= \smallbigcap\limits_{p=1}^{\outdim} \smallbigcap\limits_{k=1}^{\width} \Bigl\{ \abs{\mf W_{p,k}}^2
      \le 2 \Cvar \ln ( \tfrac{2 \outdim \width}{\poe} ) \Bigr\}, \qquad
    A_3 = \smallbigcap\limits_{i=1}^{\outdim \ndata} \smallbigcap\limits_{j=1}^{\outdim \ndata} \bigl\{
      \abs{ \gramG_{i,j} (0) - \gramGinf_{i,j} } \le \frac{\Cvar \width \lambdazero}{4 \outdim \ndata} \bigr\}, \\
    A_4 &= \Bigl\{ \smallsum\limits_{i,j=1}^{\ndata} \smallsum\limits_{p,q=1}^{\outdim} \smallsum\limits_{k=1}^{\width}
      \abs{ \mf W_{p,k} \mf W_{q,k} } \bigl(
    \ind_{ \{ \abs{ \scalprod{ W_k(0) } { \indata_i } } \in \, \mc U_R \} }
    + \ind_{ \{ \abs{ \scalprod{ W_k(0) } { \indata_j } } \in \, \mc U_R \} } \bigr)
    \le \tfrac{\Cvar \width \lambdazero}{8 (1+\cmax^2) \cderiv^2} \Bigr\}, \\
    A_5 &= \Bigl\{ \smallsum\limits_{i=1}^{\ndata} \smallsum\limits_{p,q=1}^{\outdim} \smallsum\limits_{k=1}^{\width}
      \abs{ \mf W_{p,k} \mf W_{q,k} }
      \ind_{ \{ \abs{ \scalprod{ W_k(0) } { \indata_i } } \in \, \mc U_R \} }
      \le \frac{\sqrt{2} \Cvar \width \nbp \outdim^2 \ndata R}{\sqrt{\pi \cvar} \cmin \poe} \Bigr\}, \quad \text{and} \\
    A_6 &= \smallbigcap\limits_{p=1}^{\outdim} \Bigl \{ \smallsum\limits_{q=1}^{\outdim} \smallsum\limits_{k=1}^{\width}
      \abs{ \mf W_{p,k} \mf W_{q,k} } \le \tfrac{\Cvar \width \outdim^2}{\poe} \Bigr\},
  \end{split}
  \end{equation}
  \cfload.
  Then $ \P \bigl( \bigcap_{i=1}^6 A_i \bigr) \ge 1 - 6 \poe $.
\end{lemma}

\begin{cproof}{lem:probability_A_cap}
  Throughout this proof let $ Z_{k,i} \colon \Omega \to \R $, $ k \in \{ 1, 2, \ldots, \width \} $, $ i \in \{ 1, 2, \ldots, \ndata \} $,
  satisfy for all $ k \in \{ 1, 2, \ldots, \width \} $, $ i \in \{ 1, 2, \ldots, \ndata \} $ that
  \begin{equation}
    Z_{k,i} = \tfrac{\scalprod{ W_k(0) }{ \indata_i }}{\sqrt{\cvar} \eucnorm{ \indata_i } }.
  \end{equation}
  \Nobs that the assumption that $
  \sqrt{\nicefrac{1}{\cvar}} W_1(0), \sqrt{\nicefrac{1}{\cvar}} W_2(0), \ldots, \sqrt{\nicefrac{1}{\cvar}} W_{\width}(0) $
  are standard normal \proves that for all
  $ k \in \{ 1, 2, \ldots, \width \} $, $ i \in \{ 1, 2, \ldots, \ndata \} $ it holds that $ \scalprod{ W_k(0) }{ \indata_i } $
  is a centered normal random variable with
  \begin{equation}
    \Var[\scalprod{ W_k(0) }{ \indata_i }]
    = \cvar \Var \bigl[ \scalprod{ \textstyle \sqrt{ \nicefrac{1}{\cvar} } W_k(0) }{ \indata_i } \bigr]
    = \cvar \eucnorm{\indata_i}^2.
  \end{equation}
  \Hence for all $ k \in \{ 1, 2, \ldots, \width \} $, $ i \in \{ 1, 2, \ldots, \ndata \} $ that $ Z_{k,i} $
  is a standard normal random variable.
  This and \cref{lem:gaussian_anti_concentration_inequality} (applied for every $ k \in \{ 1, 2, \ldots, \width \} $,
  $ i \in \{ 1, 2, \ldots, \ndata \} $, $ s \in \{ 1, 2, \dots, \nbp \} $ with
  $ ( \Omega, \mc F, \P ) \is ( \Omega, \mc F, \P ) $,
  $ X \is Z_{k,i} $,
  $ \kappa \is \tfrac{\bp_s}{\sqrt{\cvar} \eucnorm{\indata_i}} $,
  $ \varepsilon \is \tfrac{R}{\sqrt{\cvar} \cmin} $
  in the notation of \cref{lem:gaussian_anti_concentration_inequality})
  \prove that for all $ k \in \{ 1, 2, \ldots, \width \} $, $ i \in \{ 1, 2, \ldots, \ndata \} $ it holds that
  \begin{equation}
  \label{eq:lem_probability_expectation_ind}
  \begin{split}
    \E \bigl[ \ind_{ \{ \abs{ \scalprod{ W_k(0) } { \indata_i } } \in \, \mc U_R \} } \bigr]
    &= \P \bigl( \abs{ \scalprod{ W_k(0) }{ \indata_i } } \in \cup_{s=1}^{\nbp} [ \bp_s - R, \bp_s + R ] \bigr) \\
    &\le \smallsum\limits_{s=1}^{\nbp} \P \bigl( \abs{ \scalprod{ W_k(0) }{ \indata_i } - \bp_s } \le R \bigr)
    = \smallsum\limits_{s=1}^{\nbp} \P \Bigl( \babs{ Z_{k,i} - \tfrac{\bp_s}{\sqrt{\cvar} \eucnorm{\indata_i}} }
      \le \tfrac{R}{\sqrt{\cvar} \eucnorm{ \indata_i }} \Bigr) \\
    &\le \smallsum\limits_{s=1}^{\nbp} \P \Bigl( \babs{ Z_{k,i} - \tfrac{\bp_s}{\sqrt{\cvar} \eucnorm{\indata_i}} }
      \le \tfrac{R}{\sqrt{\cvar} \cmin} \Bigr)
    \le \smallsum\limits_{s=1}^{\nbp} \tfrac{2 R}{\sqrt{2\pi} \sqrt{\cvar} \cmin}
    = \tfrac{ \sqrt{2} \nbp R }{ \sqrt{\pi \cvar} \cmin }.
  \end{split}
  \end{equation}
  \Moreover \cref{lem:expectation_absolute_normal} (applied with
  $ ( \Omega, \mc F, \P ) \is ( \Omega, \mc F, \P ) $,
  $ X \is \nicefrac{1}{\sqrt{\Cvar}} \mf W_{1,1} $
  in the notation of \cref{lem:expectation_absolute_normal}) and
  the assumption that
  $ \sqrt{\nicefrac{1}{\Cvar}} \mf W_1, \allowbreak
    \sqrt{\nicefrac{1}{\Cvar}} \mf W_2, \allowbreak \ldots, \allowbreak
    \sqrt{\nicefrac{1}{\Cvar}} \mf W_{\outdim} $
  are independent and standard normal
  \prove that for all $ p, q \in \{ 1, 2, \ldots, \outdim \} $, $ k \in \{ 1, 2, \ldots, \width \} $ with $ p \neq q $ it holds that
  \begin{equation}\label{eq:lem_probability_expectation_abs_normal}
    \E \bigl[ \abs{ \mf W_{p,k} \mf W_{q,k} } \bigr]
    = \E \bigl[ \abs{ \mf W_{p,k} } \bigr] \E \bigl[ \abs{ \mf W_{q,k} } \bigr]
    = \bigl( \E \bigl[ \abs{ \mf W_{1,1} } \bigr] \bigr)^2
    = \Cvar \bigl( \E \bigl[ \abs{ \nicefrac{1}{\sqrt{\Cvar}} \mf W_{1,1} } \bigr] \bigr)^2
    = \tfrac{2 \Cvar}{\pi}
    \le \Cvar.
  \end{equation}
  \Moreover the assumption that
  $ \sqrt{\nicefrac{1}{\cvar}} W_1(0), \allowbreak
    \sqrt{\nicefrac{1}{\cvar}} W_2(0), \allowbreak \ldots, \allowbreak
    \sqrt{\nicefrac{1}{\cvar}} W_{\width}(0), \allowbreak
    \sqrt{\nicefrac{1}{\Cvar}} \mf W_1, \allowbreak
    \sqrt{\nicefrac{1}{\Cvar}} \mf W_2, \allowbreak \ldots, \allowbreak
    \sqrt{\nicefrac{1}{\Cvar}} \mf W_{\outdim} $
  are independent \proves that for all
  $ i \in \{ 1, 2, \ldots, \ndata \} $, $ k \in \{ 1, 2, \ldots, \width \} $, $ p, q \in \{ 1, 2, \ldots, \outdim \} $
  it holds that $ \mf W_{p,k} \mf W_{q,k} $ and $ \ind_{ \{ \abs{ \scalprod{ W_k(0) } { \indata_i } } \in \, \mc U_R \} } $
  are independent.
  Combining this, \cref{eq:lem_probability_expectation_ind}, \cref{eq:lem_probability_expectation_abs_normal}, and
  the fact that for all $ k \in \{ 1, 2, \ldots, \width \} $, $ p \in \{ 1, 2, \ldots, \outdim \} $ it holds that
  $ \mf W_{p,k} $ is a centered normal random variable with $ \Var[ \mf W_{p,k} ] = \Cvar $
  \proves that
  \begin{equation}
  \begin{split}
    \E \biggl[ \smallsum\limits_{i=1}^{\ndata} \smallsum\limits_{p,q=1}^{\outdim} \smallsum\limits_{k=1}^{\width}
      \abs{ \mf W_{p,k} \mf W_{q,k} }
      \ind_{ \{ \abs{ \scalprod{ W_k(0) } { \indata_i } } \in \, \mc U_R \} } \biggr]
    &= \smallsum\limits_{i=1}^{\ndata} \smallsum\limits_{p,q=1}^{\outdim} \smallsum\limits_{k=1}^{\width}
      \E \bigl[ \abs{ \mf W_{p,k} \mf W_{q,k} } \bigr]
      \E \bigl[ \ind_{ \{ \abs{ \scalprod{ W_k(0) } { \indata_i } } \in \, \mc U_R \} } \bigr] \\
    &\le \smallsum\limits_{i=1}^{\ndata} \smallsum\limits_{p,q=1}^{\outdim} \smallsum\limits_{k=1}^{\width}
      \Cvar \tfrac{\sqrt{2} \nbp R}{\sqrt{\pi \cvar} \cmin}
    = \frac{\sqrt{2} \Cvar \width \nbp \outdim^2 \ndata R}{\sqrt{\pi \cvar} \cmin}.
  \end{split}
  \end{equation}
  The Markov inequality thus \proves that
  \begin{equation}
  \label{eq:lem_probability_A_5}
  \begin{split}
    \P ( A_5 )
    &= \P \biggl( \smallsum\limits_{i=1}^{\ndata} \smallsum\limits_{p,q=1}^{\outdim} \smallsum\limits_{k=1}^{\width}
      \abs{ \mf W_{p,k} \mf W_{q,k} }
      \ind_{ \{ \abs{ \scalprod{ W_k(0) } { \indata_i } } \in \, \mc U_R \} }
      \le \frac{\sqrt{2} \Cvar \width \nbp \outdim^2 \ndata R}{\sqrt{\pi \cvar} \cmin \poe} \biggr) \\
    &\ge 1 - \P \biggl( \smallsum\limits_{i=1}^{\ndata} \smallsum\limits_{p,q=1}^{\outdim} \smallsum\limits_{k=1}^{\width}
      \abs{ \mf W_{p,k} \mf W_{q,k} }
      \ind_{ \{ \abs{ \scalprod{ W_k(0) } { \indata_i } } \in \, \mc U_R \} }
      \ge \frac{\sqrt{2} \Cvar \width \nbp \outdim^2 \ndata R}{\sqrt{\pi \cvar} \cmin \poe} \biggr) \\
    &\ge 1 - \frac{\E \bigl[ \smallsum_{i=1}^{\ndata} \smallsum_{p,q=1}^{\outdim} \smallsum_{k=1}^{\width} \abs{ \mf W_{p,k} \mf W_{q,k} }
      \ind_{ \{ \abs{ \scalprod{ W_k(0) } { \indata_i } } \in \, \mc U_R \} } \bigr] }
      {\sqrt{2} \Cvar \width \nbp \outdim^2 \ndata R (\sqrt{\pi \cvar} \cmin \poe)^{-1}}
    \ge 1 - \poe.
  \end{split}
  \end{equation}
  Next \nobs that the Markov inequality, \cref{eq:lem_probability_expectation_abs_normal}, and
  the fact that for all $ p \in \{ 1, 2, \ldots, \outdim \} $, $ k \in \{ 1, 2, \ldots, \width \} $ it holds that
  $ \mf W_{p,k} $ is a centered normal random variable with $ \Var[ \mf W_{p,k} ] = \Cvar $ \prove that
  \begin{equation}
  \label{eq:lem_probability_A_6}
  \begin{split}
    \P ( A_6 )
    &= \P \biggl( \smallbigcap\limits_{p=1}^{\outdim} \Bigl \{ \smallsum\limits_{q=1}^{\outdim} \smallsum\limits_{k=1}^{\width}
      \abs{ \mf W_{p,k} \mf W_{q,k} } \le \tfrac{\Cvar \width \outdim^2}{\poe} \Bigr\} \biggr)
    = 1 - \P \biggl( \smallbigcup\limits_{p=1}^{\outdim} \Bigl \{ \smallsum\limits_{q=1}^{\outdim} \smallsum\limits_{k=1}^{\width}
      \abs{ \mf W_{p,k} \mf W_{q,k} } \ge \tfrac{\Cvar \width \outdim^2}{\poe} \Bigr\} \biggr) \\
    &\ge 1 - \smallsum\limits_{p=1}^{\outdim} \P \biggl( \smallsum\limits_{q=1}^{\outdim} \smallsum\limits_{k=1}^{\width}
      \abs{ \mf W_{p,k} \mf W_{q,k} } \ge \tfrac{\Cvar \width \outdim^2}{\poe} \biggr)
    \ge 1 - \smallsum\limits_{p=1}^{\outdim} \displaystyle \frac{\E \bigl[ \smallsum\nolimits_{q=1}^{\outdim} \smallsum\nolimits_{k=1}^{\width}
      \abs{ \mf W_{p,k} \mf W_{q,k} } \bigr]}{\Cvar \width \outdim^2 \poe^{-1}} \\
    &= 1 - \tfrac{\poe}{\Cvar \width \outdim^2} \smallsum\limits_{p,q=1}^{\outdim} \smallsum\limits_{k=1}^{\width}
      \E \bigl[ \abs{ \mf W_{p,k} \mf W_{q,k} } \bigr]
    \ge 1 - \tfrac{\poe}{\Cvar \width \outdim^2} \smallsum\limits_{p,q=1}^{\outdim} \smallsum\limits_{k=1}^{\width} \Cvar
    = 1 - \poe.
  \end{split}
  \end{equation}
  \Moreover \cref{lem:error_at_initialization} (applied with
  $ \poe \is \poe $
  in the notation of \cref{lem:error_at_initialization}) and
  \cref{lem:probability_A_W} (applied with
  $ \poe \is \poe $
  in the notation of \cref{lem:probability_A_W})
  \prove that
  \begin{equation}
  \label{eq:lem_probability_A_1_2}
    \P( A_1 ) \ge 1 - \poe
    \qqandqq
    \P( A_2 ) \ge 1 - \poe.
  \end{equation}
  \Moreover the assumption that
  $ \width \ge \tfrac{32 (1+\cmax^2) \outdim \ndata}{\lambdazero}
    \ln( \frac{2 \outdim^2 \ndata^2}{\poe} ) \max \bigl\{
    \frac{4 (1+\cmax^2) \eucnorm{\slope}^4 \outdim \ndata}{\lambdazero}, \cderiv \bigr\} $ and
  \cref{lem:probability_A_gramG} (applied with
  $ \poe \is \poe $
  in the notation of \cref{lem:probability_A_gramG})
  \prove that
  \begin{equation}
  \label{eq:lem_probability_A_3}
    \P( A_3 ) \ge 1 - \poe.
  \end{equation}
  \Moreover the assumption that
  $ R \le \tfrac{ \sqrt{\pi \cvar} \cmin \poe \lambdazero }
    { 16 \sqrt{2} (1+\cmax^2) \cderiv^2 \nbp \outdim^2 \ndata^2 } $ and
  \cref{lem:probability_A_ind} (applied with
  $ \poe \is \poe $, $ R \is R $
  in the notation of \cref{lem:probability_A_ind})
  \prove that
  \begin{equation}
  \label{eq:lem_probability_A_4}
    \P( A_4 ) \ge 1 - \poe.
  \end{equation}
  Combining this \cref{eq:lem_probability_A_5}, \cref{eq:lem_probability_A_6},
  \cref{eq:lem_probability_A_1_2}, and \cref{eq:lem_probability_A_3}
  \proves that for all $ i \in \{ 1, 2, \ldots, 6 \} $
  it holds that $ \P ( A_i ) \ge 1 - \poe $.
  \Hence that
  \begin{equation}
    \textstyle
    \P \biggl( \bigcap\limits_{i=1}^6 A_i \biggr)
    = 1 - \P \biggl( \bigcup\limits_{i=1}^6 \bigl( \Omega \backslash A_i \bigr) \biggr)
    \ge 1 - \sum\limits_{i=1}^6 \P \bigl( \Omega \backslash A_i \bigr)
    = 1 - \sum\limits_{i=1}^6 \bigl( 1 - \P \bigl( A_i \bigr) \bigr)
    \ge 1 - \sum\limits_{i=1}^6 \poe
    = 1 - 6 \poe.
  \end{equation}
\end{cproof}

\cfclear
\begin{proposition}
  \label{prop:main_proposition}
  Assume \cref{setting:gradient_descent},
  assume $ \lambdazero \in (0,\infty) $,
  let $ \poe \in (0,1) $, $ R \in (0,\infty) $ satisfy that
  $ R \le \frac{\sqrt{\pi \cvar} \cmin \poe \lambdazero}{8 \sqrt{2} (1+\cmax^2) \cderiv^2 \nbp \outdim^2 (\sqrt{\ndata} + \ndata) \ndata} $,
  let $ A_1, A_2, A_3, A_4, A_5, A_6 \in \mc F $ satisfy
  \begin{equation}
  \begin{split}
    A_1 &= \Bigl \{ \eucnorm{ \fnetvalue(0) - \foutdata }^2 \le
      \bigl( \cvar \Cvar \width \outdim \eucnorm{\slope}^2 \eucnorm{ \findata }^2
      + \Cvar \width \outdim \eucnorm{\yinter}^2 \ndata + \eucnorm{ \foutdata }^2 \bigr) \poe^{-1} \Bigr \}, \\
    A_2 &= \smallbigcap\limits_{p=1}^{\outdim} \smallbigcap\limits_{k=1}^{\width} \Bigl\{ \abs{\mf W_{p,k}}^2
      \le 2 \Cvar \ln ( \tfrac{2 \outdim \width}{\poe} ) \Bigr\}, \qquad
    A_3 = \smallbigcap\limits_{i=1}^{\outdim \ndata} \smallbigcap\limits_{j=1}^{\outdim \ndata} \bigl\{
      \abs{ \gramG_{i,j} (0) - \gramGinf_{i,j} } \le \frac{\Cvar \width \lambdazero}{4 \outdim \ndata} \bigr\}, \\
    A_4 &= \Bigl\{ \smallsum\limits_{i,j=1}^{\ndata} \smallsum\limits_{p,q=1}^{\outdim} \smallsum\limits_{k=1}^{\width}
      \abs{ \mf W_{p,k} \mf W_{q,k} } \bigl(
    \ind_{ \{ \abs{ \scalprod{ W_k(0) } { \indata_i } } \in \, \mc U_R \} }
    + \ind_{ \{ \abs{ \scalprod{ W_k(0) } { \indata_j } } \in \, \mc U_R \} } \bigr)
    \le \tfrac{\Cvar \width \lambdazero}{8 (1+\cmax^2) \cderiv^2} \Bigr\}, \\
    A_5 &= \Bigl\{ \smallsum\limits_{i=1}^{\ndata} \smallsum\limits_{p,q=1}^{\outdim} \smallsum\limits_{k=1}^{\width}
      \abs{ \mf W_{p,k} \mf W_{q,k} }
      \ind_{ \{ \abs{ \scalprod{ W_k(0) } { \indata_i } } \in \, \mc U_R \} }
      \le \frac{\sqrt{2} \Cvar \width \nbp \outdim^2 \ndata R}{\sqrt{\pi \cvar} \cmin \poe} \Bigr\}, \quad \text{and} \\
    A_6 &= \smallbigcap\limits_{p=1}^{\outdim} \Bigl \{ \smallsum\limits_{q=1}^{\outdim} \smallsum\limits_{k=1}^{\width}
      \abs{ \mf W_{p,k} \mf W_{q,k} } \le \tfrac{\Cvar \width \outdim^2}{\poe} \Bigr\},
  \end{split}
  \end{equation}
  assume
  $ \lr < \min \bigl\{ \tfrac{\Cvar \width \lambdazero}{8 ( \Cvar \width \outdim^2 (1 + \cmax^2) \cderiv^2 \poe^{-1} + 1 )^2 \outdim \ndata},
    \tfrac{\ndata}{\Cvar \width \lambdazero} \bigr\} $,
  assume
  \begin{equation}
  \label{eq:thm_assumption_m}
    \tfrac{1}{\width} \ln \bigl( \tfrac{2 \outdim \width}{\poe} \bigr)
      \le \frac{ \poe \Cvar \width \lambdazero^2 R^2 }
      { 32 (1+\cmax^2)^2 \cderiv^2
      ( \cvar \Cvar \width \outdim \eucnorm{\slope}^2 \eucnorm{ \findata }^2
        + \Cvar \width \outdim \eucnorm{\yinter}^2 \ndata + \eucnorm{ \foutdata }^2 ) \outdim \ndata },
  \end{equation}
  and let $ \omega \in (\bigcap_{i=1}^6 A_i) $ \cfload.
  Then
  \begin{enumerate}[label=(\roman{*})]
    \item \label{item:thm_1} it holds for all $ \iteration \in \N_0 $, $ k \in \{ 1, 2, \ldots, \width \} $ that
      $ \cmax \eucnorm{ W_k(\iteration,\omega) - W_k(0,\omega) } + \abs{ B_k(\iteration,\omega) - B_k(0,\omega) } \le R $ and
    \item \label{item:thm_3} it holds for all $ \iteration \in \N_0 $ that
      $ \eucnorm{ \fnetvalue(\iteration,\omega) - \foutdata }^2 \le \bigl( 1 - \frac{\lr \Cvar \width \lambdazero}{\ndata} \bigr)^{\iteration}
        \eucnorm{ \fnetvalue(0,\omega) - \foutdata }^2 $.
  \end{enumerate}
\end{proposition}

\begin{cproof}{prop:main_proposition}
  Throughout this proof let
  $ I = ( I_i )_{ i \in \{ 1, 2, \ldots, \outdim \ndata \} } \colon \N_0 \to \R^{\outdim \ndata} $,
  $ J = ( J_i )_{ i \in \{ 1, 2, \ldots, \outdim \ndata \} } \colon \allowbreak \N_0 \to \R^{\outdim \ndata} $,
  $ P = ( P_{i,j} )_{ (i,j) \in \{ 1, 2, \ldots, \outdim \ndata \}^2 } \colon \N_0 \to \R^{(\outdim \ndata) \times (\outdim \ndata)} $, and
  $ Q = ( Q_{i,j} )_{ (i,j) \in \{ 1, 2, \ldots, \outdim \ndata \}^2 } \colon \N_0 \to \R^{(\outdim \ndata) \times (\outdim \ndata)} $ satisfy
  for all $ i, j \in \{ 1, 2, \ldots, \ndata \} $, $ p, q \in \{ 1, 2, \ldots, \outdim \}$, $ \iteration \in \N_0 $ that
  \begin{equation}\label{eq:thm_I}
  \begin{split}
    I_{ (i-1) \outdim + p } (\iteration)
    &= \smallsum\limits_{k=1}^{\width} \mf W_{p,k} (\omega) \Bigl[
      \act \bigl( \scalprod{ W_k(\iteration+1,\omega) }{ \indata_i } + B_k(\iteration+1,\omega) \bigr) \\
    &\quad - \act \bigl( \scalprod{ W_k(\iteration,\omega) }{ \indata_i } + B_k(\iteration,\omega) \bigr) \Bigr]
      \ind_{ \{ \abs{ \scalprod{ W_k(0) } { \indata_i } } \not\in \, \mc U_R \} } (\omega),
  \end{split}
  \end{equation}
  \begin{equation}\label{eq:thm_J}
  \begin{split}
    J_{ (i-1) \outdim + p } (\iteration)
    &= \smallsum\limits_{k=1}^{\width} \mf W_{p,k} (\omega) \Bigl[
      \act \bigl( \scalprod{ W_k(\iteration+1,\omega) }{ \indata_i } + B_k(\iteration+1,\omega) \bigr) \\
    &\quad - \act \bigl( \scalprod{ W_k(\iteration,\omega) }{ \indata_i } + B_k(\iteration,\omega) \bigr) \Bigr]
      \ind_{ \{ \abs{ \scalprod{ W_k(0) } { \indata_i } } \in \, \mc U_R \} } (\omega),
  \end{split}
  \end{equation}
  \begin{equation}
  \label{eq:thm_P}
  \begin{split}
    P_{ (i-1) \outdim + p, (j-1) \outdim + q } (\iteration)
    &= \bigl( 1 + \scalprod{ \indata_i }{ \indata_j } \bigr) \smallsum\limits_{k=1}^{\width}
      \mf W_{p,k} (\omega) \mf W_{q,k} (\omega)
      \deriv \bigl( \scalprod{ W_k (\iteration,\omega) }{ \indata_i } + B_k(\iteration,\omega) \bigr) \\
    &\quad \times \deriv \bigl( \scalprod{ W_k(\iteration,\omega) }{ \indata_j } + B_k(\iteration,\omega) \bigr)
      \ind_{ \{ \abs{ \scalprod{ W_k(0) } { \indata_i } } \in \, \mc U_R \} } (\omega),
  \end{split}
  \end{equation}
  and $ Q_{ (i-1) \outdim + p, (j-1) \outdim + q } = \ind_{ \{ p \} } (q) $
  \cfload.
  \Nobs that \cref{lem:spectral_norm_inequality} (applied for every $ \iteration \in \N_0 $ with
  $ m \is \outdim \ndata $, $ n \is \outdim \ndata $, $ A \is P(\iteration) $
  in the notation of \cref{lem:spectral_norm_inequality}),
  \cref{eq:thm_P},
  the fact that for all $ i, j \in \{ 1, 2, \ldots, \ndata \} $ it holds that
  $ \abs{ 1 + \scalprod{ \indata_i }{ \indata_j } }
    \le 1 + \abs{ \scalprod{ \indata_i }{ \indata_j } }
    \le 1 + \eucnorm{ \indata_i } \eucnorm{ \indata_j }
    \le 1 + \cmax^2 $,
  the fact that for all $ v \in \R $ it holds that $ \abs{ \deriv(v) } \le \cderiv $, and
  the fact that $ \omega \in A_5 $ \prove that for all $ \iteration \in \N_0 $ it holds that
  \begin{equation}
  \label{eq:thm_P_specnorm}
  \begin{split}
    &\specnorm{ P(\iteration) }
    \le \smallsum\limits_{i,j=1}^{\ndata} \smallsum\limits_{p,q=1}^{\outdim} \abs{ P_{ (i-1) \outdim + p, (j-1) \outdim + q } (\iteration) } \\
    &\le \smallsum\limits_{i,j=1}^{\ndata} \smallsum\limits_{p,q=1}^{\outdim} \abs{ 1 + \scalprod{ \indata_i }{ \indata_j } }
      \smallsum\limits_{k=1}^{\width} \abs{ \mf W_{p,k} (\omega) \mf W_{q,k} (\omega) }
      \babs{ \deriv \bigl( \scalprod{ W_k(\iteration,\omega) }{ \indata_i } + B_k(\iteration,\omega) \bigr) \\
    &\quad \times \deriv \bigl( \scalprod{ W_k(\iteration,\omega) }{ \indata_j } + B_k(\iteration,\omega) \bigr) }
      \ind_{ \{ \abs{ \scalprod{ W_k(0) } { \indata_i } } \in \, \mc U_R \} } (\omega) \\
    &\le ( 1 + \cmax^2 ) \cderiv^2 \smallsum\limits_{i,j=1}^{\ndata} \smallsum\limits_{p,q=1}^{\outdim} \smallsum\limits_{k=1}^{\width}
      \abs{ \mf W_{p,k} (\omega) \mf W_{q,k} (\omega) } \ind_{ \{ \abs{ \scalprod{ W_k(0) } { \indata_i } } \in \, \mc U_R \} } (\omega) \\
    &= ( 1 + \cmax^2 ) \cderiv^2 \ndata \smallsum\limits_{i=1}^{\ndata} \smallsum\limits_{p,q=1}^{\outdim} \smallsum\limits_{k=1}^{\width}
      \abs{ \mf W_{p,k} (\omega) \mf W_{q,k} (\omega) } \ind_{ \{ \abs{ \scalprod{ W_k(0) } { \indata_i } } \in \, \mc U_R \} } (\omega)
    \le \displaystyle\frac{\sqrt{2} \Cvar \width (1+\cmax^2) \cderiv^2 \nbp \outdim^2 \ndata^2 R}{\sqrt{\pi \cvar} \cmin \poe}
  \end{split}
  \end{equation}
  \cfload.  
  Next \nobs that \cref{eq:setting_W} and \cref{eq:setting_B} \prove that for all
  $ k \in \{ 1, 2, \ldots, \width \} $, $ i \in \{ 1, 2, \ldots, \ndata \} $, $ \iteration \in \N_0 $ it holds that
  \begin{equation}
  \label{eq:thm_diff_rect_arguments}
  \begin{split}
    &\Bigl[ \scalprod{ W_k(\iteration+1,\omega) }{ \indata_i } + B_k(\iteration+1,\omega) \Bigr]
      - \Bigl[ \scalprod{ W_k(\iteration,\omega) }{ \indata_i } + B_k(\iteration,\omega) \Bigr] \\
    &= \bscalprod{ W_k(\iteration+1,\omega) - W_k(\iteration,\omega) }{ \indata_i }
      + B_k(\iteration+1,\omega) - B_k(\iteration,\omega) \\
    &= \bbscalprod{ - \frac{2 \lr}{\ndata} \smallsum\limits_{j=1}^{\ndata} \smallsum\limits_{q=1}^{\outdim}
      ( \netvalue_{j}^q (\iteration,\omega) - \outdata_{j}^q ) \mf W_{q,k} (\omega)
      \deriv \bigl( \scalprod{ W_k(\iteration,\omega) }{ \indata_j } + B_k(\iteration,\omega) \bigr) \, \indata_j }{ \indata_i } \\
    &\quad - \frac{2 \lr}{\ndata} \smallsum\limits_{j=1}^{\ndata} \smallsum\limits_{q=1}^{\outdim}
      ( \netvalue_{j}^q (\iteration,\omega) - \outdata_{j}^q ) \mf W_{q,k} (\omega)
      \deriv \bigl( \scalprod{ W_k(\iteration,\omega) }{ \indata_j } + B_k(\iteration,\omega) \bigr) \\
    &= - \frac{2 \lr}{\ndata} \smallsum\limits_{j=1}^{\ndata} \smallsum\limits_{q=1}^{\outdim}
      ( \netvalue_{j}^q (\iteration,\omega) - \outdata_{j}^q ) \mf W_{q,k} (\omega)
      \deriv \bigl( \scalprod{ W_k(\iteration,\omega) }{ \indata_j } + B_k(\iteration,\omega) \bigr)
      \bigl( 1 + \scalprod{ \indata_i }{ \indata_j } \bigr).
  \end{split}
  \end{equation}
  This, the fact that $ \R \ni x \mapsto \act (x) \in \R $ is $ \cderiv $-Lipschitz continuous,
  the fact that for all $ i, j \in \{ 1, 2, \ldots, \ndata \} $ it holds that
  $ \abs{ 1 + \scalprod{ \indata_i }{ \indata_j } }
    \le 1 + \abs{ \scalprod{ \indata_i }{ \indata_j } }
    \le 1 + \eucnorm{ \indata_i } \eucnorm{ \indata_j }
    \le 1 + \cmax^2 $,
  the fact that for all $ v \in \R $ it holds that $ \abs{ \deriv(v) } \le \cderiv $,
  and \cref{lem:eucnorm_inequality} (applied for every $ \iteration \in \N_0 $, $ q \in \{ 1, 2, \ldots, \outdim \} $ with
  $ \ndata \is \ndata $, $ v \is ( \netvalue_j^q (\iteration,\omega) - \outdata_j^q )_{ j \in \{ 1, 2, \ldots, \ndata \} } $
  in the notation of \cref{lem:eucnorm_inequality})
 \prove that for all $ k \in \{ 1, 2, \ldots, \width \} $, $ i \in \{ 1, 2, \ldots, \ndata \} $, $ \iteration \in \N_0 $ it holds that
  \begin{equation}
  \label{eq:thm_abs_rect}
  \begin{split}  
    &\bbabs{ \act \bigl( \scalprod{ W_k(\iteration+1,\omega) }{ \indata_i } + B_k(\iteration+1,\omega) \bigr)
      - \act \bigl( \scalprod{ W_k(\iteration,\omega) }{ \indata_i } + B_k(\iteration,\omega) \bigr) } \\
    &\le \cderiv \bbabs{ \Bigl[ \scalprod{ W_k(\iteration+1,\omega) }{ \indata_i } + B_k(\iteration+1,\omega) \Bigr]
      - \Bigl[ \scalprod{ W_k(\iteration,\omega) }{ \indata_i } + B_k(\iteration,\omega) \Bigr] } \\
    &\le \frac{2 \lr \cderiv}{\ndata} \smallsum\limits_{q=1}^{\outdim} \smallsum\limits_{j=1}^{\ndata} \bbabs{
      ( \netvalue_{j}^q (\iteration,\omega) - \outdata_{j}^q ) \mf W_{q,k} (\omega)
      \deriv \bigl( \scalprod{ W_k(\iteration,\omega) }{ \indata_j } + B_k(\iteration,\omega) \bigr)
      \bigl( 1 + \scalprod{ \indata_i }{ \indata_j } \bigr) } \\
    &\le \frac{2 \lr \cderiv^2}{\ndata} \bigl( 1 + \cmax^2 \bigr) \smallsum\limits_{q=1}^{\outdim} \abs{ \mf W_{q,k} (\omega) }
      \smallsum\limits_{j=1}^{\ndata} \abs{ \netvalue_{j}^q (\iteration,\omega) - \outdata_{j}^q }
    \le \displaystyle\frac{2 \lr ( 1 + \cmax^2 ) \cderiv^2}{\ndata} \smallsum\limits_{q=1}^{\outdim} \abs{ \mf W_{q,k} (\omega) }
      \sqrt{\ndata} \eucnorm{ \fnetvalue(\iteration,\omega) - \foutdata } \\
    &= \frac{2 \lr ( 1 + \cmax^2 ) \cderiv^2}{\sqrt{\ndata}} \eucnorm{ \fnetvalue(\iteration,\omega) - \foutdata }
      \smallsum\limits_{q=1}^{\outdim} \abs{ \mf W_{q,k} (\omega) }.
  \end{split}
  \end{equation}
  Combining this, \cref{lem:eucnorm_inequality} (applied for every $ \iteration \in \N_0 $ with
  $ m \is \outdim \ndata $, $ v \is J(\iteration) $
  in the notation of \cref{lem:eucnorm_inequality}),
  \cref{eq:thm_J}, and the fact that $ \omega \in A_5 $
  \proves that for all $ \iteration \in \N_0 $ it holds that
  \begin{equation}
  \label{eq:thm_J_eucnorm}
  \begin{split}
    &\eucnorm{ J(\iteration) }
    \le \smallsum\limits_{i=1}^{\ndata} \smallsum\limits_{p=1}^{\outdim} \abs{ J_{ (i-1) \outdim + p } (\iteration) } \\
    &\le \smallsum\limits_{i=1}^{\ndata} \smallsum\limits_{p=1}^{\outdim} \smallsum\limits_{k=1}^{\width} \abs{ \mf W_{p,k} (\omega) }
      \bbabs{ \act \bigl( \scalprod{ W_k(\iteration+1,\omega) }{ \indata_i } + B_k(\iteration+1,\omega) \bigr) \\
    &\quad - \act \bigl( \scalprod{ W_k(\iteration,\omega) }{ \indata_i } + B_k(\iteration,\omega) \bigr) }
      \ind_{ \{ \abs{ \scalprod{ W_k(0) }{ \indata_i } } \in \, \mc U_R \} } (\omega) \\
    &\le \frac{2 \lr ( 1 + \cmax^2 ) \cderiv^2}{\sqrt{\ndata}} \eucnorm{ \fnetvalue(\iteration,\omega) - \foutdata }
      \smallsum\limits_{i=1}^{\ndata} \smallsum\limits_{p,q=1}^{\outdim} \smallsum\limits_{k=1}^{\width}
      \abs{ \mf W_{p,k} (\omega) \mf W_{q,k} (\omega) } 
      \ind_{ \{ \abs{ \scalprod{ W_k(0) }{ \indata_i } } \in \, \mc U_R \} } (\omega) \\
    &\le \frac{2 \lr ( 1 + \cmax^2 ) \cderiv^2}{\sqrt{\ndata}} \eucnorm{ \fnetvalue(\iteration,\omega) - \foutdata }
      \frac{\sqrt{2} \Cvar \width \nbp \outdim^2 \ndata R}{\sqrt{\pi \cvar} \cmin \poe}
    = \frac{2 \sqrt{2} \lr \Cvar \width ( 1 + \cmax^2 ) \cderiv^2 \nbp \outdim^2 \sqrt{\ndata} R}{\sqrt{\pi \cvar} \cmin \poe}
      \eucnorm{ \fnetvalue(\iteration,\omega) - \foutdata }.
  \end{split}
  \end{equation}
  Next we claim that for all $ \Iteration \in \N_0 $, $ \iteration \in \{ 0, 1, 2, \ldots, \Iteration \} $,
  $ k \in \{ 1, 2, \ldots, \width \} $ it holds that
  \begin{gather}
  \label{eq:thm_induction}
    \cmax \eucnorm{ W_k(\iteration,\omega) - W_k(0,\omega) } + \abs{ B_k(\iteration,\omega) - B_k(0,\omega) } \le R
    \qquad\text{and} \\
    \eucnorm{ \fnetvalue(\iteration,\omega) - \foutdata }^2 \le \bigl( 1 - \tfrac{\lr \Cvar \width \lambdazero}{\ndata} \bigr)^{\iteration}
      \eucnorm{ \fnetvalue(0,\omega) - \foutdata }^2.
  \end{gather}
  We now prove \cref{eq:thm_induction} by induction on $ \Iteration \in \N_0 $.
  \Nobs that for all $ k \in \{ 1, 2, \ldots, \width \} $ it holds that
  \begin{equation}
    \cmax \eucnorm{ W_k(0,\omega) - W_k(0,\omega) } + \abs{ B_k(0,\omega) - B_k(0,\omega) } = 0 \le R.
  \end{equation}
  This and the fact that
  $ \eucnorm{ \fnetvalue(0,\omega) - \foutdata }^2
    = \bigl( 1 - \tfrac{\lr \Cvar \width \lambdazero}{\ndata} \bigr)^0 \eucnorm{ \fnetvalue(0,\omega) - \foutdata }^2 $
  \prove \cref{eq:thm_induction} in the base case $ \Iteration = 0 $.
  For the induction step $ \N_0 \ni \Iteration \mapsto \Iteration + 1 \in \N $ assume
  for all $ \iteration \in \{ 0, 1, 2, \ldots, \Iteration \} $, $ k \in \{ 1, 2, \ldots, \width \} $ that
  \begin{gather}
  \label{eq:thm_induction_hypothesis}
    \cmax \eucnorm{ W_k(\iteration,\omega) - W_k(0,\omega) } + \abs{ B_k(\iteration,\omega) - B_k(0,\omega) } \le R
    \qquad\text{and} \\
    \eucnorm{ \fnetvalue(\iteration,\omega) - \foutdata }^2 \le \bigl( 1 - \tfrac{\lr \Cvar \width \lambdazero}{\ndata} \bigr)^{\iteration}
      \eucnorm{ \fnetvalue(0,\omega) - \foutdata }^2.
  \end{gather}
  \Nobs that this, the fact that $ \lr < \tfrac{\ndata}{\Cvar \width \lambdazero} $, and
  \cref{item:first_layer_1} and \cref{item:first_layer_2} in \cref{lem:first_layer} (applied with
  $ \poe \is \poe $, $ A \is A_2 $, $ \omega \is \omega $, $ \Iteration \is \Iteration $
  in the notation of \cref{lem:first_layer})
  \prove that for all $ k \in \{ 1, 2, \ldots, \width \} $ it holds that
  \begin{equation}
  \label{eq:thm_W}
    \eucnorm{ W_k(\Iteration+1,\omega) - W_k(0,\omega) } \le \tfrac{4 \cmax \cderiv \eucnorm{ \fnetvalue(0,\omega) - \foutdata }}
      {\width \lambdazero} \bigl( \tfrac{2 \outdim \ndata}{\Cvar} \ln( \tfrac{2 \outdim \width}{\poe} ) \bigr)^{\nicefrac{1}{2}}
  \end{equation}
  and
  \begin{equation}
  \label{eq:thm_B}
    \abs{ B_k(\Iteration+1,\omega) - B_k(0,\omega) } \le \tfrac{4 \cderiv \eucnorm{ \fnetvalue(0,\omega) - \foutdata }}
      {\width \lambdazero} \bigl( \tfrac{2 \outdim \ndata}{\Cvar} \ln( \tfrac{2 \outdim \width}{\poe} ) \bigr)^{\nicefrac{1}{2}}.
  \end{equation}
  \Moreover the fact that $ \omega \in A_1 $ and \cref{eq:thm_assumption_m} \prove that it holds that
  \begin{equation}
  \begin{split}
    \frac{4 \cderiv \eucnorm{ \fnetvalue(0,\omega) - \foutdata }}{\width \lambdazero}
      \Bigl( \tfrac{2 \outdim \ndata}{\Cvar} \ln \bigl( \tfrac{2 \outdim \width}{\poe} \bigr) \Bigr)^{\nicefrac{1}{2}}
    &\le \frac{4 \cderiv}{\lambdazero} \Bigl(
      \bigl( \cvar \Cvar \width \outdim \eucnorm{\slope}^2 \eucnorm{ \findata }^2
      + \Cvar \width \outdim \eucnorm{\yinter}^2 \ndata + \eucnorm{ \foutdata }^2 \bigr)
      \tfrac{2 \outdim \ndata}{\poe \Cvar \width^2} \ln \bigl( \tfrac{2 \outdim \width}{\poe} \bigr) \Bigr)^{\nicefrac{1}{2}} \\
    &\le \frac{4 \cderiv}{\lambdazero} \biggl(
      \frac{ \lambdazero^2 R^2 } { 16 (1+\cmax^2)^2 \cderiv^2} \biggr)^{\nicefrac{1}{2}}
    = \frac{4 \cderiv}{\width \lambdazero}
      \frac{ \lambdazero R }{4 (1+\cmax^2) \cderiv}
    = \frac{R}{1+\cmax^2}.
  \end{split}
  \end{equation}
  Combining this with \cref{eq:thm_W} and \cref{eq:thm_B} \proves that for all $ k \in \{ 1, 2, \ldots, \width \} $ it holds that
  \begin{equation}
  \label{eq:thm_W_K+1}
    \cmax \eucnorm{ W_k(\Iteration+1,\omega) - W_k(0,\omega) } + \abs{ B_k(\Iteration+1,\omega) - B_k(0,\omega) }
    \le (1+\cmax^2) \tfrac{4 \cderiv \eucnorm{ \fnetvalue(0,\omega) - \foutdata }}{\width \lambdazero}
      \Bigl( \tfrac{2 \outdim \ndata}{\Cvar} \ln \bigl( \tfrac{2 \outdim \width}{\poe} \bigr) \Bigr)^{\nicefrac{1}{2}}
    \le R.
  \end{equation}
  To complete the induction step it remains to prove that
  $ \eucnorm{ \fnetvalue(\Iteration+1,\omega) - \foutdata }^2 \le \bigl( 1 - \tfrac{\lr \Cvar \width \lambdazero}{\ndata} \bigr)^{\Iteration + 1}
    \eucnorm{ \fnetvalue(0,\omega) - \foutdata }^2 $.
  \Nobs that the fact that for all $ a, b \in \R $ it holds that $ (a+b)^2 = a^2 + b^2 + 2ab $ \proves that
  \begin{equation}
  \label{eq:thm_f_K+1}
  \begin{split}
    &\eucnorm{ \fnetvalue(\Iteration+1,\omega) - \foutdata }^2 \\
    &= \smallsum\limits_{i=1}^{\ndata} \smallsum\limits_{p=1}^{\outdim} \Bigl( \netvalue_i^p (\Iteration+1,\omega) - \outdata_i^p \Bigr)^2
    = \smallsum\limits_{i=1}^{\ndata} \smallsum\limits_{p=1}^{\outdim}
      \Bigl( \netvalue_i^p (\Iteration+1,\omega) - \netvalue_i^p (\Iteration,\omega)
      + \netvalue_i^p (\Iteration,\omega) - \outdata_i^p \Bigr)^2 \\
    &= \smallsum\limits_{i=1}^{\ndata} \smallsum\limits_{p=1}^{\outdim}
      \Bigl[ \bigl( \netvalue_i^p (\Iteration+1,\omega) - \netvalue_i^p (\Iteration,\omega) \bigr)^2
      + \bigl( \netvalue_i^p (\Iteration,\omega) - \outdata_i^p \bigr)^2
      + 2 \bigl( \netvalue_i^p (\Iteration+1,\omega) - \netvalue_i^p (\Iteration,\omega) \bigr)
      \bigl( \netvalue_i^p (\Iteration,\omega) - \outdata_i^p \bigr) \Bigr] \\
    &= \eucnorm{ \fnetvalue(\Iteration+1,\omega) - \fnetvalue(\Iteration,\omega) }^2
      + \eucnorm{ \fnetvalue(\Iteration,\omega) - \foutdata }^2
      + 2 \smallsum\limits_{i=1}^{\ndata} \smallsum\limits_{p=1}^{\outdim}
      \bigl( \netvalue_i^p (\Iteration+1,\omega) - \netvalue_i^p (\Iteration,\omega) \bigr)
      \bigl( \netvalue_i^p (\Iteration,\omega) - \outdata_i^p \bigr).
  \end{split}
  \end{equation}
  \Moreover \cref{eq:setting_mf_B} \proves that for all
  $ i \in \{ 1, 2, \ldots, \ndata \} $, $ p \in \{ 1, 2, \ldots, \outdim \}$, $ \iteration \in \N_0 $ it holds that
  \begin{equation}
  \label{eq:thm_diff_f}
  \begin{split}
    &\netvalue_i^p (\iteration+1) - \netvalue_i^p (\iteration) \\
    &= \bigl( \functionANN{\act}^p (\Phi(\iteration+1)) \bigr) (\indata_i) - \bigl( \functionANN{\act}^p (\Phi(\iteration)) \bigr) (\indata_i) \\
    &= \biggl[ \smallsum\limits_{k=1}^{\width} \mf W_{p,k} \act \bigl( \scalprod{ W_k(\iteration+1) }{ \indata_i } + B_k(\iteration+1) \bigr)
      + \mf B_p(\iteration+1) \biggr] - \biggl[ \smallsum\limits_{k=1}^{\width} \mf W_{p,k} \act \bigl(
      \scalprod{ W_k(\iteration) }{ \indata_i } + B_k(\iteration) \bigr) + \mf B_p(\iteration) \biggr] \\
    &= \smallsum\limits_{k=1}^{\width} \mf W_{p,k} \Bigl[ \act \bigl( \scalprod{ W_k(\iteration+1) }{ \indata_i } + B_k(\iteration+1) \bigr)
      - \act \bigl( \scalprod{ W_k(\iteration) }{ \indata_i } + B_k(\iteration) \bigr) \Bigr] + \mf B_p(\iteration+1) - \mf B_p(\iteration) \\
    &= \smallsum\limits_{k=1}^{\width} \mf W_{p,k} \Bigl[ \act \bigl( \scalprod{ W_k(\iteration+1) }{ \indata_i } + B_k(\iteration+1) \bigr)
      - \act \bigl( \scalprod{ W_k(\iteration) }{ \indata_i } + B_k(\iteration) \bigr) \Bigr]
      - \displaystyle \frac{2 \lr}{\ndata} \biggl(
      \smallsum\limits_{j=1}^{\ndata} ( \netvalue_{j}^p (\iteration) - \outdata_{j}^p ) \biggr) \\
  \end{split}
  \end{equation}
  \cfload.
  Combining this, \cref{eq:thm_abs_rect}, \cref{lem:eucnorm_inequality} (applied for every $ p \in \{ 1, 2, \ldots, \outdim \} $,
  $ \iteration \in \N_0 $ with
  $ m \is \ndata $, $ v \is ( \netvalue_j^p (\iteration,\omega) - \outdata_j^p )_{ j \in \{ 1, 2, \ldots, \ndata \} } $
  in the notation of \cref{lem:eucnorm_inequality}), and
  the fact that $ \omega \in A_6 $
  \prove that for all $ i \in \{ 1, 2, \ldots, \ndata \} $, $ p \in \{ 1, 2, \ldots, \outdim \} $, $ \iteration \in \N_0 $ it holds that
  \begin{equation}
  \begin{split}
    &\babs{ \netvalue_i^p (\iteration+1,\omega) - \netvalue_i^p (\iteration,\omega) } \\
    &\le \smallsum\limits_{k=1}^{\width} \abs{ \mf W_{p,k} (\omega) }
      \bbabs{ \act \bigl( \scalprod{ W_k(\iteration+1,\omega) }{ \indata_i } + B_k(\iteration+1,\omega) \bigr)
      - \act \bigl( \scalprod{ W_k(\iteration,\omega) }{ \indata_i } + B_k(\iteration,\omega) \bigr) } \\
    &\quad + \displaystyle \frac{2 \lr}{\ndata} \smallsum\limits_{j=1}^{\ndata}
      \abs{ \netvalue_{j}^p (\iteration,\omega) - \outdata_{j}^p } \\
    &\le \frac{2 \lr (1+\cmax^2) \cderiv^2}{\sqrt{\ndata}} \eucnorm{ \fnetvalue(\iteration,\omega) - \foutdata }
      \smallsum\limits_{k=1}^{\width} \smallsum\limits_{q=1}^{\outdim} \abs{ \mf W_{p,k} (\omega) \mf W_{q,k} (\omega) }
      + \displaystyle\frac{2 \lr}{\ndata} \sqrt{\ndata} \eucnorm{ \fnetvalue(\iteration,\omega) - \foutdata } \\
    &= \frac{2 \lr}{\sqrt{\ndata}} \Bigl( \bigl( 1 + \cmax^2 \bigr) \cderiv^2
      \smallsum\limits_{q=1}^{\outdim} \smallsum\limits_{k=1}^{\width} \abs{ \mf W_{p,k} (\omega) \mf W_{q,k} (\omega) }
      + 1 \Bigr) \eucnorm{ \fnetvalue(\iteration,\omega) - \foutdata } \\
    &\le \displaystyle\frac{2 \lr}{\sqrt{\ndata}} \Bigl( \tfrac{\Cvar \width \outdim^2 (1+\cmax^2) \cderiv^2}{\poe} + 1 \Bigr)
      \eucnorm{ \fnetvalue(\iteration,\omega) - \foutdata }.
  \end{split}
  \end{equation}
  Thus, we obtain that
  \begin{equation}
  \label{eq:thm_squared}
  \begin{split}
    \eucnorm{ \fnetvalue(\Iteration+1,\omega) - \fnetvalue(\Iteration,\omega) }^2
    &= \smallsum\limits_{i=1}^{\ndata} \smallsum\limits_{p=1}^{\outdim}
      \babs{ \netvalue_i^p (\Iteration+1,\omega) - \netvalue_i^p (\Iteration,\omega) }^2 \\
    &\le \smallsum\limits_{i=1}^{\ndata} \smallsum\limits_{p=1}^{\outdim} \displaystyle\frac{4 \lr^2}{\ndata} \Bigl(
      \tfrac{\Cvar \width \outdim^2 (1+\cmax^2) \cderiv^2}{\poe} + 1 \Bigr)^2
      \eucnorm{ \fnetvalue(\Iteration,\omega) - \foutdata }^2 \\
    &= 4 \lr^2 \outdim \Bigl( \tfrac{\Cvar \width \outdim^2 (1+\cmax^2) \cderiv^2}{\poe} + 1 \Bigr)^2
      \eucnorm{ \fnetvalue(\Iteration,\omega) - \foutdata }^2.
  \end{split}
  \end{equation}
  Next \nobs that \cref{eq:thm_diff_f}, the fact that it holds that
  $ 1 = \ind_{ \{ \abs{ \scalprod{ W_k(0) }{ \indata_i } } \not\in \, \mc U_R \} } (\omega)
    + \ind_{ \{ \abs{ \scalprod{ W_k(0) }{ \indata_i } } \in \, \mc U_R \} } (\omega) $,
  \cref{eq:thm_I}, and \cref{eq:thm_J} \prove that for all
  $ i \in \{ 1, 2, \ldots, \ndata \} $, $ p \in \{ 1, 2, \ldots, \outdim \}$, $ \iteration \in \N_0 $ it holds that
  \begin{equation}
  \label{eq:thm_diff_f_omega}
  \begin{split}
    &\netvalue_i^p (\iteration+1,\omega) - \netvalue_i^p (\iteration,\omega) \\
    &= \smallsum\limits_{k=1}^{\width} \mf W_{p,k} (\omega) \Bigl[
      \act \bigl( \scalprod{ W_k(\iteration+1,\omega) }{ \indata_i } + B_k(\iteration+1,\omega) \bigr)
      - \act \bigl( \scalprod{ W_k(\iteration,\omega) }{ \indata_i } + B_k(\iteration,\omega) \bigr) \Bigr]
      \ind_{ \{ \abs{ \scalprod{ W_k(0) }{ \indata_i } } \not\in \, \mc U_R \} } (\omega) \\
    &\quad + \smallsum\limits_{k=1}^{\width} \mf W_{p,k} (\omega) \Bigl[
      \act \bigl( \scalprod{ W_k(\iteration+1,\omega) }{ \indata_i } + B_k(\iteration+1,\omega) \bigr)
      - \act \bigl( \scalprod{ W_k(\iteration,\omega) }{ \indata_i } + B_k(\iteration,\omega) \bigr) \Bigr]
      \ind_{ \{ \abs{ \scalprod{ W_k(0) }{ \indata_i } } \in \, \mc U_R \} } (\omega) \\
    &\quad - \frac{2 \lr}{\ndata} \smallsum\limits_{j=1}^{\ndata} ( \netvalue_j^p (\iteration,\omega) - \outdata_j^p ) \\
    &= I_{ (i-1) \outdim + p } (\iteration) + J_{ (i-1) \outdim + p } (\iteration)
      - \frac{2 \lr}{\ndata} \smallsum\limits_{j=1}^{\ndata} ( \netvalue_j^p (\iteration,\omega) - \outdata_j^p ).
  \end{split}
  \end{equation}
  \Moreover the assumption that $ \eucnorm{ B(0) } = 0 $, the Cauchy Schwarz inequality,
  \cref{eq:thm_induction_hypothesis}, and \cref{eq:thm_W_K+1} \prove that
  for all $ \iteration \in \{ 0, 1, 2, \ldots, \Iteration, \Iteration + 1 \} $, $ i \in \{ 1, 2, \ldots, \ndata \} $ it holds that
  \begin{equation}
  \label{eq:thm_1}
  \begin{split}
    \babs{ \scalprod{ W_k(\iteration,\omega) }{ \indata_i } + B_k(\iteration,\omega) - \scalprod{ W_k(0,\omega) }{ \indata_i } }
    &= \babs{ \scalprod{ W_k(\iteration,\omega) - W_k(0,\omega) }{ \indata_i } + B_k(\iteration,\omega) - B_k(0,\omega) } \\
    &\le \babs{ \scalprod{ W_k(\iteration,\omega) - W_k(0,\omega) }{ \indata_i } } + \abs{ B_k(\iteration,\omega) - B_k(0,\omega) } \\
    &\le \eucnorm{ W_k(\iteration,\omega) - W_k(0,\omega) } \eucnorm{ \indata_i } + \abs{ B_k(\iteration,\omega) - B_k(0,\omega) } \\
    &\le \cmax \eucnorm{ W_k(\iteration,\omega) - W_k(0,\omega) } + \abs{ B_k(\iteration,\omega) - B_k(0,\omega) }
    \le R.
  \end{split}
  \end{equation}
  \Moreover for all $ a, b, z \in \R $ with $ z \not\in \mc U_R $, $ \abs{ a - z } \le R $, and $ \abs{ b - z } \le R $
  it holds that there exists $ s \in \{ 1, 2, \ldots, \nbp+1 \} $ such that
  $ a \in ( \bp_{s-1}, \bp_s ) $, $ b \in ( \bp_{s-1}, \bp_s ) $, and
  \begin{equation}
    \act( a ) - \act( b )
    = ( \slope_s a + \yinter_s ) - ( \slope_s b + \yinter_s )
    = ( a - b ) \slope_s
    = ( a - b ) \deriv( b ).
  \end{equation}
  Combining this,
  \cref{eq:thm_diff_rect_arguments}, and \cref{eq:thm_1} \proves that
  for all $ i \in \{ 1, 2, \ldots, \ndata \} $, $ p \in \{ 1, 2, \ldots, \outdim \} $, $ \iteration \in \{ 0, 1, 2, \ldots, \Iteration \} $
  it holds that
  \begin{equation}
  \label{eq:thm_I_i}
  \begin{split}
    I_{ (i-1) \outdim + p } (\iteration)
    &= \smallsum\limits_{k=1}^{\width} \mf W_{p,k} (\omega) \Bigl[
      \act \bigl( \scalprod{ W_k(\iteration + 1,\omega) }{ \indata_i } + B_k(\iteration+1,\omega) \bigr) \\
    &\qquad - \act \bigl( \scalprod{ W_k(\iteration,\omega) }{ \indata_i } + B_k(\iteration,\omega) \bigr) \Bigr]
      \ind_{ \{ \abs{ \scalprod{ W_k(0) }{ \indata_i } } \not\in \, \mc U_R \} } (\omega) \\
    &= \smallsum\limits_{k=1}^{\width} \mf W_{p,k} (\omega)
      \Bigl[ \scalprod{ W_k(\iteration + 1,\omega) }{ \indata_i } + B_k(\iteration+1,\omega) \\
    &\qquad - \scalprod{ W_k(\iteration,\omega) }{ \indata_i } - B_k(\iteration,\omega) \Bigr]
      \deriv \bigl( \scalprod{ W_k(\iteration,\omega) }{ \indata_i } + B_k(\iteration,\omega) \bigr)
      \ind_{ \{ \abs{ \scalprod{ W_k(0) }{ \indata_i } } \not\in \mc U_R \} } (\omega) \\
    &= \smallsum\limits_{k=1}^{\width} \mf W_{p,k} (\omega) \biggl[
      - \displaystyle\frac{2 \lr}{\ndata} \smallsum\limits_{j=1}^{\ndata} \smallsum\limits_{q=1}^{\outdim}
      ( \netvalue_j^q (\iteration,\omega) - \outdata_j^q ) \mf W_{q,k} (\omega)
      \deriv \bigl( \scalprod{ W_k(\iteration,\omega) }{ \indata_j } + B_k(\iteration,\omega) \bigr)
      \bigl( 1 + \scalprod{ \indata_i }{ \indata_j } \bigr) \biggr]\\
    &\qquad \times \deriv \bigl( \scalprod{ W_k(\iteration,\omega) }{ \indata_i } + B_k(\iteration,\omega) \bigr)
      \ind_{ \{ \abs{ \scalprod{ W_k(0) }{ \indata_i } } \not\in \mc U_R \} } (\omega) \\
    &= - \frac{2 \lr}{\ndata} \smallsum\limits_{j=1}^{\ndata} \smallsum\limits_{q=1}^{\outdim}
      ( \netvalue_j^q (\iteration,\omega) - \outdata_j^q )
      \bigl( 1 + \scalprod{ \indata_i }{ \indata_j } \bigr)
      \smallsum\limits_{k=1}^{\width} \mf W_{p,k} (\omega) \mf W_{q,k} (\omega)
      \deriv \bigl( \scalprod{ W_k(\iteration,\omega) }{ \indata_i } + B_k(\iteration,\omega) \bigr) \\
    &\qquad \times \deriv \bigl( \scalprod{ W_k(\iteration,\omega) }{ \indata_j } + B_k(\iteration,\omega) \bigr)
      \ind_{ \{ \abs{ \scalprod{ W_k(0) }{ \indata_i } } \not\in \mc U_R \} } ( \omega ) \\
    &= - \frac{2 \lr}{\ndata} \smallsum\limits_{j=1}^{\ndata} \smallsum\limits_{q=1}^{\outdim}
      ( \netvalue_j^q (\iteration,\omega) - \outdata_j^q )
      \bigl( 1 + \scalprod{ \indata_i }{ \indata_ j } \bigr)
      \smallsum\limits_{k=1}^{\width} \mf W_{p,k} (\omega) \mf W_{q,k} (\omega)
      \deriv \bigl( \scalprod{ W_k(\iteration,\omega) }{ \indata_i } + B_k(\iteration,\omega) \bigr) \\
    &\qquad \times \deriv \bigl( \scalprod{ W_k(\iteration,\omega) }{ \indata_j } + B_k(\iteration,\omega) \bigr)
      \Bigl[ 1 - \ind_{ \{ \abs{ \scalprod{ W_k(0) }{ \indata_i } } \in \mc U_R \} } (\omega) \Bigr] \\
    &= - \frac{2 \lr}{\ndata} \smallsum\limits_{j=1}^{\ndata} \smallsum\limits_{q=1}^{\outdim}
      ( \netvalue_j^q (\iteration,\omega) - \outdata_j^q )
      \bigl( \gramG_{i,j}^{p,q} (\iteration,\omega) - P_{ (i-1) \outdim + p, (j-1) \outdim + q } (\iteration) \bigr).
  \end{split}
  \end{equation}
  Thus, we obtain that for all $ \iteration \in \{ 0, 1, 2, \ldots, \Iteration \} $ it holds that
  \begin{equation}
    I(\iteration) = \bigl( \gramG(\iteration,\omega) - P(\iteration) \bigr)
      \bigl( \fnetvalue(\iteration,\omega) - \foutdata \bigr).
  \end{equation}
  Combining this with \cref{eq:thm_diff_f_omega} and
  the fact that for all $ i, j \in \{ 1, 2, \ldots, \ndata \} $, $ p, q \in \{ 1, 2, \ldots, \outdim \} $ it holds that
  $ Q_{ (i-1) \outdim + p, (j-1) \outdim + q } = \ind_{ \{ p \} } (q) $
  \proves that
  \begin{equation}
  \label{eq:thm_sum}
  \begin{split}
    &2 \smallsum\limits_{i=1}^{\ndata} \smallsum\limits_{p=1}^{\outdim}
      \bigl( \netvalue_i^p (\Iteration+1,\omega) - \netvalue_i^p (\Iteration,\omega) \bigr)
      \bigl( \netvalue_i^p (\Iteration,\omega) - \outdata_i^p \bigr) \\
    &= 2 \smallsum\limits_{i=1}^{\ndata} \smallsum\limits_{p=1}^{\outdim}
      \Bigl[ I_{ (i-1) \outdim + p } (\Iteration) + J_{ (i-1) \outdim + p } (\Iteration)
      - \displaystyle\frac{2 \lr}{\ndata} \smallsum\limits_{j=1}^{\ndata} ( \netvalue_j^p (\Iteration,\omega) - \outdata_j^p ) \Bigr]
      \bigl( \netvalue_i^p (\Iteration,\omega) - \outdata_i^p \bigr) \\
    &= 2 \smallsum\limits_{i=1}^{\ndata} \smallsum\limits_{p=1}^{\outdim} I_{ (i-1) \outdim + p } (\Iteration)
      \bigl( \netvalue_i^p (\Iteration,\omega) - \outdata_i^p \bigr)
      + 2 \smallsum\limits_{i=1}^{\ndata} \smallsum\limits_{p=1}^{\outdim} J_{ (i-1) \outdim + p } (\Iteration)
      \bigl( \netvalue_i^p (\Iteration,\omega) - \outdata_i^p \bigr) \\
    &\quad - \displaystyle\frac{4 \lr}{\ndata} \smallsum\limits_{i,j=1}^{\ndata} \smallsum\limits_{p,q=1}^{\outdim}
      \bigl( \netvalue_j^p (\Iteration,\omega) - \outdata_j^p \bigr) \ind_{ \{ p \} } (q)
      \bigl( \netvalue_i^p (\Iteration,\omega) - \outdata_i^p \bigr) \\
    &= 2 \bscalprod{ I(\Iteration) }{ \fnetvalue(\Iteration,\omega) - \foutdata }
     + 2 \bscalprod{ J(\Iteration) }{ \fnetvalue(\Iteration,\omega) - \foutdata }
      - \frac{4 \lr}{\ndata} \bscalprod{ \fnetvalue(\Iteration,\omega) - \foutdata }
      { Q \bigl (\fnetvalue(\Iteration,\omega) - \foutdata \bigr) } \\	
    &= 2 \bscalprod{ J(\Iteration) }{ \fnetvalue(\Iteration,\omega) - \foutdata }
      - \frac{4 \lr}{\ndata} \bscalprod{ \fnetvalue(\Iteration,\omega) - \foutdata }
        { \bigl( \gramG(\Iteration,\omega) - P (\Iteration) + Q \bigr)
      \bigl( \fnetvalue(\Iteration,\omega) - \foutdata \bigr) }.
  \end{split}
  \end{equation}
  \Moreover \cref{lem:eigenvalues_block_matrix} and the fact that $ Q $ is symmetric
  \prove that $ Q $ is positive semidefinite.
  \Moreover \cref{lem:lambdamin_gramG} (applied with
  $ \poe \is \poe $, $ R \is R $, $ A_1 \is A_3 $, $ A_2 \is A_4 $, $ \iteration \is \Iteration $, $ \omega \is \omega $
  in the notation of \cref{lem:lambdamin_gramG}),
  \cref{eq:thm_induction_hypothesis}, and
  the fact that $ \omega \in A_3 \cap A_4 $
  \prove that
  \begin{equation}
    \lambdamin( \gramG (\Iteration,\omega) ) \ge \nicefrac{\Cvar \width \lambdazero}{2}.
  \end{equation}
  Combining this, \cref{item:lambdamin_1} in \cref{lem:lambdamin_eucnorm} (applied with
  $ \ndata \is \outdim \ndata $, $ M \is \gramG(\Iteration,\omega) $, $ v \is \fnetvalue(\Iteration,\omega) - \foutdata $
  in the notation of \cref{lem:lambdamin_eucnorm}),
  the Cauchy Schwarz inequality, \cref{eq:thm_P_specnorm}, and the fact that $ Q $ is positive semidefinite
  \proves that
  \begin{equation}
  \label{eq:thm_scalprod_gramG}
  \begin{split}
    & \bscalprod{ \fnetvalue(\Iteration,\omega) - \foutdata }{ \bigl( \gramG(\Iteration,\omega) - P (\Iteration)
      + Q \bigr) \bigl( \fnetvalue(\Iteration,\omega) - \foutdata \bigr) } \\
    &= \bscalprod{ \fnetvalue(\Iteration,\omega) - \foutdata }{ \gramG(\Iteration,\omega) ( \fnetvalue(\Iteration,\omega) - \foutdata ) }
      - \bscalprod{ \fnetvalue(\Iteration,\omega) - \foutdata }{ P (\Iteration) ( \fnetvalue(\Iteration,\omega) - \foutdata ) } \\
    &\quad + \bscalprod{ \fnetvalue(\Iteration,\omega) - \foutdata }{ Q ( \fnetvalue(\Iteration,\omega) - \foutdata ) } \\
    &\ge \lambdamin( \gramG(\Iteration,\omega) ) \eucnorm{ \fnetvalue(\Iteration,\omega) - \foutdata }^2
      - \eucnorm{ \fnetvalue(\Iteration,\omega) - \foutdata }
      \eucnorm{ P (\Iteration) ( \fnetvalue(\Iteration,\omega) - \foutdata ) } \\
    &\ge \tfrac{\Cvar \width \lambdazero}{2} \eucnorm{ \fnetvalue(\Iteration,\omega) - \foutdata }^2
      - \eucnorm{ \fnetvalue(\Iteration,\omega) - \foutdata }^2 \specnorm{ P (\Iteration) } \\
    &\ge \biggl( \frac{\Cvar \width \lambdazero}{2} - \frac{\sqrt{2} \Cvar \width (1+\cmax^2) \cderiv^2 \nbp \outdim^2 \ndata^2 R}
      {\sqrt{\pi \cvar} \cmin \poe} \biggr) \eucnorm{ \fnetvalue(\Iteration,\omega) - \foutdata }^2.
  \end{split}
  \end{equation}
  \Moreover the Cauchy Schwarz inequality and \cref{eq:thm_J_eucnorm} \prove that
  \begin{equation}
    2 \bscalprod{ J(\Iteration) }{ \fnetvalue(\Iteration,\omega) - \foutdata }
    \le 2 \eucnorm{ J(\Iteration) } \eucnorm{ \fnetvalue(\Iteration,\omega) - \foutdata }
    \le \frac{4 \sqrt{2} \lr \Cvar \width ( 1 + \cmax^2 ) \cderiv^2 \nbp \outdim^2 \sqrt{\ndata} R}{\sqrt{\pi \cvar} \cmin \poe}
      \eucnorm{ \fnetvalue(\Iteration,\omega) - \foutdata }^2.
  \end{equation}
  Combining this, \cref{eq:thm_f_K+1}, \cref{eq:thm_squared}, \cref{eq:thm_sum}, and \cref{eq:thm_scalprod_gramG} \proves that
  \begin{equation}
  \label{eq:thm_final}
  \begin{split}
    &\eucnorm{ \fnetvalue(\Iteration+1,\omega) - \foutdata }^2 \\
    &= \eucnorm{ \fnetvalue(\Iteration+1,\omega) - \fnetvalue(\Iteration,\omega) }^2
      + \eucnorm{ \fnetvalue(\Iteration,\omega) - \foutdata }^2  \\
    &\quad + 2 \smallsum\limits_{i=1}^{\ndata} \smallsum\limits_{p=1}^{\outdim}
      \bigl( \netvalue_i^p (\Iteration+1,\omega) - \netvalue_i^p (\Iteration,\omega) \bigr)
      \bigl( \netvalue_i^p (\Iteration,\omega) - \outdata_i^p \bigr) \\
    &= \eucnorm{ \fnetvalue(\Iteration+1,\omega) - \fnetvalue(\Iteration,\omega) }^2
      + \eucnorm{ \fnetvalue(\Iteration,\omega) - \foutdata }^2 \\
    &\quad + 2 \bscalprod{ J(\Iteration) }{ \fnetvalue(\Iteration,\omega) - \foutdata }
      - \frac{4 \lr}{\ndata} \bscalprod{ \fnetvalue(\Iteration,\omega) - \foutdata }{ \bigl( \gramG(\Iteration,\omega)
      - P (\Iteration) + Q \bigr) \bigl( \fnetvalue(\Iteration,\omega) - \foutdata \bigr) } \\
    &\le 4 \lr^2 \outdim \Bigl( \tfrac{\Cvar \width \outdim^2 (1+\cmax^2) \cderiv^2}{\poe} + 1 \Bigr)^2
      \eucnorm{ \fnetvalue(\Iteration,\omega) - \foutdata }^2
      + \eucnorm{ \fnetvalue(\Iteration,\omega) - \foutdata }^2 \\
     &\quad + \frac{4 \sqrt{2} \lr \Cvar \width ( 1 + \cmax^2 ) \cderiv^2 \nbp \outdim^2 \sqrt{\ndata} R}{\sqrt{\pi \cvar} \cmin \poe}
      \eucnorm{ \fnetvalue(\Iteration,\omega) - \foutdata }^2 \\
    &\quad - \frac{4 \lr}{\ndata} \biggl( \frac{\Cvar \width \lambdazero}{2}
       - \frac{\sqrt{2} \Cvar \width (1+\cmax^2) \cderiv^2 \nbp \outdim^2 \ndata^2 R}{\sqrt{\pi \cvar} \cmin \poe} \biggr)
      \eucnorm{ \fnetvalue(\Iteration,\omega) - \foutdata }^2 \\
    &= \biggl[ 4 \lr^2 \outdim \Bigl( \tfrac{\Cvar \width \outdim^2 (1+\cmax^2) \cderiv^2}{\poe} + 1 \Bigr)^2 + 1
      + \frac{4 \sqrt{2} \lr \Cvar \width ( 1 + \cmax^2 ) \cderiv^2 \nbp \outdim^2 \sqrt{\ndata} R}{\sqrt{\pi \cvar} \cmin \poe}
      - \frac{2 \lr \Cvar \width \lambdazero}{\ndata} \\
    &\qquad + \frac{4 \sqrt{2} \lr \Cvar \width (1+\cmax^2) \cderiv^2 \nbp \outdim^2 \ndata R}{\sqrt{\pi \cvar} \cmin \poe} \biggr]
      \eucnorm{ \fnetvalue(\Iteration,\omega) - \foutdata }^2 \\
    &= \biggl[ 1 - \frac{2 \lr \Cvar \width \lambdazero}{\ndata}
      + 4 \lr^2 \outdim \Bigl( \tfrac{\Cvar \width \outdim^2 (1+\cmax^2) \cderiv^2}{\poe} + 1 \Bigr)^2
      + \frac{4 \sqrt{2} \lr \Cvar \width (1 + \cmax^2) \cderiv^2 \nbp \outdim^2 ( \sqrt{\ndata} + \ndata ) R}
      {\sqrt{\pi \cvar} \cmin \poe} \biggr]
      \eucnorm{ \fnetvalue(\Iteration,\omega) - \foutdata }^2.
  \end{split}
  \end{equation}
  \Moreover the fact that
  $ \lr \le \frac{\Cvar \width \lambdazero}{8 ( \Cvar \width \outdim^2 (1+\cmax^2) \cderiv^2 \poe^{-1} + 1 )^2 \outdim \ndata} $ and
  $ R \le \frac{\sqrt{\pi \cvar} \cmin \poe \lambdazero}{8 \sqrt{2} (1+\cmax^2) \cderiv^2 \nbp \outdim^2 (\sqrt{\ndata} + \ndata) \ndata} $
  \prove that
  \begin{equation}
    4 \lr \outdim \Bigl( \tfrac{\Cvar \width \outdim^2 (1+\cmax^2) \cderiv^2}{\poe} + 1 \Bigr)^2
      + \frac{4 \sqrt{2} \Cvar \width (1 + \cmax^2) \cderiv^2 \nbp \outdim^2 (\sqrt{\ndata} + \ndata) R}{\sqrt{\pi \cvar} \cmin \poe}
    \le \frac{\Cvar \width \lambdazero}{2 \ndata} + \frac{\Cvar \width \lambdazero}{2 \ndata}
    = \frac{\Cvar \width \lambdazero}{\ndata}.
  \end{equation}
  Combining this with \cref{eq:thm_final} \hence \proves that
  \begin{equation}
    \eucnorm{ \fnetvalue(\Iteration+1,\omega) - \foutdata }^2
    \le \biggl[ 1 - \frac{2 \lr \Cvar \width \lambdazero}{\ndata} + \frac{\lr \Cvar \width \lambdazero}{\ndata} \biggr]
      \eucnorm{ \fnetvalue(\Iteration,\omega) - \foutdata }^2
    = \biggl[ 1 - \frac{\lr \Cvar \width \lambdazero}{\ndata} \biggr] \eucnorm{ \fnetvalue(\Iteration,\omega) - \foutdata }^2.
  \end{equation}
  This and \cref{eq:thm_induction_hypothesis} \hence\prove that
  $ \eucnorm{ \fnetvalue(\Iteration+1,\omega) - \foutdata }^2
    \le \bigl( 1 - \tfrac{\lr \Cvar \width \lambdazero}{\ndata} \bigr)^{\Iteration + 1} \eucnorm{ \fnetvalue(0,\omega) - \foutdata }^2 $.
  Induction thus \proves \cref{eq:thm_induction}.
  \Nobs that \cref{eq:thm_induction} \proves for all $ \iteration \in \N_0 $ that
  $ \eucnorm{ \fnetvalue(\iteration,\omega) - \foutdata }^2
    \le \bigl( 1 - \frac{\lr \Cvar \width \lambdazero}{\ndata} \bigr)^{\iteration} \eucnorm{ \fnetvalue(0,\omega) - \foutdata }^2 $.
  This \proves \cref{item:thm_3}.
\end{cproof}

\subsection{Quantitative probabilistic error analysis for GD optimization algorithms}
\label{subsec:quantitative_probabilistic_error_analysis}

\cfclear
\begin{theorem}
  \label{thm:main_theorem}
  Assume \cref{setting:gradient_descent}, assume $ \lambdazero \in (0,\infty) $, let $ \poe \in (0,1) $,
  assume
  $ \lr < \min \bigl\{ \frac{\Cvar \width \lambdazero}{8(6 \Cvar \width \outdim^2 (1+\cmax^2) \cderiv^2 \poe^{-1} + 1 )^2 \outdim \ndata},
    $ $ \frac{\ndata}{\Cvar \width \lambdazero} \bigr\} $,
  $ \width \ge \tfrac{32 (1+\cmax^2) \outdim \ndata}{\lambdazero}
    \ln( \frac{12 \outdim^2 \ndata^2}{\poe} ) \max \bigl\{
    \frac{4 (1+\cmax^2) \eucnorm{\slope}^4 \outdim \ndata}{\lambdazero}, \cderiv \bigr\} $, and
  \begin{equation}\label{eq:thm:main_theorem_ln}
    \tfrac{1}{\width} \ln \bigl( \tfrac{12 \outdim \width}{\poe} \bigr)
    \le \frac{ \pi \cvar \Cvar \width \cmin^2 \poe^3 \lambdazero^4 }
      { 2^{17} \, 3^3 (1+\cmax^2)^4 \cderiv^6 \nbp^2
      ( \cvar \Cvar \width \outdim \eucnorm{\slope}^2 \eucnorm{ \findata }^2
        + \Cvar \width \outdim \eucnorm{\yinter}^2 \ndata + \eucnorm{ \foutdata }^2 ) \outdim^5 \ndata^5 },
  \end{equation}
  and let $ \loss \colon \network_{ \indim, \width, \outdim } \to [0,\infty) $ satisfy for all $ \Psi \in \network_{ \indim, \width, \outdim } $ that
  $ \loss ( \Psi ) = \tfrac{1}{\ndata} \sum_{i=1}^{\ndata} \eucnorm{ ( \functionANN{\act} (\Psi) ) (\indata_i) - \outdata_i }^2 $ \cfload.
  Then
  \begin{equation}
    \P \Bigl( \Forall \iteration \in \N_0 \colon \loss(\Phi(\iteration))
      \le \bigl( 1 - \tfrac{ \lr \Cvar \width \lambdazero }{ \ndata } \bigr)^{\iteration} \loss(\Phi(0)) \Bigr) \ge 1 - \poe.
  \end{equation}
\end{theorem}

\begin{cproof}{thm:main_theorem}
  Throughout this proof let $ \delta \in (0,1) $ satisfy $ \delta = \nicefrac{\poe}{6} $, let $ R \in (0,\infty) $ satisfy
  \begin{equation}\label{eq:thm:main_theorem_R}
    R = \tfrac{ \sqrt{\pi \cvar} \cmin \delta \lambdazero }{ 16 \sqrt{2} (1+\cmax^2) \cderiv^2 \nbp \outdim^2 \ndata^2 },
  \end{equation}
  and let $ A_1, A_2, A_3, A_4, A_5, A_6 \in \mc F $ satisfy
  \begin{equation}
  \begin{split}
    A_1 &= \Bigl \{ \eucnorm{ \fnetvalue(0) - \foutdata }^2 \le
      \bigl( \cvar \Cvar \width \outdim \eucnorm{\slope}^2 \eucnorm{ \findata }^2
      + \Cvar \width \outdim \eucnorm{\yinter}^2 \ndata + \eucnorm{ \foutdata }^2 \bigr) \delta^{-1} \Bigr \}, \\
    A_2 &= \smallbigcap\limits_{p=1}^{\outdim} \smallbigcap\limits_{k=1}^{\width} \Bigl\{ \abs{\mf W_{p,k}}^2
      \le 2 \Cvar \ln ( \tfrac{2 \outdim \width}{\delta} ) \Bigr\}, \qquad
    A_3 = \smallbigcap\limits_{i=1}^{\outdim \ndata} \smallbigcap\limits_{j=1}^{\outdim \ndata} \bigl\{
      \abs{ \gramG_{i,j} (0) - \gramGinf_{i,j} } \le \frac{\Cvar \width \lambdazero}{4 \outdim \ndata} \bigr\}, \\
    A_4 &= \Bigl\{ \smallsum\limits_{i,j=1}^{\ndata} \smallsum\limits_{p,q=1}^{\outdim} \smallsum\limits_{k=1}^{\width}
      \abs{ \mf W_{p,k} \mf W_{q,k} } \bigl(
    \ind_{ \{ \abs{ \scalprod{ W_k(0) } { \indata_i } } \in \, \mc U_R \} }
    + \ind_{ \{ \abs{ \scalprod{ W_k(0) } { \indata_j } } \in \, \mc U_R \} } \bigr)
    \le \tfrac{\Cvar \width \lambdazero}{8 (1+\cmax^2) \cderiv^2} \Bigr\}, \\
    A_5 &= \Bigl\{ \smallsum\limits_{i=1}^{\ndata} \smallsum\limits_{p,q=1}^{\outdim} \smallsum\limits_{k=1}^{\width}
      \abs{ \mf W_{p,k} \mf W_{q,k} }
      \ind_{ \{ \abs{ \scalprod{ W_k(0) } { \indata_i } } \in \, \mc U_R \} }
      \le \frac{\sqrt{2} \Cvar \width \nbp \outdim^2 \ndata R}{\sqrt{\pi \cvar} \cmin \delta} \Bigr\}, \quad \text{and} \\
    A_6 &= \smallbigcap\limits_{p=1}^{\outdim} \Bigl \{ \smallsum\limits_{q=1}^{\outdim} \smallsum\limits_{k=1}^{\width}
      \abs{ \mf W_{p,k} \mf W_{q,k} } \le \tfrac{\Cvar \width \outdim^2}{\delta} \Bigr\},
  \end{split}
  \end{equation}
  First, \nobs that the fact that $ \poe = 6 \delta $ \proves that
  \begin{equation}
  \label{eq:thm_width}
    \width
    \ge \tfrac{32 (1+\cmax^2) \outdim \ndata}{\lambdazero}
      \ln \bigl( \tfrac{12 \outdim^2 \ndata^2}{\poe} \bigr) \max \Bigl\{
      \tfrac{4 (1+\cmax^2) \eucnorm{\slope}^4 \outdim \ndata}{\lambdazero}, \cderiv \Bigr\}
    = \tfrac{32 (1+\cmax^2) \outdim \ndata}{\lambdazero}
      \ln \bigl( \tfrac{2 \outdim^2 \ndata^2}{\delta} \bigr) \max \Bigl\{
      \tfrac{4 (1+\cmax^2) \eucnorm{\slope}^4 \outdim \ndata}{\lambdazero}, \cderiv \Bigr\}.
  \end{equation}
  \cref{lem:probability_A_cap} (applied with
  $ \poe \is \delta $, $ R \is R $, $ ( A_i )_{ i \in \{ 1, 2, \ldots, 6 \} } \is ( A_i )_{ i \in \{ 1, 2, \ldots, 6 \} } $
  in the notation of \cref{lem:probability_A_cap}) and
  \cref{eq:thm:main_theorem_R}
  \hence \prove that
  \begin{equation}
  \label{eq:thm_probability}
    \P \biggl( \smallbigcap\limits_{i=1}^6 A_i \biggr)
    \ge 1 - 6 \delta
    = 1 - \poe.
  \end{equation}
  Next \nobs that
  the fact that $ \sqrt{m} + \ndata \le 2 \ndata $ \proves that
  \begin{equation}
  \label{eq:thm_R}
    R
    = \tfrac{ \sqrt{\pi \cvar} \cmin \delta \lambdazero }{ 16 \sqrt{2} (1+\cmax^2) \cderiv^2 \nbp \outdim^2 \ndata^2 }
    \le \tfrac{\sqrt{\pi \cvar} \cmin \delta \lambdazero}{8 \sqrt{2} (1+\cmax^2) \cderiv^2 \nbp \outdim^2 (\sqrt{\ndata} + \ndata) \ndata}.
  \end{equation}
  \Moreover \cref{eq:thm:main_theorem_ln} \proves that
  \begin{equation}
  \label{eq:thm_ln}
  \begin{split}
    \tfrac{1}{\width} \ln \bigl( \tfrac{2 \outdim \width}{\delta} \bigr)
    &= \tfrac{1}{\width} \ln \bigl( \tfrac{12 \outdim \width}{\poe} \bigr) \\
    &\le \frac{ \pi \cvar \Cvar \width \cmin^2 \poe^3 \lambdazero^4 }
      { 2^{17} \, 3^3 (1+\cmax^2)^4 \cderiv^6 \nbp^2
      ( \cvar \Cvar \width \outdim \eucnorm{\slope}^2 \eucnorm{ \findata }^2
        + \Cvar \width \outdim \eucnorm{\yinter}^2 \ndata + \eucnorm{ \foutdata }^2 ) \outdim^5 \ndata^5 } \\
    &= \frac{ \pi \cvar \Cvar \width \cmin^2 \delta^3 \lambdazero^4 }
      { 2^{14} (1+\cmax^2)^4 \cderiv^6 \nbp^2
      ( \cvar \Cvar \width \outdim \eucnorm{\slope}^2 \eucnorm{ \findata }^2
        + \Cvar \width \outdim \eucnorm{\yinter}^2 \ndata + \eucnorm{ \foutdata }^2 ) \outdim^5 \ndata^5 } \\
    &= \frac{ \delta \Cvar \width \lambdazero^2 R^2 }
      { 32 (1+\cmax^2)^2 \cderiv^2
      ( \cvar \Cvar \width \outdim \eucnorm{\slope}^2 \eucnorm{ \findata }^2
        + \Cvar \width \outdim \eucnorm{\yinter}^2 \ndata + \eucnorm{ \foutdata }^2 ) \outdim \ndata }.
  \end{split}
  \end{equation}
  \Moreover the fact that $ \poe = 6 \delta $ \proves that
  \begin{equation}
  \label{eq:thm_eta}
    \lr
    < \min \Bigl\{ \frac{\Cvar \width \lambdazero}{8 (6 \Cvar \width \outdim^2 (1+\cmax^2) \cderiv^2 \poe^{-1} + 1 )^2 \outdim \ndata},
      \frac{\ndata}{\Cvar \width \lambdazero} \Bigr\}
    = \min \Bigl\{ \frac{\Cvar \width \lambdazero}{8 (\Cvar \width \outdim^2 (1+\cmax^2) \cderiv^2 \delta^{-1} + 1 )^2 \outdim \ndata},
      \frac{\ndata}{\Cvar \width \lambdazero} \Bigr\}.
  \end{equation}
  Combining this, \cref{eq:thm_R}, \cref{eq:thm_ln}, and \cref{item:thm_3} in \cref{prop:main_proposition} (applied with
  $ \poe \is \delta $, $ R \is R $, $ ( A_i )_{ i \in \{ 1, 2, \ldots, 6 \} } \is ( A_i )_{ i \in \{ 1, 2, \ldots, 6 \} } $
  in the notation of \cref{prop:main_proposition})
  \proves that for all $ \omega \in ( \smallbigcap_{i=1}^6 A_i ) $, $ \iteration \in \N_0 $ it holds that
  \begin{equation}
    \eucnorm{ \fnetvalue(\iteration,\omega) - \foutdata }^2 \le \Bigl( 1 - \tfrac{\lr \Cvar \width \lambdazero}{\ndata} \Bigr)^{\iteration}
      \eucnorm{ \fnetvalue(0,\omega) - \foutdata }^2.
  \end{equation}
  Thus, we obtain for all $ \omega \in ( \smallbigcap_{i=1}^6 A_i ) $, $ \iteration \in \N_0 $ that
  \begin{equation}
  \begin{split}
    \loss( \Phi(\iteration,\omega) )
    &= \tfrac{1}{\ndata} \smallsum\limits_{i=1}^{\ndata}
      \eucnorm{ ( \functionANN{\act}( \Phi(\iteration,\omega) )) ( \indata_i ) - \outdata_i }^2
    = \tfrac{1}{\ndata} \smallsum\limits_{i=1}^{\ndata} \smallsum\limits_{p=1}^{\outdim}
      \abs{ \netvalue_i^p (\iteration,\omega) - \outdata_i^p }^2
    = \tfrac{1}{\ndata} \eucnorm{ \fnetvalue(\iteration,\omega) - \foutdata }^2 \\
    &\le \tfrac{1}{\ndata} \Bigl( 1 - \tfrac{\lr \Cvar \width \lambdazero}{\ndata} \Bigr)^{\iteration}
      \eucnorm{ \fnetvalue(0,\omega) - \foutdata }^2
    = \tfrac{1}{\ndata} \Bigl( 1 - \tfrac{\lr \Cvar \width \lambdazero}{\ndata} \Bigr)^{\iteration}
      \smallsum\limits_{i=1}^{\ndata} \smallsum\limits_{p=1}^{\outdim} \abs{ \netvalue_i^p (0,\omega) - \outdata_i^p }^2 \\
    &= \tfrac{1}{\ndata} \Bigl( 1 - \tfrac{\lr \Cvar \width \lambdazero}{\ndata} \Bigr)^{\iteration}
      \smallsum\limits_{i=1}^{\ndata} \eucnorm{ ( \functionANN{\act}( \Phi(0,\omega) )) ( \indata_i ) - \outdata_i }^2
    = \Bigl( 1 - \tfrac{\lr \Cvar \width \lambdazero}{\ndata} \Bigr)^{\iteration} \loss( \Phi(0,\omega) ).
  \end{split}
  \end{equation}
  Combining this and \cref{eq:thm_probability} \proves that
  \begin{equation}
    \P \Bigl( \Forall \iteration \in \N_0 \colon \loss(\Phi(\iteration))
      \le \bigl( 1 - \tfrac{\lr \Cvar \width \lambdazero}{\ndata} \bigr)^{\iteration} \loss(\Phi(0)) \Bigr) \ge 1 - \poe.
  \end{equation}
\end{cproof}

\cfclear
\begin{definition}[$1$-Clipping functions]
\label{def:clipping_functions}
  \cfconsiderloaded{def:clipping_functions}
  Let $ z \in \R $.
  Then we denote by $ \maxone{z}, \minone{z} \in \R $ the real numbers which satisfy
  \begin{equation}\label{eq:def:rounding_functions:1}
    \maxone{z} = \max\{ 1, z \}
    \qqandqq
    \minone{z} = \min\{ 1, z \}.
  \end{equation}
\end{definition}

\cfclear
\begin{cor}
  \label{cor:assumption_lambdazero_ln}
  Assume \cref{setting:gradient_descent},
  assume $ \lambdazero \in (0,\infty) $,
  let $ \poe \in (0,1) $,
  assume
  \begin{equation}
  \label{eq:cor_assumption_lambdazero_ln}
    \lr < \frac{\poe^2 \min\{ \Cvar \width, (\Cvar \width)^{-1} \} \min\{ \lambdazero, \lambdazero^{-1} \} }
      {1352 \maxone{\cmax^4} \maxone{\cderiv^4} \outdim^4 \ndata}
    \qandq
    \tfrac{1}{\width} \ln \bigl( \tfrac{12 \outdim \width}{\poe} \bigr)
      \le \frac{ \pi \minone{ \cvar } \minone{ \Cvar \width } \cmin^2 \poe^3 \min\{ \lambdazero, \lambdazero^4 \} }
      { 2^{21} \, 3^4 \maxone{ \cmax^{10} } \maxone{ \cout^2 } \maxone{ \cderiv^6 } \maxone{ \eucnorm{\slope}^2 }
      \maxone{ \eucnorm{\yinter}^2 } \nbp^2 \outdim^6 \ndata^6 },
  \end{equation}
  and let $ \loss \colon \network_{ \indim, \width, \outdim } \to [0,\infty) $ satisfy for all $ \Psi \in \network_{ \indim, \width, \outdim } $ that
  $ \loss ( \Psi ) = \tfrac{1}{\ndata} \sum_{i=1}^{\ndata} \eucnorm{ ( \functionANN{\act} (\Psi) ) (\indata_i) - \outdata_i }^2 $ \cfload.
  Then
  \begin{equation}
    \P \Bigl( \Forall \iteration \in \N_0 \colon \loss(\Phi(\iteration))
      \le \bigl( 1 - \tfrac{ \lr \Cvar \width \lambdazero }{ \ndata } \bigr)^{\iteration} \loss(\Phi(0)) \Bigr) \ge 1 - \poe.
  \end{equation}
\end{cor}

\begin{cproof}{cor:assumption_lambdazero_ln}
  First, \nobs that the fact that $ 1 + \cmax^2 \le 2 \max\{ 1, \cmax^2 \} = 2 \maxone{ \cmax^2 } $
  \proves that
  \begin{equation}
  \begin{split}
    &\bigl( 6 \Cvar \width \outdim^2 (1+\cmax^2) \cderiv^2 \poe^{-1} + 1 \bigr)^2 \\
    &= 36 \Cvar^2 \width^2 \outdim^4 (1+\cmax^2)^2 \cderiv^4 \poe^{-2}
      + 12 \Cvar \width \outdim^2 (1+\cmax^2) \cderiv^2 \poe^{-1}
      + 1 \\
    &\le 36 \maxone{ \Cvar^2 \width^2 } \outdim^4 ( 2 \maxone{ \cmax^2 } )^2 \maxone{ \cderiv^4 } \poe^{-2}
      + 24 \maxone{ \Cvar \width } \outdim^2 \maxone{ \cmax^2 } \maxone{ \cderiv^2 } \poe^{-1}
      + 1 \\
    &\le 144 \maxone{ \Cvar^2 \width^2 } \outdim^4 \maxone{ \cmax^4 } \maxone{ \cderiv^4 } \poe^{-2}
      + 24 \maxone{ \Cvar^2 \width^2 } \outdim^4 \maxone{ \cmax^4 } \maxone { \cderiv^4 } \poe^{-2}
      + \maxone{ \Cvar^2 \width^2 } \outdim^4 \maxone{ \cmax^4 } \maxone{ \cderiv^4 } \poe^{-2} \\
    &= 169 \maxone{ \Cvar^2 \width^2 } \outdim^4 \maxone{ \cmax^4 } \maxone{ \cderiv^4 } \poe^{-2}.
  \end{split}
  \end{equation}
  Combining this with \cref{eq:cor_assumption_lambdazero_ln} and
  the fact that for all $ z \in (0,\infty) $ it holds that $ \tfrac{z}{ \maxone{z^2} } = \min\{ z, z^{-1} \} $
  \proves that
  \begin{equation}
  \label{eq:cor_assumption_lambdazero_ln_eta}
  \begin{split}
    \min \biggl\{ \frac{\Cvar \width \lambdazero}{8 (6 \Cvar \width \outdim^2 (1+\cmax^2) \cderiv^2 \poe^{-1} + 1 )^2 \ndata},
      \frac{\ndata}{\Cvar \width \lambdazero} \biggr\}
    &\ge \min \biggl\{ \frac{\poe^2 \Cvar \width \lambdazero }{1352 \maxone{ \Cvar^2 \width^2 } \maxone{ \cmax^4 } \maxone{ \cderiv^4 } \outdim^4 \ndata},
      \frac{\ndata}{\Cvar \width \lambdazero} \biggr\} \\
    &= \min \biggl\{ \frac{\poe^2 \min\{ \Cvar\width, ( \Cvar \width )^{-1} \} \lambdazero }
      {1352 \maxone{ \cmax^4 } \maxone{ \cderiv^4 } \outdim^4 \ndata},
      \frac{\ndata}{\Cvar \width \lambdazero} \biggr\} \\
    &\ge \min \biggl\{ \frac{\poe^2 \min\{ \Cvar\width, ( \Cvar \width )^{-1} \} \lambdazero }
      {1352 \maxone{ \cmax^4 } \maxone{ \cderiv^4 } \outdim^4 \ndata},
      \frac{\poe^2 \min\{ \Cvar\width, ( \Cvar \width )^{-1} \} }{1352 \maxone{ \cmax^4 } \maxone{ \cderiv^4 } \outdim^4 \ndata \lambdazero} \biggr\} \\
    &= \frac{\poe^2 \min\{ \Cvar\width, ( \Cvar \width )^{-1} \} \min\{ \lambdazero, \lambdazero^{-1} \} }
      {1352 \maxone{ \cmax^4 } \maxone{ \cderiv^4 } \outdim^4 \ndata}
    > \lr.
  \end{split}
  \end{equation}
  Next \nobs that
  the fact that $ \eucnorm{ \findata }^2 = \smallsum_{i=1}^{\ndata} \eucnorm{ \indata_i }^2 \le \cmax^2 \ndata $ and
  the fact that $ \eucnorm{ \foutdata }^2 = \smallsum_{i=1}^{\ndata} \eucnorm{ \outdata_i }^2 \le \cout^2 \ndata \le \cout^2 \outdim \ndata $
  \prove that
  \begin{equation}\label{eq:cor_assumption_lambdazero_ln_sum}
  \begin{split}
    \cvar \Cvar \width \outdim \eucnorm{\slope}^2 \eucnorm{ \findata }^2
      + \Cvar \width \outdim \eucnorm{\yinter}^2 \ndata + \eucnorm{ \foutdata }^2
    &\le \cvar \Cvar \width \outdim \eucnorm{\slope}^2 \cmax^2 \ndata + \Cvar \width \outdim \eucnorm{\yinter}^2 \ndata + \cout^2 \outdim \ndata \\
    &= \bigl( \cvar \Cvar \width \cmax^2 \eucnorm{\slope}^2 + \Cvar \width \eucnorm{\yinter}^2 + \cout^2 \bigr) \outdim \ndata.
  \end{split}
  \end{equation}
  \Moreover the fact that for all $ z \in [0,\infty) $ it holds that $ \minone{z} \le 1 \le \maxone{z} $
  \proves that
  \begin{equation}
  \begin{split}
    \frac{ \cvar \Cvar \width \cmax^2 \eucnorm{\slope}^2 + \Cvar \width \eucnorm{\yinter}^2 + \cout^2 }{ \cvar \Cvar \width }
    &= \cmax^2 \eucnorm{\slope}^2 + \frac{\eucnorm{\yinter}^2}{\cvar} + \frac{\cout^2}{\cvar \Cvar \width}
    \le \frac{ \maxone{ \cmax^2 } \maxone{ \eucnorm{\slope}^2 } }{ \minone{ \cvar } \minone{ \Cvar \width } }
      + \frac{ \maxone{ \eucnorm{\yinter}^2 } }{ \minone{ \cvar } \minone{ \Cvar \width } }
      + \frac{ \maxone{ \cout^2 } }{ \minone{ \cvar } \minone{ \Cvar \width } } \\
    &\le \frac{ 3 \maxone{ \cmax^2 } \maxone{ \cout^2 } \maxone{ \eucnorm{\slope}^2 } \maxone{ \eucnorm{\yinter}^2 } }
      { \minone{ \cvar } \minone{ \Cvar \width } }.
  \end{split}
  \end{equation}
  Combining this, \cref{eq:cor_assumption_lambdazero_ln}, \cref{eq:cor_assumption_lambdazero_ln_sum}, and
  the fact that $ 1 + \cmax^2 \le 2 \max\{ 1, \cmax^2 \} = 2 \maxone{ \cmax^2 } $
  \proves that
  \begin{equation}
  \label{eq:cor_assumption_lambdazero_ln_ln}
  \begin{split}
    &\frac{ \pi \cvar \Cvar \width \cmin^2 \poe^3 \lambdazero^4 }
      { 2^{17} \, 3^3 (1+\cmax^2)^4 \cderiv^6 \nbp^2
      ( \cvar \Cvar \width \outdim \eucnorm{\slope}^2 \eucnorm{ \findata }^2
        + \Cvar \width \outdim \eucnorm{\yinter}^2 \ndata + \eucnorm{ \foutdata }^2 ) \outdim^5 \ndata^5 } \\
    &\ge \frac{ \pi \cvar \Cvar \width \cmin^2 \poe^3 \lambdazero^4 }
      { 2^{17} \, 3^3 ( 2 \maxone{ \cmax^2 } )^4 \cderiv^6 \nbp^2
      ( \cvar \Cvar \width \cmax^2 \eucnorm{\slope}^2 + \Cvar \width \eucnorm{\yinter}^2 + \cout^2 ) \outdim^6 \ndata^6 } \\
    &\ge \frac{ \pi \minone{ \cvar } \minone{ \Cvar \width } \cmin^2 \poe^3 \lambdazero^4 }
      { 2^{21} \, 3^4 \maxone{ \cmax^8 } \cderiv^6 \nbp^2
      \maxone{ \cmax^2 } \maxone{ \cout^2 } \maxone{ \eucnorm{\slope}^2 } \maxone{ \eucnorm{\yinter}^2 } \outdim^6 \ndata^6 } \\
    &\ge \frac{ \pi \minone{ \cvar } \minone{ \Cvar \width } \cmin^2 \poe^3 \min\{ \lambdazero, \lambdazero^4 \} }
      { 2^{21} \, 3^4 \maxone{ \cmax^{10} } \maxone{ \cout^2 } \maxone{ \cderiv^6 } \maxone{ \eucnorm{\slope}^2 }
      \maxone{ \eucnorm{\yinter}^2 } \nbp^2 \outdim^6 \ndata^6 }
    \ge \tfrac{1}{\width} \ln \bigl( \tfrac{12 \outdim \width}{\poe} \bigr).
  \end{split}
  \end{equation}
  \Moreover \cref{eq:cor_assumption_lambdazero_ln} and
  the fact that for all $ z \in (12,\infty) $ it holds that $ \ln(z) \ge 1 $
  \prove that
  \begin{equation}
    \width
    \ge \ln \bigl( \tfrac{12 \outdim \width}{\poe} \bigr)
      \frac{ 2^{21} \, 3^4 \maxone{ \cmax^{10} } \maxone{ \cout^2 } \maxone{ \cderiv^6 } \maxone{ \eucnorm{\slope}^2 }
      \maxone{ \eucnorm{\yinter}^2 } \nbp^2 \outdim^6 \ndata^6 }
      { \pi \minone{ \cvar } \minone{ \Cvar \width } \cmin^2 \poe^3 \min\{ \lambdazero, \lambdazero^4 \} }
    \ge \frac{ 2^{21} \, 3^4 \maxone{ \cmax^{10} } \maxone{ \cout^2 } \maxone{ \cderiv^6 } \maxone{ \eucnorm{\slope}^2 }
      \maxone{ \eucnorm{\yinter}^2 } \nbp^2 \outdim^6 \ndata^6 }
      { \pi \minone{ \cvar } \minone{ \Cvar \width } \cmin^2 \poe^3 \min\{ \lambdazero, \lambdazero^4 \} }.
  \end{equation}
  The fact that for all $ z \in (0,\infty) $ it holds that $ \ln(z) \le z $,
  the fact that $ \eucnorm{\slope}^4 = ( \sum_{i=1}^{\nbp+1} \abs{\slope_i}^2 )^2 \le ( (\nbp+1) \cderiv^2 )^2
    \le ( 2 \nbp \cderiv^2 )^2 = 4 \nbp^2 \cderiv^4 $, and
  the fact that $ \nicefrac{\cmax^2}{\cmin^2} \le 1 $ \hence \prove that
  \begin{equation}
  \begin{split}
    &\frac{32 (1+\cmax^2) \outdim \ndata}{\lambdazero}
      \ln \bigl( \tfrac{12 \outdim^2 \ndata^2}{\poe} \bigr) \max \Bigl\{
      \frac{4 (1+\cmax^2) \eucnorm{\slope}^4 \outdim \ndata}{\lambdazero}, \cderiv \Bigr\} \\
    &\le \frac{384 (1+\cmax^2) \outdim^3 \ndata^3}{\poe \lambdazero} \max \Bigl\{
      \frac{16 (1+\cmax^2) \nbp^2 \cderiv^4 \outdim \ndata}{\lambdazero}, \cderiv \Bigr\}
    \le \frac{6144 (1+\cmax^2)^2 \nbp^2 \outdim^4 \ndata^4}{\poe \lambdazero} \max \Bigl\{
      \frac{\cderiv^4}{\lambdazero}, \cderiv \Bigr\} \\
    &\le \frac{6144 (2 \maxone{ \cmax^2 } )^2 \nbp^2 \outdim^4 \ndata^4 }{\poe \lambdazero }
      \frac{ \maxone{ \cderiv^4 } }{ \minone{ \lambdazero } }
    = \frac{24576 \maxone{ \cmax^4 } \maxone{ \cderiv^4 } \nbp^2 \outdim^4 \ndata^4 }{\poe \min\{ \lambdazero, \lambdazero^{-1} \} }
    \le \frac{24576 \maxone{ \cmax^6 } \maxone{ \cderiv^4 } \nbp^2 \outdim^4 \ndata^4 }{\cmin^2 \poe \min\{ \lambdazero, \lambdazero^{-1} \} } \\
    &\le \frac{ 2^{21} \, 3^4 \maxone{ \cmax^{10} } \maxone{ \cout^2 } \maxone{ \cderiv^6 } \maxone{ \eucnorm{\slope}^2 }
      \maxone{ \eucnorm{\yinter}^2 } \nbp^2 \outdim^6 \ndata^6 }
      { \pi \minone{ \cvar } \minone{ \Cvar \width } \cmin^2 \poe^3 \min\{ \lambdazero, \lambdazero^4 \} }
    \le \width.
  \end{split}
  \end{equation}
  Combining this, \cref{eq:cor_assumption_lambdazero_ln_eta}, \cref{eq:cor_assumption_lambdazero_ln_ln} and
  \cref{thm:main_theorem} (applied with
  $ \poe \is \poe $, $ \loss \is \loss $
  in the notation of \cref{thm:main_theorem})
  \proves that
  \begin{equation}
    \P \Bigl( \Forall \iteration \in \N_0 \colon \loss(\Phi(\iteration))
      \le \bigl( 1 - \tfrac{\lr \Cvar \width \lambdazero}{\ndata} \bigr)^{\iteration} \loss(\Phi(0)) \Bigr) \ge 1 - \poe.
  \end{equation}
\end{cproof}

\cfclear
\begin{lemma}
  \label{lem:log_inequality}
  Let $ \varepsilon \in (0,\infty) $.
  Then it holds for all $ z \in (0, \infty) $ that
  $ \ln( z ) \le z^{\varepsilon} \varepsilon^{-1} $.
\end{lemma}

\begin{cproof}{lem:log_inequality}
  \Nobs that for all $ z \in (0,1) $ it holds that $ \ln(z) \le 0 \le z^{\varepsilon} \varepsilon^{-1} $.
  \Moreover the fundamental theorem of calculus and the fact that for all $ s \in [1,\infty) $ it holds that $ 1 \le s^{\varepsilon} $
  \prove that for all $ z \in [1,\infty) $ it holds that
  \begin{equation}
    \ln(z)
    = \ln(z) - \ln(1)
    = \int_1^z s^{-1} \diff s
    \le \int_1^z s^{\varepsilon - 1} \diff s
    = \Bigl[ \tfrac{1}{\varepsilon} s^{\varepsilon} \Bigr]_1^z
    = \tfrac{1}{\varepsilon} z^{\varepsilon} - \tfrac{1}{\varepsilon}
    \le \tfrac{1}{\varepsilon} z^{\varepsilon}.
  \end{equation}
\end{cproof}

\cfclear
\begin{cor}
  \label{cor:assumption_lambdazero_delta}
  Assume \cref{setting:gradient_descent},
  assume $ \lambdazero \in (0,\infty) $,
  let $ \poe, \delta \in (0,1) $,
  assume
  \begin{equation}\label{eq:cor:assumption_lambdazero_delta}
    \lr \le \frac{\poe^2 \min\{ \Cvar \width, (\Cvar \width)^{-1} \} \min\{ \lambdazero, \lambdazero^{-1} \} }
      {2^{11} \maxone{\cmax^4} \maxone{\cderiv^4} \outdim^4 \ndata}
    \qqandqq
    \width \ge \Biggl[ \frac{ 2^{29} \maxone{ \cmax^{10} } \maxone{ \cout^2 } \maxone{ \cderiv^6 }
      \maxone{ \eucnorm{\slope}^2 } \maxone{ \eucnorm{\yinter}^2 } \nbp^2 \outdim^6 \ndata^6 }
      { \delta \minone{ \cvar } \minone{ \Cvar \width } \cmin^2 \poe^3 \min\{ \lambdazero, \lambdazero^4 \} } \Biggr]^{1+\delta},
  \end{equation}
  and let $ \loss \colon \network_{ \indim, \width, \outdim } \to [0,\infty) $ satisfy for all $ \Psi \in \network_{ \indim, \width, \outdim } $ that
  $ \loss ( \Psi ) = \tfrac{1}{\ndata} \sum_{i=1}^{\ndata} \eucnorm{ ( \functionANN{\act} (\Psi) ) (\indata_i) - \outdata_i }^2 $ \cfload.
  Then
  \begin{equation}
    \P \Bigl( \Forall \iteration \in \N_0 \colon \loss(\Phi(\iteration))
      \le \bigl( 1 - \tfrac{ \lr \Cvar \width \lambdazero }{ \ndata } \bigr)^{\iteration} \loss(\Phi(0)) \Bigr) \ge 1 - \poe.
  \end{equation}
\end{cor}

\begin{cproof}{cor:assumption_lambdazero_delta}
  First, \nobs that the fact that $ 2^{11} > 1352 $ \proves that
  \begin{equation}
  \label{eq:cor_assumption_lambdazero_delta_eta}
    \lr
    \le \frac{\poe^2 \min\{ \Cvar \width, (\Cvar \width)^{-1} \} \min\{ \lambdazero, \lambdazero^{-1} \} }
      {2^{11} \maxone{\cmax^4} \maxone{\cderiv^4} \outdim^4 \ndata}
    < \frac{\poe^2 \min\{ \Cvar \width, (\Cvar \width)^{-1} \} \min\{ \lambdazero, \lambdazero^{-1} \} }
      {1352 \maxone{\cmax^4} \maxone{\cderiv^4} \outdim^4 \ndata}.
  \end{equation}
  Next \nobs that the fact that $ 3 \delta \ge 0 $ \proves that
  \begin{equation}\label{eq:cor_assumption_lambdazero_delta_ineq}
  \begin{gathered}
    \frac{ 3 \delta }{ 4 + 3 \delta } - 1
    = - \frac{ 4 }{ 4 + 3 \delta }
    = - \frac{ 1 }{ 1 + \delta } \frac{ 4 + 4 \delta }{ 4 + 3 \delta }, \\
    \frac{ 3 \delta }{ 4 + 3 \delta } - \frac{ 6 ( 4 + 4 \delta ) }{ 4 + 3 \delta }
    = - \frac{ 24 + 21 \delta }{ 4 + 3 \delta }
    \le - \frac{ 24 + 18 \delta }{ 4 + 3 \delta }
    = -6, \quad \text{and} \\
    \frac{ 3 \delta }{ 4 + 3\delta } - \frac{ 3 ( 4 + 4 \delta ) }{ 4 + 3 \delta }
    = - \frac{ 12 + 9 \delta }{ 4 + 3 \delta }
    = -3.
  \end{gathered}
  \end{equation}
  \Moreover the fact that $ 3^{\nicefrac{24}{7}} \, 7 \le 2^{\nicefrac{50}{7}} \pi $ \proves that
  \begin{equation}
    \frac{ 12^{\nicefrac{3}{7}} \, 7 }{ 2^{29} \, 3 }
    = \frac{ 7 }{ 2^{\nicefrac{197}{7}} \, 3^{\nicefrac{4}{7}} }
    \le \frac{ 2^{\nicefrac{50}{7}} \pi }{ 2^{\nicefrac{197}{7}} 3^4 }
    = \frac{ \pi }{ 2^{21} 3^4 }.
  \end{equation}
  \cref{lem:log_inequality} (applied with
  $ \varepsilon \is \tfrac{ 3 \delta }{ 4 + 3 \delta } $, $ z \is \tfrac{12 \outdim \width}{\poe} $
  in the notation of \cref{lem:log_inequality}),
  \cref{eq:cor:assumption_lambdazero_delta}, \cref{eq:cor_assumption_lambdazero_delta_ineq},
  the fact that $ \tfrac{3}{7} \delta \le \tfrac{ 3 \delta }{ 4 + 3 \delta } \le \tfrac{3}{7} $, and
  the fact that $ \tfrac{4+4\delta}{4+3\delta} \ge 1 $
  \hence \prove that
  \begin{equation}
  \begin{split}
    \tfrac{1}{\width} \ln \bigl( \tfrac{12 \outdim \width}{\poe} \bigr)
    &\le \width^{-1} \frac{ 12^{3\delta(4+3\delta)^{-1}} \outdim^{3\delta(4+3\delta)^{-1}} \width^{3\delta(4+3\delta)^{-1}} }
      { \poe^{3\delta(4+3\delta)^{-1}} 3\delta(4+3\delta)^{-1} }
    \le \Bigl[ \width^{-(1+\delta)^{-1}} \Bigr]^{(4+4\delta)(4+3\delta)^{-1}}
      \frac{ 12^{\nicefrac{3}{7}} 7 \outdim^{3\delta(4+3\delta)^{-1}} }
      { 3 \poe^{3\delta(4+3\delta)^{-1}} \delta } \\
    &\le \biggl[ \frac{ \delta \minone{ \cvar } \minone{ \Cvar \width } \cmin^2 \poe^{3} \min\{ \lambdazero, \lambdazero^4 \} }
      { 2^{29} \maxone{ \cmax^{10} } \maxone{ \cout^2 } \maxone{ \cderiv^6 }
      \maxone{ \eucnorm{\slope}^2 } \maxone{ \eucnorm{\yinter}^2 } \nbp^2 \outdim^{6} \ndata^6 } \biggr]^{(4+4\delta)(4+3\delta)^{-1}}
      \frac{ 12^{\nicefrac{3}{7}} 7 \outdim^{3\delta(4+3\delta)^{-1}} }
      { 3 \poe^{3\delta(4+3\delta)^{-1}} \delta } \\
    &= \biggl[ \frac{ \delta \minone{ \cvar } \minone{ \Cvar \width } \cmin^2 \min\{ \lambdazero, \lambdazero^4 \} }
      { 2^{29} \maxone{ \cmax^{10} } \maxone{ \cout^2 } \maxone{ \cderiv^6 }
      \maxone{ \eucnorm{\slope}^2 } \maxone{ \eucnorm{\yinter}^2 } \nbp^2 \ndata^6 } \biggr]^{(4+4\delta)(4+3\delta)^{-1}}
      \frac{ 12^{\nicefrac{3}{7}} 7 \outdim^{3\delta(4+3\delta)^{-1} - 6(4+4\delta)(4+3\delta)^{-1}} }
      { 3 \poe^{3\delta(4+3\delta)^{-1} - 3(4+4\delta)(4+3\delta)^{-1}} \delta } \\
    &\le \frac{ 12^{\nicefrac{3}{7}} \, 7 \minone{ \cvar } \minone{ \Cvar \width } \cmin^2 \poe^{3} \min\{ \lambdazero, \lambdazero^4 \} }
      { 2^{29} \, 3 \maxone{ \cmax^{10} } \maxone{ \cout^2 } \maxone{ \cderiv^6 }
      \maxone{ \eucnorm{\slope}^2 } \maxone{ \eucnorm{\yinter}^2 } \nbp^2 \outdim^{6} \ndata^6 }
    \le \frac{ \pi \minone{ \cvar } \minone{ \Cvar \width } \cmin^2 \poe^3 \min\{ \lambdazero, \lambdazero^4 \} }
      { 2^{21} \, 3^4 \maxone{ \cmax^{10} } \maxone{ \cout^2 } \maxone{ \cderiv^6 } \maxone{ \eucnorm{\slope}^2 }
      \maxone{ \eucnorm{\yinter}^2 } \nbp^2 \outdim^6 \ndata^6 }.
  \end{split}
  \end{equation}
  Combining this, \cref{eq:cor_assumption_lambdazero_delta_eta}, and \cref{cor:assumption_lambdazero_ln} (applied with
  $ \poe \is \poe $, $ \loss \is \loss $
  in the notation of \cref{cor:assumption_lambdazero_ln})
  \proves that
  \begin{equation}
    \P \Bigl( \Forall \iteration \in \N_0 \colon \loss(\Phi(\iteration))
      \le \bigl( 1 - \tfrac{ \lr \Cvar \width \lambdazero }{ \ndata } \bigr)^{\iteration} \loss(\Phi(0)) \Bigr) \ge 1 - \poe.
  \end{equation}
\end{cproof}

\cfclear
\begin{cor}
  \label{cor:assumption_pairwise_distinct}
  Assume \cref{setting:gradient_descent},
  assume $ \# \{ \indata_1, \indata_2, \ldots, \indata_{\ndata} \} = \ndata $ and $ \# \{ \slope_1, \slope_2, \ldots, \slope_{\nbp+1} \} \ge 2 $,
  let $ \poe, \delta \in (0,1) $,
  assume
  \begin{equation}\label{eq:cor:assumption_pairwise_distinct}
    \lr \le \frac{\poe^2 \min\{ \Cvar \width, (\Cvar \width)^{-1} \} \min\{ \lambdazero, \lambdazero^{-1} \} }
      {2^{11} \maxone{\cmax^4} \maxone{\cderiv^4} \outdim^4 \ndata}
    \qqandqq
    \width \ge \Biggl[ \frac{ 2^{29} \maxone{ \cmax^{10} } \maxone{ \cout^2 } \maxone{ \cderiv^6 }
      \maxone{ \eucnorm{\slope}^2 } \maxone{ \eucnorm{\yinter}^2 } \nbp^2 \outdim^6 \ndata^6 }
      { \delta \minone{ \cvar } \minone{ \Cvar \width } \cmin^2 \poe^3 \min\{ \lambdazero, \lambdazero^4 \} } \Biggr]^{1+\delta}
  \end{equation}
  and let $ \loss \colon \network_{ \indim, \width, \outdim } \to [0,\infty) $ satisfy for all $ \Psi \in \network_{ \indim, \width, \outdim } $ that
  $ \loss ( \Psi ) = \tfrac{1}{\ndata} \sum_{i=1}^{\ndata} \eucnorm{ ( \functionANN{\act} (\Psi) ) (\indata_i) - \outdata_i }^2 $ \cfload.
  Then
  \begin{enumerate}[(i)]
    \item \label{cor:assumption_pairwise_distinct:item1} it holds that $ \lambdazero > 0 $ and
    \item \label{cor:assumption_pairwise_distinct:item2} it holds that
      $ \P \bigl( \Forall \iteration \in \N_0 \colon \loss(\Phi(\iteration))
        \le \bigl( 1 - \tfrac{ \lr \Cvar \width \lambdazero }{ \ndata } \bigr)^{\iteration} \loss(\Phi(0)) \bigr) \ge 1 - \poe $.
  \end{enumerate}
\end{cor}

\begin{cproof}{cor:assumption_pairwise_distinct}
  \Nobs that \cref{lem:lambdazero_positive} \proves \cref{cor:assumption_pairwise_distinct:item1}.
  \Moreover \cref{cor:assumption_pairwise_distinct:item1} and \cref{cor:assumption_lambdazero_delta} (applied with
  $ \poe \is \poe $, $ \delta \is \delta $, $ \loss \is \loss $
  in the notation of \cref{cor:assumption_lambdazero_delta})
  \prove that
  \begin{equation}
    \P \Bigl( \Forall \iteration \in \N_0 \colon \loss(\Phi(\iteration))
      \le \bigl( 1 - \tfrac{ \lr \Cvar \width \lambdazero }{ \ndata } \bigr)^{\iteration} \loss(\Phi(0)) \Bigr) \ge 1 - \poe.
  \end{equation}
\end{cproof}

\cfclear
\begin{cor}
  \label{cor:assumption_pairwise_distinct_const}
  Assume \cref{setting:gradient_descent},
  assume $ \# \{ \indata_1, \indata_2, \ldots, \indata_{\ndata} \} = \ndata $ and $ \# \{ \slope_1, \slope_2, \ldots, \slope_{\nbp+1} \} \ge 2 $,
  let $ \Lambda, \poe \in (0,1) $, $ \constvarwidth \in (0,\infty) $ satisfy
  $ \Cvar = \frac{\constvarwidth}{\width} $ and
  \begin{equation}
  \label{eq:cor:assumption_pairwise_distinct_const}
    \Lambda = \frac{ \minone{ \cvar^2 } \min\{ \constvarwidth^{-1}, \constvarwidth^2 \} \cmin^{\nicefrac{17}{8}} }
      { 2^{36} \maxone{ \cmax^{11} } \maxone{ \cout^3 } \maxone{ \cderiv^7 } \maxone{ \eucnorm{\slope}^3 }
      \maxone{ \eucnorm{\yinter}^3 } \nbp^3 \outdim^7 },
  \end{equation}
  assume
  $ \lr \le \Lambda \poe^2 \ndata^{-1} \min\{ \lambdazero, \lambdazero^{-1} \} $ and
  $ \width \ge \Lambda^{-1} \poe^{-4} \ndata^7 \max\{ \lambdazero^{-1}, \lambdazero^{-5} \} $,
  and let $ \loss \colon \network_{ \indim, \width, \outdim } \to [0,\infty) $ satisfy for all $ \Psi \in \network_{ \indim, \width, \outdim } $ that
  $ \loss ( \Psi ) = \tfrac{1}{\ndata} \sum_{i=1}^{\ndata} \eucnorm{ ( \functionANN{\act} (\Psi) ) (\indata_i) - \outdata_i }^2 $ \cfload.
  Then
  \begin{enumerate}[(i)]
    \item \label{cor:assumption_pairwise_distinct_const:item1} it holds that $ \lambdazero > 0 $ and
    \item \label{cor:assumption_pairwise_distinct_const:item2} it holds that
      $ \P \bigl( \Forall \iteration \in \N_0 \colon \loss(\Phi(\iteration))
        \le \bigl( 1 - \tfrac{ \lr \Cvar \width \lambdazero }{ \ndata } \bigr)^{\iteration} \loss(\Phi(0)) \bigr) \ge 1 - \poe $.
  \end{enumerate}
\end{cor}

\begin{cproof}{cor:assumption_pairwise_distinct_const}
  Throughout this proof let $ \delta \in (0,1) $ satisfy $ \delta = 2^{-4} $.
  First, \nobs that the fact that $ \cmin \le \cmax $ \proves that
  \begin{equation}
    \frac{ \cmin^{\nicefrac{17}{8}} }{ \maxone{ \cmax^{11} } }
    \le \frac{ \cmax^{\nicefrac{17}{8}} }{ \maxone{ \cmax^{11} } }
    \le \frac{1}{ \maxone{ \cmax^8 } }
    \le \frac{1}{ \maxone{ \cmax^4 } }.
  \end{equation}
  This, \cref{eq:cor:assumption_pairwise_distinct_const},
  the fact that
  $ \min \{ \constvarwidth^{-1}, \constvarwidth^2 \}
    \le \min\{ \constvarwidth, \constvarwidth^{-1} \}
    = \min\{ \Cvar \width, (\Cvar \width)^{-1} \} $, and
  the assumption that $ \lr \le \Lambda \poe^2 \ndata^{-1} \min\{ \lambdazero, \lambdazero^{-1} \}$
  \prove that
  \begin{equation}
  \begin{split}
    \lr
    \le \Lambda \poe^2 \ndata^{-1} \min\{ \lambdazero, \lambdazero^{-1} \}
    &= \frac{ \poe^2 \minone{ \cvar^2 } \min\{ \constvarwidth^{-1}, \constvarwidth^2 \} \cmin^{\nicefrac{17}{8}} \min\{ \lambdazero, \lambdazero^{-1} \} }
      { 2^{36} \maxone{ \cmax^{11} } \maxone{ \cout^3 } \maxone{ \cderiv^7 } \maxone{ \eucnorm{\slope}^3 }
      \maxone{ \eucnorm{\yinter}^3 } \nbp^3 \outdim^7 \ndata } \\
    &\le \frac{\poe^2 \min\{ \Cvar \width, (\Cvar \width)^{-1} \} \min\{ \lambdazero, \lambdazero^{-1} \} }
      {2^{11} \maxone{\cmax^4} \maxone{\cderiv^4} \outdim^4 \ndata}.
  \end{split}
  \end{equation}
  \Moreover \cref{eq:cor:assumption_pairwise_distinct_const},
  the fact that
  $ \min \{ \constvarwidth^{-1}, \constvarwidth^2 \} \le \minone{ \constvarwidth^2 } = \minone{ ( \Cvar \width )^2 } $, and
  the assumption that $ \width \ge \Lambda^{-1} \poe^{-4} \ndata^7 \max\{ \lambdazero^{-1}, \lambdazero^{-5} \} $ \prove that
  \begin{equation}
  \begin{split}
    \width
    \ge \Lambda^{-1} \poe^{-4} \ndata^7 \max\{ \lambdazero^{-1}, \lambdazero^{-5} \}
    &= \frac{ 2^{36} \maxone{ \cmax^{11} } \maxone{ \cout^3 } \maxone{ \cderiv^7 } \maxone{ \eucnorm{\slope}^3 }
      \maxone{ \eucnorm{\yinter}^3 } \nbp^3 \outdim^7 \ndata^7 }
      { \minone{ \cvar^2 } \min\{ \constvarwidth^{-1}, \constvarwidth^2 \} \cmin^{\nicefrac{17}{8}} \poe^4 \min\{ \lambdazero, \lambdazero^5 \} } \\
    &\ge \Biggl[ \frac{ 2^{29} \maxone{ \cmax^{10} } \maxone{ \cout^2 } \maxone{ \cderiv^6 }
      \maxone{ \eucnorm{\slope}^2 } \maxone{ \eucnorm{\yinter}^2 } \nbp^2 \outdim^6 \ndata^6 }
      { \delta \minone{ \cvar } \minone{ \Cvar \width } \cmin^2 \poe^3 \min\{ \lambdazero, \lambdazero^4 \} } \Biggr]^{1+\delta}.
  \end{split}
  \end{equation}
  \cref{cor:assumption_pairwise_distinct} (applied with
  $ \poe \is \poe $, $ \delta \is \delta $, $ \loss \is \loss $
  in the notation of \cref{cor:assumption_pairwise_distinct})
  \hence \proves \cref{cor:assumption_pairwise_distinct_const:item1} and \cref{cor:assumption_pairwise_distinct_const:item2}.
\end{cproof}

\subsection{Qualitative probabilistic error analysis for GD optimization algorithms}
\label{subsec:qualitative_probabilistic_error_analysis}

\cfclear
\begin{samepage}
\begin{cor}
  \label{cor:vectorized}
  Let $ \indim, \outdim \in \N $,
  let $ \mc P \colon \N \to \N $ satisfy for all $ \width \in \N $ that
  $ \mc P ( \width ) = \width \indim + \width + \outdim \width + \outdim $,
  let $ \varphi_s \in C( \R, \R ) $, $ s \in \N_0 $, satisfy for all $ s \in \N $ that $ \varphi_s $ is differentiable,
  assume that $ \varphi_0 $ is piecewise affine and not affine,
  let $ \grad \colon \R \to \R $,
  assume for all $ v \in \R $ that
  $ \limsup_{s \to \infty} \abs{ \varphi_s (v) - \varphi_0 (v) } = \limsup_{s \to \infty} \abs{ (\varphi_s)'(v) - \grad(v) } = 0 $,
  assume for all $ v \in \{ u \in \R \colon \varphi_0 \text{ is differentiable at } u \} $ that $ (\varphi_0)'(v) = \grad(v) $,
  for every $ \width, s \in \N_0 $, $ \theta = ( \theta_1, \ldots, \theta_{\mc P(\width)} ) \in \R^{\mc P(\width)} $
  let $ \mc N_{\theta}^{\width,s} = ( \mc N_{\theta}^{\width,s,1}, \ldots, \mc N_{\theta}^{\width,s,\outdim} ) \colon \R^{\indim} \to \R^{\outdim} $
  satisfy for all $ \ell \in \{ 1, 2, \ldots, \outdim \} $, $ v = (v_1, \ldots, v_{\indim}) \in \R^{\indim} $ that
  \begin{equation}\label{cor:vectorized:realization}
    \mc N_{\theta}^{\width,s,\ell} (v) = \theta_{ \width \indim + \width + \outdim \width + \ell }
      + \smallsum\limits_{k=1}^{\width} \theta_{ \width \indim + \width + (\ell-1) \width + k } \,
      \varphi_s \Bigl( \theta_{ \width \indim + k } + \smallsum\limits_{j=1}^{\indim} \theta_{ (k-1) \indim + j } v_{j} \Bigr),
  \end{equation}
  let $ \cvar, \Cvar, \const \in (0,\infty) $, for every $ \ndata \in \N $ let
  $ \indata_1^{\ndata}, \indata_2^{\ndata}, \ldots, \indata_{\ndata}^{\ndata} \in \R^{\indim} $,
  $ \outdata_1^{\ndata}, \outdata_2^{\ndata}, \ldots, \outdata_{\ndata}^{\ndata} \in \R^{\outdim} $
  satisfy for all $ i \in \{ 1, 2, \ldots, \ndata \}$ that
  $ \nicefrac{1}{\const} \le \eucnorm{ \indata_i^{\ndata} } \le \const $,
  $ \eucnorm{ \outdata_i^{\ndata} } \le \const $, and
  $ \# \{ \indata_1^{\ndata}, \allowbreak \indata_2^{\ndata}, \allowbreak \ldots, \allowbreak \indata_{\ndata}^{\ndata} \} = \ndata $,
  for every $ \width, s \in \N_0 $, $ \ndata \in \N $ let $ \loss^{\width,\ndata,s} \colon \R^{ \mc P ( \width ) } \to \R $
  satisfy for all $ \theta \in \R^{ \mc P ( \width ) } $ that
  \begin{equation}\label{cor:vectorized:loss}
    \loss^{\width,\ndata,s} (\theta)
    = \frac{1}{\ndata} \biggl[ \smallsum\limits_{i=1}^{\ndata} \eucnorm{ \mc N_{\theta}^{\width,s} (\indata_i^{\ndata}) - \outdata_i^{\ndata} }^2 \biggr],
  \end{equation}
  let $ (\Omega, \mc F, \P) $ be a probability space,
  for every $ \width, \ndata \in \N $, $ \lr \in \R $ let
  $ \Theta^{\width, \ndata, \lr} = ( \Theta_1^{\width, \ndata, \lr}, \ldots, \Theta_{ \mc P( \width ) }^{\width, \ndata, \lr} )
    \colon \N_0 \times \Omega \to \R^{ \mc P( \width ) } $
  be a stochastic process which satisfies
  for all $ \iteration \in \N $, $ i \in \{ 1, 2, \ldots, \mc P( \width ) \} $ that
  $ \smallsum_{k=1}^{\width} \abs{ \Theta_{ \width \indim + k }^{\width, \ndata, \lr} (0) }
    + \smallsum_{j=1}^{\outdim} \abs{ \Theta_{ \width \indim + \width + \outdim \width + j }^{\width, \ndata, \lr} (0) } = 0 $ and
  \begin{equation}\label{cor:vectorized:training}
    \Theta_i^{\width, \ndata, \lr} ( \iteration ) = \Theta_i^{\width, \ndata, \lr} ( \iteration - 1 )
    - \lr \bigl[ \lim\nolimits_{s \to \infty} ( \tfrac{\partial}{\partial \theta_i} \loss^{\width, \ndata, s} )
    ( \Theta^{\width, \ndata, \lr} ( \iteration - 1 ) ) \bigr] \ind_{ \R \backslash (\width \indim + \width, \mc P( \width ) - \outdim ] } ( i ),
  \end{equation}
  assume for all $ \width, \ndata \in \N $, $ \lr \in \R $ that
  $ \sqrt{ \nicefrac{1}{\cvar} } \mspace{1.5mu} \Theta_1^{\width, \ndata, \lr} (0), \allowbreak
    \sqrt{ \nicefrac{1}{\cvar} } \mspace{1.5mu} \Theta_2^{\width, \ndata, \lr} (0), \allowbreak \ldots, \allowbreak
    \sqrt{ \nicefrac{1}{\cvar} } \mspace{1.5mu} \Theta_{\width \indim}^{\width, \ndata, \lr} (0), \allowbreak
    \sqrt{ \nicefrac{\width}{\Cvar} } \mspace{1.5mu} \Theta_{\width \indim + \width + 1}^{\width, \ndata, \lr} (0), \allowbreak
    \sqrt{ \nicefrac{\width}{\Cvar} } \mspace{1.5mu} \Theta_{\width \indim + \width + 2 }^{\width, \ndata, \lr} (0), \allowbreak \ldots, \allowbreak
    \sqrt{ \nicefrac{\width}{\Cvar} } \mspace{1.5mu} \Theta_{\width \indim + \width + \outdim \width}^{\width, \ndata, \lr} (0) $
  are independent and standard normal,
  let $ Z \colon \Omega \to \R^{\indim} $ be standard normal, and
  for every $ \ndata \in \N $ let $ \lambda_{\ndata} \in \R $ satisfy
  \begin{equation}\label{cor:vectorized:lambda}
    \lambda_{\ndata} = \inf\limits_{ z = ( z_1, \ldots, z_{\ndata} ) \in \R^{\ndata}, \, \eucnorm{ z } = 1 }
      \biggl( \smallsum\limits_{i,j=1}^{\ndata}
      \bigl( 1 + \scalprod{ \indata_i^{\ndata} }{ \indata_j^{\ndata} } \bigr) \,
      \E \biggl[ \smallprod\limits_{ k \in \{ i, j \} } z_k
        \grad \bigl( \sqrt{ \nicefrac{\width}{\Cvar} } \mspace{1mu} \scalprod{ Z }{ \indata_k^{\ndata} } \bigr) \biggr] \biggr).
  \end{equation}
  Then
  \begin{enumerate}[(i)]
    \item \label{cor:vectorized:item1} it holds for all $ \ndata \in \N $ that $ \lambda_{\ndata} > 0 $ and
    \item \label{cor:vectorized:item2} there exists $ \Lambda \in (0, \nicefrac{1}{\Cvar}) $ such that for all $ \ndata \in \N $, $ \poe \in (0, 1) $,
      $ \lr \in (0, \Lambda \poe^2 \ndata^{-1} \min \{ \lambda_{\ndata}, ( \lambda_{\ndata} )^{-1} \} ] $,
      $ \width \in \N \cap [ \Lambda^{-1} \poe^{-4} \ndata^7 \max\{ ( \lambda_{\ndata} )^{-1}, ( \lambda_{\ndata} )^{-5} \}, \infty) $
      it holds that
  \begin{equation}
  \label{eq:cor:vectorized}
    \P \Bigl( \Forall \iteration \in \N_0 \colon \loss^{\width, \ndata, 0} ( \Theta^{\width,\ndata,\lr} (\iteration) )
      \le \bigl( 1 - \tfrac{ \lr \Cvar \lambda_{\ndata} }{ \ndata } \bigr)^{\iteration} \loss^{\width, \ndata, 0}
        ( \Theta^{\width,\ndata,\lr} (0) ) \Bigr)
    \ge 1 - \varepsilon.
  \end{equation}
  \end{enumerate}
\end{cor}
\end{samepage}

\renewcommand{\cmax}{C}
\renewcommand{\indata}{x}
\renewcommand{\outdata}{y}
\begin{cproof}{cor:vectorized}
  Throughout this proof
  let $ \cmin, \cmax, \cout, \constvarwidth, \cderiv \in (0,\infty) $ satisfy
  \begin{equation}
    \cmin \le \tfrac{1}{\const}, \quad
    \cmax \ge \const, \quad
    \cout \ge \const, \quad
    \constvarwidth = \Cvar, \qandq
    \sup\nolimits_{ v \in \R } \abs{ \grad( v ) } \le \cderiv.
  \end{equation}
  First, \nobs that the assumption that $ \varphi_0 $ is piecewise affine and not affine \proves that there exist
  $ \nbp \in \N $, $ \bp_0, \bp_1, \dots, \bp_{\nbp+1} \in [-\infty,\infty] $,
  $ \slope = ( \slope_1, \dots, \slope_{\nbp+1} ) \in \R^{\nbp+1} \backslash \{ 0 \} $,
  $ \yinter = ( \yinter_1, \dots, \yinter_{\nbp+1} ) \in \R^{\nbp+1} $
  which satisfy for all $ i \in \{ 1, 2, \dots, \nbp+1 \} $, $ v \in ( \bp_{i-1}, \bp_i ) $ that
  \begin{equation}\label{cor:vectorized:st}
    -\infty = \bp_0 < \bp_1 < \dots < \bp_{\nbp+1} = \infty, \qquad
    \varphi_0 ( v ) = \slope_i v + \yinter_i, \qqandqq
    \# \{ \slope_1, \slope_2, \ldots, \slope_{\nbp+1} \} \ge 2.
  \end{equation}
  \Moreover \cref{cor:vectorized:st} and
  the assumption that for all $ v \in \{ u \in \R \colon \varphi_0 \text{ is differentiable at } u \} $ it holds that $ (\varphi_0)'(v) = \grad(v) $
  \prove that for all $ i \in \{ 1, 2, \dots, \nbp+1 \} $, $ v \in ( \bp_{i-1}, \bp_i ) $ it holds that $ \grad(v) = \slope_i $.
  Next let $ \Lambda \in (0, \nicefrac{1}{\Cvar}) $ satisfy
  \begin{equation}\label{eq:cor_vectorized_lambda}
    \Lambda = \frac{ \minone{ \cvar^2 } \min\{ \constvarwidth^{-1}, \constvarwidth^2 \} \cmin^{\nicefrac{17}{8}} }
      { 2^{36} \maxone{ \cmax^{11} } \maxone{ \cout^3 } \maxone{ \cderiv^7 } \maxone{ \eucnorm{\slope}^3 }
      \maxone{ \eucnorm{\yinter}^3 } \nbp^3 \outdim^7 },
  \end{equation}
  let $ \ndata \in \N $, $ \poe \in (0,1) $, $ \lr \in (0, \Lambda \poe^2 \ndata^{-1} \min \{ \lambda_{\ndata}, \abs{ \lambda_{\ndata} }^{-1} \} ] $,
  $ \width \in \N \cap [ \Lambda^{-1} \poe^{-4} \ndata^7 \max\{ \abs{ \lambda_{\ndata} }^{-1}, \abs{ \lambda_{\ndata} }^{-5} \}, \infty) $,
  let
  $ \indata = ( ( \indata_i^p )_{ p \in \{ 1, 2, \ldots, \indim \} } )_{ i \in \{ 1, 2, \ldots, \ndata \} } \in ( \R^{\indim} )^{\ndata} $,
  $ \outdata = ( ( \outdata_i^p )_{ p \in \{ 1, 2, \ldots, \outdim \} } )_{ i \in \{ 1, 2, \ldots, \ndata \} } \in ( \R^{\outdim} )^{\ndata} $
  satisfy for all $ i \in \{ 1, 2, \ldots, \ndata \} $ that
  $ \indata_i = \ms x_i^{\ndata} $ and $ \outdata_i = \ms y_i^{\ndata} $,
  let
  $ W = ( W_1, \ldots, W_{\width} ) = ( W_{i,j} )_{ (i,j) \in \{ 1, 2, \ldots, \width \} \times \{ 1, 2, \ldots, \indim \} }
    \colon \N_0 \times \Omega \to \R^{\width \times \indim} $,
  $ B = ( B_1, \ldots, B_{\width} ) \colon \N_0 \times \Omega \to \R^{\width} $,
  $ \mf W = ( \mf W_1, \ldots, \mf W_{\outdim} ) = ( \mf W_{i,j} )_{ (i,j) \in \{ 1, 2, \ldots, \outdim \} \times \{ 1, 2, \ldots, \width \} }
    \colon \Omega \to \R^{\outdim \times \width} $, and
  $ \mf B = ( \mf B_1, \ldots, \mf B_{\outdim} ) \colon \N_0 \times \Omega \to \R^{\outdim} $
  satisfy for all $ j \in \{ 1, 2, \ldots, \indim \} $, $ k \in \{ 1, 2, \ldots, \width \} $, $ \ell \in \{ 1, 2, \ldots, \outdim \} $,
  $ \iteration \in \N_0 $, $ \omega \in \Omega $ that
  \begin{equation}\label{eq:cor_vectorized_WB}
  \begin{gathered}
    W_{k,j} ( \iteration, \omega ) = \Theta^{\width, \ndata, \lr}_{ (k-1) \indim + j } ( \iteration, \omega ), \qquad
    B_k ( \iteration, \omega ) = \Theta^{\width, \ndata, \lr}_{ \width \indim + k } ( \iteration, \omega ), \\
    \mf W_{\ell,k} ( \omega ) = \Theta^{\width, \ndata, \lr}_{ \width \indim + \width + (\ell-1) \width + k }, ( 0, \omega ) \qqandqq
    \mf B_{\ell} ( \iteration, \omega ) = \Theta^{\width, \ndata, \lr}_{ \width \indim + \width + \outdim \width + \ell } ( \iteration, \omega ),
  \end{gathered}
  \end{equation}
  let $ \Phi \colon \N_0 \times \Omega \to \network_{ \indim, \width, \outdim } $ and
  $ f = ( ( f_i^p )_{ p \in \{ 1, 2, \ldots, \outdim \} } )_{ i \in \{ 1, 2, \ldots, \ndata \} } \colon \N_0 \times \Omega \to (\R^{\outdim})^{\ndata} $
  satisfy for all $ i \in \{ 1, 2, \ldots, \ndata \} $, $ \iteration \in \N_0 $, $ \omega \in \Omega $ that
  \begin{equation}
    \Phi ( \iteration, \omega ) = \bigl( ( W( \iteration, \omega ), B( \iteration, \omega ) ), ( \mf W( \omega ), \mf B( \iteration, \omega ) ) \bigr)
    \qqandqq
    f_i ( \iteration, \omega ) = \bigl( \functionANN{\varphi_0} ( \Phi( \iteration, \omega ) ) \bigr) ( \indata_i ),
  \end{equation}
  let $ \mf E \colon \network_{ \indim, \width, \outdim } \to \R $ satisfy for all $ \Psi \in \network_{ \indim, \width, \outdim } $ that
  $ \mf E ( \Psi ) = \tfrac{1}{\ndata} \sum_{i=1}^{\ndata} \eucnorm{ ( \functionANN{\varphi_0} (\Psi) ) (\indata_i) - \outdata_i }^2 $, and
  let
  $ \gramGinf = ( \gramGinf_{i,j} )_{ (i,j) \in \{ 1, 2, \ldots, \outdim \ndata \}^2 } \in \R^{(\outdim \ndata) \times (\outdim \ndata)} $,
  $ G = ( G_{i,j} )_{ (i,j) \in \{ 1, 2, \ldots, \ndata \}^2 } \in \R^{ \ndata \times \ndata } $,
  $ \lambdazero \in \R $
  satisfy for all $ i, j \in \{ 1, 2, \ldots, \allowbreak \ndata \} $, $ p, q \in \{ 1, 2, \ldots, \outdim \} $ that
  \begin{equation}\label{eq:cor_vectorized_gramGinf}
    \gramGinf_{ (i-1) \outdim + p, (j-1) \outdim + q } = 2 \bigl( 1 + \scalprod{ \indata_i }{ \indata_j } \bigr)
      \E \bigl[ \grad \bigl( \scalprod{ W_1(0) }{ \indata_i } \bigr)
      \grad \bigl( \scalprod{ W_1(0) }{ \indata_j } \bigr)\bigr] \ind_{ \{ p \} } (q),
  \end{equation}
  \begin{equation}\label{eq:cor_vectorized_G}
    G_{i,j} = 2 \bigl( 1 + \scalprod{ \indata_i }{ \indata_j } \bigr)
      \E \bigl[ \grad \bigl( \scalprod{ W_1(0) }{ \indata_i } \bigr)
      \grad \bigl( \scalprod{ W_1(0) }{ \indata_j } \bigr)\bigr],
  \end{equation}
  and $ \lambdazero = \lambdamin( \tfrac{1}{2} \gramGinf ) $
  \cfload.
  \Nobs that \cref{eq:cor_vectorized_WB} and
  the assumption that
  $ \sqrt{\nicefrac{\indim}{2}} \mspace{2mu} \Theta_1^{\width, \ndata, \lr} (0), \allowbreak
    \sqrt{\nicefrac{\indim}{2}} \mspace{2mu} \Theta_2^{\width, \ndata, \lr} (0), \allowbreak \ldots, \allowbreak
    \sqrt{\nicefrac{\indim}{2}} \mspace{2mu} \Theta_{\width \indim}^{\width, \ndata, \lr} (0), \allowbreak
    \sqrt{\nicefrac{\width}{2}} \mspace{2mu} \Theta_{\width \indim + \width + 1}^{\width, \ndata, \lr} (0), \allowbreak
    \sqrt{\nicefrac{\width}{2}} \mspace{2mu} \Theta_{\width \indim + \width + 2 }^{\width, \ndata, \lr} (0), \allowbreak \ldots, \allowbreak
    \sqrt{\nicefrac{\width}{2}} \mspace{2mu} \Theta_{\width \indim + \width + \outdim \width}^{\width, \ndata, \lr} (0) $
  are independent and standard normal
  \prove that
  \begin{equation}
    \textstyle
    \sqrt{\nicefrac{1}{\cvar}} W_1(0),
    \sqrt{\nicefrac{1}{\cvar}} W_2(0), \ldots,
    \sqrt{\nicefrac{1}{\cvar}} W_{\width}(0),
    \sqrt{\nicefrac{1}{\Cvar}} \mf W_1,
    \sqrt{\nicefrac{1}{\Cvar}} \mf W_2, \ldots,
    \sqrt{\nicefrac{1}{\Cvar}} \mf W_{\outdim}
  \end{equation}
  are independent and standard normal.
  \Moreover the assumption that
  $ \smallsum_{k=1}^{\width} \abs{ \Theta_{ \width \indim + k }^{\width, \ndata, \lr} (0) }
    + \smallsum_{j=1}^{\outdim} \abs{ \Theta_{ \width \indim + \width + \outdim \width + j }^{\width, \ndata, \lr} (0) } = 0 $
  and \cref{eq:cor_vectorized_WB}
  \prove that
  $ \eucnorm{ B(0) } = \eucnorm{ \mf B(0) } = 0 $.
  Next \nobs that for all $ q \in \{ 1, 2, \ldots, \allowbreak \mc P( \width ) \} $, $ \iteration \in \N_0 $, $ s \in \N $ it holds that
  \begin{equation}
  \label{eq:cor_vectorized_error}
  \begin{split}
    \tfrac{\partial}{\partial \theta_q} \loss^{\width, \ndata, s} ( \Theta^{\width, \ndata, \lr} (\iteration) )
    &= \tfrac{\partial}{\partial \theta_q} \biggl( \tfrac{1}{\ndata} \smallsum\limits_{i=1}^{\ndata}
      \beucnorm{ \mc N_{\Theta^{\width, \ndata, \lr} (\iteration)}^{\width, s} ( \indata_i ) - \outdata_i }^2 \biggr)
    = \tfrac{\partial}{\partial \theta_q} \biggl( \tfrac{1}{\ndata} \smallsum\limits_{i=1}^{\ndata} \smallsum\limits_{p=1}^{\outdim}
      \babs{ \mc N_{\Theta^{\width, \ndata, \lr} (\iteration)}^{\width, s, p} ( \indata_i ) - \outdata_i^p }^2 \biggr) \\
    &= \tfrac{2}{\ndata} \smallsum\limits_{i=1}^{\ndata} \smallsum\limits_{p=1}^{\outdim} 
      \bigl( \mc N_{\Theta^{\width, \ndata, \lr} (\iteration)}^{\width, s, p} ( \indata_i ) - \outdata_i^p \bigr)
      \bigl( \tfrac{\partial}{\partial \theta_q} \mc N_{\Theta^{\width, \ndata, \lr} (\iteration)}^{\width, s, p} ( \indata_i ) \bigr).
  \end{split}
  \end{equation}
  \Moreover
  the fact that for all $ v \in \R $ it holds that $ \lim_{s\to\infty} \varphi_s (v) = \varphi_0 (v) $
  and the fact that for all
  $ k \in \{ 1, 2, \ldots, \width \} $, $ \ell \in \{ 1, 2, \ldots, \outdim \} $, $ \iteration \in \N_0 $ it holds that
  $ \Theta^{\width, \ndata, \lr}_{ \width \indim + \width + (\ell-1) \width + k } (\iteration)
    = \Theta^{\width, \ndata, \lr}_{ \width \indim + \width + (\ell-1) \width + k } (0) = \mf W_{\ell,k} $
  \proves that for all $ p \in \{ 1, 2, \ldots, \outdim \} $, $ n \in \N_0 $, $ v = ( v_1, \ldots, v_{\indim} ) \in \R^{\indim} $ it holds that
  \begin{equation}
  \label{eq:cor_vectorized_realization}
  \begin{split}
    &\lim_{s \to \infty} \mc N_{\Theta^{\width, \ndata, \lr} (\iteration)}^{\width, s, p} ( v ) \\
    &= \lim_{s \to \infty} \biggl[ \Theta_{ \width \indim + \width + \outdim \width + p }^{\width, \ndata, \lr} (\iteration)
      + \smallsum\limits_{k=1}^{\width} \Theta_{ \width \indim + \width + (p-1) \width + k }^{\width, \ndata, \lr} (\iteration) \,
      \varphi_s \Bigl( \Theta_{ \width \indim + k }^{\width, \ndata, \lr} (\iteration)
      + \smallsum\limits_{j=1}^{\indim} \Theta_{ (k-1) \indim + j }^{\width, \ndata, \lr} (\iteration) v_{j} \Bigr) \biggr] \\
    &= \Theta_{ \width \indim + \width + \outdim \width + p }^{\width, \ndata, \lr} (\iteration)
      + \smallsum\limits_{k=1}^{\width} \Theta_{ \width \indim + \width + (p-1) \width + k }^{\width, \ndata, \lr} (\iteration) \,
      \varphi_0 \Bigl( \Theta_{ \width \indim + k }^{\width, \ndata, \lr} (\iteration)
      + \smallsum\limits_{j=1}^{\indim} \Theta_{ (k-1) \indim + j }^{\width, \ndata, \lr} (\iteration) v_{j} \Bigr) \\
    &= \mf B_{p} (\iteration) + \smallsum\limits_{k=1}^{\width} \mf W_{p,k}
      \varphi_0 \bigl( \scalprod{ W_k(\iteration) }{ v } + B_k(\iteration) \bigr)
    = \bigl( \functionANN{\varphi_0}^{p} ( \Phi( \iteration ) ) \bigr) ( v ).
  \end{split}
  \end{equation}
  \Moreover the assumption that for all $ s \in \N $ it holds that $ \varphi_s $ is differentiable
  \proves that for all
  $ s \in \N $, $ v = ( v_1, \ldots, v_{\indim} ) \in \R^{\indim} $,
  $ j \in \{ 1, 2, \ldots, \indim \} $, $ k \in \{ 1, 2, \ldots, \width \} $, $ \ell, p \in \{ 1, 2, \ldots, \outdim \} $,
  $ q, r, t \in \N $ with $ q = (k-1) \indim + j $, $ r = \width \indim + k $, and $ t = \width \indim + \width + \outdim \width + \ell $
  it holds that
  \begin{equation}
  \begin{split}
    \tfrac{\partial}{\partial \theta_q} \mc N_{\Theta^{\width, \ndata, \lr} (\iteration)}^{\width, s, p} ( v )
    &= \tfrac{\partial}{\partial \theta_q} \biggl( \Theta_{ \width \indim + \width + \outdim \width + p }^{\width, \ndata, \lr} (\iteration)
      + \smallsum\limits_{u=1}^{\width} \Theta_{ \width \indim + \width + (p-1) \width + u }^{\width, \ndata, \lr} (\iteration) \,
      \varphi_s \Bigl( \Theta_{ \width \indim + u }^{\width, \ndata, \lr} (\iteration)
      + \smallsum\limits_{b=1}^{\indim} \Theta_{ (u-1) \indim + b }^{\width, \ndata, \lr} (\iteration) v_{b} \Bigr) \biggr) \\
    &= \Theta_{\width \indim + \width + (p-1) \width + k}^{\width, \ndata, \lr} (\iteration) (\varphi_s)' \Bigl(
      \Theta_{ \width \indim + u }^{\width, \ndata, \lr} (\iteration)
      + \smallsum\limits_{b=1}^{\indim} \Theta_{ (u-1) \indim + b }^{\width, \ndata, \lr} (\iteration) v_{b} \Bigr)
      v_j \\
    &= \mf W_{p,k} (\varphi_s)' \bigl( \scalprod{ W_k(\iteration) }{ v } + B_k(\iteration) \bigr) v_j,
  \end{split}
  \end{equation}
  \begin{equation}
  \begin{split}
    \tfrac{\partial}{\partial \theta_r} \mc N_{\Theta^{\width, \ndata, \lr} (\iteration)}^{\width, s, p} ( v )
    &= \tfrac{\partial}{\partial \theta_r} \biggl( \Theta_{ \width \indim + \width + \outdim \width + p }^{\width, \ndata, \lr} (\iteration)
      + \smallsum\limits_{u=1}^{\width} \Theta_{ \width \indim + \width + (p-1) \width + u }^{\width, \ndata, \lr} (\iteration) \,
      \varphi_s \Bigl( \Theta_{ \width \indim + u }^{\width, \ndata, \lr} (\iteration)
      + \smallsum\limits_{b=1}^{\indim} \Theta_{ (u-1) \indim + b }^{\width, \ndata, \lr} (\iteration) v_{b} \Bigr) \biggr) \\
    &= \Theta_{\width \indim + \width + (p-1) \width + k}^{\width, \ndata, \lr} (\iteration) (\varphi_s)' \Bigl(
      \Theta_{ \width \indim + u }^{\width, \ndata, \lr} (\iteration)
      + \smallsum\limits_{b=1}^{\indim} \Theta_{ (u-1) \indim + b }^{\width, \ndata, \lr} (\iteration) v_{b} \Bigr) \\
    &= \mf W_{p,k} (\varphi_s)' \bigl( \scalprod{ W_k(\iteration) }{ v } + B_k(\iteration) \bigr),
  \end{split}
  \end{equation}
  and
  \begin{equation}
  \begin{split}
    \tfrac{\partial}{\partial \theta_t} \mc N_{\Theta^{\width, \ndata, \lr} (\iteration)}^{\width, s, p} ( v )
    &= \tfrac{\partial}{\partial \theta_t} \biggl( \Theta_{ \width \indim + \width + \outdim \width + p }^{\width, \ndata, \lr} (\iteration)
      + \smallsum\limits_{u=1}^{\width} \Theta_{ \width \indim + \width + (p-1) \width + u }^{\width, \ndata, \lr} (\iteration) \,
      \varphi_s \Bigl( \Theta_{ \width \indim + u }^{\width, \ndata, \lr} (\iteration)
      + \smallsum\limits_{b=1}^{\indim} \Theta_{ (u-1) \indim + b }^{\width, \ndata, \lr} (\iteration) v_{b} \Bigr) \biggr) \\
    &= \ind_{ \{ \ell \} } (p).
  \end{split}
  \end{equation}
  Combining this with \cref{eq:cor_vectorized_error}, \cref{eq:cor_vectorized_realization}, and
  the fact that for all $ v \in \R $ it holds that $ \lim_{s\to\infty} (\varphi_s)' (v) = \grad(v) $
  \proves that for all
  $ j \in \{ 1, 2, \ldots, \indim \} $, $ k \in \{ 1, 2, \ldots, \width \} $, $ \ell \in \{ 1, 2, \ldots, \outdim \} $,
  $ q, r, t \in \N $ with $ q = (k-1) \indim + j $, $ r = \width \indim + k $, and $ t = \width \indim + \width + \outdim \width + \ell $
  it holds that
  \begin{equation}
  \begin{split}
    \lim_{s \to \infty} \tfrac{\partial}{\partial \theta_q} \loss^{\width, \ndata, s} ( \Theta^{\width, \ndata, \lr} (\iteration) )
    &= \lim_{s \to \infty} \biggl[ \tfrac{2}{\ndata} \smallsum\limits_{i=1}^{\ndata} \smallsum\limits_{p=1}^{\outdim} \bigl(
      \mc N_{\Theta^{\width, \ndata, \lr} (\iteration)}^{\width, s, p} ( \indata_i ) - \outdata_i^p \bigr)
       \bigl( \tfrac{\partial}{\partial \theta_q} \mc N_{\Theta^{\width, \ndata, \lr} (\iteration)}^{\width, s, p} ( \indata_i ) \bigr) \biggr] \\
    &= \tfrac{2}{\ndata} \smallsum\limits_{i=1}^{\ndata} \smallsum\limits_{p=1}^{\outdim}
      \bigl( ( \functionANN{\act}^p ( \Phi (\iteration) ) ) ( \indata_i ) - \outdata_i^p \bigr)
      \mf W_{p,k} \grad \bigl( \scalprod{ W_k(\iteration) }{ \indata_i } + B_k(\iteration) \bigr) \indata_i^j \\
    &= \tfrac{2}{\ndata} \smallsum\limits_{i=1}^{\ndata} \smallsum\limits_{p=1}^{\outdim} ( f_i^p (\iteration) - \outdata_i^p )
      \mf W_{p,k} \grad \bigl( \scalprod{ W_k(\iteration) }{ \indata_i } + B_k(\iteration) \bigr) \indata_i^j,
  \end{split}
  \end{equation}
  \begin{equation}
  \begin{split}
    \lim_{s \to \infty} \tfrac{\partial}{\partial \theta_r} \loss^{\width, \ndata, s} ( \Theta^{\width, \ndata, \lr} (\iteration) )
    &= \lim_{s \to \infty} \biggl[ \tfrac{2}{\ndata} \smallsum\limits_{i=1}^{\ndata} \smallsum\limits_{p=1}^{\outdim} \bigl(
      \mc N_{\Theta^{\width, \ndata, \lr} (\iteration)}^{\width, s, p} ( \indata_i ) - \outdata_i^p \bigr)
       \bigl( \tfrac{\partial}{\partial \theta_r} \mc N_{\Theta^{\width, \ndata, \lr} (\iteration)}^{\width, s, p} ( \indata_i ) \bigr) \biggr] \\
    &= \tfrac{2}{\ndata} \smallsum\limits_{i=1}^{\ndata} \smallsum\limits_{p=1}^{\outdim}
      \bigl( ( \functionANN{\act}^p ( \Phi (\iteration) ) ) ( \indata_i ) - \outdata_i^p \bigr)
      \mf W_{p,k} \grad \bigl( \scalprod{ W_k(\iteration) }{ \indata_i } + B_k(\iteration) \bigr) \\
    &= \tfrac{2}{\ndata} \smallsum\limits_{i=1}^{\ndata} \smallsum\limits_{p=1}^{\outdim} ( f_i^p (\iteration) - \outdata_i^p )
      \mf W_{p,k} \grad \bigl( \scalprod{ W_k(\iteration) }{ \indata_i } + B_k(\iteration) \bigr),
  \end{split}
  \end{equation}
  and
  \begin{equation}
  \begin{split}
    \lim_{s \to \infty} \tfrac{\partial}{\partial \theta_t} \loss^{\width, \ndata, s} ( \Theta^{\width, \ndata, \lr} (\iteration) )
    &= \lim_{s \to \infty} \biggl[ \tfrac{2}{\ndata} \smallsum\limits_{i=1}^{\ndata} \smallsum\limits_{p=1}^{\outdim} \bigl(
      \mc N_{\Theta^{\width, \ndata, \lr} (\iteration)}^{\width, s, p} ( \indata_i ) - \outdata_i^p \bigr)
       \bigl( \tfrac{\partial}{\partial \theta_t} \mc N_{\Theta^{\width, \ndata, \lr} (\iteration)}^{\width, s, p} ( \indata_i ) \bigr) \biggr] \\
    &= \tfrac{2}{\ndata} \smallsum\limits_{i=1}^{\ndata} \smallsum\limits_{p=1}^{\outdim}
      \bigl( ( \functionANN{\act}^p ( \Phi (\iteration) ) ) ( \indata_i ) - \outdata_i^p \bigr)
      \ind_{ \{ \ell \} } (p) \\
    &= \tfrac{2}{\ndata} \smallsum\limits_{i=1}^{\ndata} ( f_i^{\ell} (\iteration) - \outdata_i^{\ell} ).
  \end{split}
  \end{equation}
  This and the assumption that for all $ i \in \{ 1, 2, \ldots, \mc P( \width ) \} $, $ \iteration \in \N $ it holds that
  $ \Theta_i^{\width, \ndata, \lr} ( \iteration ) = \Theta_i^{\width, \ndata, \lr} ( \iteration - 1 )
    - \lr \bigl[ \lim\nolimits_{s \to \infty} ( \tfrac{\partial}{\partial \theta_i} \loss^{\width, \ndata, s} )
    ( \Theta^{\width, \ndata, \lr} ( \iteration - 1 ) ) \bigr] \ind_{ \R \backslash (\width \indim + \width, \mc P( \width ) - \outdim ] } ( i ) $
  \prove that for all
  $ j \in \{ 1, 2, \ldots, \indim \} $, $ k \in \{ 1, 2, \ldots, \width \} $, $ \ell \in \{ 1, 2, \ldots, \outdim \} $,
  $ \iteration \in \N_0 $, $ \omega \in \Omega $ it holds that
  \begin{equation}
    W_{k,j} (\iteration+1,\omega) = W_{k,j} (\iteration,\omega) - \frac{2 \lr}{\ndata} \biggl(
      \smallsum\limits_{i=1}^{\ndata} \smallsum\limits_{p=1}^{\outdim} ( \netvalue_i^p (\iteration,\omega) - \outdata_i^p )
      \mf W_{p,k}(\omega)
      \grad \bigl( \scalprod{ W_k(\iteration,\omega) }{ \indata_i } + B_k (\iteration,\omega) \bigr) \, \indata_i^j \biggr),
  \end{equation}
  \begin{equation}
    B_k(\iteration+1,\omega) = B_k(\iteration,\omega) - \frac{2 \lr}{\ndata} \biggl(
      \smallsum\limits_{i=1}^{\ndata} \smallsum\limits_{p=1}^{\outdim} ( \netvalue_i^p (\iteration,\omega) - \outdata_i^p )
      \mf W_{p,k}(\omega)
      \grad \bigl( \scalprod{ W_k(\iteration,\omega) }{ \indata_i } + B_k (\iteration,\omega) \bigr) \biggr),
  \end{equation}
  and
  \begin{equation}
    \mf B_{\ell} (\iteration+1,\omega) = \mf B_{\ell} (\iteration,\omega) - \frac{2 \lr}{\ndata} \biggl(
      \smallsum\limits_{i=1}^{\ndata} ( \netvalue_i^{\ell} (\iteration,\omega) - \outdata_i^{\ell} ) \biggr).
  \end{equation}
  \Moreover \cref{cor:vectorized:lambda}, \cref{eq:cor_vectorized_G},
  the fact that $ \sqrt{\nicefrac{2}{\indim}} Z $ and $ W_1 (0) $ are identically distributed, and
  \cref{item:lambdamin_2} in \cref{lem:lambdamin_eucnorm} (applied with
  $ n \is \ndata $, $ A \is G $ 
  in the notation of \cref{lem:lambdamin_eucnorm})
  \prove that
  \begin{equation}
    \lambda_{\ndata}
    = \min\nolimits_{ v \in \R^{\ndata}, \, \eucnorm{ v } = 1 } \bscalprod{ v }{ \tfrac{1}{2} G v }
    = \lambdamin( \tfrac{1}{2} G )
    = \tfrac{1}{2} \lambdamin( G ).
  \end{equation}
  Combining this,
  \cref{lem:eigenvalues_block_matrix} (applied with
  $ \outdim \is \outdim $, $ \ndata \is \ndata $, $ A \is \gramGinf $, $ B \is G $
  in the notation of \cref{lem:eigenvalues_block_matrix}), and
  the assumption that $ \lambdazero = \lambdamin( \tfrac{1}{2} \gramGinf ) $
  \prove that
  \begin{equation}
    \lambda_{\ndata}
    = \tfrac{1}{2} \lambdamin( G )
    = \tfrac{1}{2} \lambdamin( \gramGinf )
    = \lambdazero. 
  \end{equation}
  \cref{cor:assumption_pairwise_distinct_const} (applied with
  $ \deriv \is \grad $, $ \act \is \varphi_0 $, $ \mf C \is \cmax $, $ \Cvar \is \tfrac{2}{\width} $,
  $ \ms x \is \indata $, $ \ms y \is \outdata $, $ \Lambda \is \Lambda $, $ \poe \is \poe $, $ \constvarwidth \is 2 $, $ \loss \is \mf E $
  in the notation of \cref{cor:assumption_pairwise_distinct_const})
  and the fact that $ \# \{ \indata_1, \indata_2, \ldots, \indata_{\ndata} \} = \ndata $
  \hence \prove that
  $ \lambda_{\ndata} > 0 $ and
  \begin{equation}\label{eq:cor_vectorized_probability}
  \begin{split}
    &\P \Bigl( \Forall \iteration \in \N_0 \colon \loss^{\width, \ndata, 0} ( \Theta^{\width,\ndata,\lr} (\iteration) )
      \le \bigl( 1 - \tfrac{ 2 \lr \lambda_{\ndata} }{ \ndata } \bigr)^{\iteration} \loss^{\width, \ndata, 0}
        ( \Theta^{\width,\ndata,\lr} (0) ) \Bigr) \\
    &= \P \Bigl( \Forall \iteration \in \N_0 \colon \mf E(\Phi(\iteration))
      \le \bigl( 1 - \tfrac{ 2 \lr \lambdazero }{ \ndata } \bigr)^{\iteration} \mf E(\Phi(0)) \Bigr)
    \ge 1 - \varepsilon.
  \end{split}
  \end{equation}
  This \proves \cref{cor:vectorized:item1} and \cref{cor:vectorized:item2}.
\end{cproof}

\begin{lemma}
  \label{lem:convolution}
  Let $ \nbp \in \N $, $ \bp_0, \bp_1, \ldots, \bp_{\nbp+1} \in [-\infty,\infty] $,
  $ \slope_1, \slope_2, \ldots, \slope_{\nbp+1} $, $ \yinter_1, \yinter_2, \ldots, \yinter_{\nbp+1} \in \R $
  satisfy $ - \infty = \bp_0 < \bp_1 < \dots < \bp_{\nbp+1} = \infty $,
  let $ \act \in C( \R, \R ) $ satisfy for all $ i \in \{ 1, 2, \ldots, \nbp+1 \} $, $ v \in ( \bp_{i-1}, \bp_i ) $ that
  $ \act( v ) = \slope_i v + \yinter_i $,
  and let $ \varphi_k \colon \R \to \R $, $ k \in \N $, satisfy for all $ k \in \N $, $ v \in \R $ that
  \begin{equation}
    \varphi_k (v) = k \int_{v-1/k}^{v} \act(u) \, \diff u.
  \end{equation}
  Then
  \begin{enumerate}[(i)]
    \item \label{item:convolution_1} it holds for all $ v \in \R $ that $ \lim_{k\to\infty} \varphi_k (v) = \act(v) $,
    \item \label{item:convolution_2} it holds for all $ k \in \N $ that $ \varphi_k $ is differentiable, and
    \item \label{item:convolution_3} it holds for all $ i \in \{ 1, 2, \ldots, \nbp+1 \} $, $ v \in \R \cap ( \bp_{i-1}, \bp_i ] $ that
      $ \lim_{k\to\infty} ( \varphi_k )'(v) = \slope_i = ( \nablam \act )( v ) $.
  \end{enumerate}
\end{lemma}

\begin{cproof}{lem:convolution}
  Throughout this proof let $ \cderiv \in (0,\infty) $ satisfy $ \max_{ i \in \{ 1, 2, \ldots, \nbp+1 \} } \abs{\slope_i} \le \cderiv $.
  \Nobs that the fact that $ \act $ is $ \cderiv $-Lipschitz continuous
  \proves that for all $ k \in \N $, $ v \in \R $ it holds that
  \begin{equation}
  \begin{split}
    \babs{ \varphi_k (v) - \act(v) }
    &= \bbbabs{ k \int_{v-1/k}^{v} \bigl( \act(u) - \act(v) \bigr) \, \diff u }
    \le k \int_{v-1/k}^{v} \abs{ \act(u) - \act(v) } \, \diff u \\
    &\le \cderiv \mspace{1.5mu} k \int_{v-1/k}^{v} \abs{ u - v } \, \diff u
    \le \cderiv \mspace{1.5mu} k \int_{v-1/k}^{v} \frac{1}{k} \, \diff u
    = \frac{\cderiv}{k}.
  \end{split}
  \end{equation}
  \Hence that for all $ v \in \R $ it holds that $ \lim_{k\to\infty} \varphi_k (v) = \act(v) $.
  This \proves \cref{item:convolution_1}.
  \Moreover for all $ k \in \N $, $ v \in \R $ there exists $ z \in \R $ such that
  $ v - \tfrac{1}{k} < z < v $ and
  \begin{equation}\label{lem:convolution:ftc}
    \varphi_k (v)
    = k \int_{v-1/k}^{v} \act(u) \, \diff u
    = k \int_{v-1/k}^{z} \act(u) \, \diff u + \int_{z}^{v} \act(u) \, \diff u
    = k \int_{z+1/k}^{v} - \act \bigl( u - \tfrac{1}{k} \bigr) \, \diff u + k \int_{z}^{v} \act(u) \, \diff u.
  \end{equation}
  The fundamental theorem of calculus and
  the assumption that $ \act $ is continuous \hence \prove that for all $ k \in \N $, $ v \in \R $
  it holds that $ \varphi_k $ is differentiable at $ v $.
  This \proves \cref{item:convolution_2}.
  \Moreover \cref{lem:convolution:ftc} and
  the fundamental theorem of calculus
  \prove that for all $ i \in \{ 1, 2, \ldots, \nbp+1 \} $, $ v \in \R \cap ( \bp_{i-1}, \bp_i ] $, $ k \in \N $
  with $ v - \tfrac{1}{k} > \bp_{i-1} $ it holds that
  \begin{equation}
    ( \varphi_k )' (v)
    = k \Bigl[ \act (v) - \act \bigl( v - \tfrac{1}{k} \bigr) \Bigr]
    = k \Bigl[ \bigl( \slope_i v + \yinter_i \bigr) - \bigl( \slope_i \bigl( v - \tfrac{1}{k} \bigr) + \yinter_i \bigr) \Bigr]
    = \slope_i
    = ( \nablam \act )( v ).
  \end{equation}
  \Hence that for all $ i \in \{ 1, 2, \ldots, \nbp+1 \} $, $ v \in \R \cap ( \bp_{i-1}, \bp_i ] $ it holds that
  $ \lim_{k\to\infty} ( \varphi_k )'(v) = \slope_i = ( \nablam \act )( v ) $.
  This \proves \cref{item:convolution_3}.
\end{cproof}

\cfclear
\begin{lemma}
  \label{lem:left_derivative}
  Let $ M \subseteq C( \R, \R ) $ satisfy
  \begin{equation}
    M = \Bigl\{ f \in C( \R, \R ) \colon \Bigl[ \Forall v \in \R \colon \Bigl( \Exists u \in \R \colon
        \limsup\nolimits_{\varepsilon \nearrow 0} \bbabs{ \tfrac{ f(v+\varepsilon) - f(v) }{ \varepsilon } - u } = 0 \Bigr) \Bigr] \Bigr\}.
  \end{equation}
  Then
  \begin{enumerate}[(i)]
    \item \label{lem:left_derivative:item1} it holds for all $ f, g \in M $, $ \lambda \in \R $ that
      $ \nablam ( f + \lambda g ) = ( \nablam f ) + \lambda ( \nablam g ) $,
    \item \label{lem:left_derivative:item2} it holds for all $ f \in C^1( \R, \R ) $, $ g \in M $ that
      $ \nablam ( f \circ g ) = ( f' \circ g ) \times ( \nablam g )  $, and
    \item \label{lem:left_derivative:item3} it holds for all $ f, g \in M $ with $ g $ non-decreasing that
      $ \nablam ( f \circ g ) = ( ( \nablam f ) \circ g ) \times ( \nablam g ) $.
  \end{enumerate}
\end{lemma}

\begin{cproof}{lem:left_derivative}
  \Nobs that for all $ f, g \in M $, $ \lambda, v \in \R $ it holds that
  \begin{equation}
  \begin{split}
    \nablam \bigl( f(v) + \lambda g(v) \bigr)
    &= \lim\nolimits_{\varepsilon \nearrow 0}
      \frac{ [ f(v+\varepsilon) + \lambda g(v+\varepsilon) ] - [ f(v) + \lambda g(v) ] }{\varepsilon} \\
    &= \biggl[ \lim\nolimits_{\varepsilon \nearrow 0} \frac{ f(v+\varepsilon) - f(v) }{\varepsilon} \biggr]
      + \lambda \biggl[ \lim\nolimits_{\varepsilon \nearrow 0} \frac{ g(v+\varepsilon) - g(v) }{\varepsilon} \biggr] \\
    &= \bigl( \nablam f(v) \bigr) + \lambda \bigl( \nablam g(v) \bigr).
  \end{split}
  \end{equation}
  This \proves \cref{lem:left_derivative:item1}.
  \Nobs that for all $ f \in C^1( \R, \R ) $, $ g \in M $, $ v \in \R $ it holds that
  \begin{equation}
  \begin{split}
    \nablam \bigl( f( g(v) ) \bigr)
    &= \lim\nolimits_{\varepsilon \nearrow 0} \frac{ f( g(v+\varepsilon) ) - f( g(v) ) }{\varepsilon} \\
    &= \biggl[ \lim\nolimits_{\varepsilon \nearrow 0} \frac{ f( g(v+\varepsilon) ) - f( g(v) ) }{ g(v+\varepsilon) - g(v) } \biggr]
      \biggl[ \lim\nolimits_{\varepsilon \nearrow 0} \frac{ g(v+\varepsilon) - g(v) }{\varepsilon} \biggr] \\
    &= f'( g(v) ) ( \nablam g (v) ).
  \end{split}
  \end{equation}
  This \proves \cref{lem:left_derivative:item2}.
  \Nobs that for all $ f, g \in M $, $ v \in \R $ with $ g $ non-decreasing it holds that
  \begin{equation}
  \begin{split}
    \nablam \bigl( f( g(v) ) \bigr)
    &= \lim\nolimits_{\varepsilon \nearrow 0} \frac{ f( g(v+\varepsilon) ) - f( g(v) ) }{\varepsilon} \\
    &= \biggl[ \lim\nolimits_{\varepsilon \nearrow 0} \frac{ f( g(v+\varepsilon) ) - f( g(v) ) }{ g(v+\varepsilon) - g(v) } \biggr]
      \biggl[ \lim\nolimits_{\varepsilon \nearrow 0} \frac{ g(v+\varepsilon) - g(v) }{\varepsilon} \biggr] \\
    &= ( \nablam f(g(v)) ) ( \nablam g (v) ).
  \end{split}
  \end{equation}
  This \proves \cref{lem:left_derivative:item3}.
\end{cproof}

\subsection*{Acknowledgements}
This work has been partially funded by the European Union (ERC, MONTECARLO, 101045811). The views and the opinions expressed in this work are however those of the authors only and do not necessarily reflect those of the European Union or the European Research Council (ERC). Neither the European Union nor the granting authority can be held responsible for them.
Moreover, this work has been partially funded by the Deutsche Forschungsgemeinschaft (DFG, German Research Foundation)
under Germany's Excellence Strategy EXC 2044/2--390685587,
Mathematics M\"unster: Dynamics--Geometry--Structure.

\bibliographystyle{acm}
\bibliography{references}

\end{document}